\documentclass[12pt]{article}
\usepackage{amsmath,amsthm}
\usepackage{authblk}
\usepackage{bbm, bm}
\usepackage{amsfonts}
\usepackage{graphicx,psfrag,epsf}
\usepackage{subcaption}
\usepackage{enumerate}
\usepackage{natbib}
\usepackage{url} % not crucial - just used below for the URL 
\usepackage{algorithm}
\usepackage[noend]{algpseudocode}
\usepackage{booktabs}
\usepackage[usenames,dvipsnames]{xcolor}
\usepackage{hyperref}
\hypersetup{
    colorlinks = true,
    linkbordercolor = {white},
    citecolor  ={blue}
}

\newcommand{\blind}{0}

\theoremstyle{definition}

\theoremstyle{plain}
\newtheorem{theorem}{Theorem}

\newtheorem{condition}{Condition}

\DeclareMathOperator*{\argmin}{arg\,min}
\DeclareMathOperator*{\ind}{\mathbbm{1}}

\begin{document}

\def\spacingset#1{\renewcommand{\baselinestretch}%
{#1}\small\normalsize} \spacingset{1}

\newcommand{\jelena}[1]{\textcolor{blue}{Jelena: #1}}
\newcommand{\hanbo}[1]{\textcolor{red}{Hanbo: #1}}

%%%%%%%%%%%%%%%%%%%%%%%%%%%%%%%%%%%%%%%%%%%%%%%%%%%%%%%%%%%%%%%%%%%%%%%%%%%%%%

\if0\blind
{
  \title{\Large \bf{Censored Quantile Regression Forests }}
  
  \author{Alexander Hanbo Li}
\author{Jelena Bradic}

\affil{Department of Mathematics \\ University of California San Diego}
   
%\author{Alexander Hanbo Li
%  \and
%Jelena Bradic  
%}

\date{}

  \maketitle
} \fi

\if1\blind
{
  \bigskip
  \bigskip
  \bigskip
  \begin{center}
    {\LARGE\bf Censored Quantile Regression Forests}
\end{center}
  \medskip
} \fi

\bigskip
\begin{abstract}
Random forests are powerful non-parametric regression method but are severely limited in their usage in the presence of randomly censored observations, and naively applied can exhibit poor predictive performance due to the incurred biases. Based on a local adaptive representation of random forests, we develop its regression adjustment for randomly censored regression quantile models. Regression adjustment is based on new estimating equations that adapt to censoring and lead to quantile score whenever the data do not exhibit censoring. The proposed procedure named {\it censored quantile regression forest}, allows us to estimate quantiles of time-to-event without any parametric modeling assumption. We establish its consistency under mild model specifications. Numerical studies showcase a clear advantage of the proposed procedure.
\end{abstract}

\noindent%
{\it Keywords:}  Random Forest; Censored quantile regression; Nonparametric regression; Kaplan-Meier estimation.
 \vfill
  \newpage
\spacingset{1.45} % DON'T change the spacing!

\section{Introduction}
\label{sec:intro}

In many applications, we want to predict and estimate the effect of a covariate on survival time of interests. Examples include treatment, surgical procedure, or immunization on survival time of patients, who for example, could be individuals who have metastatic breast cancer,  military casualties suffering from various injuries, or survival time of infectious diseases. Classically, most datasets have been too small to meaningfully examine the heterogeneity of the data beyond dividing them into a few sub-populations. In the past few years, however, there has been an explosion of experimental settings where it is potentially feasible to explore heterogeneity to its full extent.

An impediment to exploring heterogeneous effects is the fear that scientists with two opposite agendas could hypothetically string together two opposite but coherent results by searching through many different possible models and then reporting only the very extreme ones -- highlighting solely spurious results \citep{olken2015promises}. Thus, protocols for clinical trials must specify in advance the pre-analysis plans and then learn from the data. However, such restrictions can make it challenging to discover unexpected effects due to heterogeneity. Here,  we aim to address this challenge by developing a robust, non-parametric method for estimation in regression settings with censored response variable, which yields consistent estimator that adapts to the heterogeneity of the data and hence can be broadly applied to many different models and achieves further improvement. One example is the accelerated failure time model.

Classical approaches to accelerated failure time model include non-parametric maximum likelihood method, semi-parametric approaches, as well as traditional  parametric approaches; see e.g., \cite{zeng2007efficient}, \cite{robins1992semiparametric},    \cite{robins1992estimation}, as well as \cite{koul1981regression},  \cite{wei1992accelerated},  \cite{huang2007least} etc. These methods perform well in settings where the model error is correctly specified, but quickly break down when the distribution of the error has any heterogeneous, asymmetric, or outlier structure. Other non-parametric methods, like \cite{doi:10.1093/biomet/68.2.381}, \cite{doi:10.1093/biomet/85.4.809}, and rank-based methods, like  \cite{jin2003rank}, strive to achieve certain level of assumption-lean modeling of the mean in the accelerated failure time model. In this paper, we explore the use of ideas from the machine learning literature to improve the performance of these classical methods in a non-parametric fashion that is adaptive to heterogeneous or heavy-tailed error distributions.

Random forest algorithms introduced by \cite{breiman2001random} allow for flexible modeling of covariate interactions, and are related to kernels and nearest-neighbor methods in the sense that they make predictions using a weighted average of ``nearby'' observations. However, random forests depart from the above principles in that they have a data-driven way to determine which nearby observations receive more weight, something that is especially important in environments
with many covariates or complex interactions among covariates.

Despite their wide-spread success at estimation and prediction, application of random forests to censored regression models is far from being easily understood. Not all response variables being observed makes it difficult to understand how to evaluate the prediction arising from simple tree structures.  Namely, a simple average is only enough for conditional mean estimation without censoring. Median estimates, and quantiles more generally, are not easy to construct. In particular, quantile random forests \citep{meinshausen2006quantile} have mainly been undeveloped for censored observations.

This paper addresses these limitations, developing a forest-based method for estimation of quantiles for randomly censored observations (right or left censored). 
More formally, let $T$ be a real-valued latent variable (e.g., survival time) and $X$ be a (possibly high-dimensional) predictor variable. Let $C$ denote the censoring variable, which prevents us from observing all of the information regarding $T$. In left-censored data, we only observe $Y_i = \max(T_i,C_i)$, whereas  in right-censored data, we only observe $Y_i = \min(T_i,C_i)$. Our goal is to estimate the $\tau$-th quantile of $T_i$, non-parametrically, using observations $(Y_i,X_i, \delta_i)$ where $\delta_i =\ind \{ T_i \leq C_i\}$.

We take on the perspective of random forests as that of an nonparametric, 
adaptive kernel smoothing method. This interpretation follows work by  \cite{athey2016generalized}, \cite{Bloniarz2016SupervisedNF}, \cite{hothorn2004bagging},   \cite{li2017forest}, and  \cite{meinshausen2006quantile}, and supplements  the customary view of forests as an ensemble method (i.e., an
average of predictions made by separate trees).  In this view, random forest predictions, of the mean for example, evaluated at a new test point $x$, can be represented by
\[
\sum_{i=1}^n w_i(X_i,x) Y_i, \qquad \sum_{i=1}^n w_i(X_i,x) =1, \qquad w_i(X_i,x)\geq 0
\]
where the weights encode the similarity between the new test point $x$ and the observed covariates $X_i$.
 It is worth pointing that for the conditional mean, the averaging and weighting views of random forests are equivalent; however, once we move to more general settings, the weighting-based perspective proves substantially more powerful.
The goal is then to utilize these local neighborhood weights and quantile regression adjustments to design a new non-parametric quantile estimate of 
\begin{equation}\label{eq:1}
Q_{T|x}(\tau|x), \qquad \tau \in(0,1),
\end{equation}
based on observations $(Y_i, X_i, \delta_i)_{i=1}^n$.  In their simplest form, censored quantile forests just take the forest weights, $w_i(X_i,x)$, and use them for quantile regression:
\[
\hat q_\tau (x) = \arg\min_q \left\{  \Bigl\| \sum_{i=1}^n w_i(X_i,x) {\mathcal{S}}_\tau (X_i, Y_i, q,x) \Bigl\|_l\right\},
\]
where $\| \cdot \|_l$ denotes any $l$-norm with $l\geq 1$.
Here, we define a new score function, ${\mathcal{S}}_\tau $, that is   censoring and quantile adaptive:
\[
{\mathcal{S}}_\tau (X_i, Y_i, q,x)= (1-\tau)\hat{G}(q|x) - \ind (Y_i > q)
\]
with $\hat G$ denoting an estimate of $G (u) = \mathbb{P}(C \geq u| X =x)$, the   survival function of the censoring time $C $ given $X$. In the rest of the document we   focus on  the case of $l=1$; however, results remain true for general cases of $l$ as well.

 Formally, we study the performance of the above non-parametric quantile estimator $\hat q_\tau(x)$ of  \eqref{eq:1} while considering right-censored observations. All of our results can be trivially extended to the left-censored case. Most of the theoretical work focuses on establishing Theorem \ref{thm:consistency} that ensures the consistency of the censored quantile forest estimator $\hat q_\tau(x)$  at a given test point $x$ while allowing minimal assumptions on the regression model as well as the distribution of the model error. This result also allows us to construct prediction intervals regarding  \eqref{eq:1} that adapt to the model structure of the data generating process. Censored quantile forest estimator improves on prediction error compared with many state-of-the-art censored regression methods, and yet it retains the flexibility of random forest methods designed for non-censored observations. Finally, we note that our method can also be seen as an improvement over classical non-parametric approaches for censored observations. While the latter only perform well in low-dimensional problems, ours performs well even with a large number of covariates. One important reason is that random forest weights adapt reasonably well to the dimensionality increase, whereas kernel methods suffer from the curse of dimensionality.

\subsection{Related Work}

There has been a long-time understanding that proportional hazard models, and cox's model, in particular,  are especially powerful for right-censored observations and regression problems.  However, they do not adapt to possibly left-censored observations; besides, they heavily rely on the proportionality assumption which can sometimes be inappropriate, necessitating stratification of the baseline hazard or some other weakening of the proportional hazards condition \citep{koenker2008censored}.

A more flexible approach for random censoring problem is to model the conditional quantiles of the response variable directly. This approach offers greater flexibility as it does not restrict the structure of the hazard function \citep{koenker2008censored} and merely is more intuitive. To estimate the conditional quantiles, \cite{portnoy2003censored} proposed a recursive method which estimates a sequence of linear conditional quantile functions recursively. It can be treated as a generalization to regression of the Kaplan Meier estimator. Another closely related quantile regression model proposed by \cite{peng2008survival} instead makes the linkage to the Nelson-Aalen estimator of the cumulative hazard function, upon which they developed a complete asymptotic theory. The closest work to ours is that of \cite{doi:10.1080/01621459.2018.1469996} where authors propose a new censored loss function. However, they only discuss the properties of quantile estimates in the case of linear quantile model.  

The parametric methods, including those mentioned above, always rely on the linearity assumption on the conditional quantiles, that is,
\begin{equation}\label{eq:linear_quantile}
    Q_{\log(T)|x}(\tau|x) = x^\top \beta.
\end{equation}
Here, the $\log$ transformation is arbitrary but popular in survival analysis and can be replaced by any monotone function. This linearity assumption is too restrictive in many cases, especially when data lie on a complex manifold. Therefore, non-parametric methods are necessary and play an important role in modeling data heterogeneity.

In the case of right censoring, most non-parametric recursive partitioning algorithms in the existing literature rely on survival tree or its ensembles. \cite{ishwaran2008random} proposed random survival forest (RSF) algorithm in which each tree is built by maximizing the between-node log-rank statistic. However, it is not directly estimating the conditional quantiles but instead estimating the cumulative hazard. \cite{zhu2012recursively} proposed the recursively imputed survival trees (RIST) algorithm with the same splitting criterion for each individual tree but different ensemble scheme. Other similar methods relying on different kinds of survival trees were proposed in \cite{gordon1985tree}, \cite{segal1988regression}, \cite{davis1989exponential}, \cite{leblanc1992relative}, and \cite{leblanc1993survival}. All these methods as mentioned above use splitting rules specifically designed to deal with the right censored data. Despite different splitting strategies, they all rely on the proportional hazard assumption and cannot reduce to a loss-based method that might ordinarily be used in the situation with no censoring.

\cite{molinaro2004tree} proposed a tree method based on the inverse probability censoring \citep{robins1994estimation} weighted (IPCW) loss function which reduces to the full data loss function used by CART in the absence of censoring. \cite{hothorn2005survival} then extended the IPCW idea and proposed a forest-type method in which each tree is trained on resampled observations according to inverse probability censoring weights. However, the censored data always get weights zero and hence only uncensored observations will be resampled. As pointed out by \cite{robins1994estimation}, the inverse probability weighted estimators are inefficient because of their failure to utilize all the information available on observations with missing or partially missing data. 

This work aims to build a non-parametric conditional quantile estimator for randomly censored data that reduces to the ordinal quantile forest estimators when full data are observed, and efficiently utilizes all available information on both uncensored and censored observations. Furthermore, it does not require the specific modification of ordinal regression tree (e.g., CART) to survival tree, and hence works on both left and right censored problems.

Fundamentally different from the aforementioned forests methods, in which the censoring information is considered directly in the tree constructing process, our method avoids this complexity and only requires building ordinal regression trees (e.g., CART) for the first step, treating all observations equally. The censoring effects are then taken care of in the second step by solving a locally weighted estimating equation. 
These local weights can be directly calculated from the random forest constructed in the first step;   weights derived from the fraction of trees in which an observation appears in the same leaf as the target value of the covariate vector.  This locally weighted view of random forests was previously advocated by \cite{hothorn2005survival} and \cite{Bloniarz2016SupervisedNF}; original random forest algorithm \citep{breiman2001random} utilized ensemble learning literature.    Differently from kernel weights, typically employed in local maximum likelihood method, for example, and whose performance suffers greatly whenever the dimensionality of the covariate space is more than two or three, our random-forests weights adapt well to moderate dimensionality increases.

Additional challenges arise due to the random censoring nature of the observations. For fixed censoring, one observes all the censoring values and hence can straightforwardly modify the objective used in the general framework of \cite{athey2016generalized}, for example. However, it is unclear how to develop a non-parametric estimator that adapts to unknown censoring in the observations.

\subsection{Organization}

The paper is organized as follows. 
 In Section \ref{ssec:rf_weights}  we provide local adaptive nature of random forests and regression adjustment utilizing those weights. 
  In Sections \ref{subsec:meth1} and \ref{subsec:meth2}, we showcase the development of a new loss function and its power in  predicting any conditional quantile of the latent variable $T$ by solving an ingenious estimating equation, which is   designed to correct the censoring effect. The Algorithm is then described in details in Section \ref{subsec:meth3}.  In Section \ref{sec:theory}, we analyze the time complexity of our algorithm and prove  consistency of the proposed estimator. Section \ref{sec:experiments} contains extensive numerical studies where we compare our algorithm with other forest algorithms on simulated and real censored data sets. Proofs are collected in the Appendix.

\section{Censored Quantile Regression Forest}
\label{sec:meth}
The quantile random forest cannot be directly applied to censored data $\{(X_i, Y_i)\}$ because the conditional quantile of $Y$ is different from that of the latent variable $T$ due to the censoring. Moreover, there is no explicitly defined quantile loss function for randomly censored data.   In this section, we design a new approach to achieve both tasks.

\subsection{Regression adjustment for random forests}
\label{ssec:rf_weights}

Let $\theta$ denote the random parameter determining how a tree is grown, and $\{(X_i, Y_i): i=1,\ldots,n\} \in \mathcal{X} \times \mathcal{Y} \subset \mathbb{R}^p \times \mathbb{R}$ denote the training data. For each tree $T(\theta)$, let $R_l$ denotes its $l$-th terminal leaf. Since the space $\mathcal{X}$ is split into disjoint leaves by $T(\theta)$, we know for any $x \in \mathcal{X}$, there is exactly one leaf containing $x$. We let the index of the leaf be $l(x;\theta)$ and we say $x \in R_{l(x;\theta)}$.

Then for any single tree $T(\theta)$, the prediction on any data point $x \in \mathcal{X}$ is $  \sum_{i=1}^n w(X_i, x; \theta) Y_i$ where
\begin{equation} \label{eq:tree_weight}
    w(X_i, x; \theta) = \frac{\ind_{\{ X_i \in R_{l(x;\theta)} \}}}{\# \{ j: X_j \in R_{l(x;\theta)} \}}.
\end{equation}
Then  a random forest containing $m$ trees formulates a  prediction of $\mathbb{E}[Y|X=x]$ as
%\begin{equation} \label{eq:rf_pred}
\[
 \sum_{i=1}^n w(X_i, x) Y_i
\]
where
\begin{equation} \label{eq:rf_weight}
    w(X_i,x) = \frac{1}{m} \sum_{t=1}^m w(X_i, x; \theta_t).
\end{equation}
From now on, we call the weight $w(X_i,x)$ in \eqref{eq:rf_weight} \textit{random forest weight}. One can easily show that $\sum_{i=1}^n w(X_i,x) = 1$.
The above representation of the random forest prediction of the mean can be equivalently obtained as a solution to the following least-squares optimization problem
\[
\min_{\lambda \in \mathbb{R}} \sum_{i=1}^n w(X_i,x) (Y_i - \lambda)^2.
\]
Therefore, a least-squares regression adjustment, as the above, is equivalent to  \cite{breiman2001random} representation of random forests. However, when we move to estimation quantities that are not the mean, the latter representation is very powerful. Namely, a quantile random forest of \cite{meinshausen2006quantile} can be seen as a quantile regression adjustment \citep{li2017forest}, i.e., as a solution to the following optimization problem
\[
\min_{\lambda \in \mathbb{R}} \sum_{i=1}^n w(X_i, x) \rho_{\tau}(Y_i - \lambda),
\]
where $\rho_\tau$ is the $\tau$-th quantile loss function, defined as $\rho_{\tau}(u) = u(\tau - \ind(u < 0))$.
Local linear regression adjustment was also recently utilized in \cite{athey2016generalized} to obtain a smoother and more poweful random forest algorithm.

\subsection{Motivation}
\label{subsec:meth1}

Let us consider the case of no censored observations. Full data serve as a motivation for developing suitable estimating equations. Following the regression adjustment reasoning, for the case of fully observed data, we 
  could estimate the $\tau$-th quantile of $T_i$ at $x$, denoted as $q_{\tau, x}$, as a solution to 
\[
\min_{q \in \mathbb{R}} \sum_{i=1}^n w(X_i, x) \rho_{\tau}(T_i - q).
\]
Equivalently, such estimate would solve the following 
 estimating equations 
\begin{eqnarray} \label{eq:fop_latent}
U_n(q) &=& \sum_{i=1}^n w(X_i, x) \Bigl\{ (1 - \tau) - \ind (T_i > q) \Bigr\} \nonumber \\
&=& (1-\tau) - \sum_{i=1}^n w(X_i, x) \ind(T_i > q) \approx 0,
\end{eqnarray}
where the second equality is true because $\sum_{i=1}^n w(X_i, x) = 1$.

For simplicity and better illustration of the idea, we first assume the latent variable $T_i$ has the same conditional probability in a neighborhood $R_x$ of $x$. Out of the $n$ data points, assume $\{X_1, \cdots, X_k\} \subset R_x$ and $w(X_i, x) = 1/k$ when $X_i \in R_x$ and $0$ otherwise. Now the estimating equation becomes
\begin{eqnarray} \label{eq:fop_approx}
U_k(q) &=& \frac{1}{k} \sum_{i=1}^k \Bigl\{ (1 - \tau) - \ind (T_i > q) \Bigr\} \nonumber \\
&=& (1-\tau) - \frac{1}{k} \sum_{i=1}^k \ind (T_i > q).
\end{eqnarray}
Now conditional on $\{x\} \cup \{X_i\}_{i=1}^k$,
\begin{equation*}
\mathbb{E} \left[ U_k(q) \right] = (1-\tau) - \mathbb{P}(T > q | x)
\end{equation*}
which will be zero at $q^*$ where $\mathbb{P}(T > q^* | x) = 1-\tau$, that is, when $q^* = q_{\tau,x}$

Let's now consider the case of right-censored setting, where 
 we further have the censoring variable $C_i$, which is independent of $T_i$ conditional on $X_i$, and we could only observe $Y_i = \min\{T_i, C_i\}$ and censoring indicator $\delta_i = \ind (T_i \le C_i)$. In order to estimate $q_{\tau,x}$, we cannot simply replace $T_i$ with $Y_i$ in \eqref{eq:fop_approx} as the $\tau$-th quantile of $T_i$ is no longer the $\tau$-th quantile of $Y_i$ because of the censoring. However, we can observe and utilize the following relationship
\begin{equation*}
\mathbb{P}(Y_i > q_{\tau,x}|x) = \mathbb{P}(T_i > q_{\tau,x}|x) \mathbb{P}(C_i > q_{\tau,x}|x) = (1-\tau) G(q_{\tau,x}|x),
\end{equation*}
where $G(u|x)$ is the survival function of $C_i$ at $x$. That is to say, the $\tau$-th quantile of $T_i$ is actually the $1 - (1-\tau)G(q_{\tau,x}|x)$-th quantile of $Y_i$ at $x$. Now, we define a new estimating equation that resembles  \eqref{eq:fop_approx} as follows
\begin{equation} \label{eq:fop_Y_approx}
S^o_k(q) = \frac{1}{k}\sum_{i=1}^k \Bigl\{ (1-\tau)G(q|x) - \ind (Y_i > q) \Bigr\} \approx 0.
\end{equation}
If we substitute $G(q|x)$ with $G(q_{\tau,x}|x)$, an intuitive explanation for \eqref{eq:fop_Y_approx} is that as  the $\tau$-th quantile of $T_i$ happens to be the $1 - (1-\tau)G(q_{\tau,x}|x)$-th quantile of $Y_i$ at $x$, instead of estimating the former which is not available because of the censoring, we turn to estimate the later one. Namely, the conditional expectation, $\mathbb{E}[S^o_k(q)]$, will still be zero at the same root $q^*$ for \eqref{eq:fop_approx}. 

The survival function $G(\cdot|x)$ can be estimated by a consistent estimate, for example the Kaplan-Meier estimator $\hat{G}(\cdot|x)$ using $\{Y_i\}_{i=1}^k$ and $\{\delta_i\}_{i=1}^k$, and we  can then define  
\begin{equation} \label{eq:fop_Y_approx_KM}
S_k(q) = \frac{1}{k}\sum_{i=1}^k \Bigl\{ (1-\tau)\hat{G}(q|x) - \ind (Y_i > q) \Bigr\} \approx 0.
\end{equation}

\subsection{Full model}
\label{subsec:meth2}
In the previous subsection, we made an assumption that $\mathbb{P}(T|X) = \mathbb{P}(T|x)$ for all $X \in R_x$, where $R_x$ is a neighborhood of $x$. But in reality, this assumption is not always true, and that is why $w(X_i, x)$ plays an important rule in our final estimator, as it ``corrects" the empirical probability of each $T_i$ at $x$.

For example, say we have $n$ data points $\{(X_i, T_i)\}_{i=1}^n$ and have two cases: (1) at all $X_i$'s we have the same conditional probability of $T$, i.e. $\mathbb{P}(T|X_i) = \mathbb{P}(T|X_j)$ for all $i,j$; (2) $T$ has different conditional probabilities at different locations. In the setting (1), $X_i$'s become irrelevant and the point mass on each $T_i$ is $1/n$. We share the mass uniformly to the $n$ points $T_i$'s as they are equally important. When $n \to \infty$, it is known that for any $q$,
\begin{equation} \label{eq:point_mass}
\frac{1}{n} \sum_{i=1}^n \ind (T_i \le q) \to \mathbb{P}(T \le q|x).
\end{equation}
However, in the case (2), the convergence \eqref{eq:point_mass} is no longer valid. We cannot simply put a mass $1/n$ on each $T_i$ because the probability of $T_i$ showing up at $X_i$ could be severely different than the probability it shows up at $x$. An extreme example is when $\mathbb{P}(T|x) = \text{Unif}(x-1,x+1)$. Then if $|X_i - x| > 1$, any $T_i$ showing up at $X_i$ should not even be counted when estimating $\mathbb{P}(T|x)$ because $\mathbb{P}(T_i|x) = 0$. In another word, we should give $T_i$ mass $0$ instead of $1/n$.

Therefore, a measure of ``similarity" between points $X_i$ and $x$ needs to come into play, because we can no longer uniformly distribute the mass; observe that   some $T_i$'s are more important than others for estimating $\mathbb{P}(T|x)$. For instance, if $X_i = x + 0.01$ and $X_j = x + 2$ in the previous example, then $T_i$ should be assigned much more weight than $T_j$.

Now let $w(X_i,x)$ denote the weight (mass) we assign to $T_i$ when we are estimating $\mathbb{P}(T|x)$. In the setting (1), we just have $w(X_i,x) = 1/n$ uniformly. But in the setting (2), we should have $w(X_i, x) > w(X_j, x)$ when $X_i$ is more similar to $x$ than $X_j$ in some sense. Therefore, the estimator for $\mathbb{P}(T \le q|x)$ is then
\begin{equation*}
\sum_{i=1}^n w(X_i, x) \ind(T_i \le q)
\end{equation*}
and it becomes clear that a proper weight $w(X_i,x)$ needs to satisfy: 
\begin{equation} \label{eq:conditions}
(1) \; \sum_{i=1}^n w(X_i, x) = 1; \; \;
(2) \sum_{i=1}^n w(X_i, x) \ind(T_i \le q) \overset{p}{\to} \mathbb{P}(T \le q | x) \; \forall q.
\end{equation}
One may naively think that any fixed Kernel weights, $K(X_i, x)$, could be a suitable choice. However, they would not be able to satisfy the second condition   in \eqref{eq:conditions} for any distribution $\mathbb{P}(T | x)$.
  Fortunately, as shown in \cite{meinshausen2006quantile}, the data-adaptive random forest weight $w(X_i, x)$ introduced in Section \ref{ssec:rf_weights} perfectly satisfy both conditions in \eqref{eq:conditions}. And therefore going back to \eqref{eq:fop_latent}, we  now consider estimating equations,
\begin{equation}
U_n(q_{\tau,x}) = (1-\tau) - \sum_{i=1}^n w(X_i, x) \ind(T_i > q_{\tau,x}) \overset{p}{\to} 0
\end{equation}
when $n \to \infty$. Then following the same logic of how we get \eqref{eq:fop_Y_approx_KM}, a heuristic extension of \eqref{eq:fop_latent} to censoring case will be
\begin{eqnarray} \label{eq:final_fop}
S_n(q; \tau) &=& \sum_{i=1}^n w(X_i, x) \Bigl\{ (1-\tau)\hat{G}(q|x) - \ind (Y_i > q) \Bigr\} \nonumber \\
&=& (1-\tau)\hat{G}(q|x) - \sum_{i=1}^n w(X_i, x)\ind (Y_i > q).
\end{eqnarray}

\subsection{Forest Algorithm}
\label{subsec:meth3}
In the simplified example in Section \ref{subsec:meth1}, we assume that $Y$ has the same conditional probability $\mathbb{P}(Y|X)$ in a neighborhood $R_x$ of $x$, and hence, we can estimate $G(q|x)$ by Kaplan-Meier estimator \citep{kaplan1958nonparametric} (assuming no tied events)
\begin{eqnarray} \label{eq:KM}
\hat{G}(q|x) &=& \prod_{i:X_{(i)} \in R_x, Y_{(i)}\le q} \left( 1 - \frac{1}{k-i+1} \right)^{1-\delta_{(i)}} \nonumber \\
&=& \prod_{i:X_i \in R_x, Y_i\le q} \left( 1 - \frac{1}{\sum_{j=1}^n \ind(Y_j \ge Y_i) \ind(X_j \in R_x)} \right)^{1-\delta_i}
\end{eqnarray}
where $k = |R_x|$. In the more complex case like in Section \ref{subsec:meth2}, many consistent estimators for the conditional survival functions exist. For example, the nonparametric estimator proposed by \cite{beran1981nonparametric}
\begin{equation} \label{eq:Beran}
\tilde{G}(q|x) = \prod_{Y_i \le q} \left\{ 1 - \frac{W_i(x, a_n)}{\sum_{j=1}^n \ind(Y_j \ge Y_i) W_j(x,a_n)} \right\}^{1-\delta_i}
\end{equation}
is shown to be consistent \citep{beran1981nonparametric,dabrowska1987non,dabrowska1989uniform,gonzalez1994asymptotic,akritas1994nearest,li1995approach,van1996uniform}. Here, $W_i(x,a_n)$ are the Nadaraya-Watson weights
\begin{equation*}
W_i(x,a_n) = \frac{K((x-X_i)/a_n)}{\sum_{j=1}^n K((x-X_j) / a_n)},
\end{equation*}
$K(\cdot)$ is a known kernel and $\{a_n\}$ is a bandwidth sequence tending to zero as n tends to infinity. We can then simply use $\tilde{G}(q|x)$ as $\hat{G}(q|x)$ in \eqref{eq:final_fop}.

However, since we already have an adaptive version of kernel -- the random forest weights $w(X_i, x)$, we will propose the following two  new estimators for $G(q|x)$.

\paragraph{Kaplan-Meier using nearest neighbors.}
The first estimator is resembles that  of \eqref{eq:KM}. We first find the $k$ nearest neighbors of $x$ according to the weights $w(X_i, x)$. Denoting these points as a set $N_x$, then we can simply use the Kaplan-Meier estimator on $N_x$
\begin{equation} \label{eq:KM_kNN}
\hat{G}(q|x) = \prod_{i:X_i \in N_x, Y_i\le q} \left( 1 - \frac{1}{\sum_{j=1}^n \ind(Y_j \ge Y_i) \ind(X_j \in N_x)} \right)^{1-\delta_i}.
\end{equation}
Here, the number of neighbors $k$ will be a tuning parameter.

\paragraph{Beran estimator with random forest weights.}
In the second proposal, we will replace the Nadaraya-Watson weights in \eqref{eq:Beran} with random forest weights and get
\begin{equation} \label{eq:Beran_rf}
\hat{G}(q|x) = \prod_{Y_i \le q} \left\{ 1 - \frac{w(X_i, x)}{\sum_{j=1}^n \ind(Y_j \ge Y_i) w(X_j, x)} \right\}^{1-\delta_i}.
\end{equation}
One could observe that \eqref{eq:KM_kNN} is a special case of \eqref{eq:Beran_rf} when the weight $w(X_i, x) = 1/k$ for $X_i \in R_x$ and $0$ otherwise.

Finally, we summarize our main algorithm in Algorithm \ref{alg:main}. The details for choosing the candidate set $\mathcal{C}$ is in Section \ref{subsec:alg}. The choice to minimize the absolute value of $S_n(q;\tau)$ is arbitrary. The goal is to find the approximate root of $S_n(q;\tau) = 0$.

\begin{algorithm}
\caption{Censored quantile regression forest}
\label{alg:main}
\begin{algorithmic}[2]
\State {\hskip -5pt  \small All tuning parameters, including the number of trees, $B$, in the forest as well as the minimum size of the leaves, $m$, in the individual trees  are pre-determined. If using the KM-estimator \eqref{eq:KM_kNN}, the number of nearest neighbors, $k$, is also given.}
\Procedure{Forest-CQR} {test point $x$, training set $ \mathcal{D}=\{(X_i, Y_i, \delta_i)\}_{i=1}^n$, quantile $\tau$} 
\State tree $\mathcal{T} \gets$ {\sc regression forest} $(\mathcal{D})$ with $B$ trees and leaf sizes $m$
%Grow a normal regression forest on the training set with each leaf containing at least $k$ observations.
\State random forest weights $w(x, X_i) \gets$  \eqref{eq:rf_weight}, for all $i$
\State survival function $\hat{G}(q|x) \gets$  \eqref{eq:KM_kNN} or \eqref{eq:Beran_rf}
\State quantile estimator $\hat q$ such that 
$$\hat q \gets \argmin_{q \in \mathcal{C}} |S_n(q; \tau)|$$

\Comment{$\mathcal{C}$ is a candidate set as discussed in Section \ref{subsec:alg}}
\Comment{ $S_n$ is \eqref{eq:final_fop}} \label{step:4}

\Return  $\hat q $ \Comment{The $\tau$-th quantile of  survival time $T$ at $x$}
\State{\hskip -20pt {\small The function {\sc random forest} creates  a regression tree and can be implemented to be as any of the standard forest algorithms.} }
\EndProcedure
\end{algorithmic}
\end{algorithm}

\section{Theoretical Develoments}
\label{sec:theory}
In this section, we will assume the random forest has terminal node size $m$, feature vector $X_i \in \mathbb{R}^{p}$, sample size is $n$, and $k$ nearest neighbors are chosen if using \eqref{eq:KM_kNN}.

\subsection{Time complexity} \label{subsec:alg}
The step \ref{step:4} in Algorithm \ref{alg:main} involves of finding the $q^*$ in a candidate set $\mathcal{C}$ that sets the estimating equation $S_n(q; \tau)$ closest to zero. We simply evaluate the function $S_n(q;\tau)$ for all possible $q$ in $\mathcal{C}$ and find the minimum point. Note that for any fixed $\tau$, $S_n(q; \tau)$ is a step function in $q$ with jumps at $Y_i$'s because the discontinuities only happen at $Y_i$'s for $\hat{G}(q|x)$ (both \eqref{eq:KM_kNN} and \eqref{eq:Beran_rf}) and $\sum_{i=1}^n w(X_i, x) \ind(Y_i > q)$. Therefore, the candidate set $\mathcal{C} \subset \{Y_i\}_{i=1}^n$, and $|\mathcal{C}| = n$ in the worst case.

But in fact, for any fixed $x$, only $Y_i$'s with the corresponding feature vector $X_i \in R_x$ \eqref{eq:KM_kNN} or with $w(X_i, x) > 0$ \eqref{eq:Beran_rf} will be jump points, and hence, we can refine $\mathcal{C} = \{Y_i: X_i \in R_x\}$ for \eqref{eq:KM_kNN} or $\mathcal{C} = \{Y_i: w(X_i, x) > 0\}$ for \eqref{eq:Beran_rf}. We then have the following theorem. The proof is given in the Appendix \ref{sec:appex}. 

\begin{theorem}\label{thm:complexity}
For a fixed test point $x$, depending on whether $G(q|X)$ is estimated by \eqref{eq:KM_kNN} or \eqref{eq:Beran_rf}, the time complexity for Algorithm \ref{alg:main} is $O(n \max\{k, \log(n)\})$ or $O(n m \log(n)^{p-1})$, respectively.
\end{theorem}

\subsection{Consistency}
\label{subsec:consistency}
In this section, we will show that for any fixed $\tau \in (0,1)$, $S_n(q; \tau)$ in \eqref{eq:final_fop} will converge in probability to $(1-\tau)G(q|x) - \mathbb{P}(Y_i > q)$ uniformly for $q$.

\begin{condition}\label{cond:1}
The density of $X$ is positive and bounded from above and below by positive constants on the support $\mathcal{X}$.
\end{condition}

We note that Condition \ref{cond:1} is a very primitive condition on the distribution of the covariates. It is satisfied for example for Gaussian distribution and more broadly for most symmetric, continuous distributions with unbounded support. The case of bounded or discrete covariates is beyond the scope of the current work.

\begin{condition}\label{cond:2}
The terminal node size $m \to \infty$ and $m/n \to 0$ as $n \to \infty$. Furthermore, for each tree splitting, the probability that each variable is chosen for the split point is bounded from below by a positive constant, and every child node contains at least $\gamma$ proportion of the data in the parent node, for some $\gamma \in (0, 0.5]$.
\end{condition}

The two requirements of Condition \ref{cond:2} are also required in \cite{meinshausen2006quantile} (see   Assumptions 2 and 3 therein). This condition states that the leaf node size of each tree should increase with the sample size $n$, but at a slower rate. Intuitively, first, the trees that we are using need to be shallow (i.e., with large leaves) in order to estimate a more complex model, reliably. Secondly, there can not be leaves with no samples, i.e., each leaf must be large enough to capture the local estimating equations more adequately. Our experiments also justify the necessity of Condition \ref{cond:2}, as the performance of our model, will deteriorate if we keep a small leaf node size but increase the sample size. We will talk about this in detail in Section \ref{sssec:nodesize}.

\begin{condition}\label{cond:3}
Denote $F(y|x) = \mathbb{P}(Y \le y | x)$. There exists a constant $L$ such that $F(y|x)$ is Lipschitz continuous with parameter $L$, that is, for all $x, x^{'} \in \mathcal{X}$,
\begin{equation*}
\sup_{y} |F(y|x) - F(y|x^{'})| \le L \|x - x^{'}\|_1.
\end{equation*}
\end{condition}

We note that Condition \ref{cond:3} appears in all existing work related to quantile regression and inference thereafter.  

\begin{condition}\label{cond:4}
The  response variable $T$ and the censoring variable $C$ are   independent conditional on $X$, and the conditional distribution $\mathbb{P}(T \le q|x)$ and $\mathbb{P}(C \le q|x)$ are both positive and strictly increasing in $q$ for all $x \in \mathcal{X}$.
\end{condition}

Conditional independence of $T$ and $C$ is a very standard assumption and can be traced back to \cite{robins1992semiparametric} among other works.

\begin{condition}\label{cond:5}
For any $x \in \mathcal{X}$, the estimator $\hat{G}(q|x)$ converges pointwisely to the true conditional survival function $G(q|x)$.
\end{condition}

Condition \ref{cond:5} is satisfied, for example, by the Kaplan-Meier estimator \eqref{eq:Beran} \citep{dabrowska1989uniform}. Please take a look at Figure \ref{fig:g_comparison} and Figure \ref{fig:g_comparison_high} where we compare finite sample properties of the newly introduced estimators \eqref{eq:KM_kNN} and \eqref{eq:Beran_rf}. We observe that the new distributional estimators are more adaptive and yet seemingly inherit consistency to that of the traditional KM estimator.

We proceed to showcase asymptotic properties of the proposed estimating equations. We begin by illustrating a concentration of measure phenomenon for the introduced score equations.

\begin{theorem}
\label{thm:part1}
Define
\begin{equation}
S(q;\tau) = (1-\tau)G(q|x) - \mathbb{P}(Y > q).
\end{equation}
Under Conditions \ref{cond:1} -- \ref{cond:5}, for any $x \in \mathcal{X}$, $r > 0$, $\tau \in (0,1)$,
\begin{equation*}
\sup_{q \in [-r,r]} | S_n(q; \tau) - S(q; \tau) | = o_p(1).
\end{equation*}
\end{theorem}

Next, we present our main result that illustrates an asymptotic consistency of the proposed conditional quantile estimator. The proof is given in Appendix \ref{sec:appex}.

\begin{theorem}
\label{thm:consistency}
Under Conditions \ref{cond:1} -- \ref{cond:5}, for fixed $\tau \in (0,1)$ and $x \in \mathcal{X}$, define $q^*$ to be the root of $S(q;\tau) = 0$, and $r > 0$ to be some constant so that $q^* \in [-r, r]$. Also define $q_n$ to be $\argmin_{q \in [-r, r]} \left|S_n(q; \tau)\right|$. Then 
$$\mathbb{P}(T \le q^*|x) = \tau,
$$ 
and as $n \to \infty$,
$$q_n \overset{p}{\to} q^*.$$

\end{theorem}

\section{Experiments}
\label{sec:experiments}
In this section, we will compare our model, censored forest regression (\textit{crf}) with generalized random forest (\textit{grf}) \citep{athey2016generalized}, quantile random forest (\textit{qrf}) \citep{meinshausen2006quantile} and random survival forest (\textit{rsf}) \citep{hothorn2005survival} on simulated and real data sets.

On the simulated data sets, we will apply \textit{qrf} and \textit{grf} to the censored data directly, and get biased models which we denote by \textit{qrf} and \textit{grf}, respectively. We also apply \textit{qrf} and \textit{grf} to the data with uncensored responses, and call the resulted models \textit{qrf-oracle} and \textit{grf-oracle}.

Throughout this section, we fix the number of trees for each forest to be 1000. The only tuning parameter we have is the node size of each tree. All other parameters are kept as default.

\subsection{Illustrative  example}
\label{subsec:one-d}
In this section, we generate the true response, i.e., latent,  variables $T_i \sim \textrm{Unif}(0,1)$, and consider the censoring variables $C_i \sim \mathcal{N}(0.8, 0.2^2)$. The censored responses, i.e. observed responses, are then taken as $Y_i = \min(T_i, C_i)$. We compare the estimating equation on the latent variables $T_i$
\begin{equation*}
U_1(q) = (1-\tau) - \frac{1}{n} \sum_{i=1}^n \ind(T_i > q)
\end{equation*}
to the estimating equation of our proposed algorithm
\begin{equation*}
U_2(q) = (1-\tau)\hat{G}(q) - \frac{1}{n} \sum_{i=1}^n \ind(Y_i > q),
\end{equation*}
where $\hat{G}(q)$ is the one-dimensional Kaplan-Meier estimator for the survival function of censoring variable $C$. The results are  shown in Figure \ref{fig:loss_plot}. We consider $\tau=0.5$ but the results persist for many other choices of $\tau$.

In Figure \ref{fig:loss_plot} we present the two estimating equations as functions of $q$ and illustrate that the solutions to  $U_1(q)=0$ and $U_2(q)=0$ are closer and closer together when the sample size grows. The solution for $U_1(q)=0$ can be treated as an oracle solution where the oracle observes ``uncensored" (true) response variable. Figure \ref{fig:loss_plot} therefore indicated that 
 the root of our method's estimating equation is very close to the oracle root  and that we are therefore finding a good approximation to the unknown parameter of interest.

%It means that by solving $U_2(q) = 0$, we have a good approximation of the root of $U_1(q)=0$, and hence find the $\tau$-th quantile of the latent variable $T$.

\begin{figure}[!htb]
\small
\centering
\includegraphics[width=0.618\linewidth]{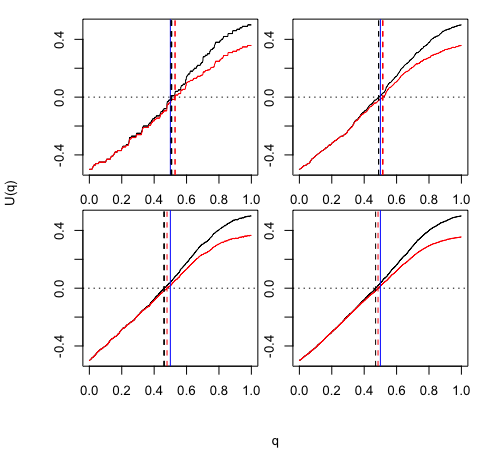}
\caption{Sample loss plot for different sample size when $\tau = 0.5$. For upper left figure, we have sample size $100$, for upper right, $n = 500$, for lower left, $n = 1000$, and for lower right, $n = 5000$. The black curve is the value of $U_1(q)$, the red cureve is the value of $U_2(q)$, the black dotted vertical line is the root of $U_1(q)$, the red dotted vertical line is the root of $U_2(q)$, and the blue vertical line is $q = \tau$. \label{fig:loss_plot}}
\end{figure}

\subsection{Simulation results}
\label{ssec:simulation}
In this section, we will compare different forest algorithms on simulated data sets including accelerated failure time model (AFT) and many non-parametric censored regression models.

\subsubsection{One-dimensional AFT model}
\label{sssec:1d-aft}
We simulate data from an one-dimensional AFT model
\begin{equation*}
    \log(T) = X + \epsilon
\end{equation*}
where $X \sim \textrm{Unif}(0,2)$ and $\epsilon \sim \mathcal{N}(0, 0.3^2)$. Then the censoring variable $C \sim \textrm{Exp}(\lambda = 0.08)$, and the observed response $Y = \min(T, C)$ and the censoring indicator $\delta = \ind(T \le C)$. The average censoring rate is about $20\%$. The number of training data, validation data and test data are all 300. All the forests consist of 1000 trees. The node size of each forest is determined by cross--validation. We plot out one set of training data and the corresponding quantile predictions for $\tau = 0.3, 0.5, 0.7$ on a set of test data in Figure \ref{fig:aft_1d}. We only show the results of \textit{crf}, \textit{grf}, and \textit{grf-oracle} because in one dimension, \textit{qrf}'s performance is visually indistinguishable from \textit{grf}. There we observe a consistency of our method as well as superior behavior to the competing methods. Namely, the   generalized random forest that ignores the censoring component of the data, incurs large bias; due to the right censoring, bias is larger for lower values of the quantiles. We observe that the proposed {\it crf} follows closely the oracle estimator and is extremely close to the true quantile regardless of the $\tau$ in the study. 

\begin{figure}[!htb]
    \small
    \centering
    \begin{subfigure}[b]{0.4\linewidth}
        \includegraphics[width=\textwidth]{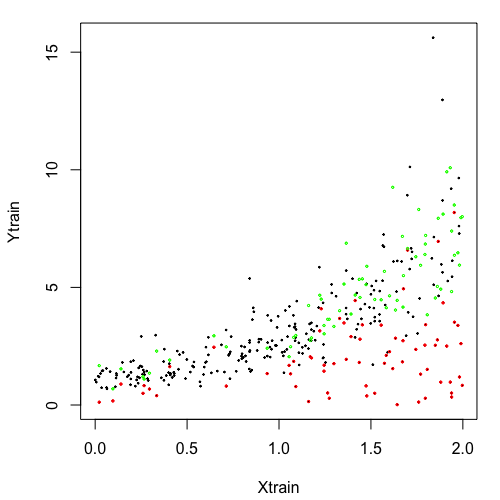}
        \caption{Training data}
        \label{fig:aft_train}
    \end{subfigure}
    ~ %add desired spacing between images, e. g. ~, \quad, \qquad, \hfill etc. 
    \begin{subfigure}[b]{0.4\linewidth}
        \includegraphics[width=\textwidth]{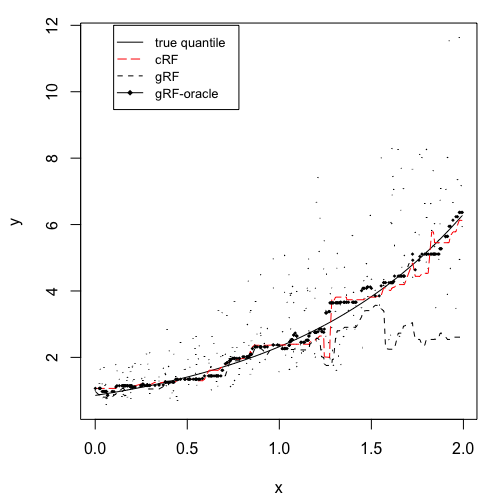}
        \caption{$\tau = 0.3$}
        \label{fig:aft_1d_tau_03}
    \end{subfigure}
    
    \vspace{-0.05in}
    
    \begin{subfigure}[b]{0.4\linewidth}
        \includegraphics[width=\textwidth]{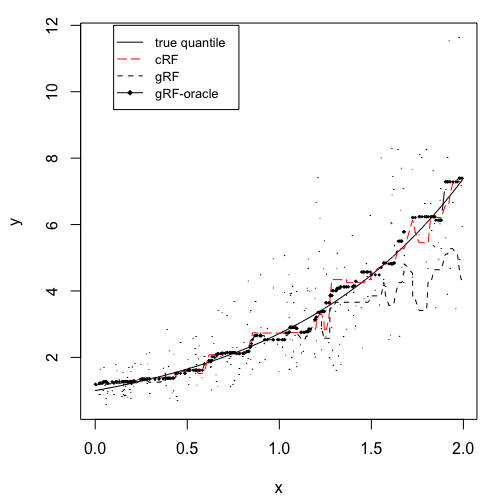}
        \caption{$\tau = 0.5$}
        \label{fig:aft_1d_tau_05}
    \end{subfigure}
    ~
    \begin{subfigure}[b]{0.4\linewidth}
        \includegraphics[width=\textwidth]{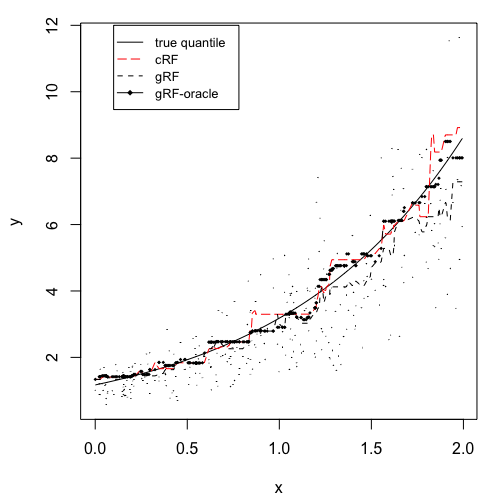}
        \caption{$\tau = 0.7$}
        \label{fig:aft_1d_tau_07}
    \end{subfigure}
    
    \caption{One-dimensional AFT model results with $n=300$ and $B=1000$. In (a), black points stand for observations that are not censored; red points are observations that are censored, i.e. the truly observed data points, and the green points are the counterpart of the red points, that is, they are the latent values of those red points if they were not censored.}
    \label{fig:aft_1d}
\end{figure}

Moreover, we proceed further and   for a set of values  $\tau \in \{0.1,0.3,0.5,0.7,0.9\}$, we repeat the process   40 times, and for each time, we calculate the MSE and MAD between the estimated quantiles and the true quantiles, and the $\tau$-th quantile loss. To be more specific, let $T_i$ be the response in test set (all uncensored), $Q^{\tau}_i$ be the true $\tau$-th quantile, and $\hat{Q}^{\tau}_i$ be the estimated quantile, then
\begin{equation}\label{eq:L_MSE}
    L_{MSE} = \frac{1}{n} \sum_{i=1}^n (\hat{Q}^{\tau}_i - Q^{\tau}_i)^2,
\end{equation}
\begin{equation}\label{eq:L_MAD}
    L_{MAD} = \frac{1}{n} \sum_{i=1}^n |\hat{Q}^{\tau}_i - Q^{\tau}_i|,
\end{equation}
\begin{equation}\label{eq:L_quantile}
    L_{quantile} = \frac{1}{n} \sum_{i=1}^n \rho_{\tau}(T_i - \hat{Q}^{\tau}_i).
\end{equation}
The reason we use $L_{quantile}$ to measure the quality of quantile predictions is that, by \cite{meinshausen2006quantile}, the $\tau$-th quantile of $T$ at $x$ equals to $\argmin_{q \in \mathbb{R}} \mathbb{E}[\rho_{\tau}(T - q) | X=x]$. The results are illustrated in Figure \ref{fig:aft_1d_box} where besides the abose three measures we compare the concordance index (C-index) \citep{harrell1982evaluating}, which is related to the area under the ROC curve \citep{heagerty2005survival}. It estimates the probability that, in a randomly selected pair of cases, the case that fails first had a worse predicted outcome. In \cite{ishwaran2008random}, they use the ensemble mortality as the predictive outcome for their random survival forest, and the predicted survival time for random forest regression. For our method \textit{crf} and the other two methods, \textit{qrf} and \textit{grf}, we will use the $\tau$-th conditional quantile as the predicted outcome. Since the outcomes will be different for different $\tau$, we report the results for all $\tau \in \{0.1,0.3,0.5,0.7,0.9\}$.

In Figure \ref{fig:aft_1d_box} we observe an oracle like behavior of the proposed {\it crf} method in terms all four measures of the quality of estimation and/or prediction. Namely, we observe that MAD, MSE and quantile losses are extremely small whereas C-index is high and all are close to the corresponding oracle estimators (colored purple and blue). Moreover, we observe that the proposed {\it crf} method, although not primarily build for the hazard rates,  is even better than survival random forest: see for example discrepancies between red and brown boxplots in the last row of Figure \ref{fig:aft_1d_box} where the larger the C-index is  the better the method is.

\begin{figure}[!htb]
    \small
    \centering
    \begin{subfigure}[b]{0.182\linewidth}
        \includegraphics[width=\textwidth]{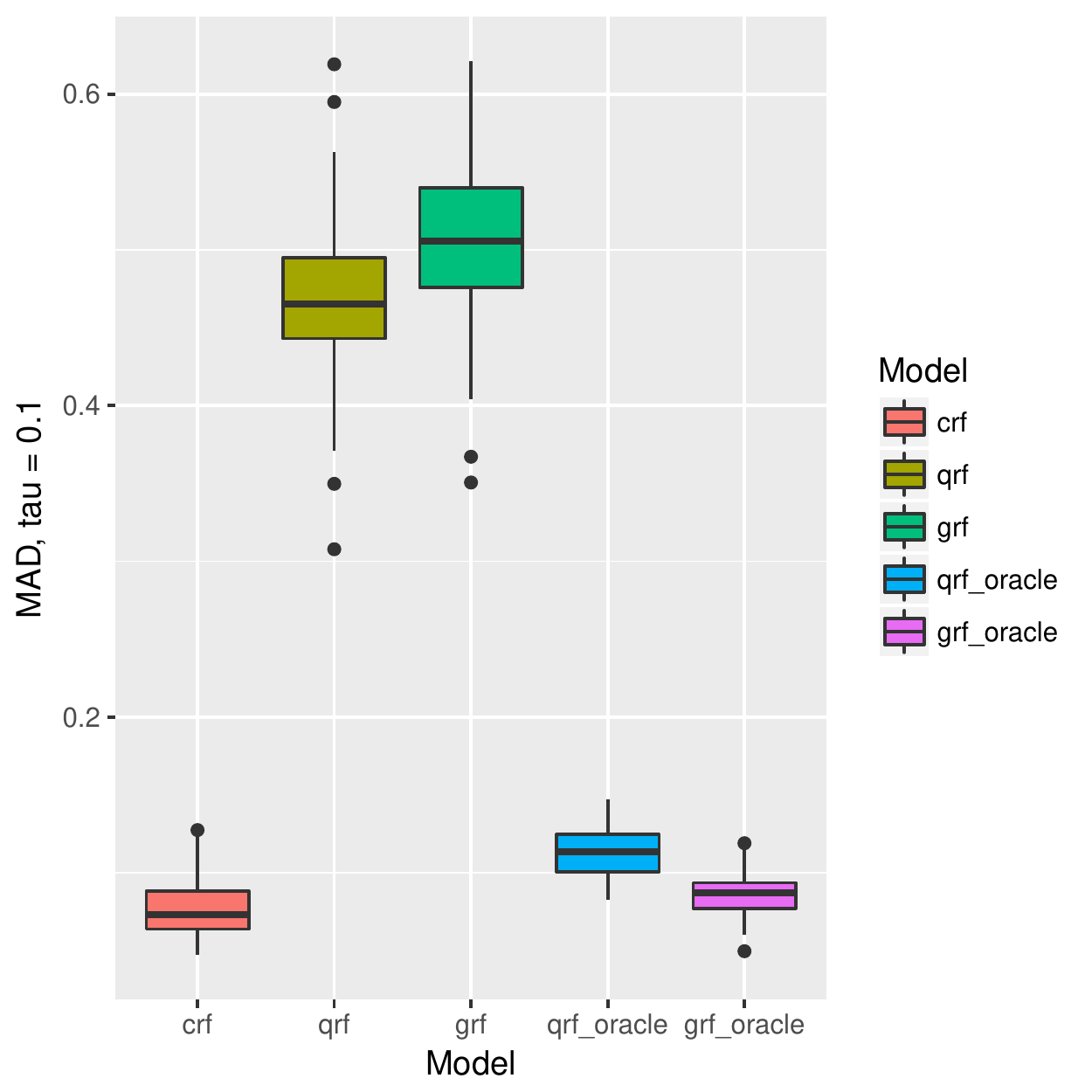}
        \caption{MAD: $\tau = 0.1$}
    \end{subfigure}
    ~
    \begin{subfigure}[b]{0.182\linewidth}
        \includegraphics[width=\textwidth]{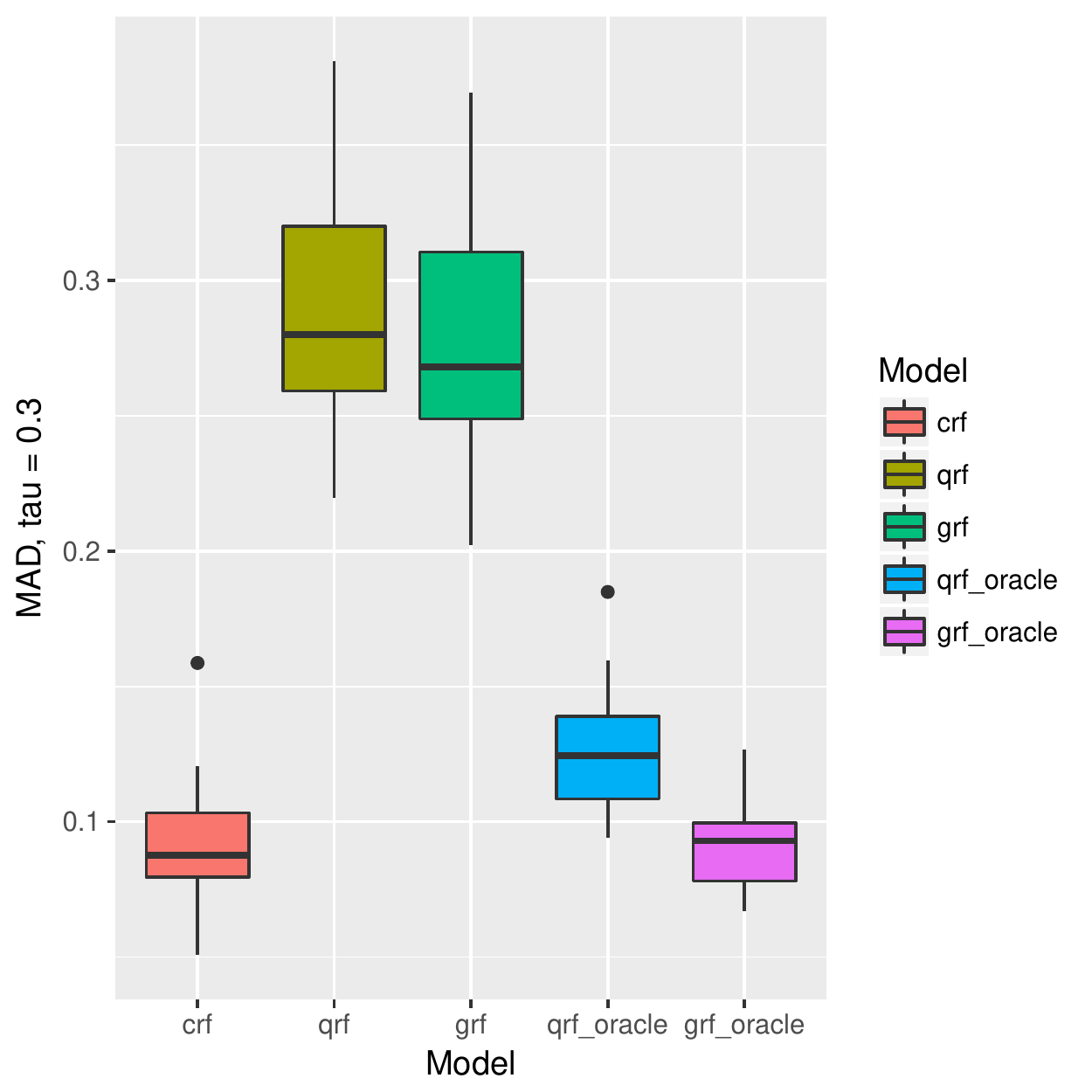}
        \caption{MAD: $\tau = 0.3$}
    \end{subfigure}
    ~
    \begin{subfigure}[b]{0.182\linewidth}
        \includegraphics[width=\textwidth]{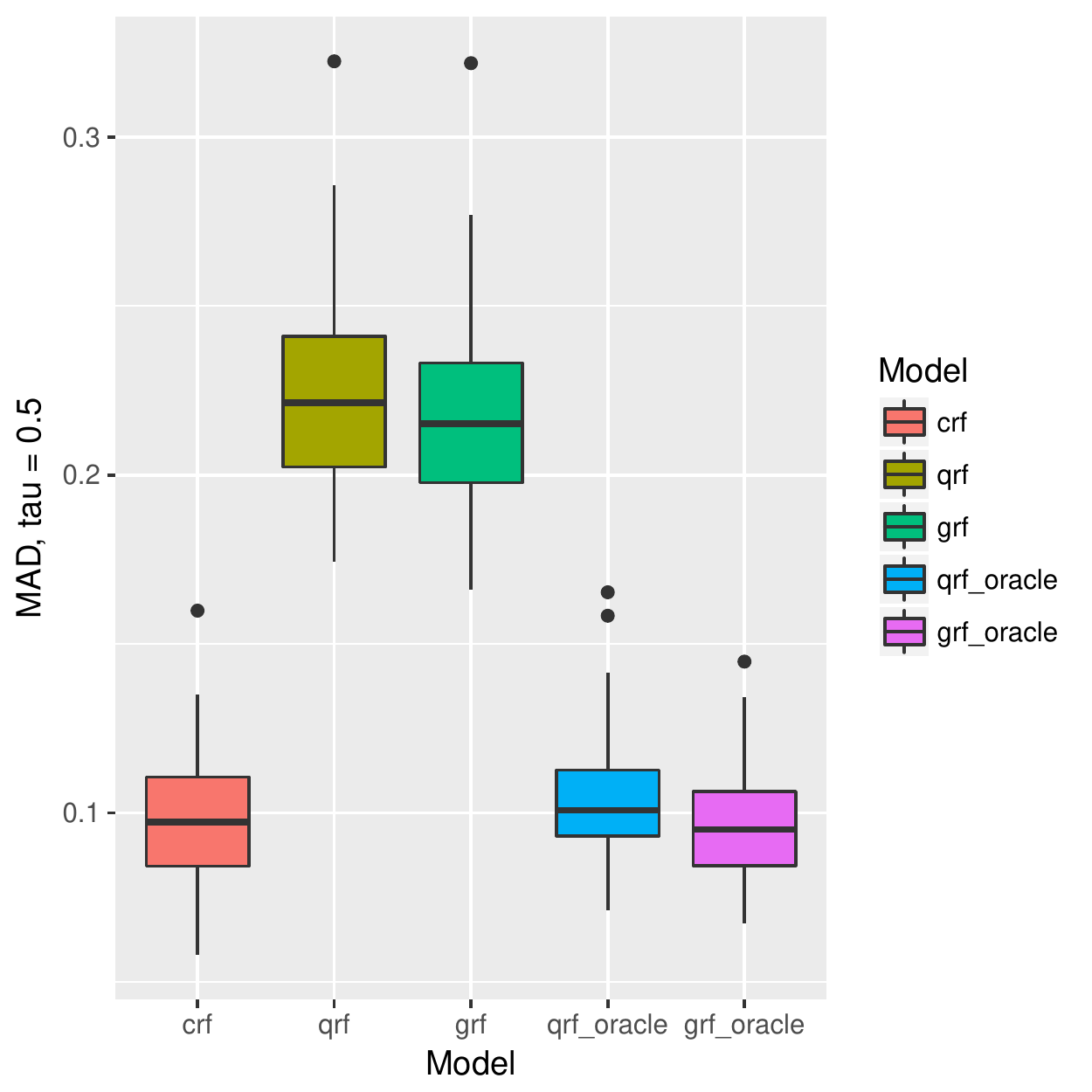}
        \caption{MAD: $\tau = 0.5$}
    \end{subfigure}
    ~
    \begin{subfigure}[b]{0.182\linewidth}
        \includegraphics[width=\textwidth]{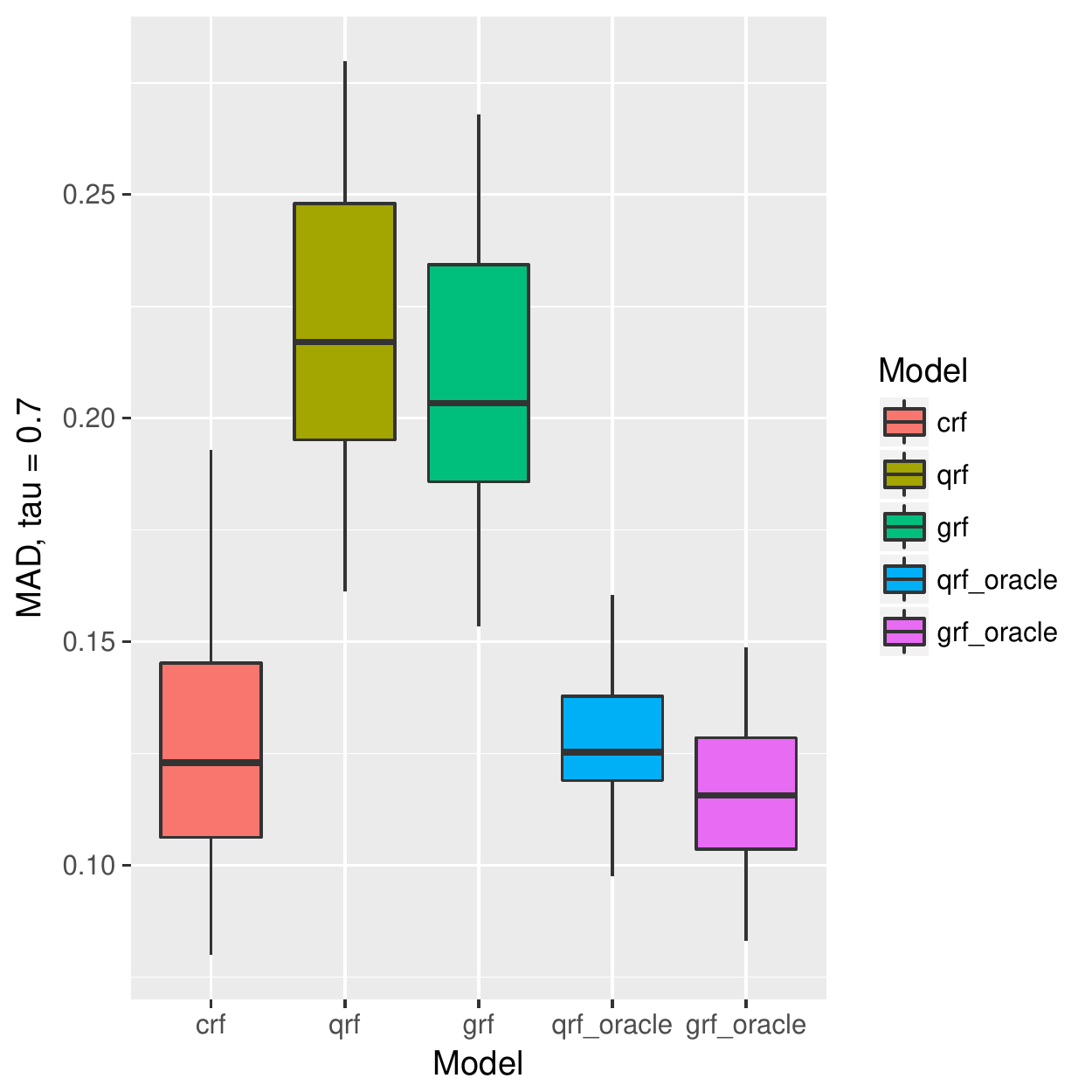}
        \caption{MAD: $\tau = 0.7$}
    \end{subfigure}
    ~
    \begin{subfigure}[b]{0.182\linewidth}
        \includegraphics[width=\textwidth]{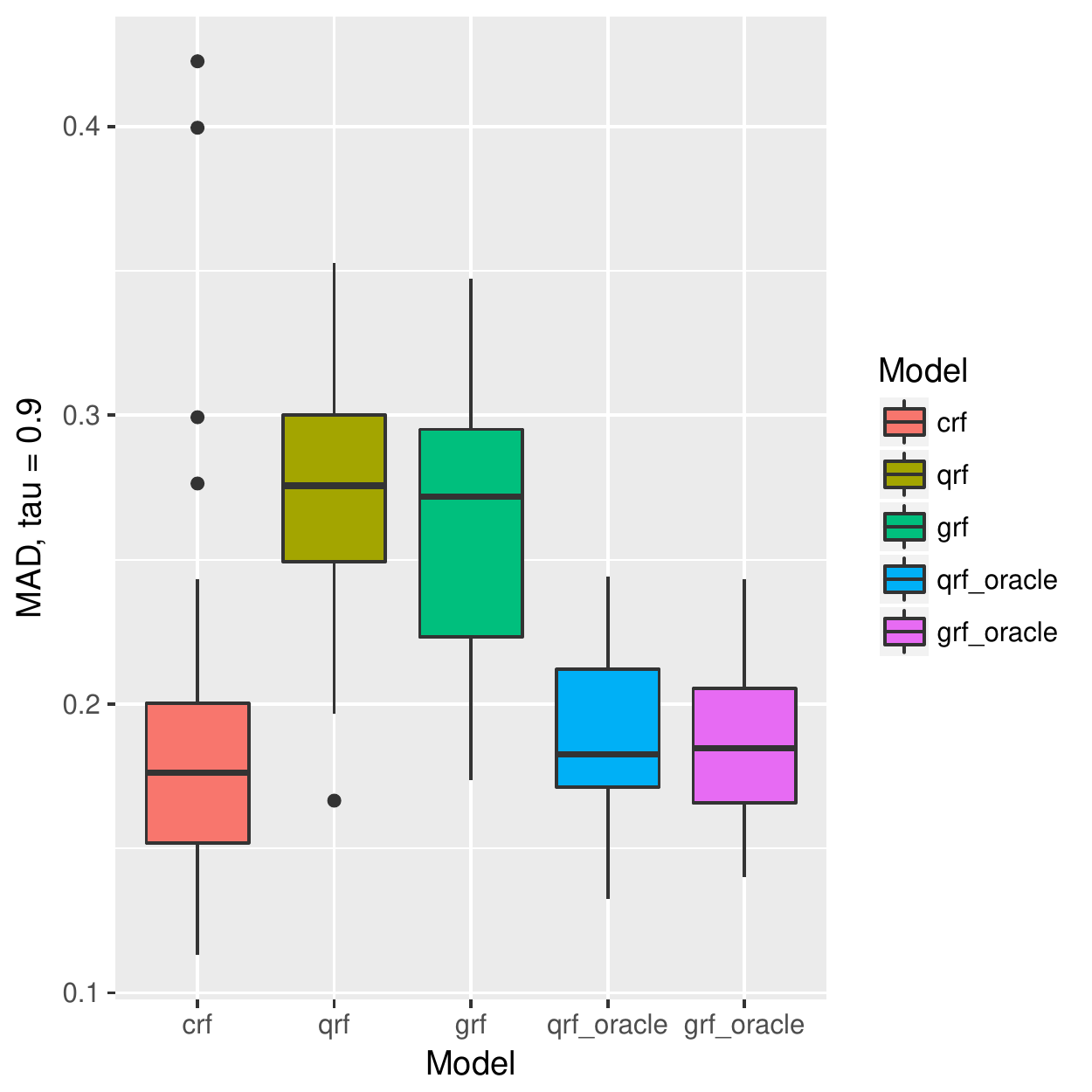}
        \caption{MAD: $\tau = 0.9$}
    \end{subfigure}
    
    \vspace{-0.05in}
    
    \begin{subfigure}[b]{0.18\linewidth}
        \includegraphics[width=\textwidth]{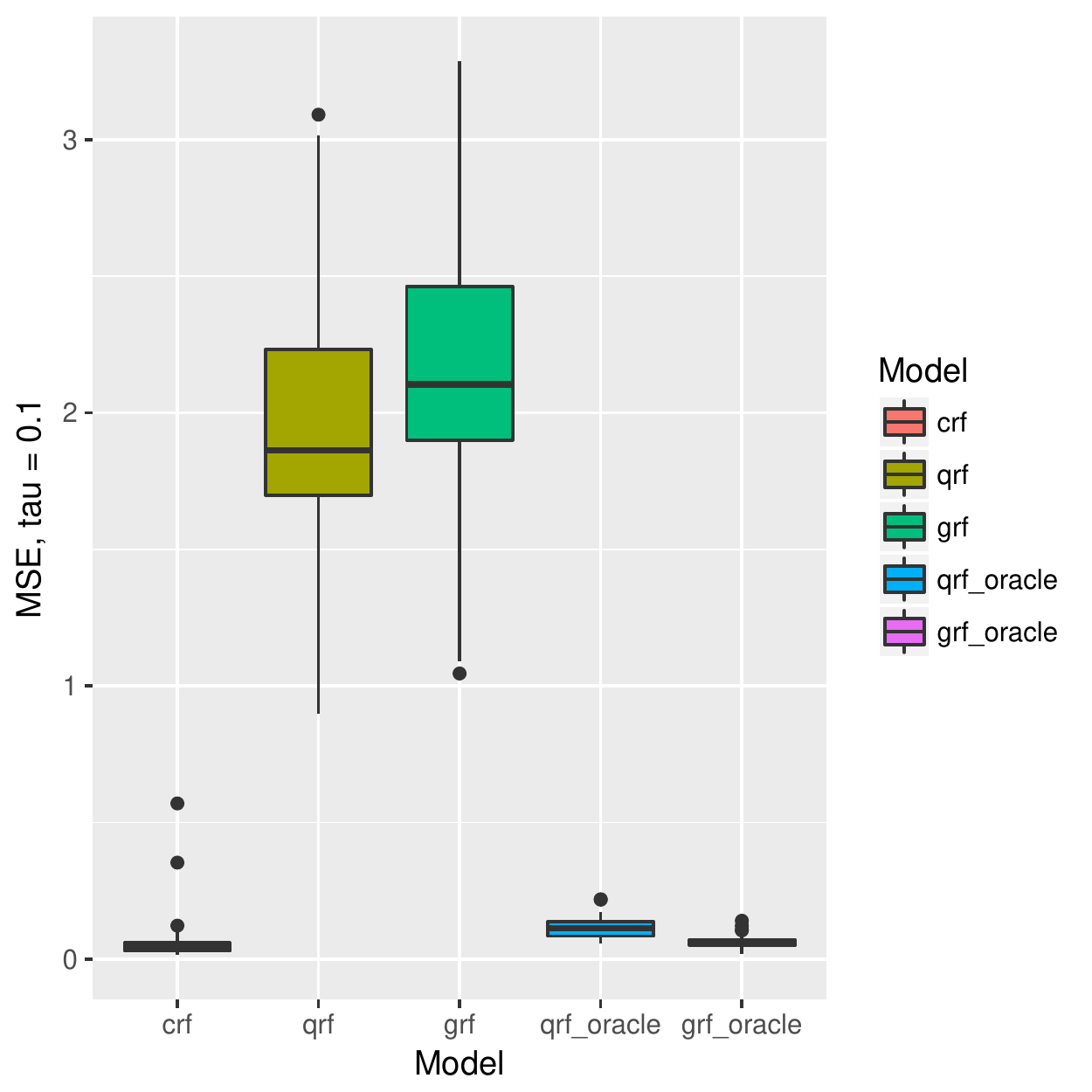}
        \caption{MSE: $\tau = 0.1$}
    \end{subfigure}
    ~
    \begin{subfigure}[b]{0.18\linewidth}
        \includegraphics[width=\textwidth]{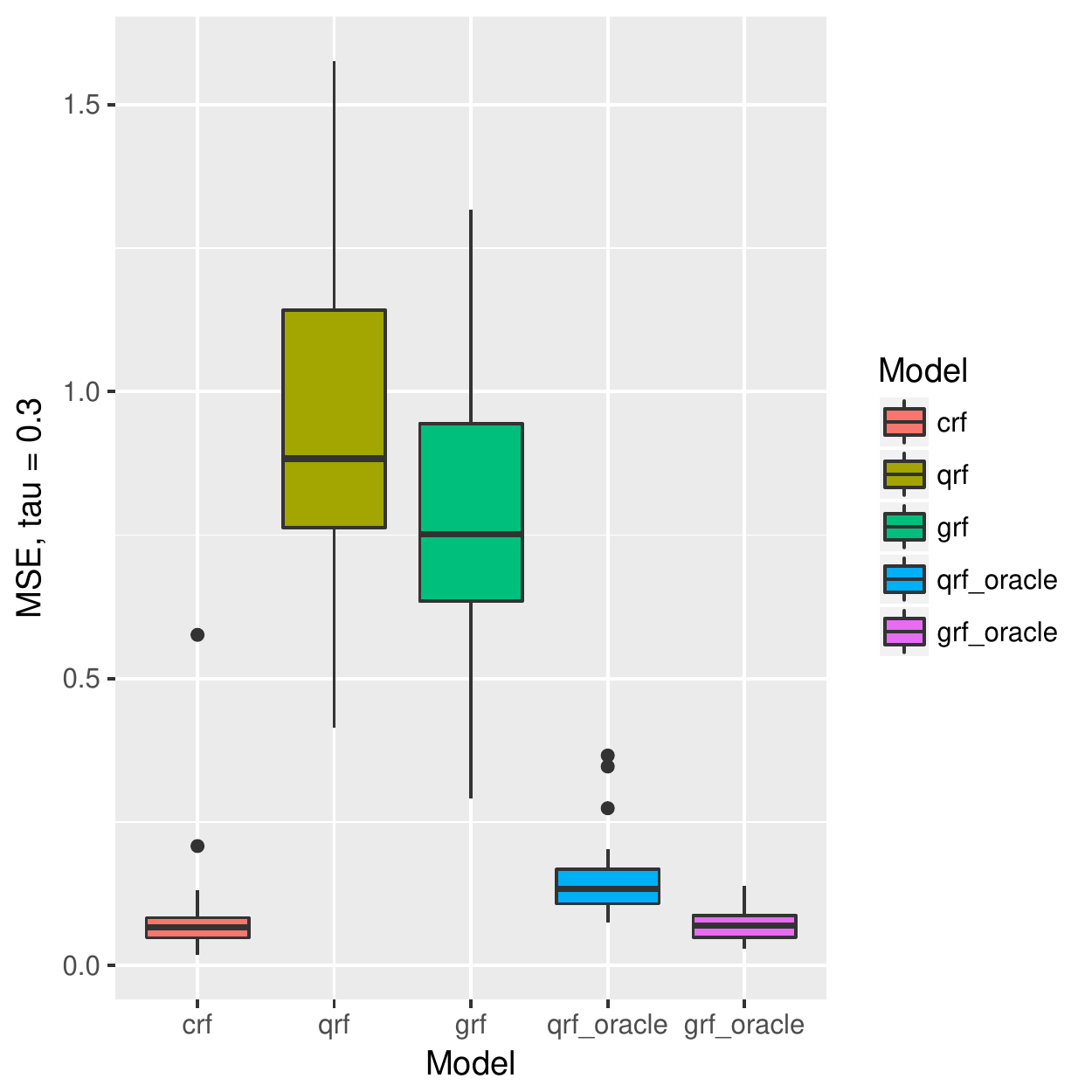}
        \caption{MSE: $\tau = 0.3$}
    \end{subfigure}
    ~ %add desired spacing between images, e. g. ~, \quad, \qquad, \hfill etc. 
    \begin{subfigure}[b]{0.18\linewidth}
        \includegraphics[width=\textwidth]{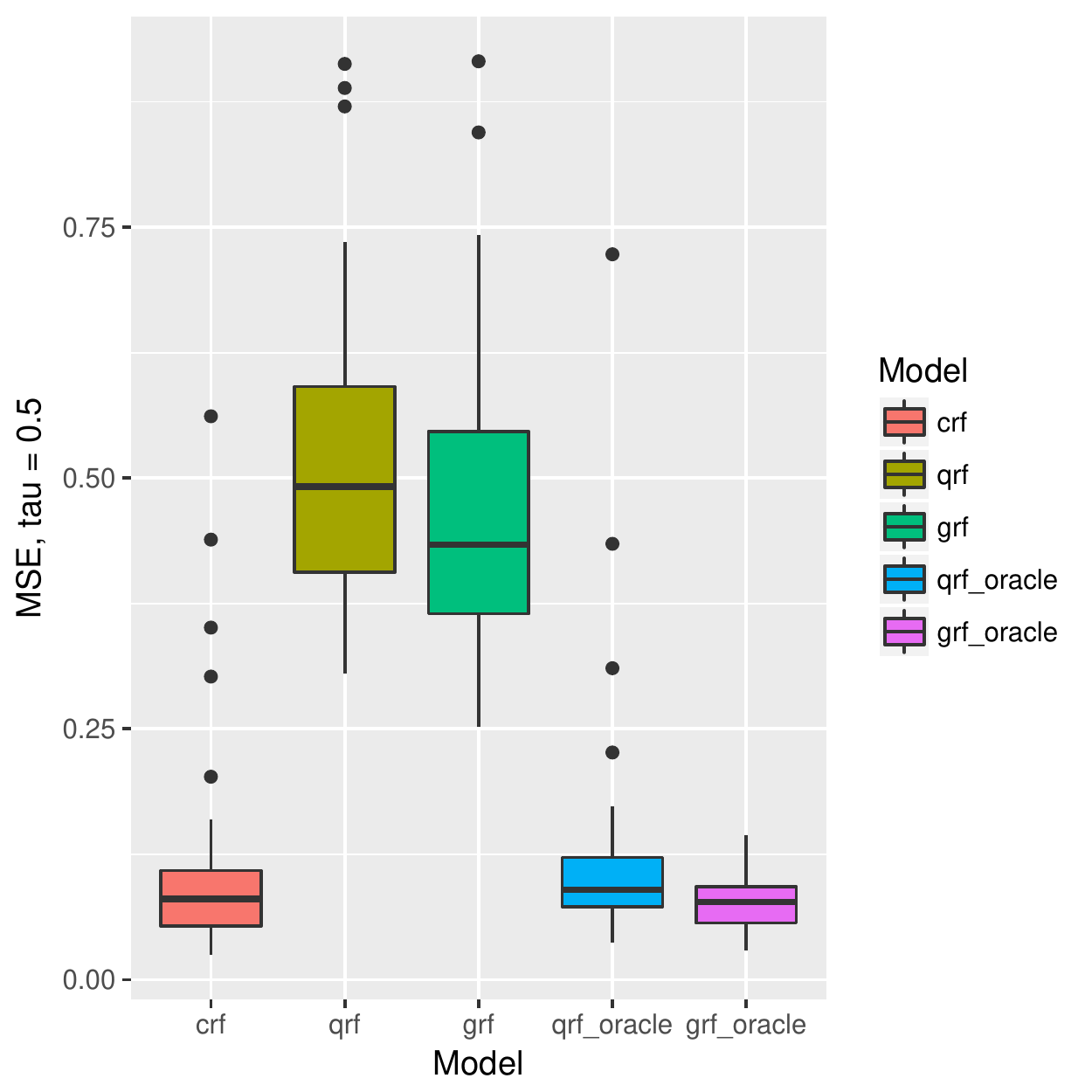}
        \caption{MSE: $\tau = 0.5$}
    \end{subfigure}
    ~
    \begin{subfigure}[b]{0.18\linewidth}
        \includegraphics[width=\textwidth]{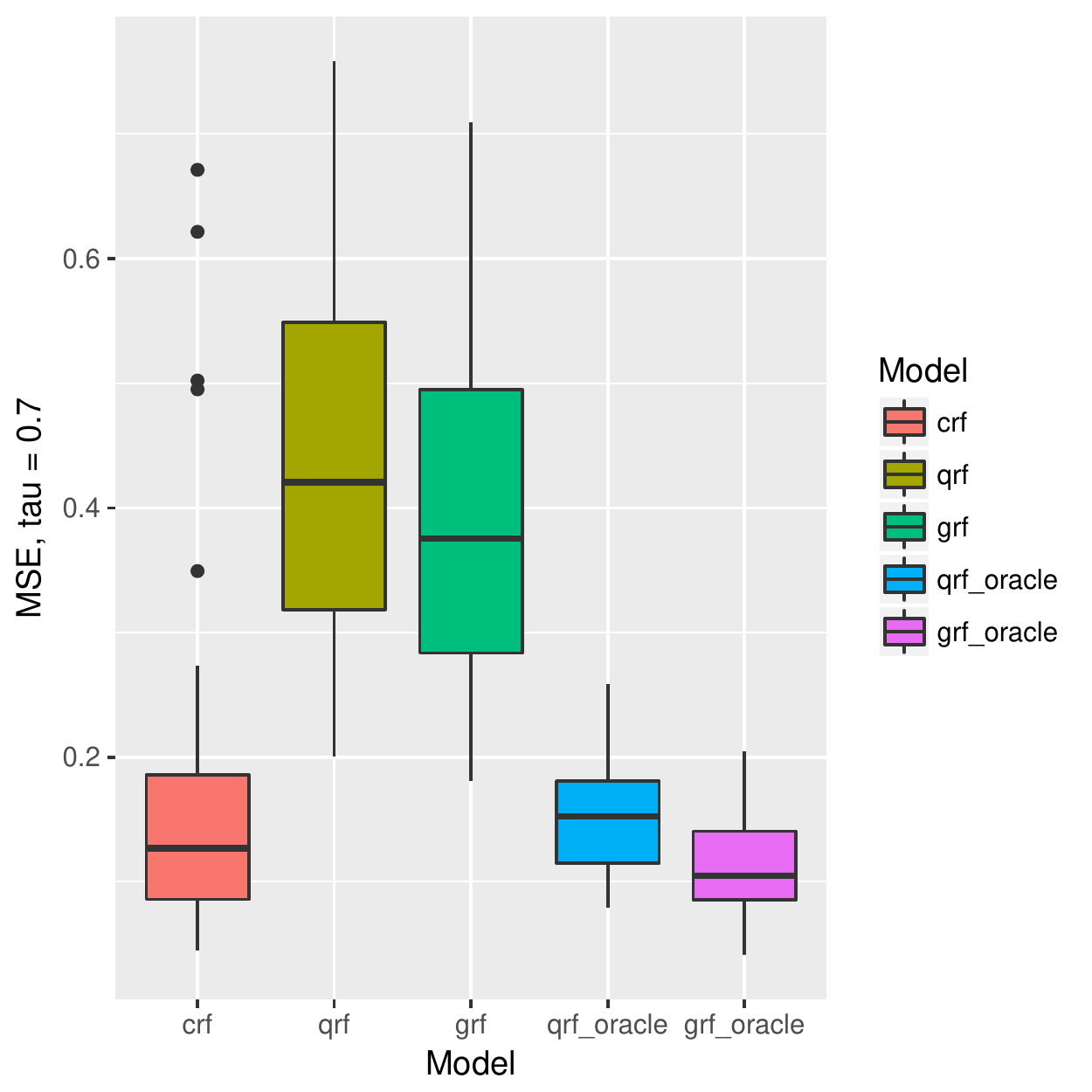}
        \caption{MSE: $\tau = 0.7$}
    \end{subfigure}
    ~
    \begin{subfigure}[b]{0.18\linewidth}
        \includegraphics[width=\textwidth]{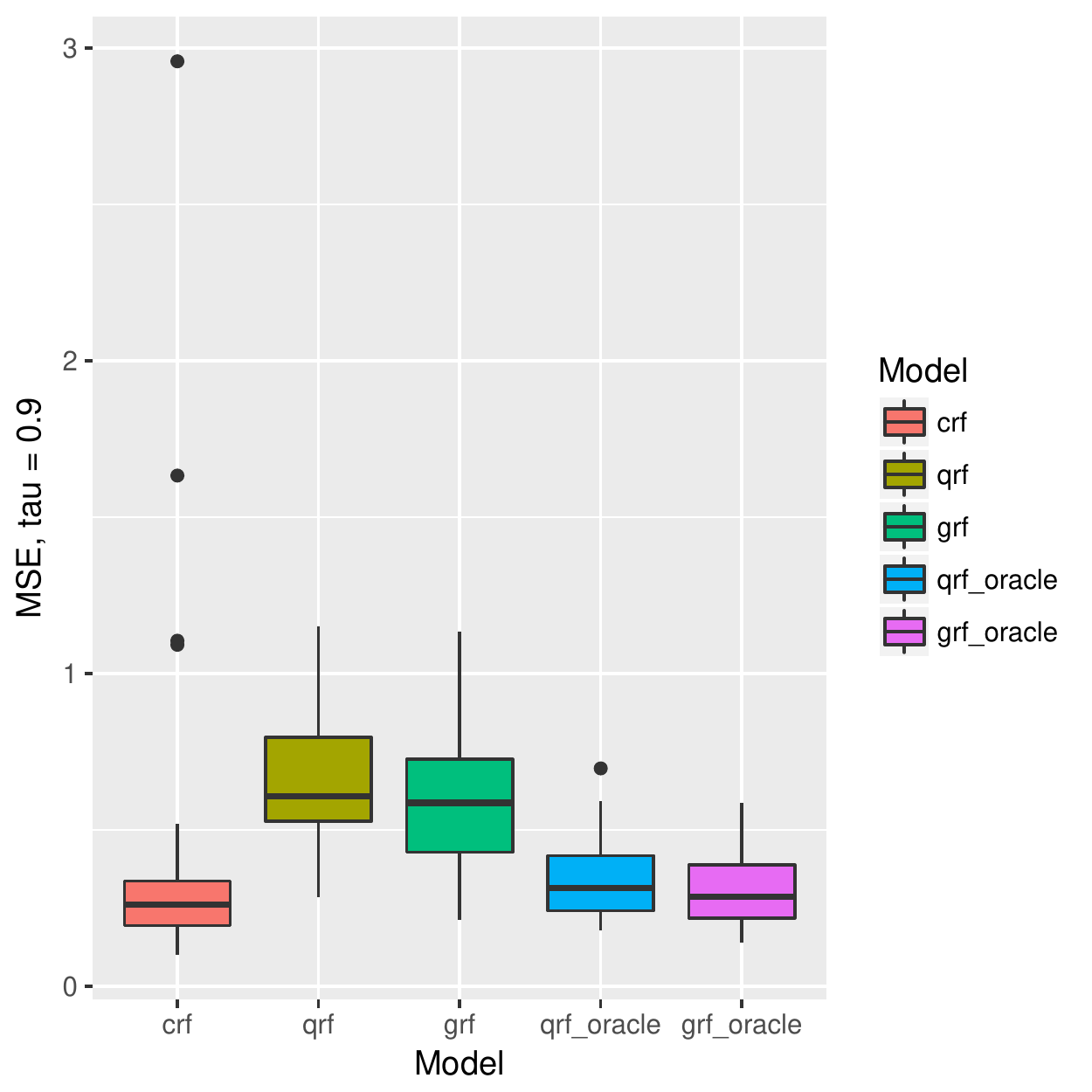}
        \caption{MSE: $\tau = 0.9$}
    \end{subfigure}
    
    \vspace{-0.05in}
    
    \begin{subfigure}[b]{0.182\linewidth}
        \includegraphics[width=\textwidth]{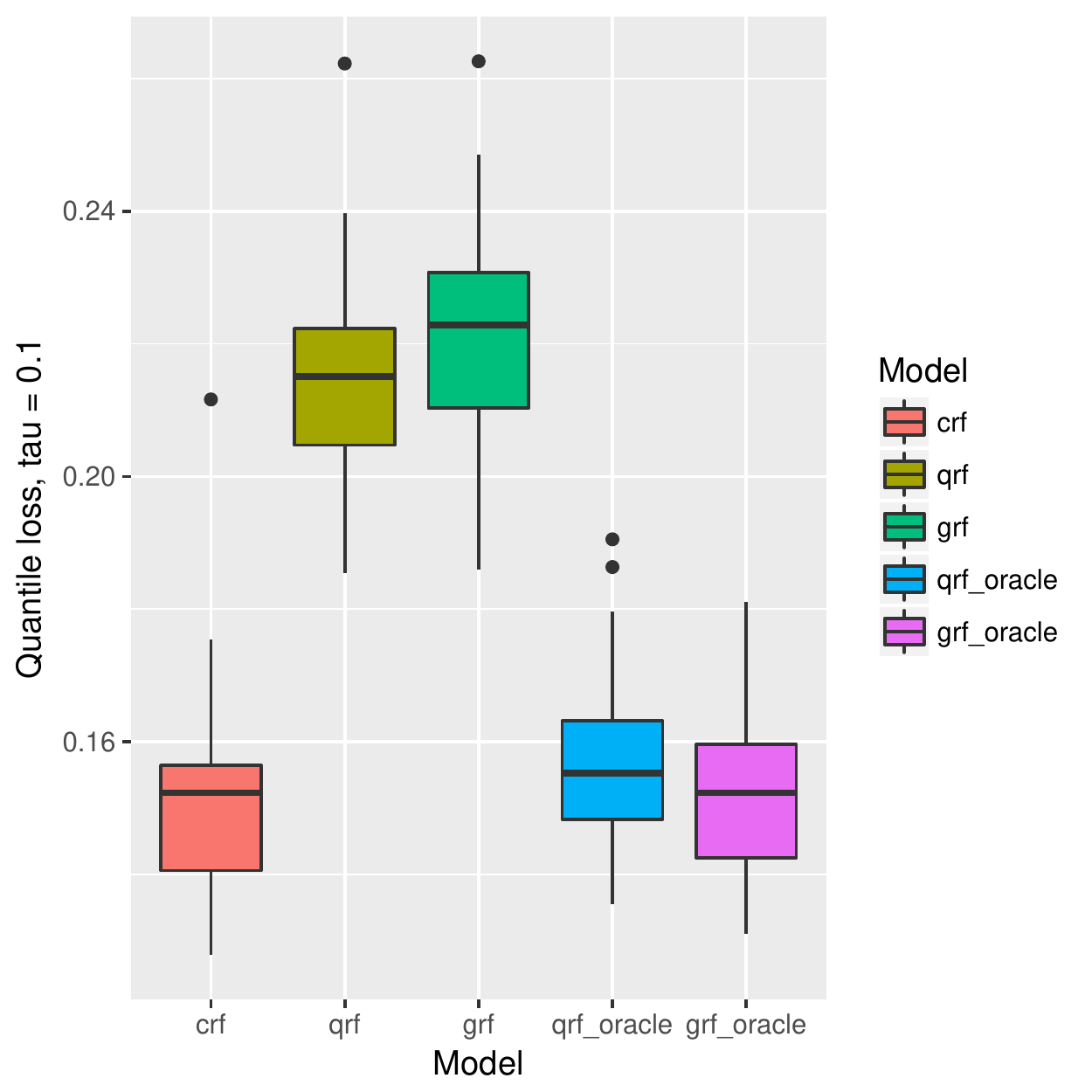}
        \caption{Quantile loss: $\tau = 0.1$}
    \end{subfigure}
    ~
    \begin{subfigure}[b]{0.182\linewidth}
        \includegraphics[width=\textwidth]{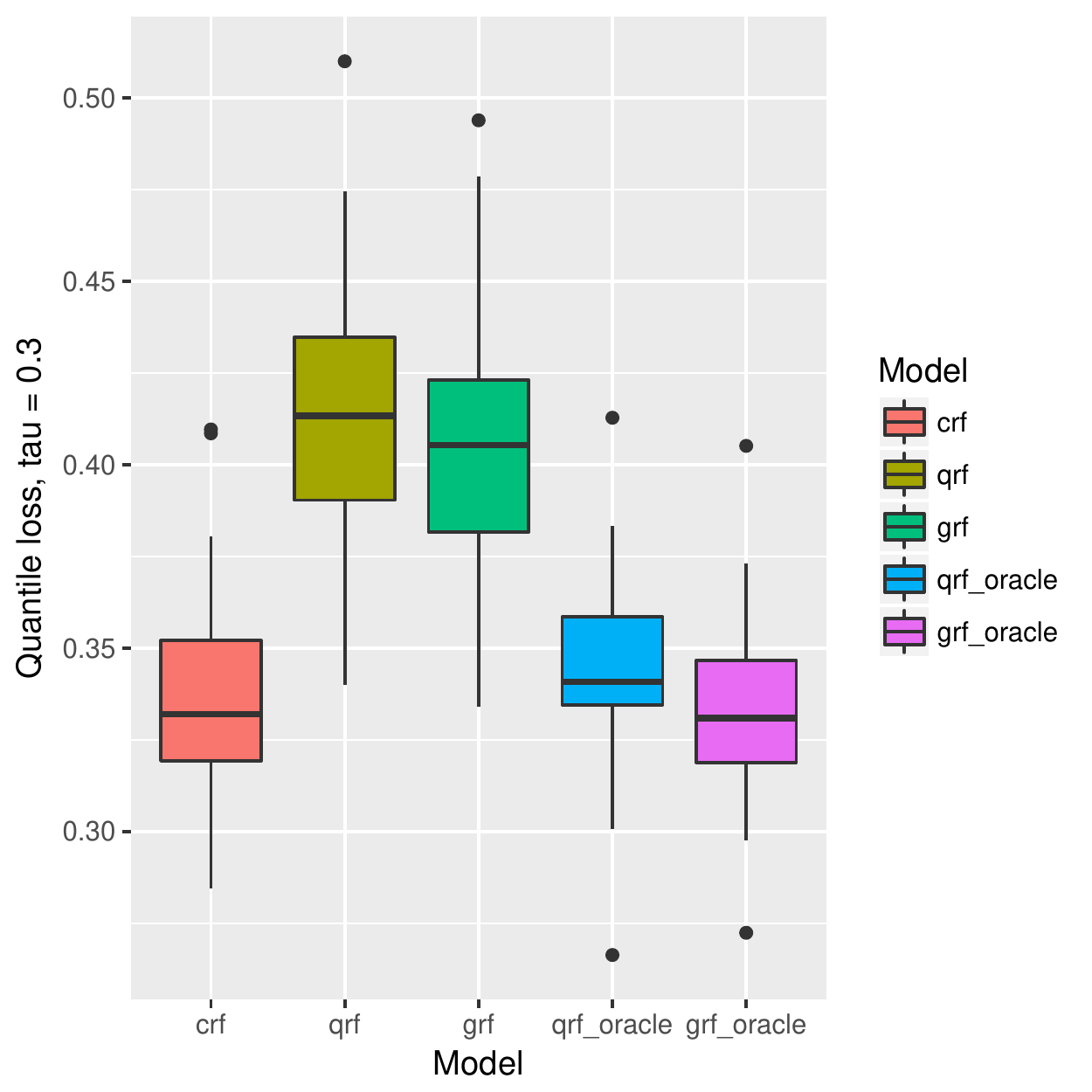}
        \caption{Quantile loss: $\tau = 0.3$}
    \end{subfigure}
    ~ %add desired spacing between images, e. g. ~, \quad, \qquad, \hfill etc. 
    \begin{subfigure}[b]{0.182\linewidth}
        \includegraphics[width=\textwidth]{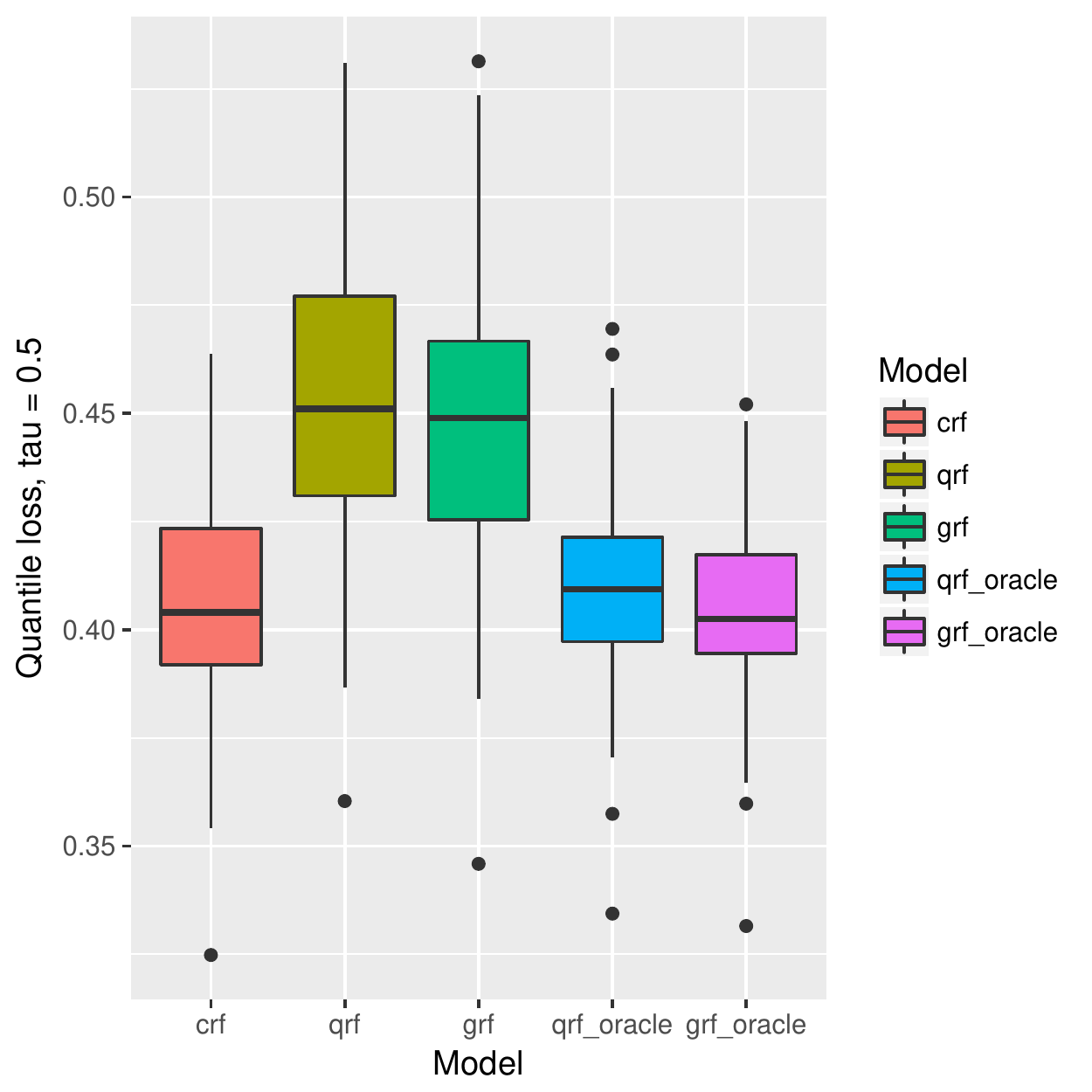}
        \caption{Quantile loss: $\tau = 0.5$}
    \end{subfigure}
    ~
    \begin{subfigure}[b]{0.182\linewidth}
        \includegraphics[width=\textwidth]{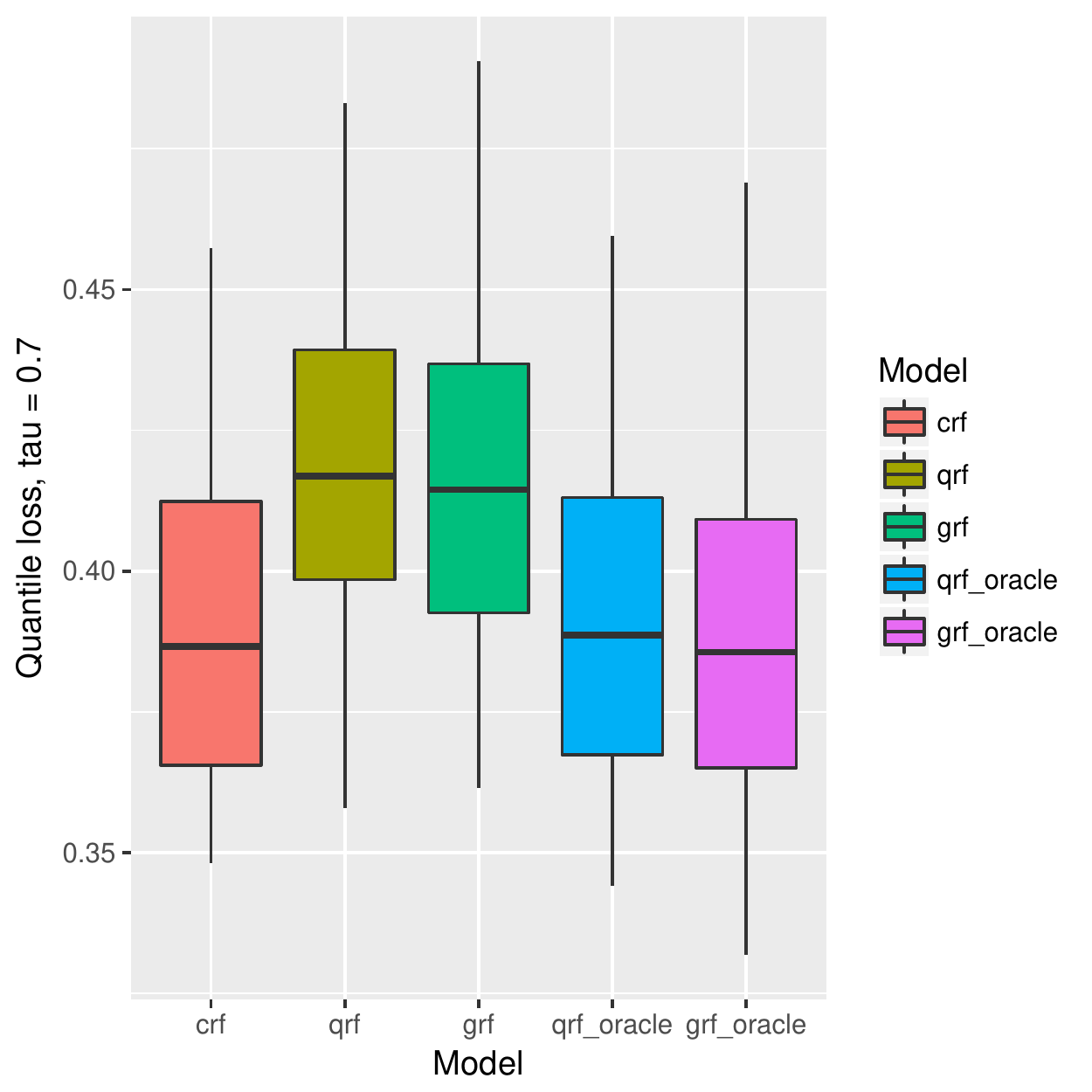}
        \caption{Quantile loss: $\tau = 0.7$}
    \end{subfigure}
    ~
    \begin{subfigure}[b]{0.182\linewidth}
        \includegraphics[width=\textwidth]{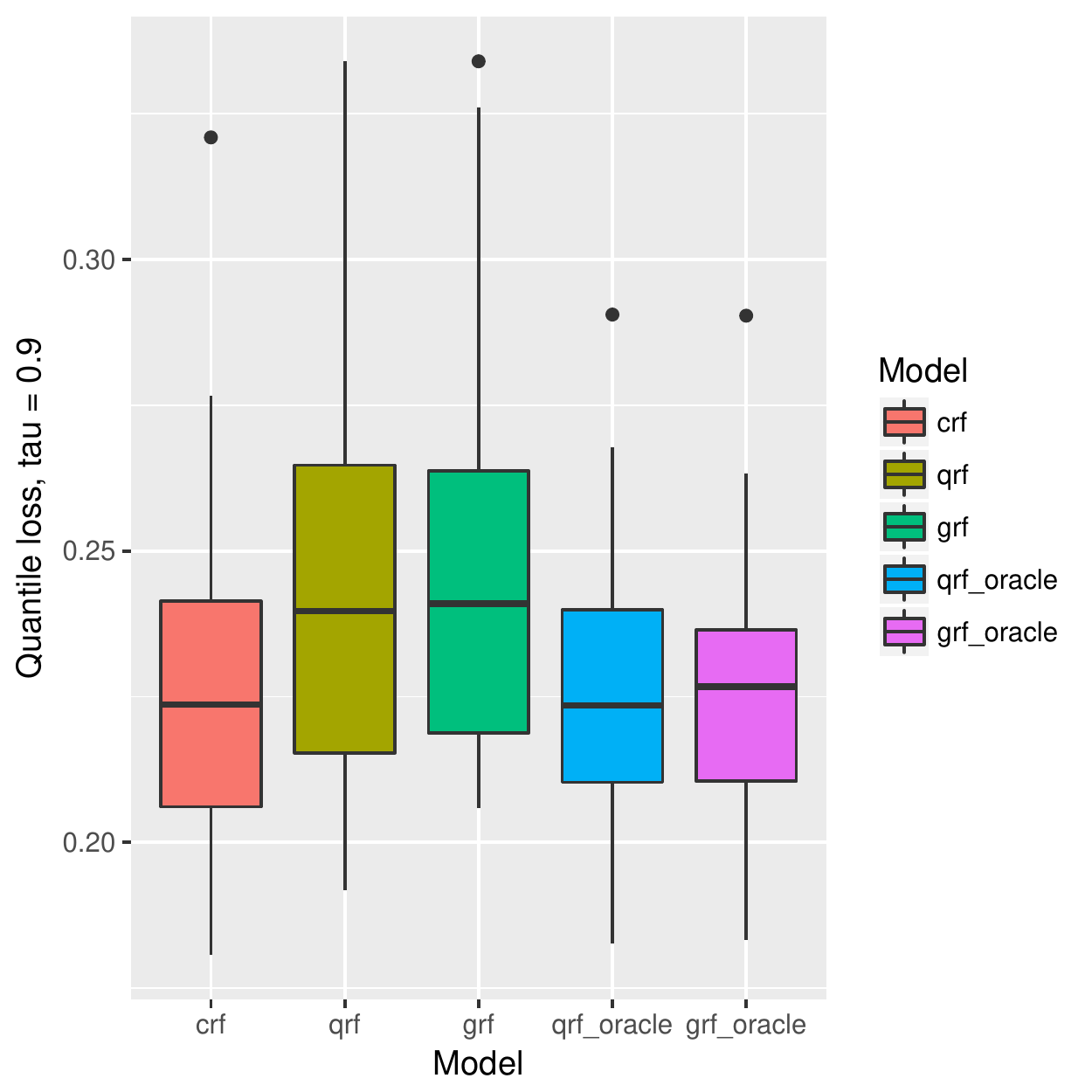}
        \caption{Quantile loss: $\tau = 0.9$}
    \end{subfigure}
    
    \vspace{-0.05in}
    
    \begin{subfigure}[b]{0.18\linewidth}
        \includegraphics[width=\textwidth]{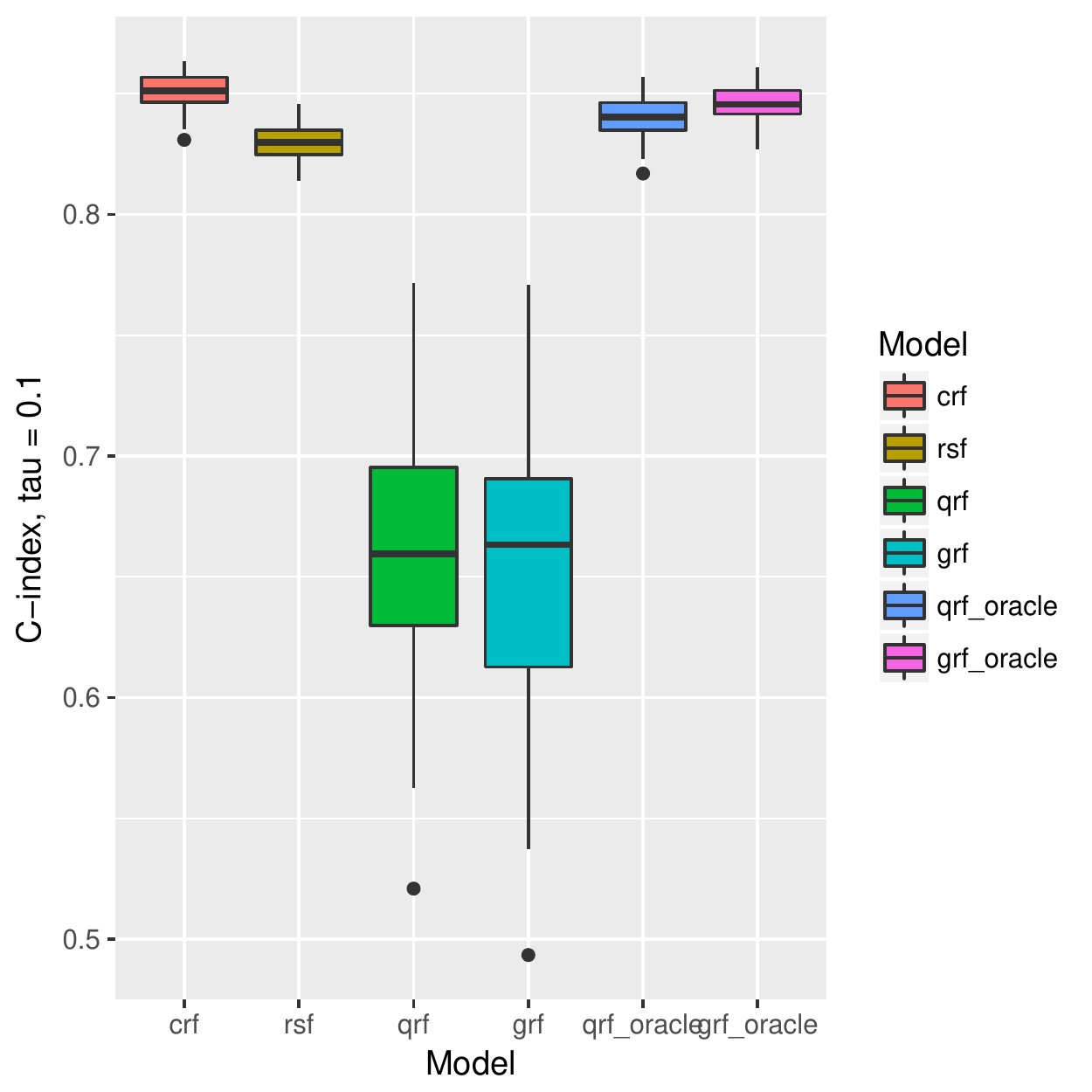}
        \caption{C-index: $\tau = 0.1$}
    \end{subfigure}
    ~
    \begin{subfigure}[b]{0.18\linewidth}
        \includegraphics[width=\textwidth]{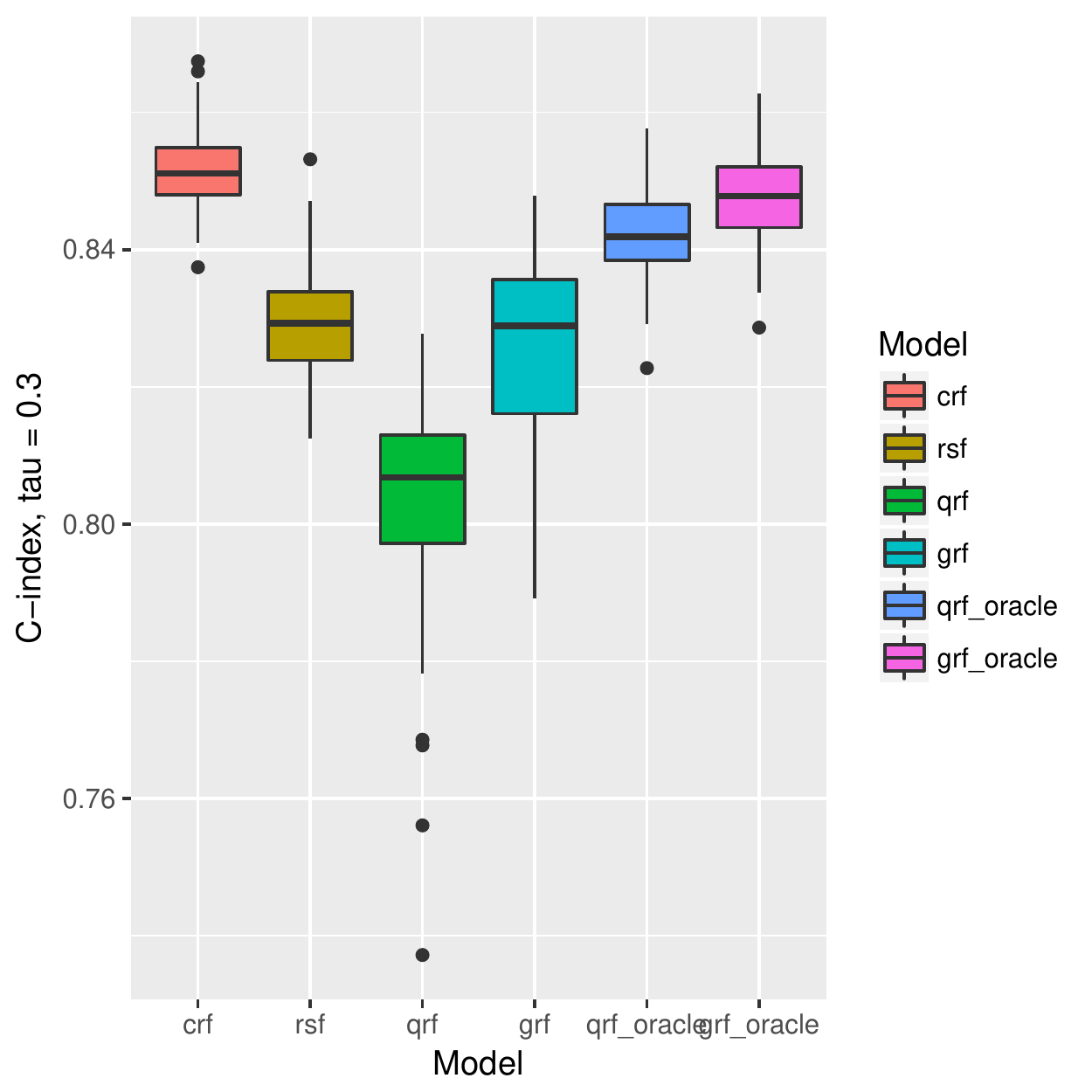}
        \caption{C-index: $\tau = 0.3$}
    \end{subfigure}
    ~ %add desired spacing between images, e. g. ~, \quad, \qquad, \hfill etc. 
    \begin{subfigure}[b]{0.18\linewidth}
        \includegraphics[width=\textwidth]{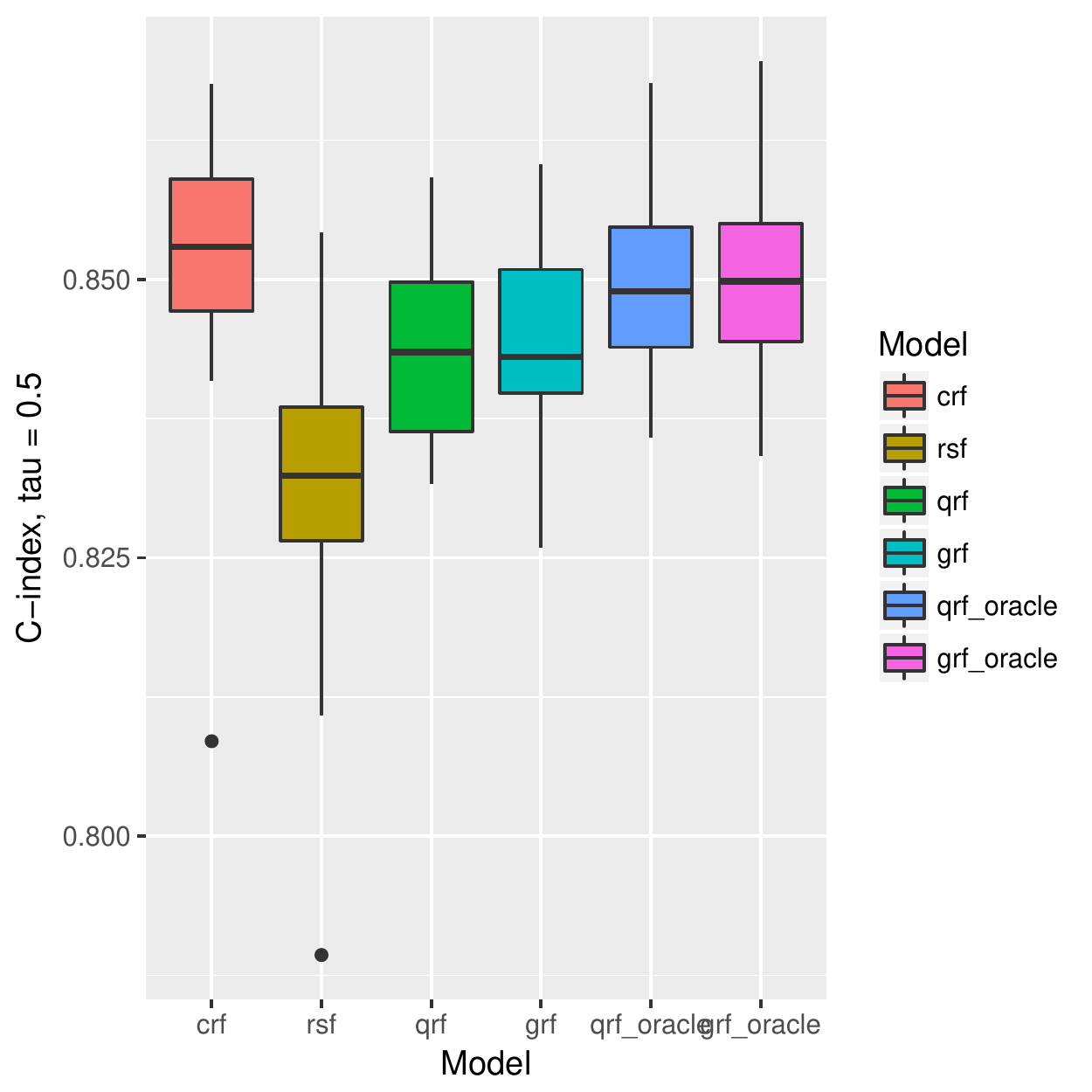}
        \caption{C-index: $\tau = 0.5$}
    \end{subfigure}
    ~
    \begin{subfigure}[b]{0.18\linewidth}
        \includegraphics[width=\textwidth]{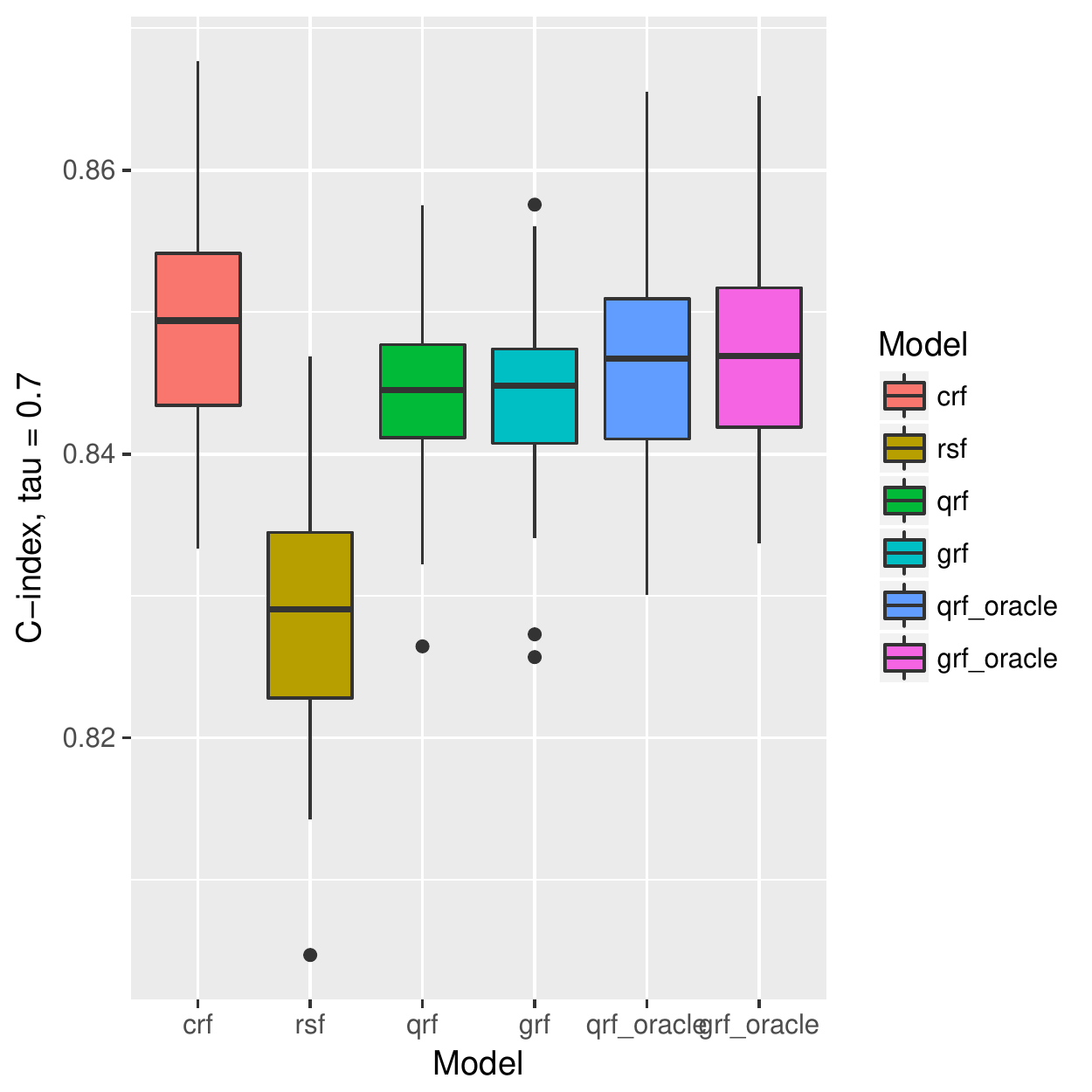}
        \caption{C-index: $\tau = 0.7$}
    \end{subfigure}
    ~
    \begin{subfigure}[b]{0.18\linewidth}
        \includegraphics[width=\textwidth]{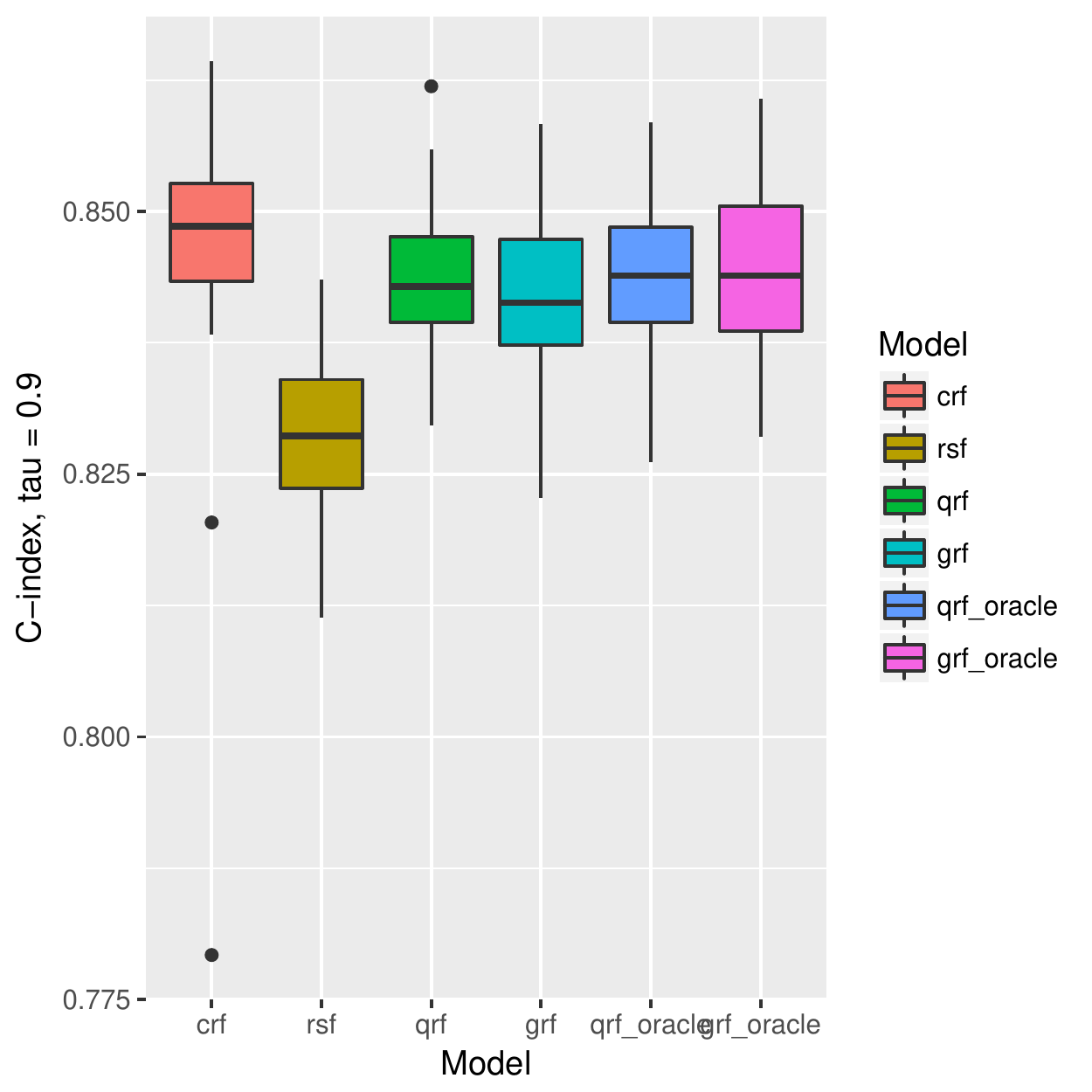}
        \caption{C-index: $\tau = 0.9$}
    \end{subfigure}
    
    \caption{One-dimensional AFT model box plots with $n=300$ and $B=1000$. For the metrics MSE \eqref{eq:L_MSE}, MAD \eqref{eq:L_MAD} and quantile loss \eqref{eq:L_quantile}, the smaller the value is the better. For C-index, the larger it is the better.}
    \label{fig:aft_1d_box}
\end{figure}

\subsubsection{Comparison of different conditional survival estimators}
In this section, we will compare the two different conditional survival function estimators \eqref{eq:KM_kNN} and \eqref{eq:Beran_rf}. We generate training data and test data from the one-dimensional AFT model defined in the previous section, but with two different censoring rate:
\begin{itemize}
    \item $C \sim \textrm{Exp}(\lambda = 0.08)$, in this case, the censoring rate is about $20\%$.
    \item $C \sim \textrm{Exp}(\lambda = 0.20)$, in this case, the censoring rate is about $50\%$.
\end{itemize}
We then choose four test points $\{x_1=0.4, x_2=0.8, x_3=1.2, x_4=1.6\}$, and then plot out the conditional survival function estimators $\hat{G}(q|x_i)$ by the two different methods \eqref{eq:KM_kNN} and \eqref{eq:Beran_rf} on these four points. The results are shown in Figure \ref{fig:g_comparison} and \ref{fig:g_comparison_high} for three different training sample sizes $n \in \{300, 2000, 5000\}$. For the nearest neighbor estimator \eqref{eq:KM_kNN}, we set the number of neighbors to be $n/10$, which is also the node size we choose.

We can observe that when $n$ increases, two curves become closer and are both good approximations of the true survival curve. But the first method \eqref{eq:KM_kNN} does have an extra tuning parameter $k$ -- the number of nearest neighbors, so in the experiments, we always choose to use the second estimator \eqref{eq:Beran_rf}, which is more adaptive and parameter free.

Note that the estimated survival function will degenerate at the tail of the distribution when the test point $x$ is small; see  the first two columns in Figure \ref{fig:g_comparison} and \ref{fig:g_comparison_high}. This is a common phenomenon even for the regular KM estimator because there is no censored observations beyond some time point. In the AFT model, the conditional distribution of the latent variable depends on the location $x$. When $x$ is small, the conditional mean of $T$ is also small, and we could not observe most of the censoring values where $C_i > T_i$, leading to degenerated survival curves. However, if we continue increasing the sample size $n$, we should be able to recover the entire curve even for smaller $x$. In fact, when we increase the censoring level from $20\%$ (Figure \ref{fig:g_comparison}) to $50\%$ (Figure \ref{fig:g_comparison_high}), we find that both estimators give better performance because we can observe more censored values.

\begin{figure}[!htb]
    \small
    \centering
    \begin{subfigure}[b]{0.23\linewidth}
        \includegraphics[width=\textwidth]{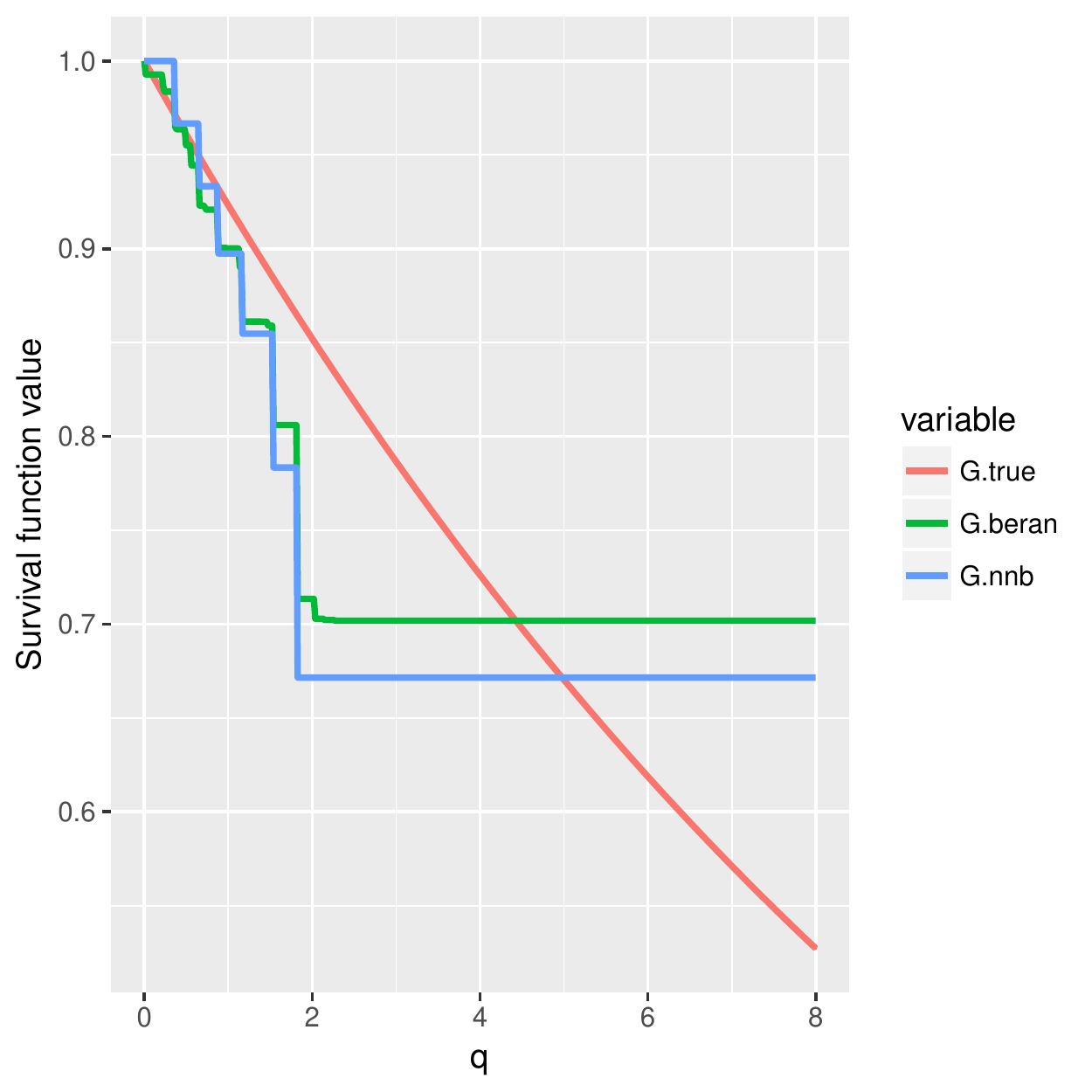}
        \caption{sample size: $300$}
    \end{subfigure}
    ~
    \begin{subfigure}[b]{0.23\linewidth}
        \includegraphics[width=\textwidth]{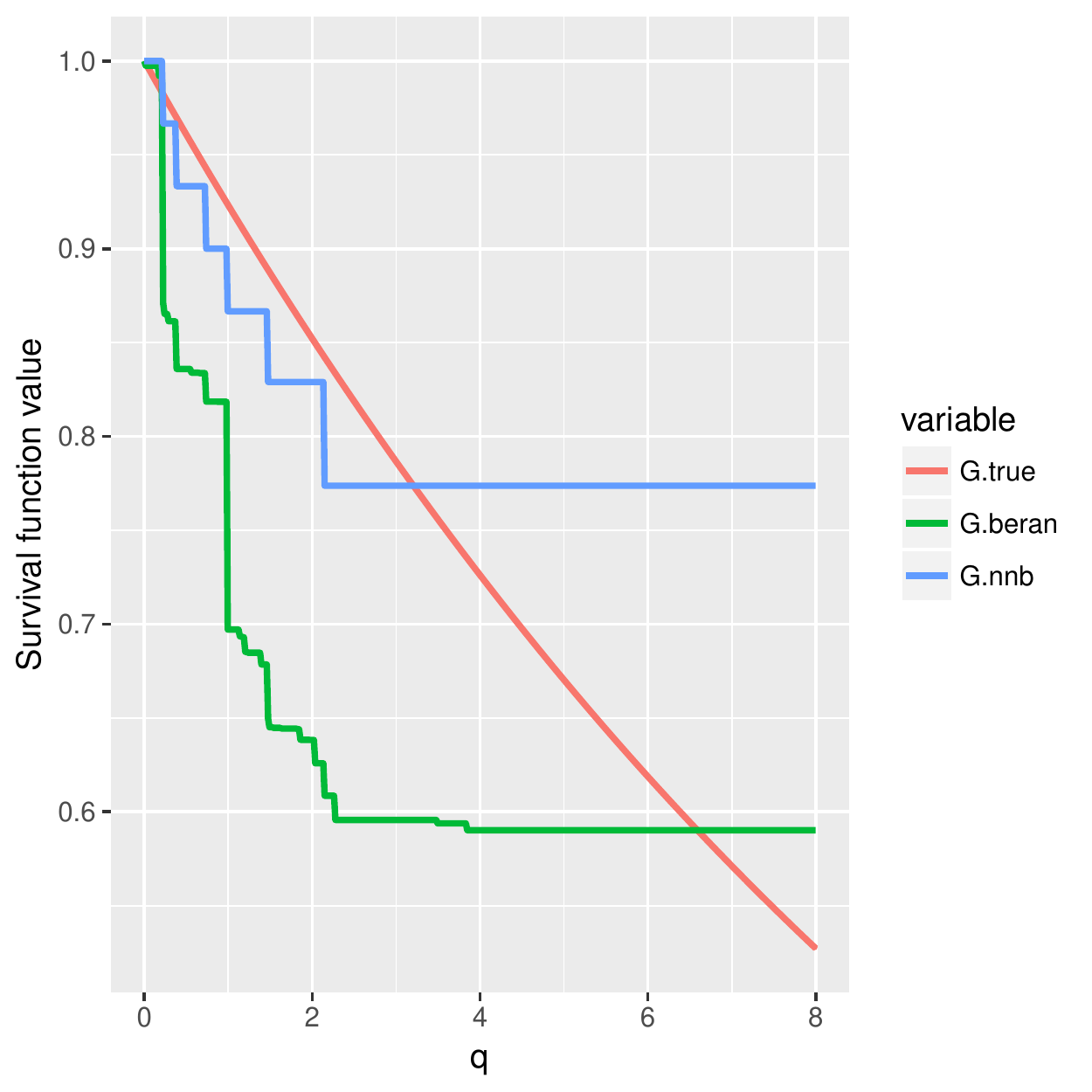}
        \caption{sample size: $300$}
    \end{subfigure}
    ~
    \begin{subfigure}[b]{0.23\linewidth}
        \includegraphics[width=\textwidth]{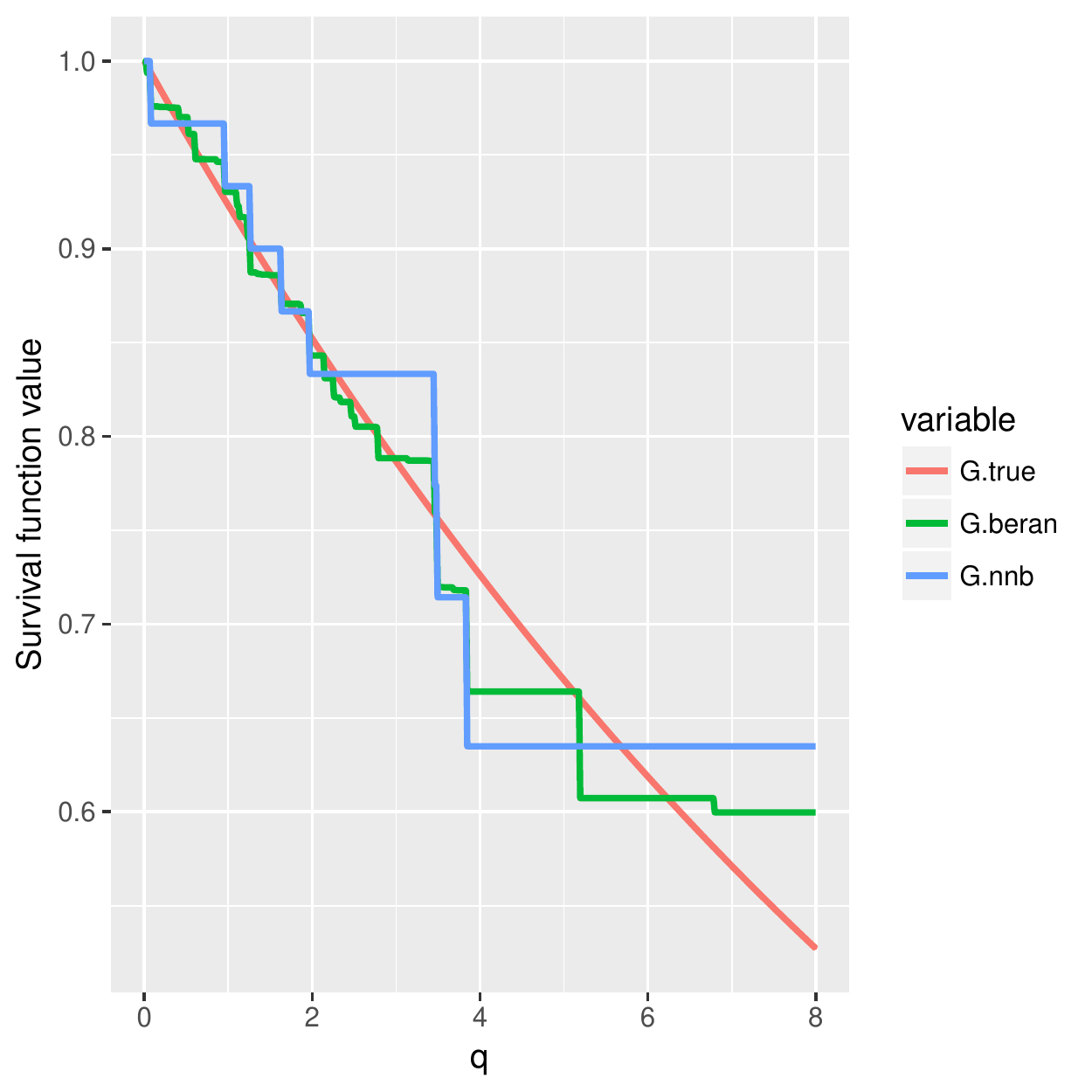}
        \caption{sample size: $300$}
    \end{subfigure}
    ~
    \begin{subfigure}[b]{0.23\linewidth}
        \includegraphics[width=\textwidth]{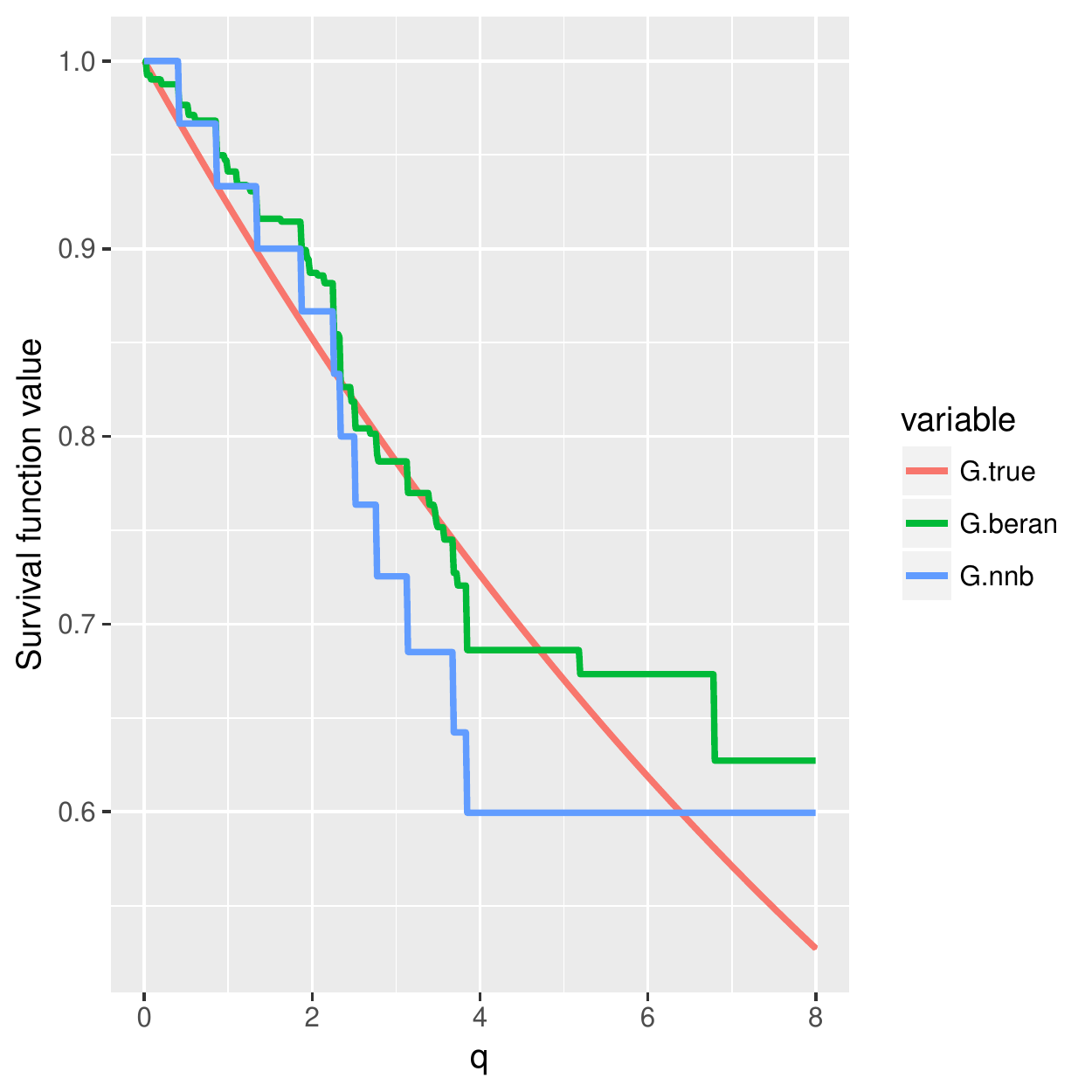}
        \caption{sample size: $300$}
    \end{subfigure}
    
    \vspace{-0.05in}
    
    \begin{subfigure}[b]{0.23\linewidth}
        \includegraphics[width=\textwidth]{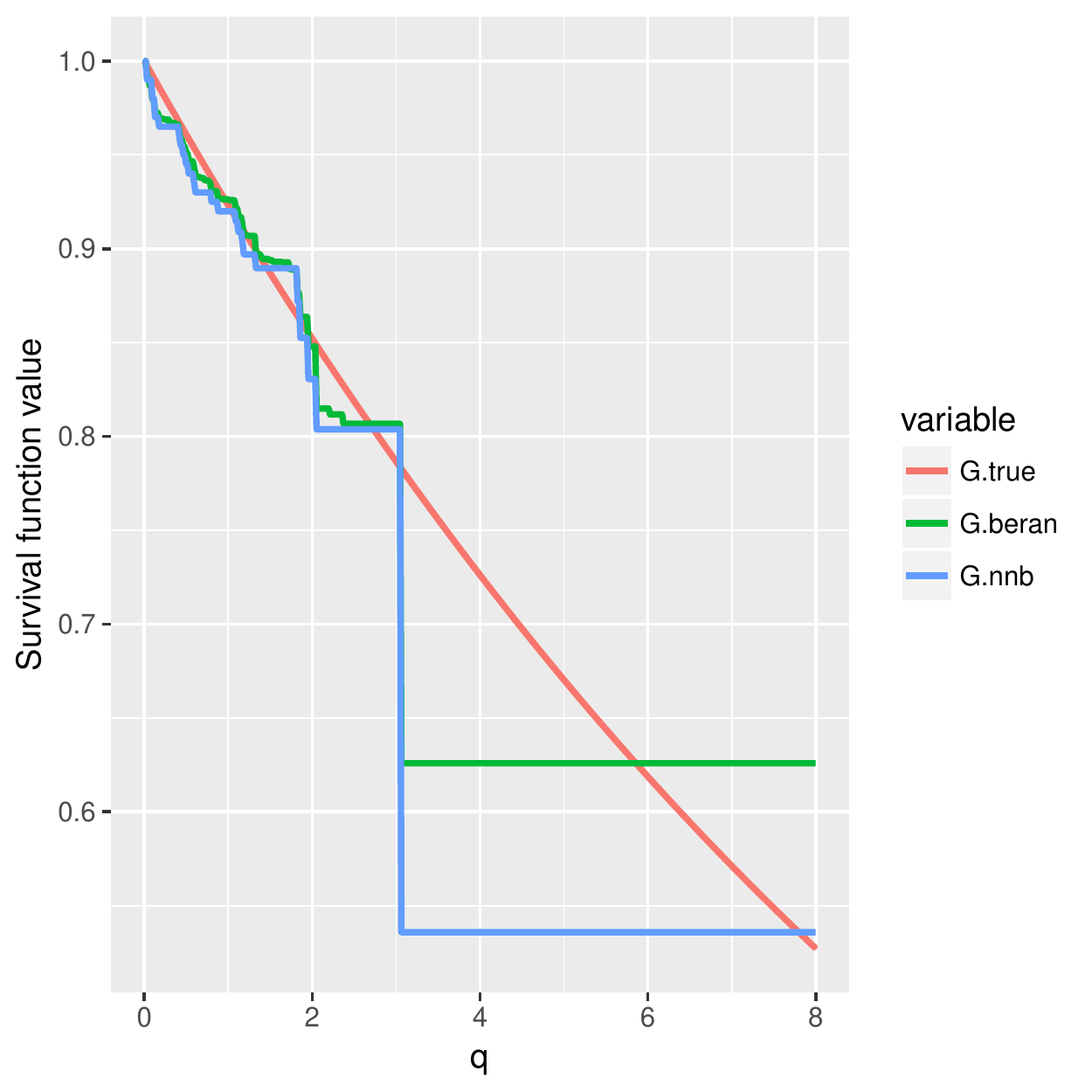}
        \caption{sample size: $2000$}
    \end{subfigure}
    ~
    \begin{subfigure}[b]{0.23\linewidth}
        \includegraphics[width=\textwidth]{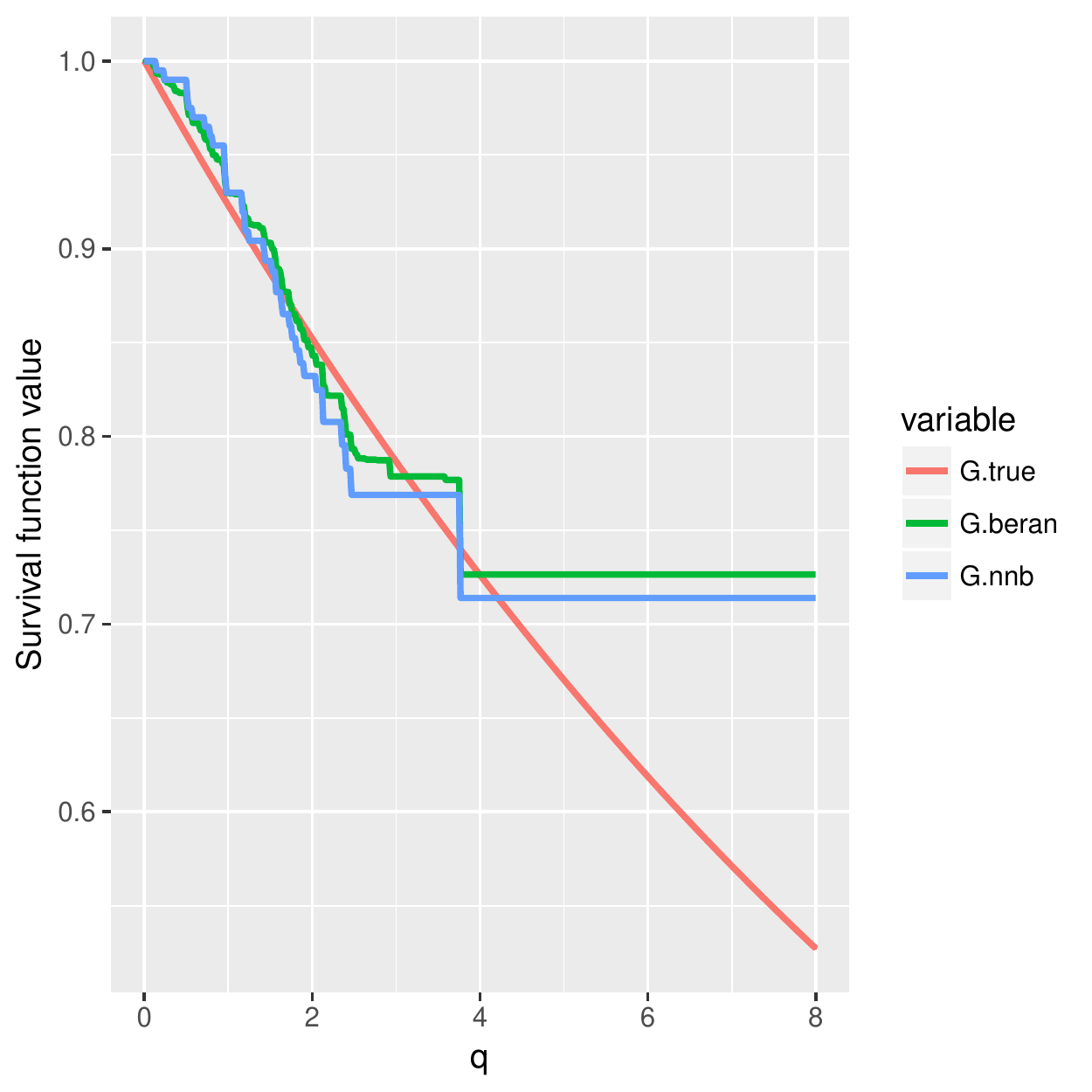}
        \caption{sample size: $2000$}
    \end{subfigure}
    ~
    \begin{subfigure}[b]{0.23\linewidth}
        \includegraphics[width=\textwidth]{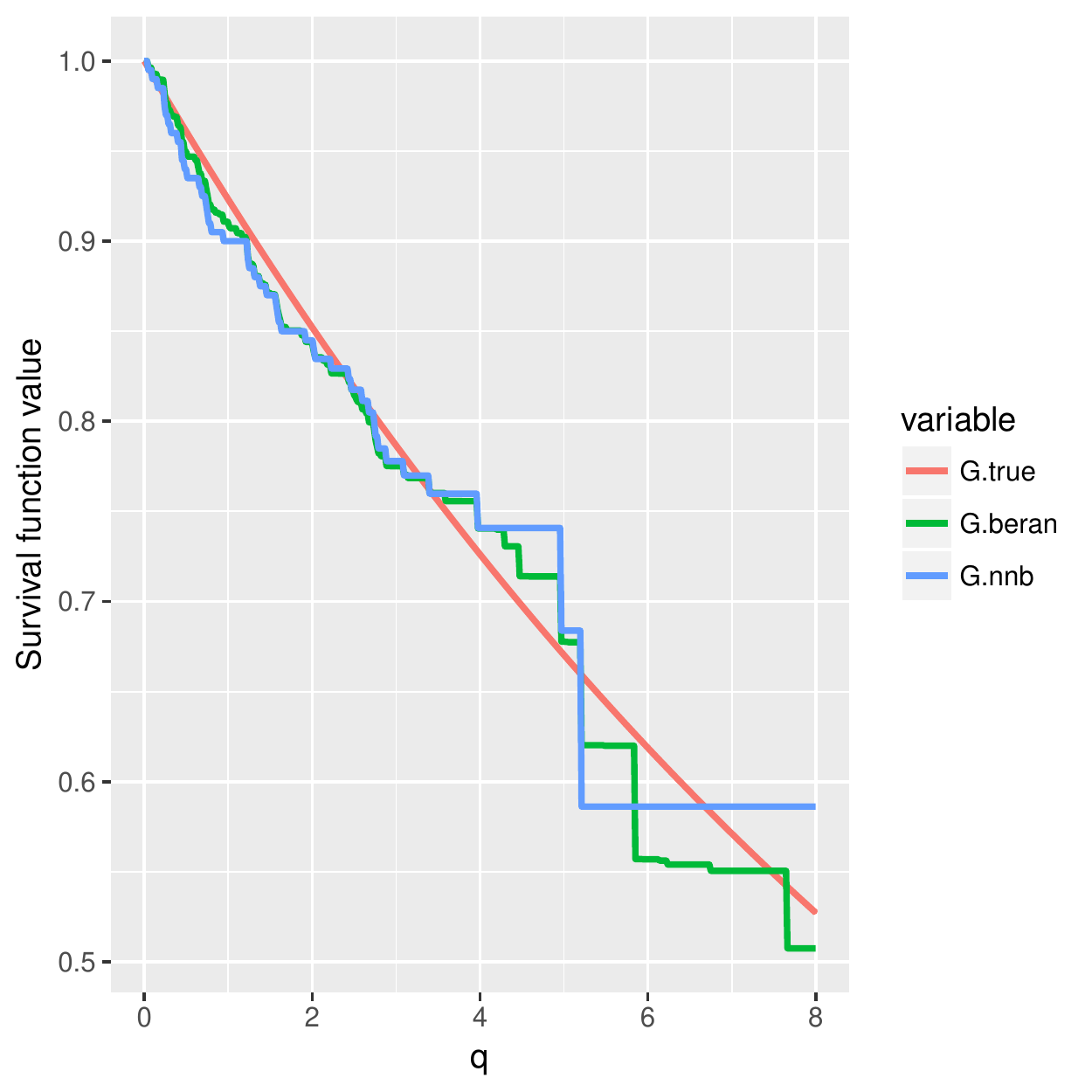}
        \caption{sample size: $2000$}
    \end{subfigure}
    ~
    \begin{subfigure}[b]{0.23\linewidth}
        \includegraphics[width=\textwidth]{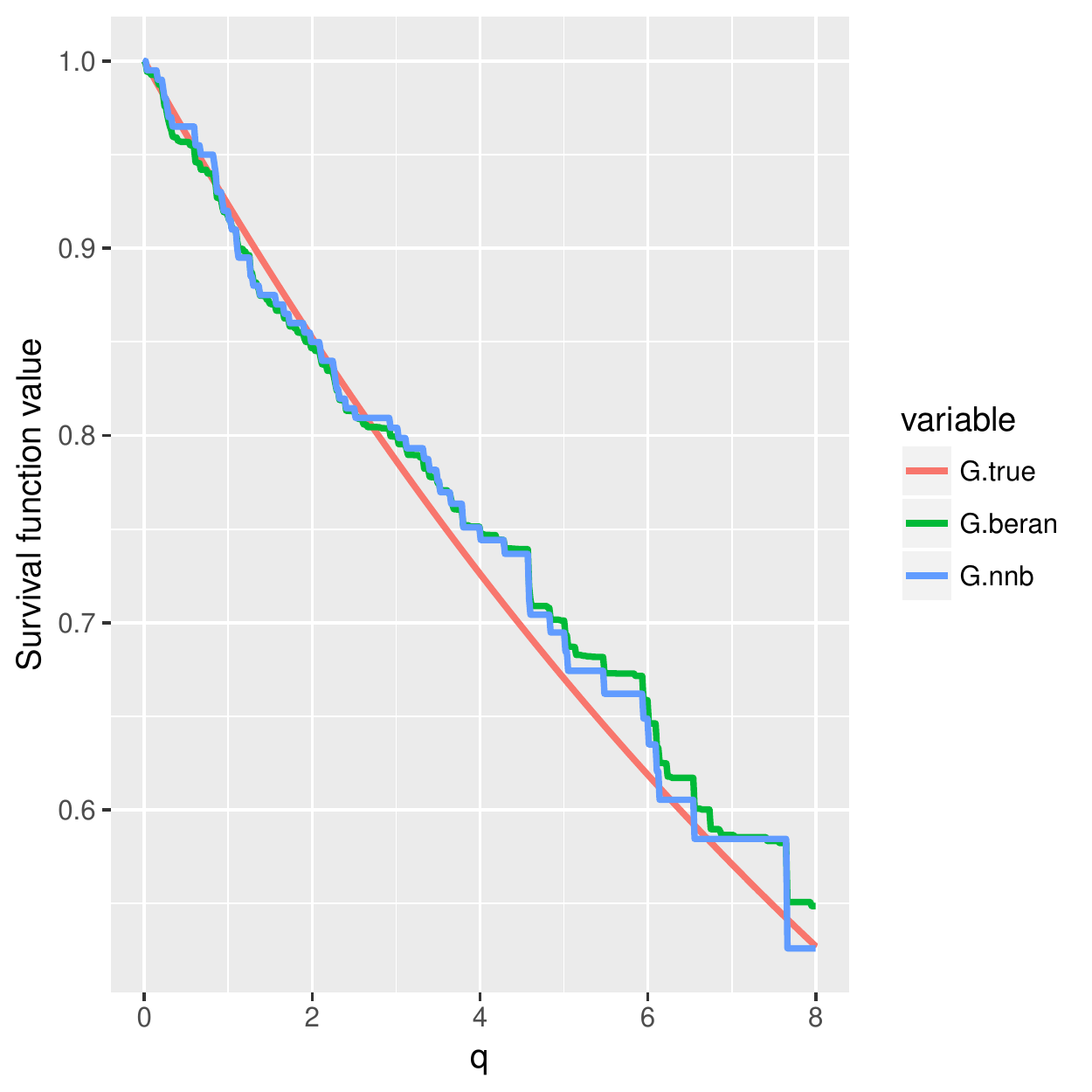}
        \caption{sample size: $2000$}
    \end{subfigure}
    
    \vspace{-0.05in}
    
    \begin{subfigure}[b]{0.23\linewidth}
        \includegraphics[width=\textwidth]{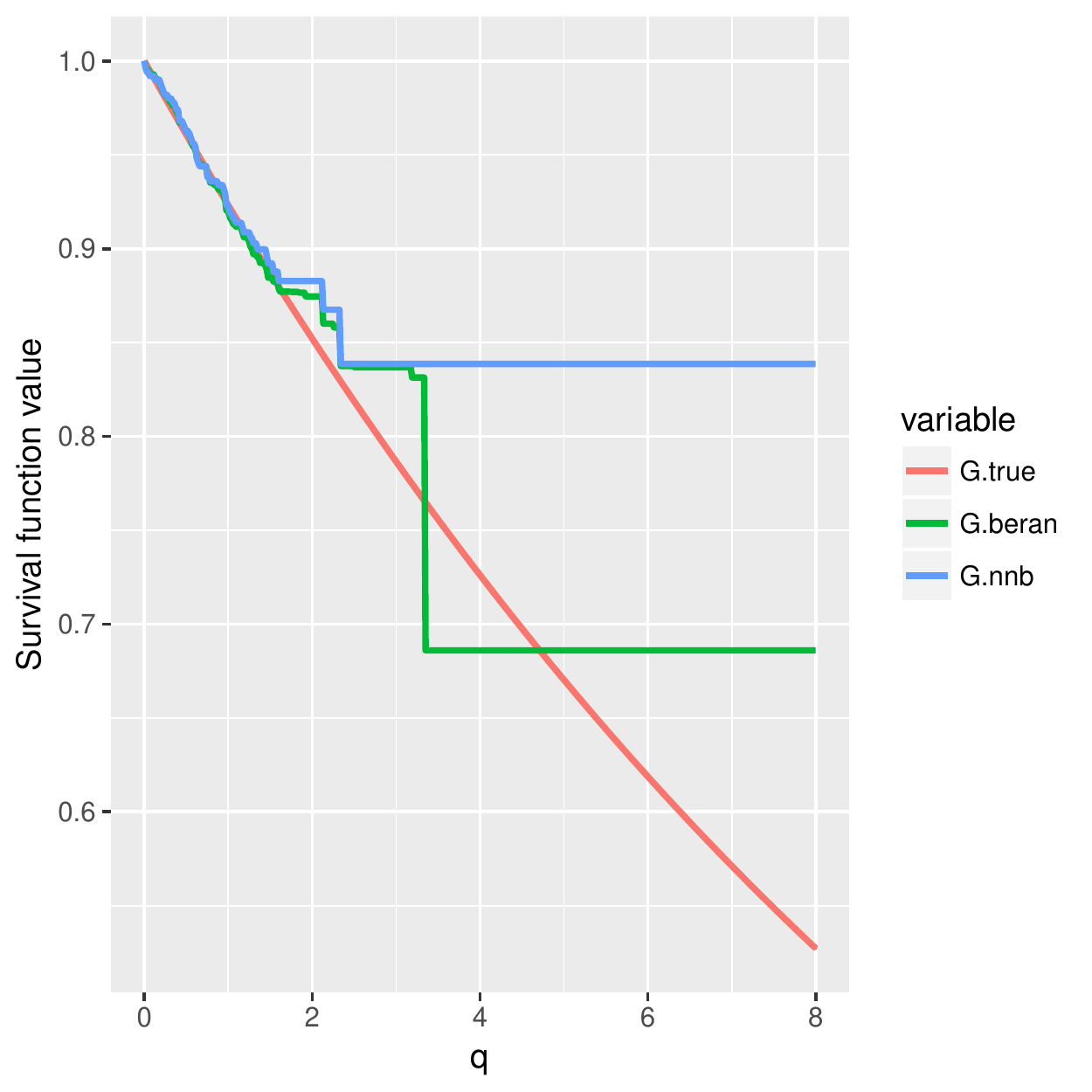}
        \caption{sample size: $5000$}
    \end{subfigure}
    ~
    \begin{subfigure}[b]{0.23\linewidth}
        \includegraphics[width=\textwidth]{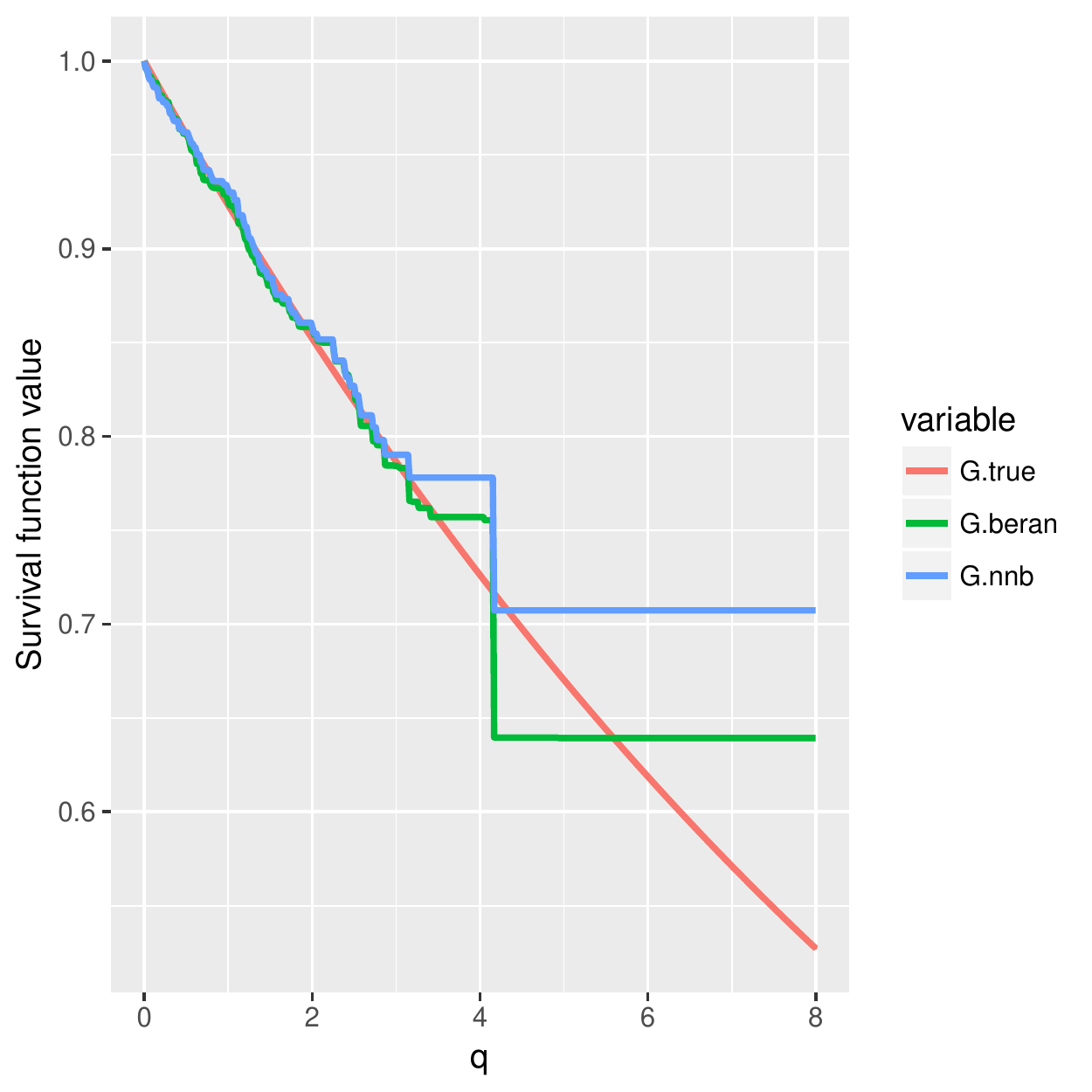}
        \caption{sample size: $5000$}
    \end{subfigure}
    ~
    \begin{subfigure}[b]{0.23\linewidth}
        \includegraphics[width=\textwidth]{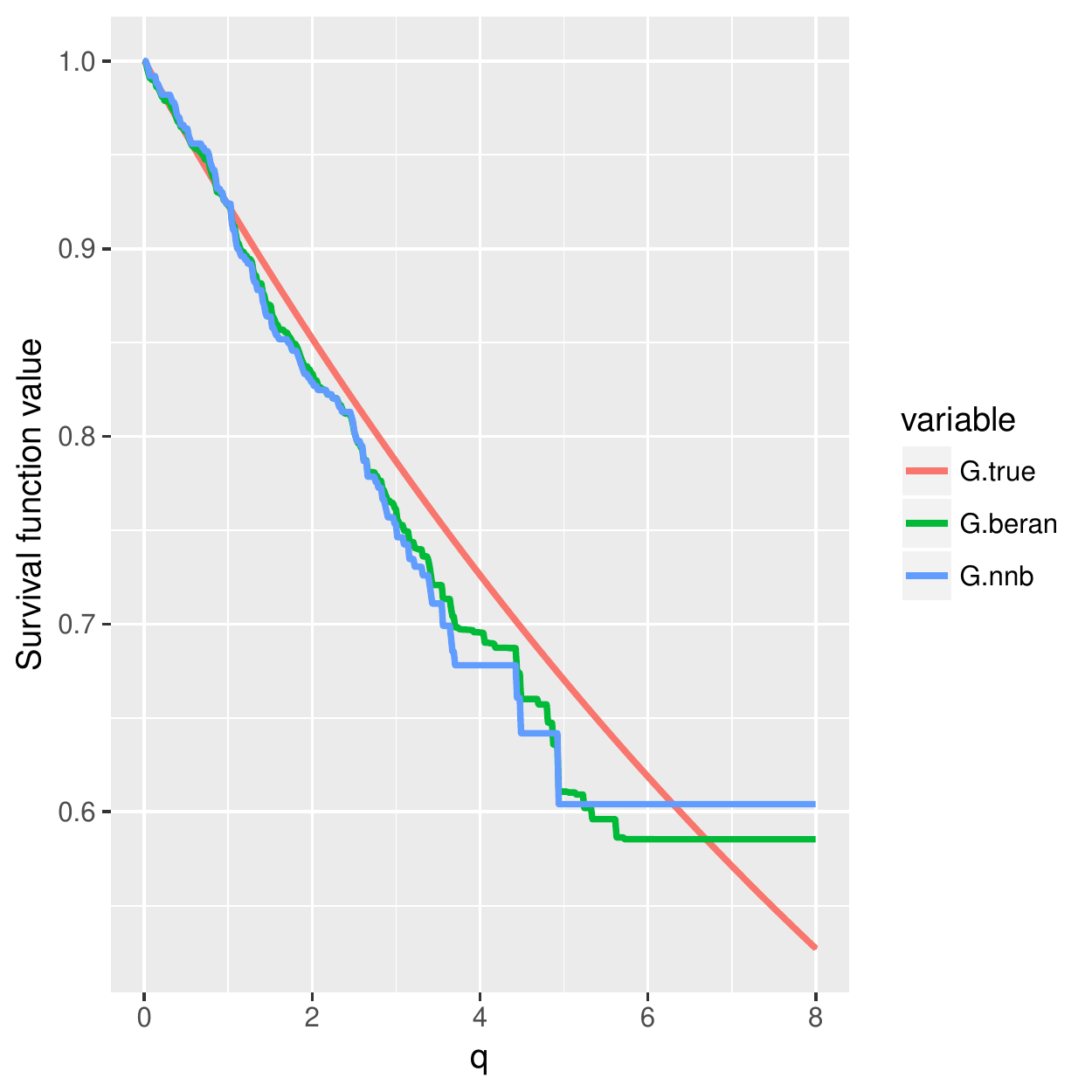}
        \caption{sample size: $5000$}
    \end{subfigure}
    ~
    \begin{subfigure}[b]{0.23\linewidth}
        \includegraphics[width=\textwidth]{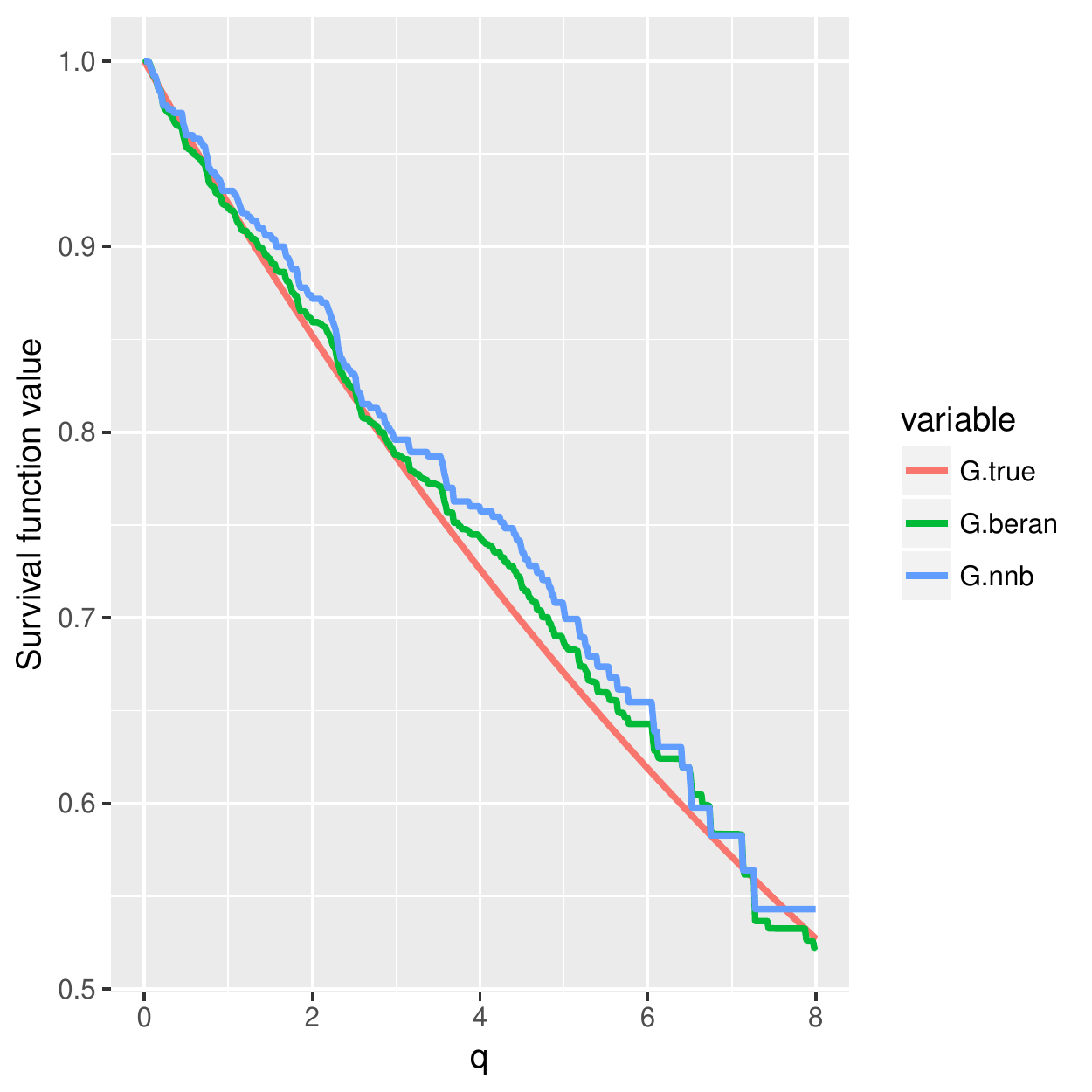}
        \caption{sample size: $5000$}
    \end{subfigure}
    
    \caption{Comparison of different conditional survival estimators on the one-dimensional AFT model. In this case, the censoring variable $C \sim \textrm{Exp}(\lambda=0.08)$, and the average censoring rate is around $20\%$. From left-most column to right-most column, we plot the conditional survival estimators for four test points, $x = 0.4, 0.8, 1.2, 1.6$.}
    \label{fig:g_comparison}
\end{figure}

\begin{figure}[!htb]
    \small
    \centering
    \begin{subfigure}[b]{0.23\linewidth}
        \includegraphics[width=\textwidth]{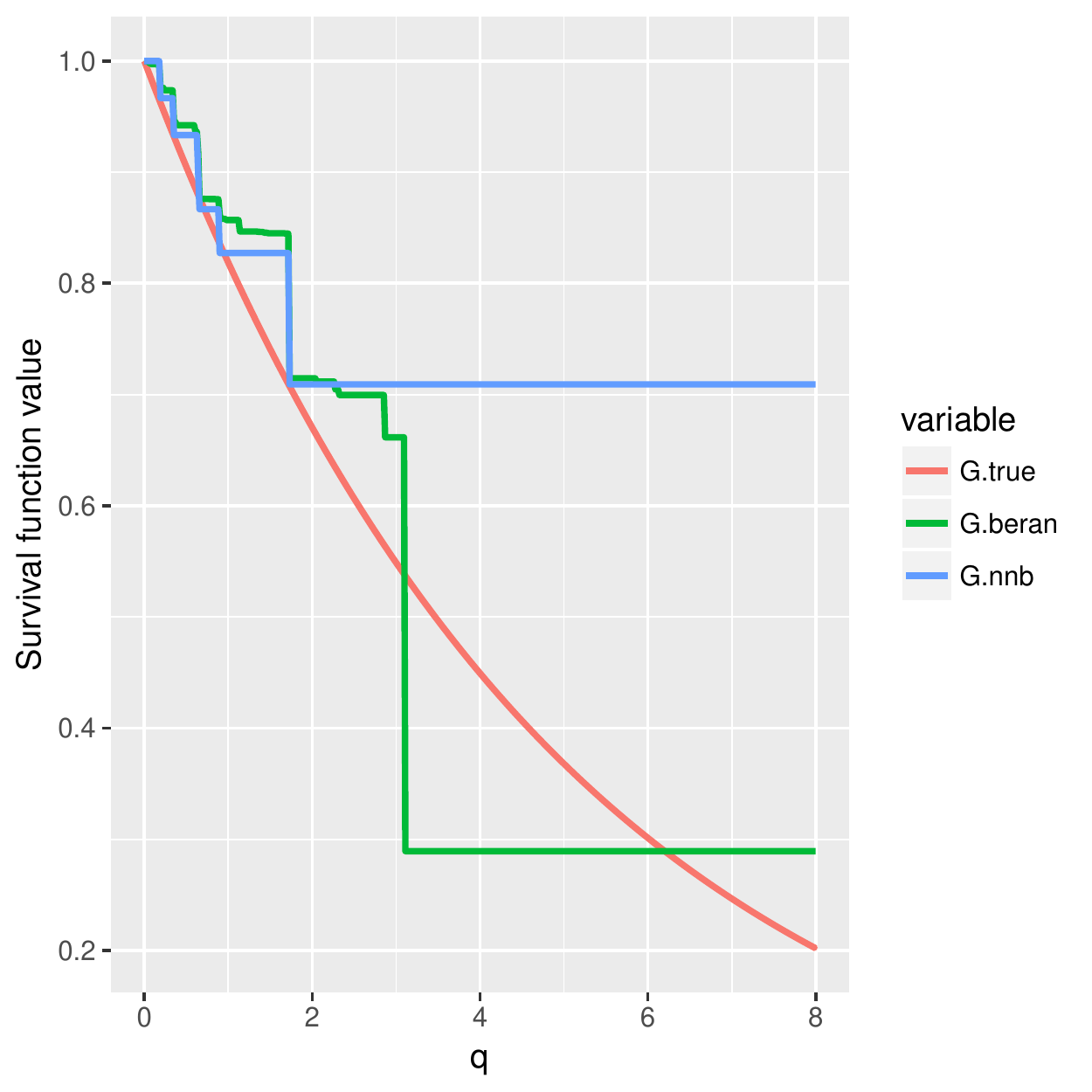}
        \caption{sample size: $300$}
    \end{subfigure}
    ~
    \begin{subfigure}[b]{0.23\linewidth}
        \includegraphics[width=\textwidth]{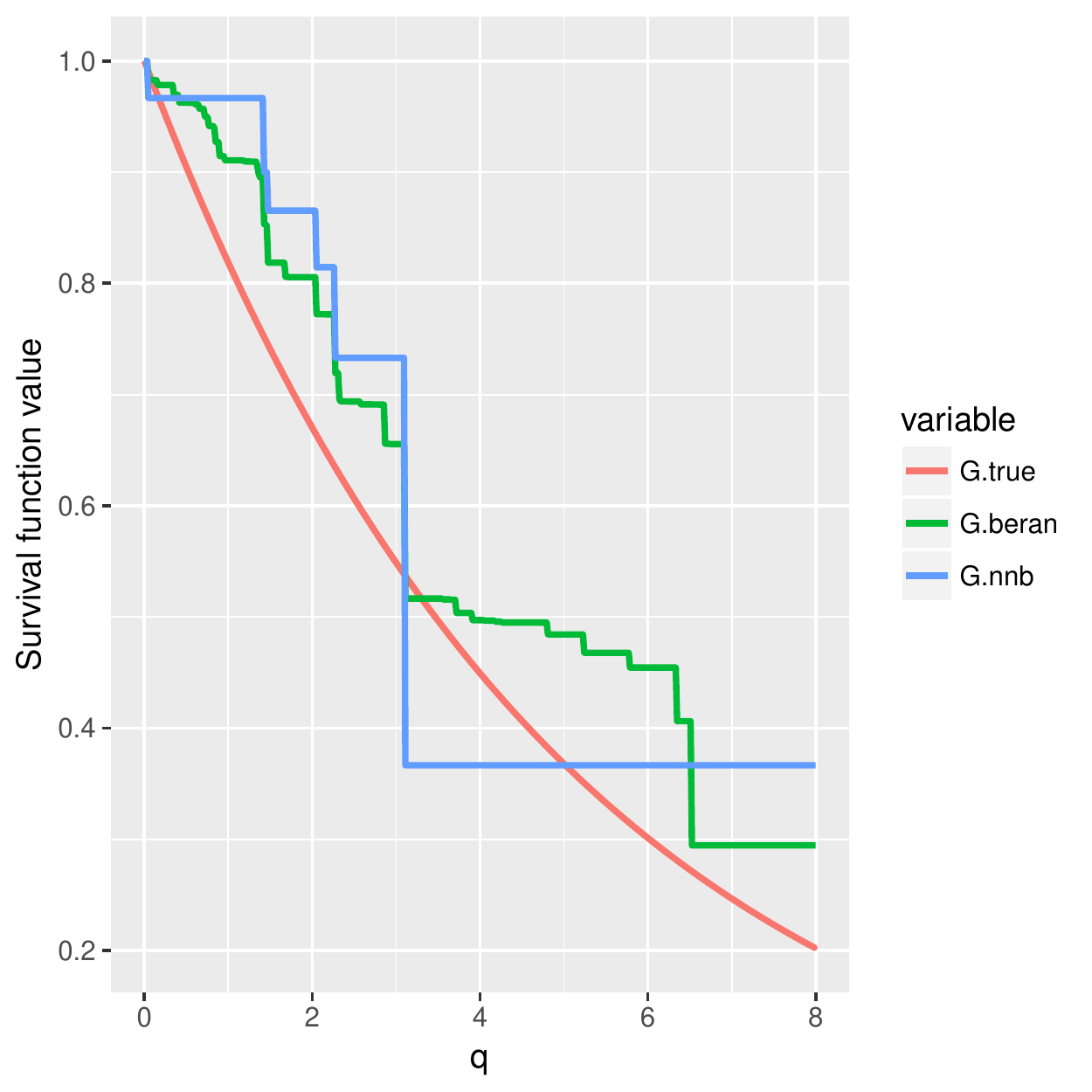}
        \caption{sample size: $300$}
    \end{subfigure}
    ~
    \begin{subfigure}[b]{0.23\linewidth}
        \includegraphics[width=\textwidth]{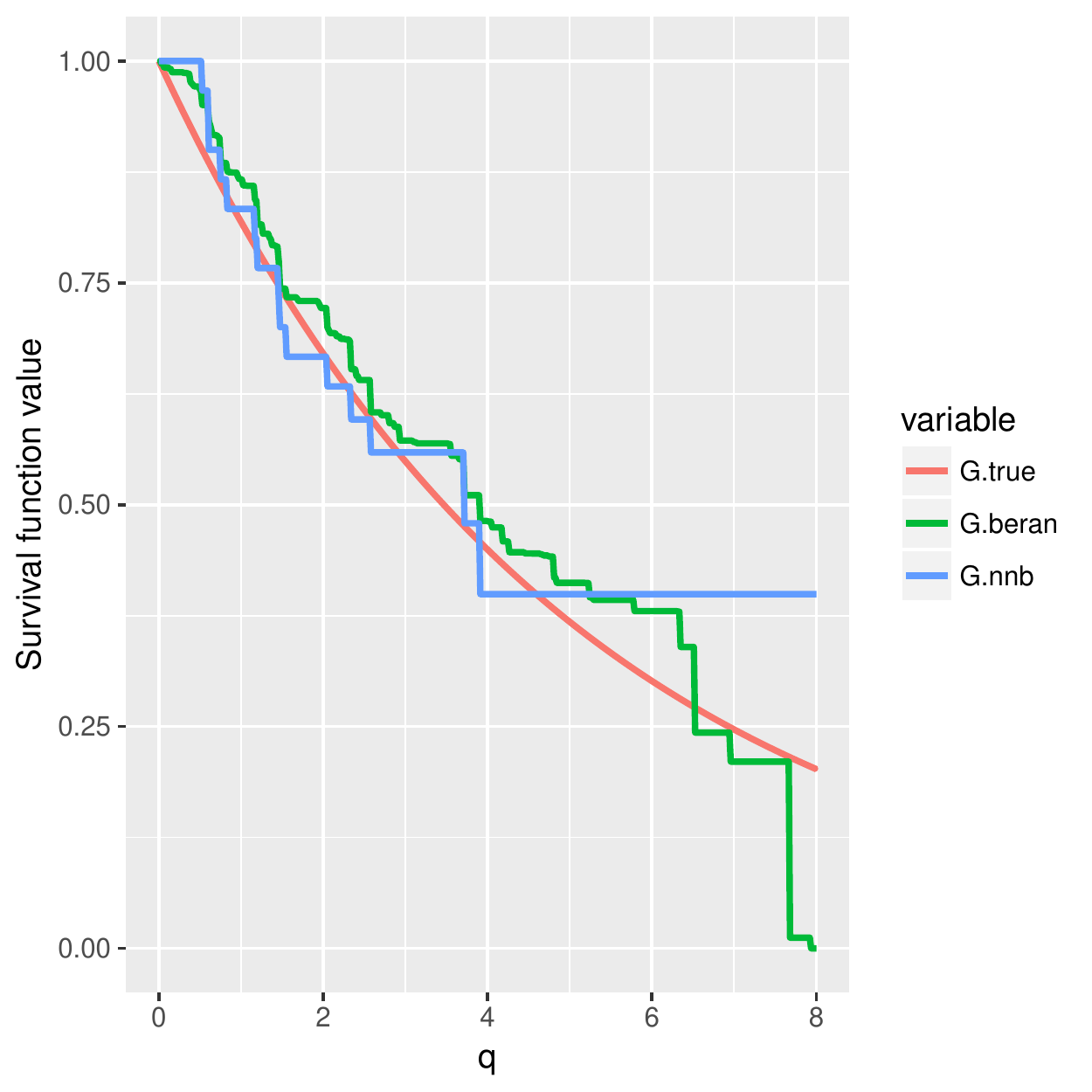}
        \caption{sample size: $300$}
    \end{subfigure}
    ~
    \begin{subfigure}[b]{0.23\linewidth}
        \includegraphics[width=\textwidth]{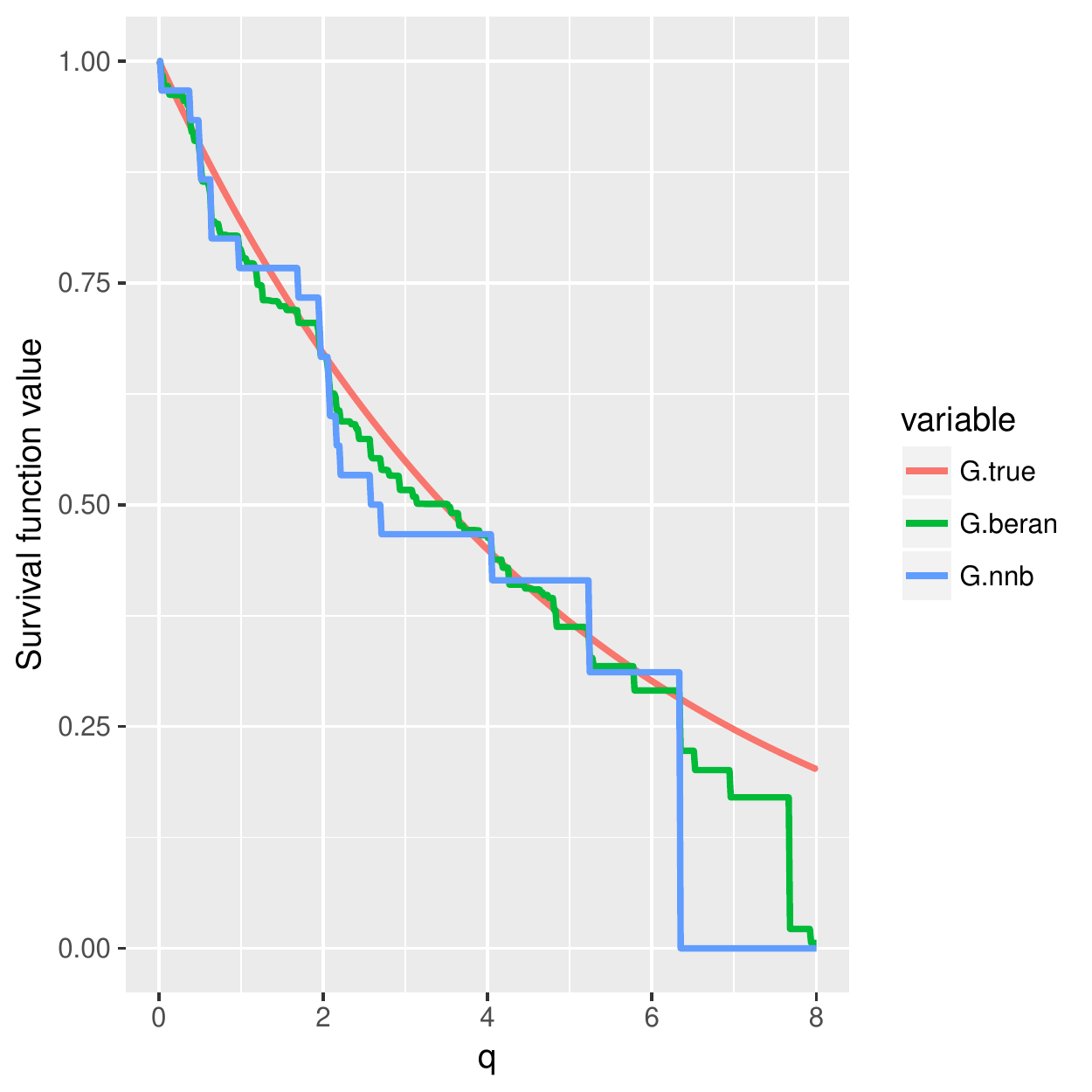}
        \caption{sample size: $300$}
    \end{subfigure}
    
    \vspace{-0.05in}
    
    \begin{subfigure}[b]{0.23\linewidth}
        \includegraphics[width=\textwidth]{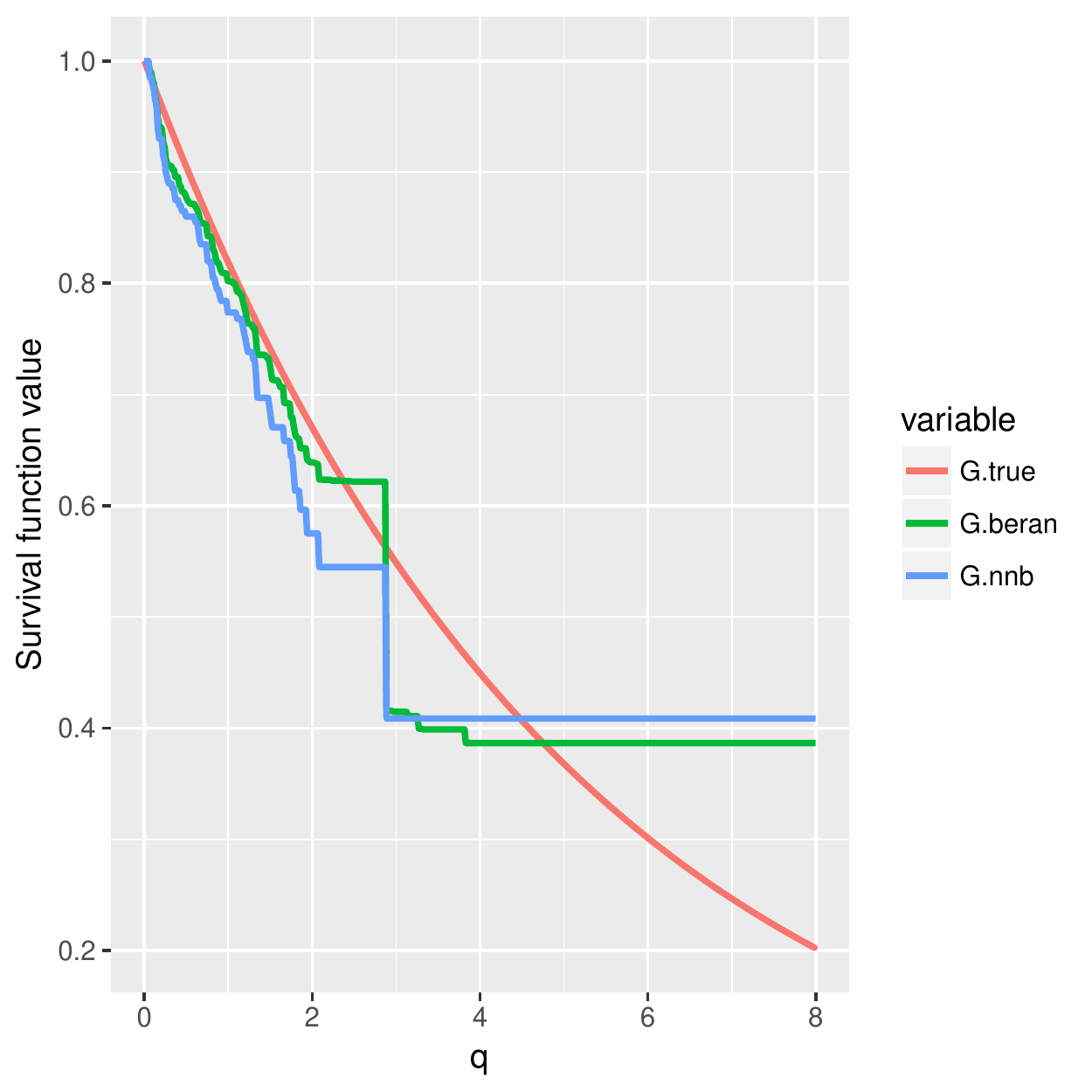}
        \caption{sample size: $2000$}
    \end{subfigure}
    ~
    \begin{subfigure}[b]{0.23\linewidth}
        \includegraphics[width=\textwidth]{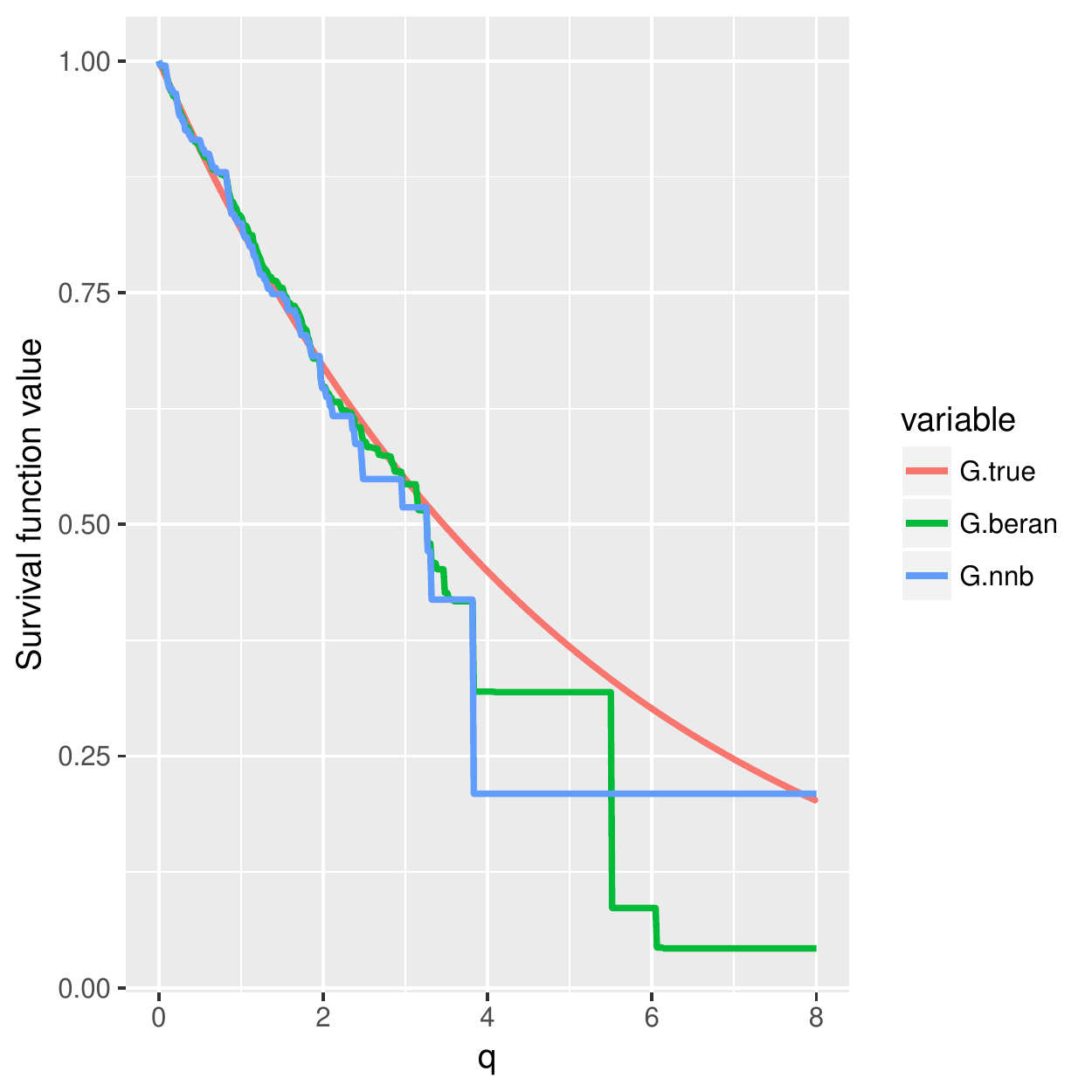}
        \caption{sample size: $2000$}
    \end{subfigure}
    ~
    \begin{subfigure}[b]{0.23\linewidth}
        \includegraphics[width=\textwidth]{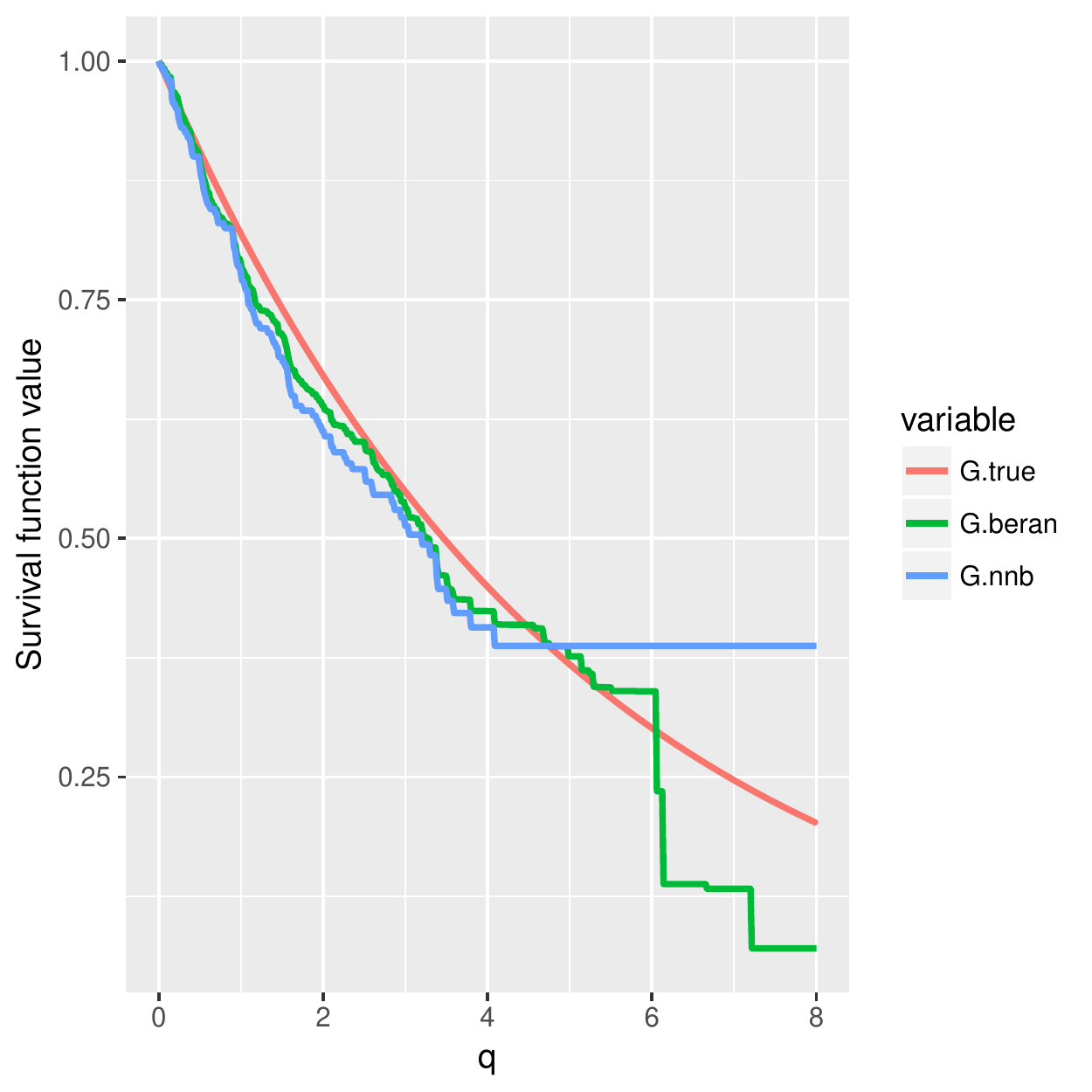}
        \caption{sample size: $2000$}
    \end{subfigure}
    ~
    \begin{subfigure}[b]{0.23\linewidth}
        \includegraphics[width=\textwidth]{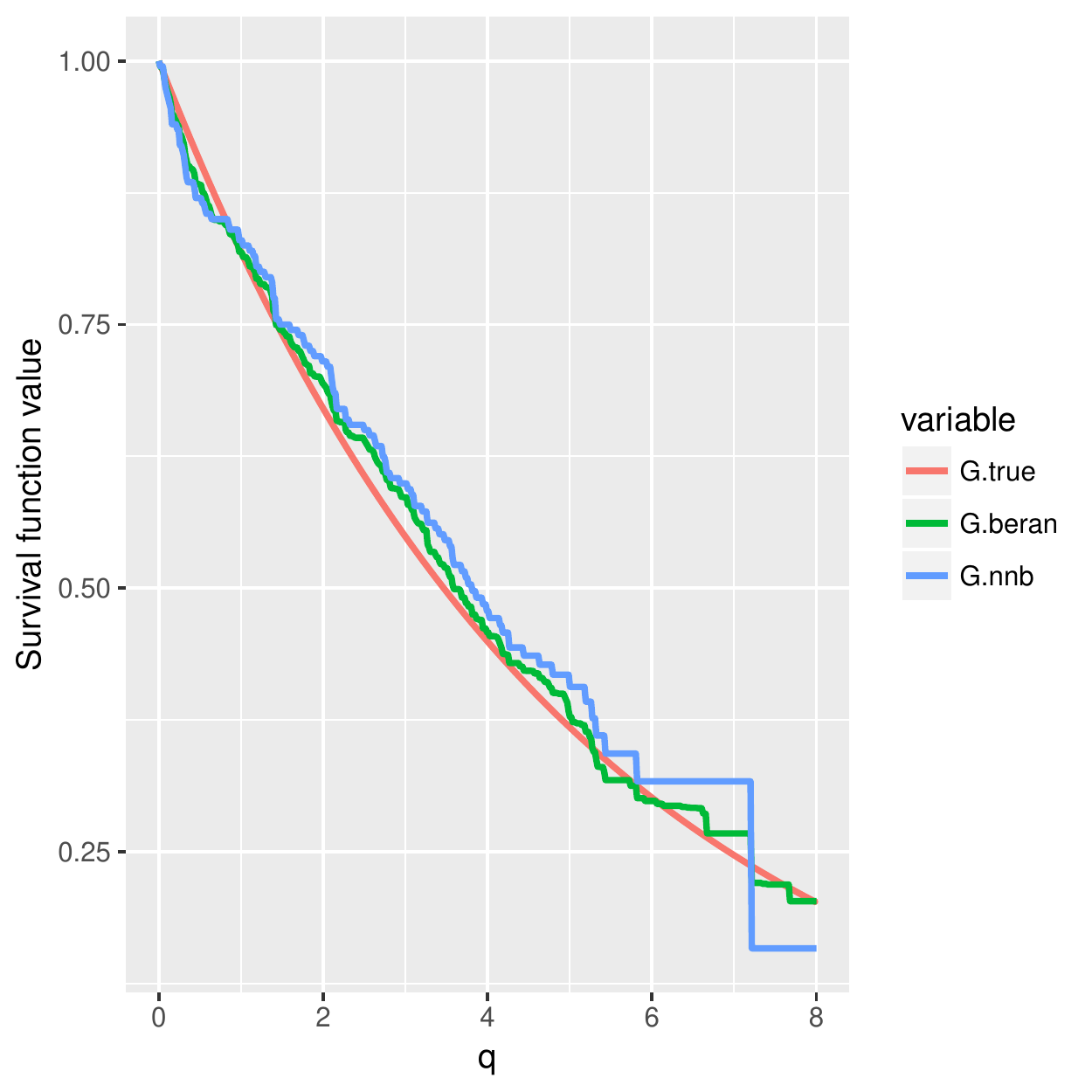}
        \caption{sample size: $2000$}
    \end{subfigure}
    
    \vspace{-0.05in}
    
    \begin{subfigure}[b]{0.23\linewidth}
        \includegraphics[width=\textwidth]{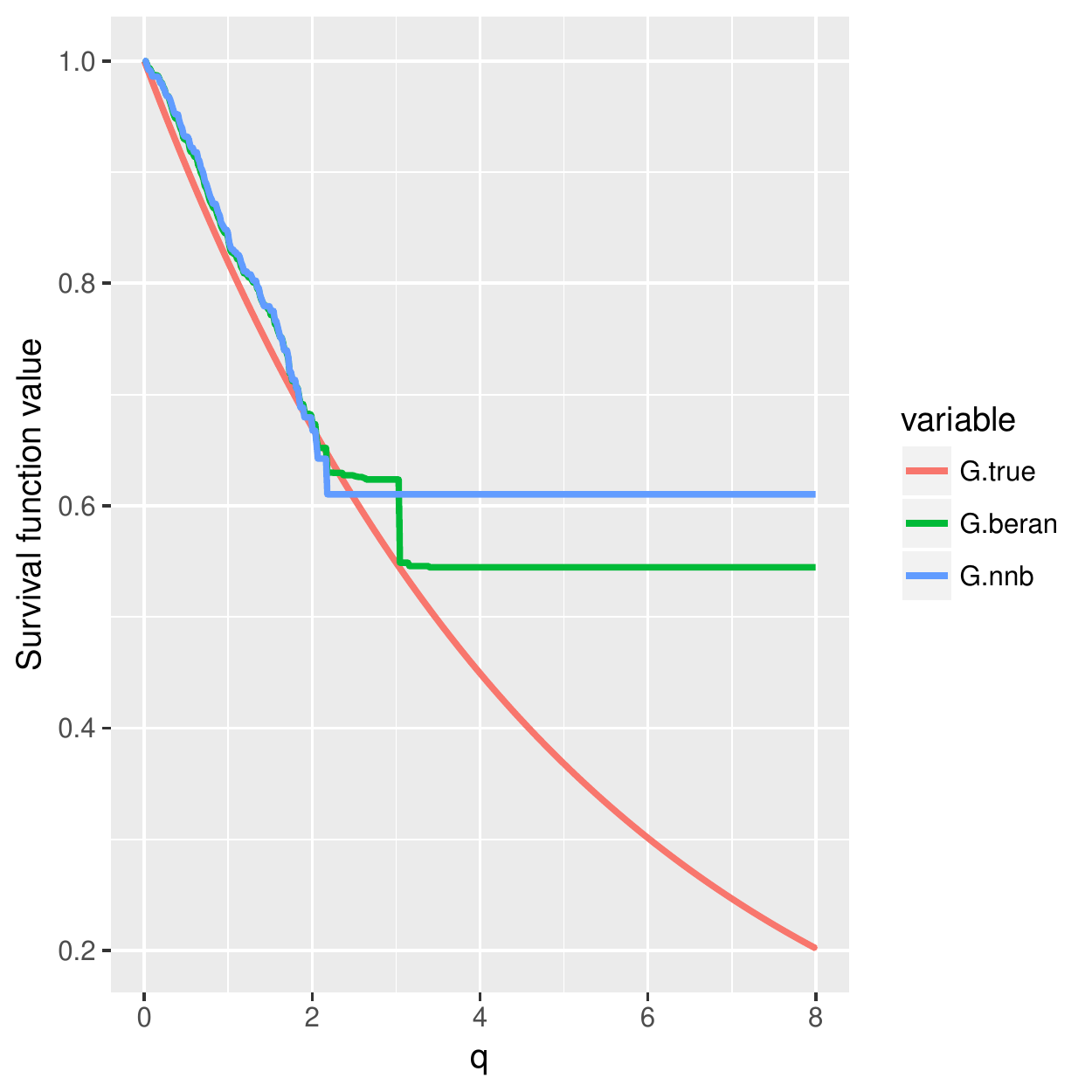}
        \caption{sample size: $5000$}
    \end{subfigure}
    ~
    \begin{subfigure}[b]{0.23\linewidth}
        \includegraphics[width=\textwidth]{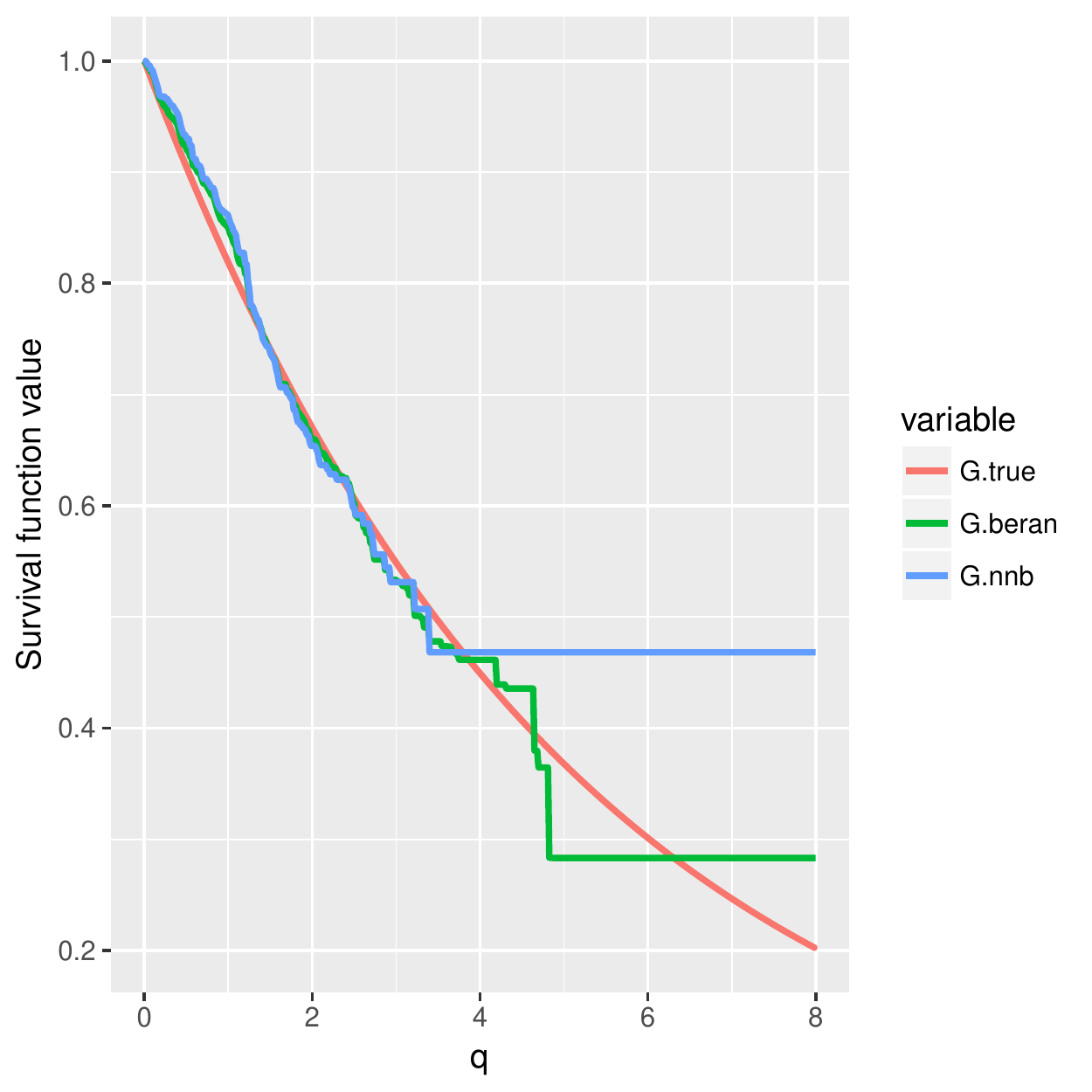}
        \caption{sample size: $5000$}
    \end{subfigure}
    ~
    \begin{subfigure}[b]{0.23\linewidth}
        \includegraphics[width=\textwidth]{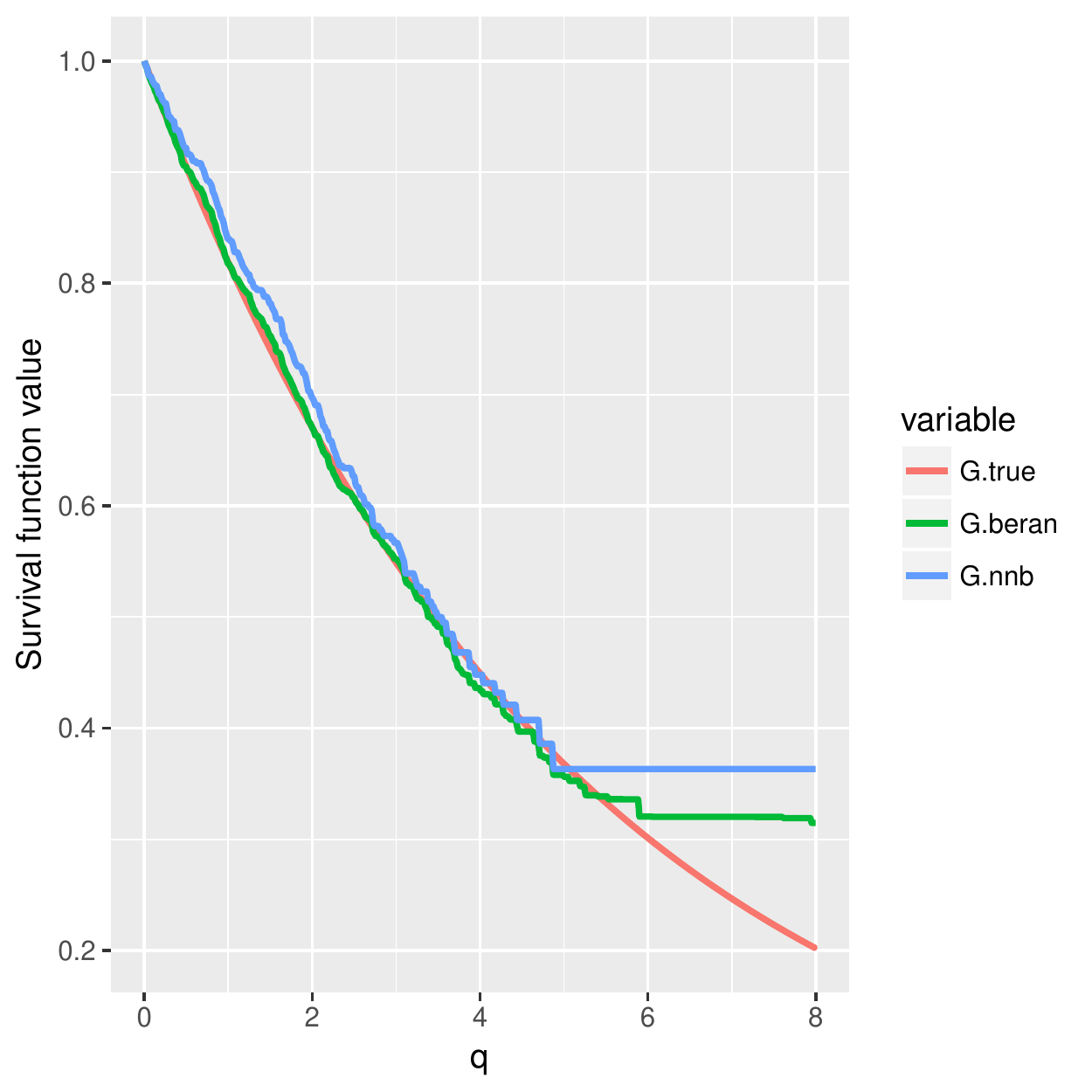}
        \caption{sample size: $5000$}
    \end{subfigure}
    ~
    \begin{subfigure}[b]{0.23\linewidth}
        \includegraphics[width=\textwidth]{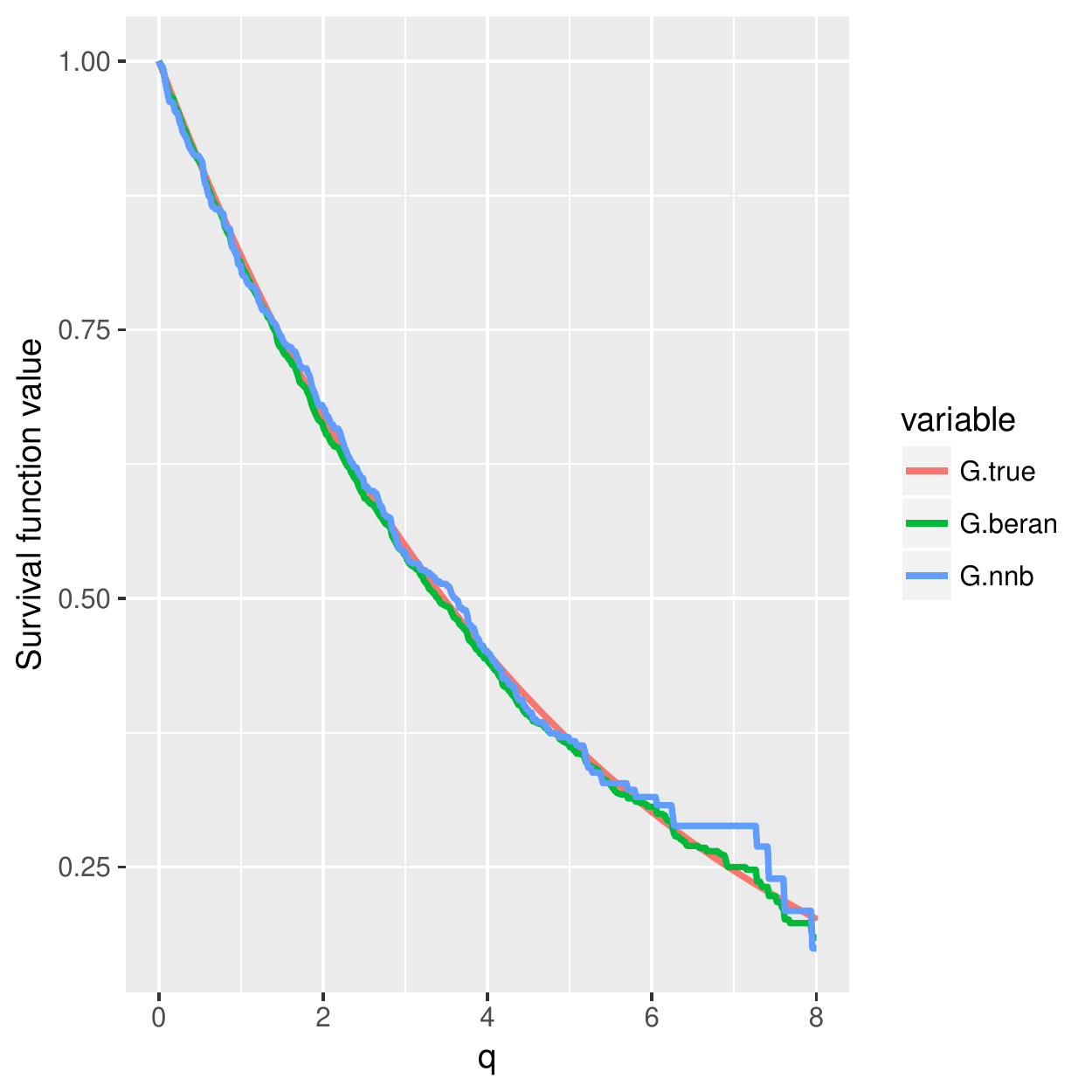}
        \caption{sample size: $5000$}
    \end{subfigure}
    
    \caption{Comparison of different conditional survival estimators on the one-dimensional AFT model. In this case, the censoring variable $C \sim \textrm{Exp}(\lambda=0.20)$, and the average censoring rate is around $50\%$. From left-most column to right-most column, we plot the conditional survival estimators for four test points, $x = 0.4, 0.8, 1.2, 1.6$.}
    \label{fig:g_comparison_high}
\end{figure}

\subsubsection{One-dimensional censored sine model}
Since our forest regression method \textit{crf} is nonparametric and does not rely on any parametric assumption between response and explanatory variables, it can be used to estimate quantiles for any general model $T = f(X) + \epsilon$. In this section, we let $f(x) = \sin(x)$ and have the model
\begin{equation*}
    T = 2.5 + \sin(X) + \epsilon
\end{equation*}
where $X \sim \textrm{Unif}(0,2 \pi)$ and $\epsilon \sim \mathcal{N}(0, 0.3^2)$. Then the censoring variable $C \sim 1 + \sin(X) + \textrm{Exp}(\lambda = 0.2)$, and the responses $Y = \min(T, C)$. All the other settings are the same as in Section \ref{sssec:1d-aft}. We plot out the training data and the quantile predictions for $\tau = 0.3, 0.5, 0.7$ in Figure \ref{fig:sine_1d}. The censoring level is about $25\%$. We observe that for all $\tau \in \{0.3,0.5,0.7\}$, \textit{crf} can produce comparable quantile predictions to \textit{grf-oracle}. Especially when $\tau=0.3$, the quantile prediction by $grf$ (blue dotted curve) severely deviates from the true quantile, while our method \textit{crf} can still predict the correct quantile and performs as good as \textit{grf-oracle}. We want to emphasize that \textit{grf-oracle} uses the latent responses $T_i$ while our method only uses the observed responses $Y_i$ and censoring indicators $\delta_i$. We then repeat the experiments for 40 times and report the box plots in Figure \ref{fig:sine_1d_box}. Again we can see that for all quantiles, our method \textit{crf} behaves almost as good as \textit{qrf-oracle} and \textit{grf-oracle}, and consistently better than \textit{qrf} and \textit{grf}. For example the order of magnitude of our error is twice and sometimes more than two times smaller than that of quantile or generalized  random forest. 

\begin{figure}[!htb]
    \small
    \centering
    \begin{subfigure}[b]{0.4\linewidth}
        \includegraphics[width=\textwidth]{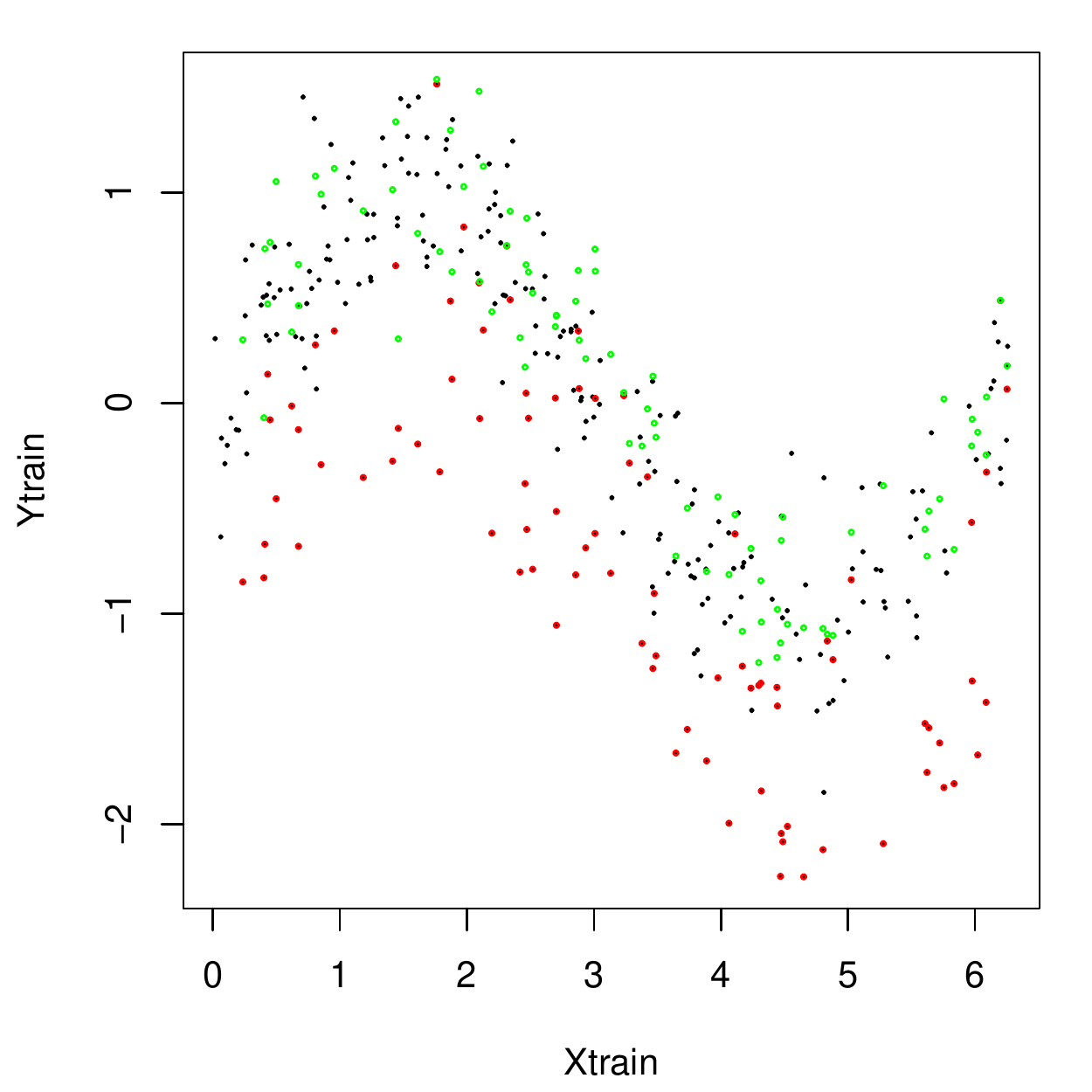}
        \caption{Training data}
        \label{fig:sine_train}
    \end{subfigure}
    ~
    \begin{subfigure}[b]{0.4\linewidth}
        \includegraphics[width=\textwidth]{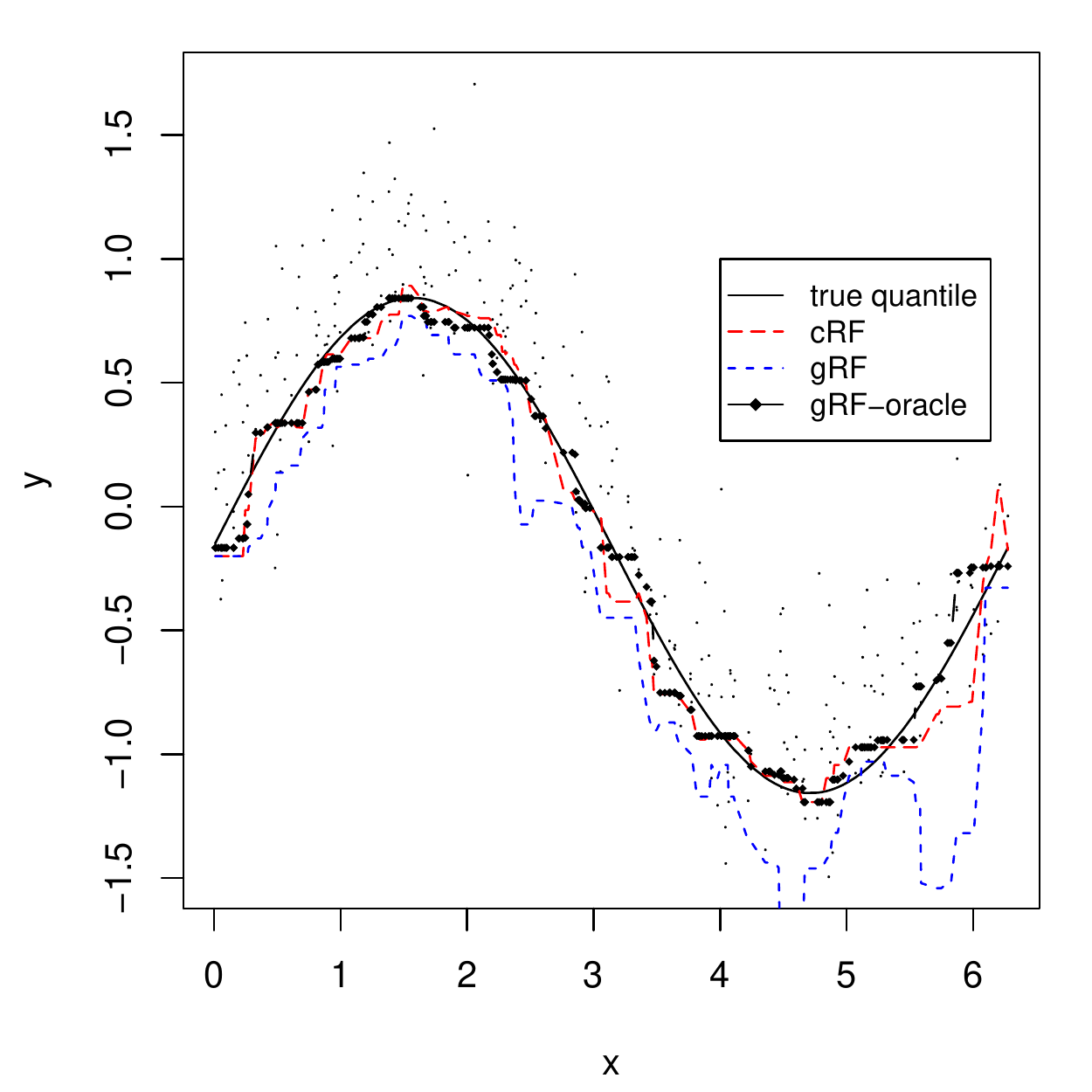}
        \caption{$\tau = 0.3$}
        \label{fig:sine_1d_tau_03}
    \end{subfigure}
    
    \vspace{-0.05in}
    
    \begin{subfigure}[b]{0.4\linewidth}
        \includegraphics[width=\textwidth]{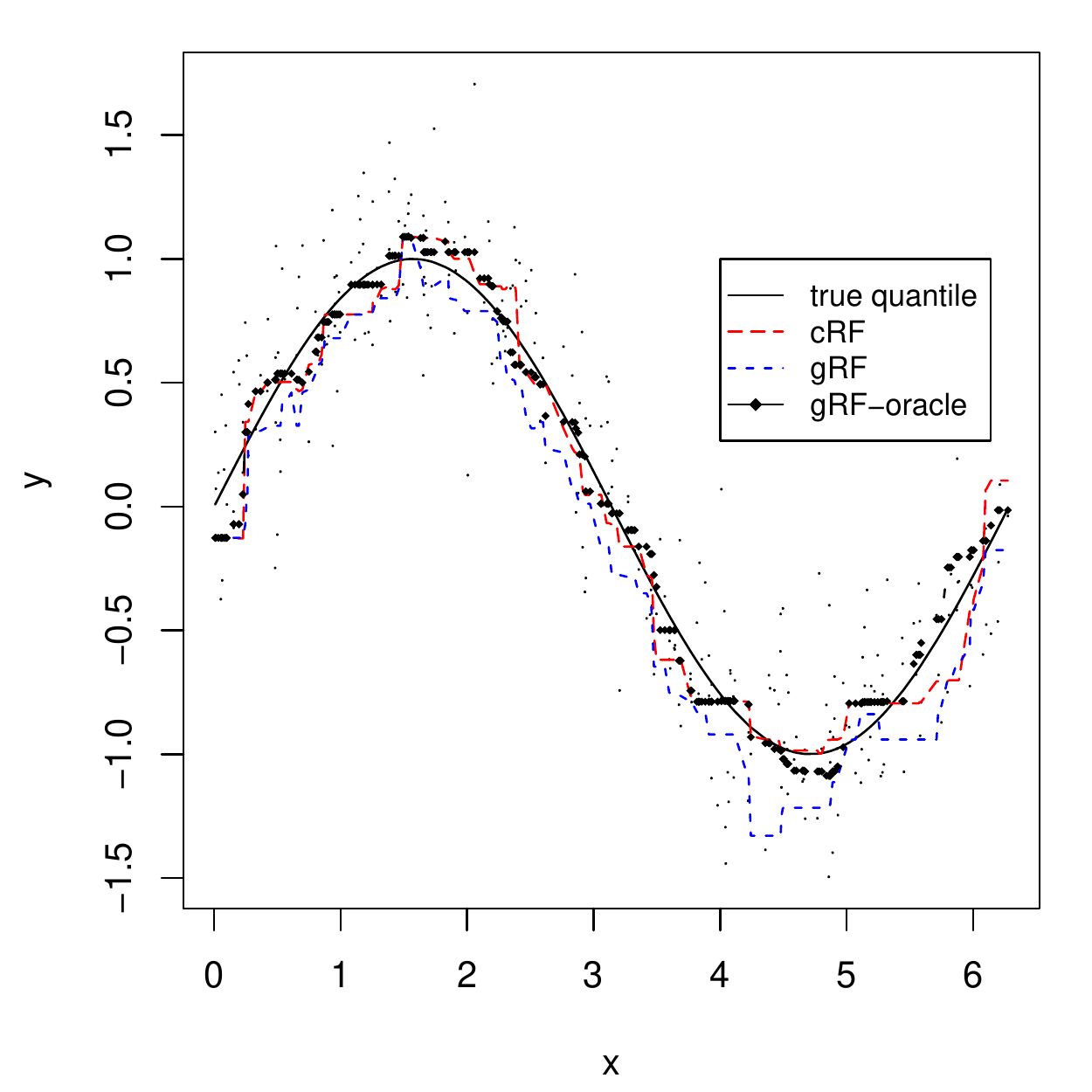}
        \caption{$\tau = 0.5$}
        \label{fig:sine_1d_tau_05}
    \end{subfigure}
    ~
    \begin{subfigure}[b]{0.4\linewidth}
        \includegraphics[width=\textwidth]{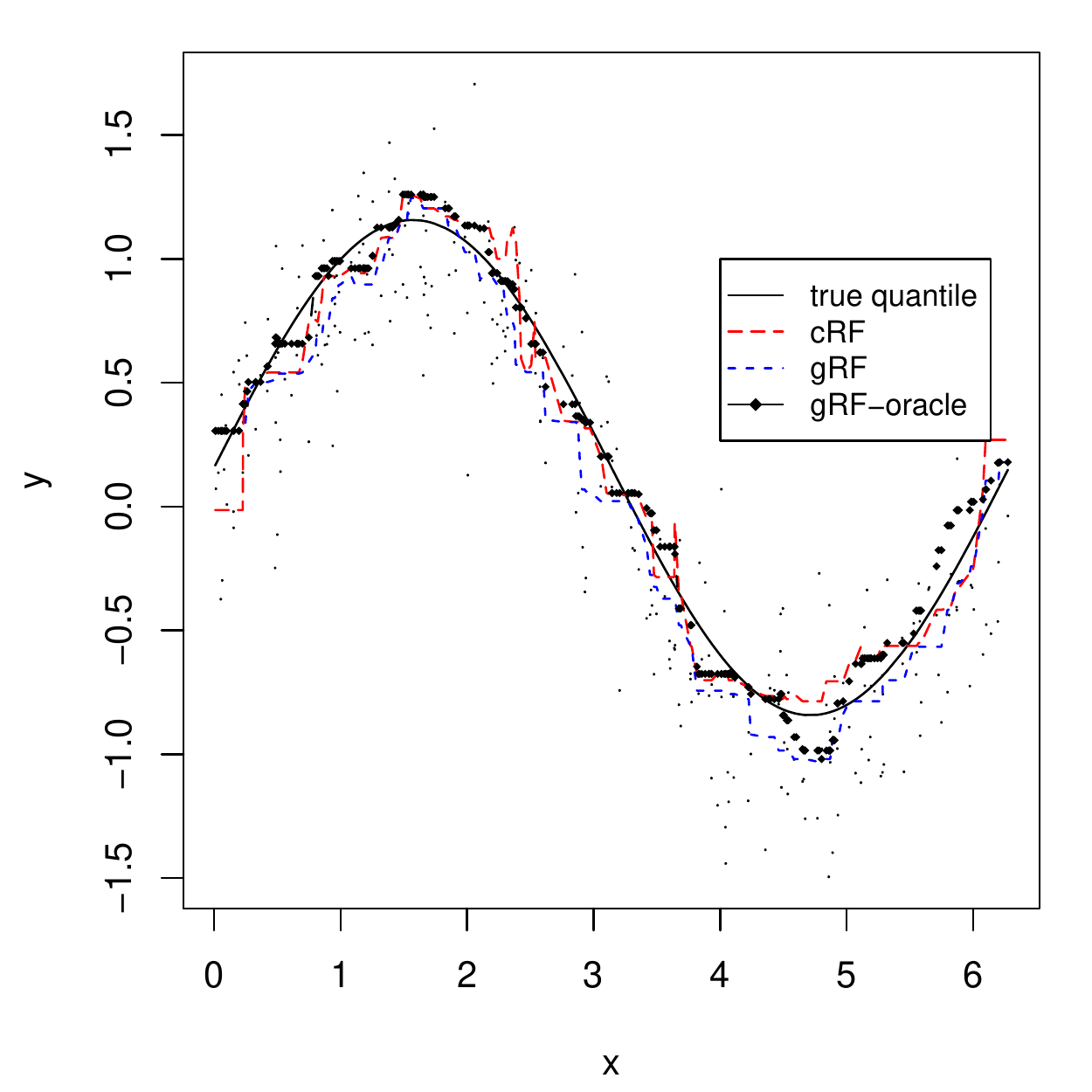}
        \caption{$\tau = 0.7$}
        \label{fig:sine_1d_tau_07}
    \end{subfigure}
    
    \caption{One-dimensional Sine model. In (a), black points stand for observation that are not censored; red points are observations that are censored, and the green points are the counterpart of the red points, that is, they are the latent values of those red points if they were not censored.}
    \label{fig:sine_1d}
\end{figure}

\begin{figure}[!htb]
    \small
    \centering
    \begin{subfigure}[b]{0.182\linewidth}
        \includegraphics[width=\textwidth]{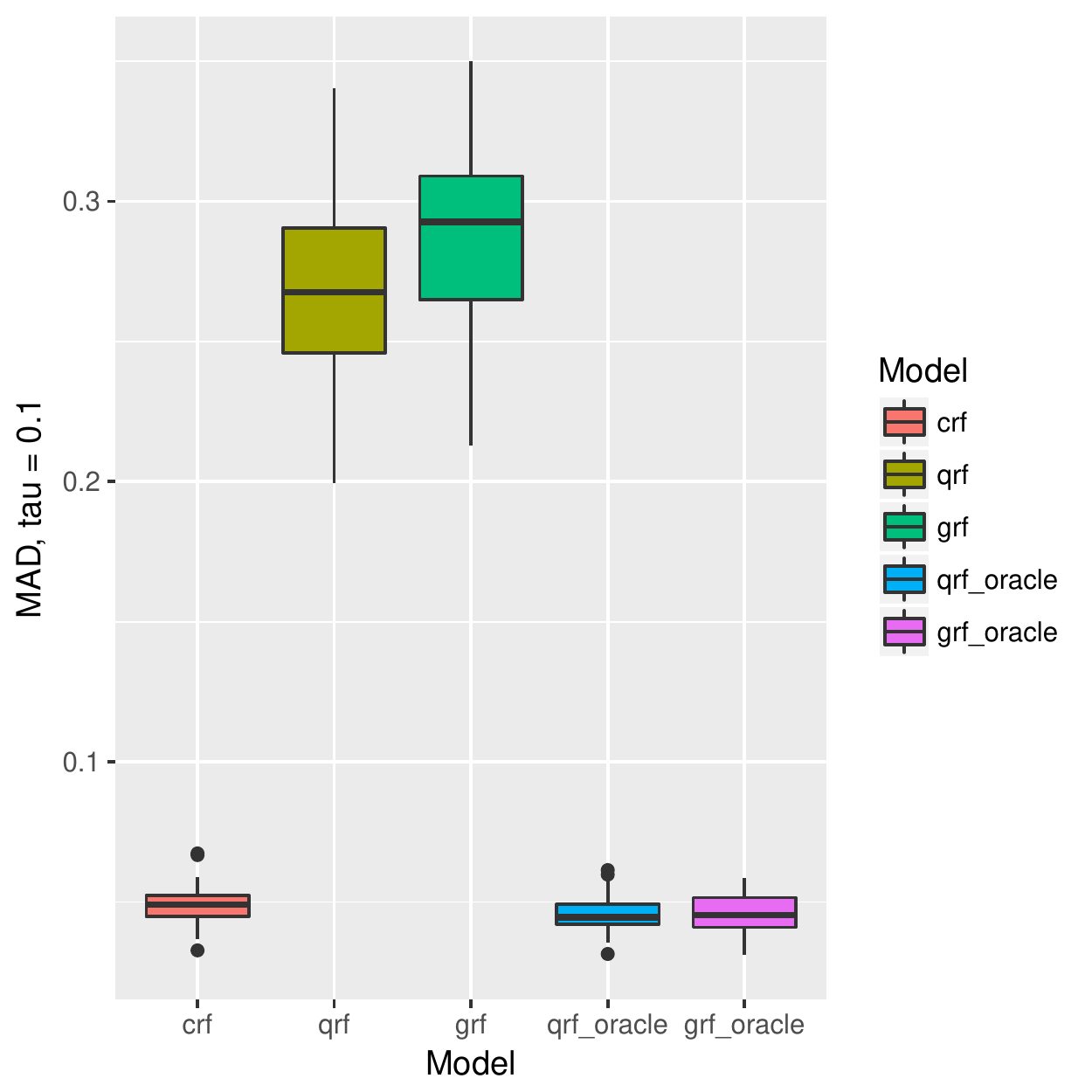}
        \caption{MAD: $\tau = 0.1$}
    \end{subfigure}
    ~
    \begin{subfigure}[b]{0.182\linewidth}
        \includegraphics[width=\textwidth]{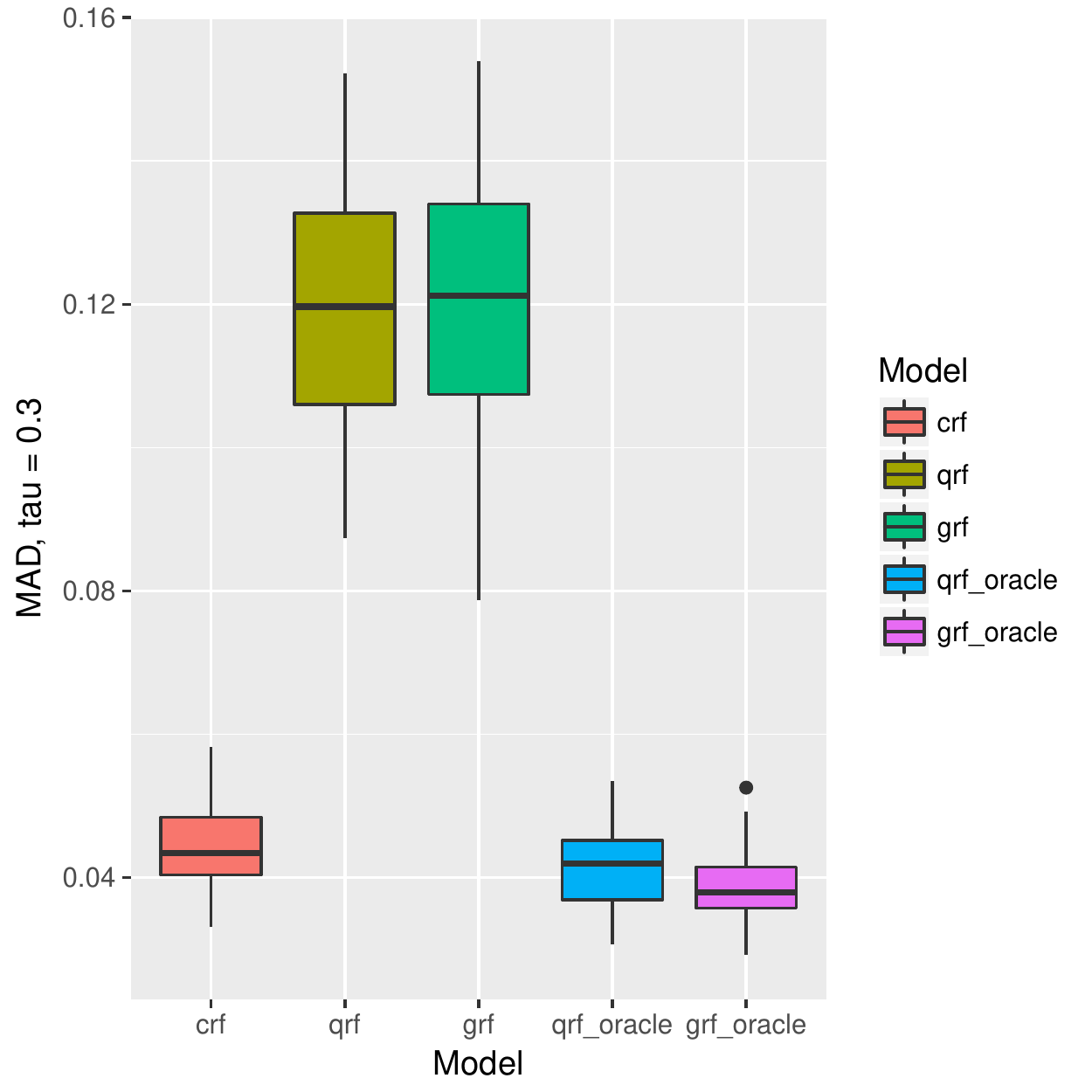}
        \caption{MAD: $\tau = 0.3$}
    \end{subfigure}
    ~
    \begin{subfigure}[b]{0.182\linewidth}
        \includegraphics[width=\textwidth]{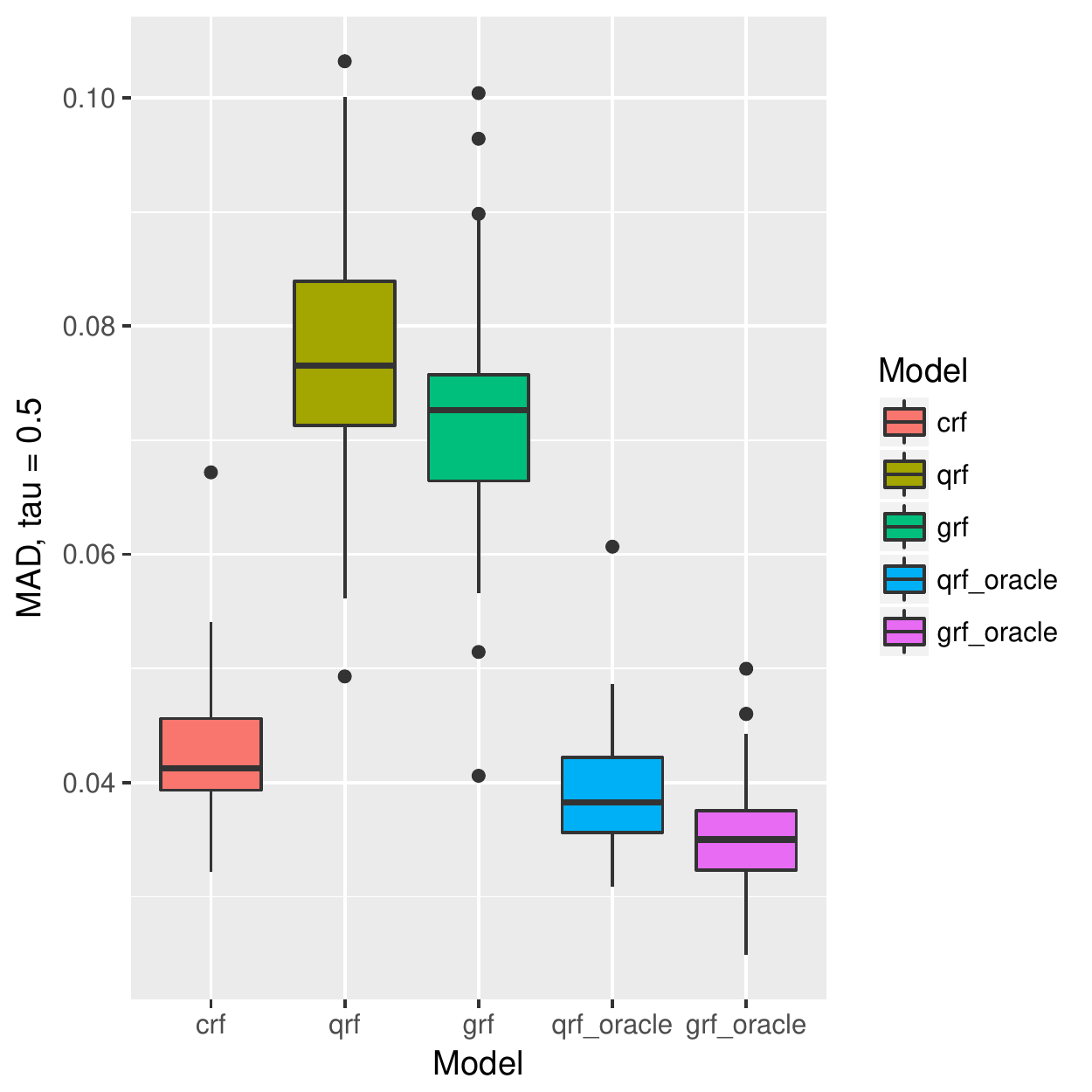}
        \caption{MAD: $\tau = 0.5$}
    \end{subfigure}
    ~
    \begin{subfigure}[b]{0.182\linewidth}
        \includegraphics[width=\textwidth]{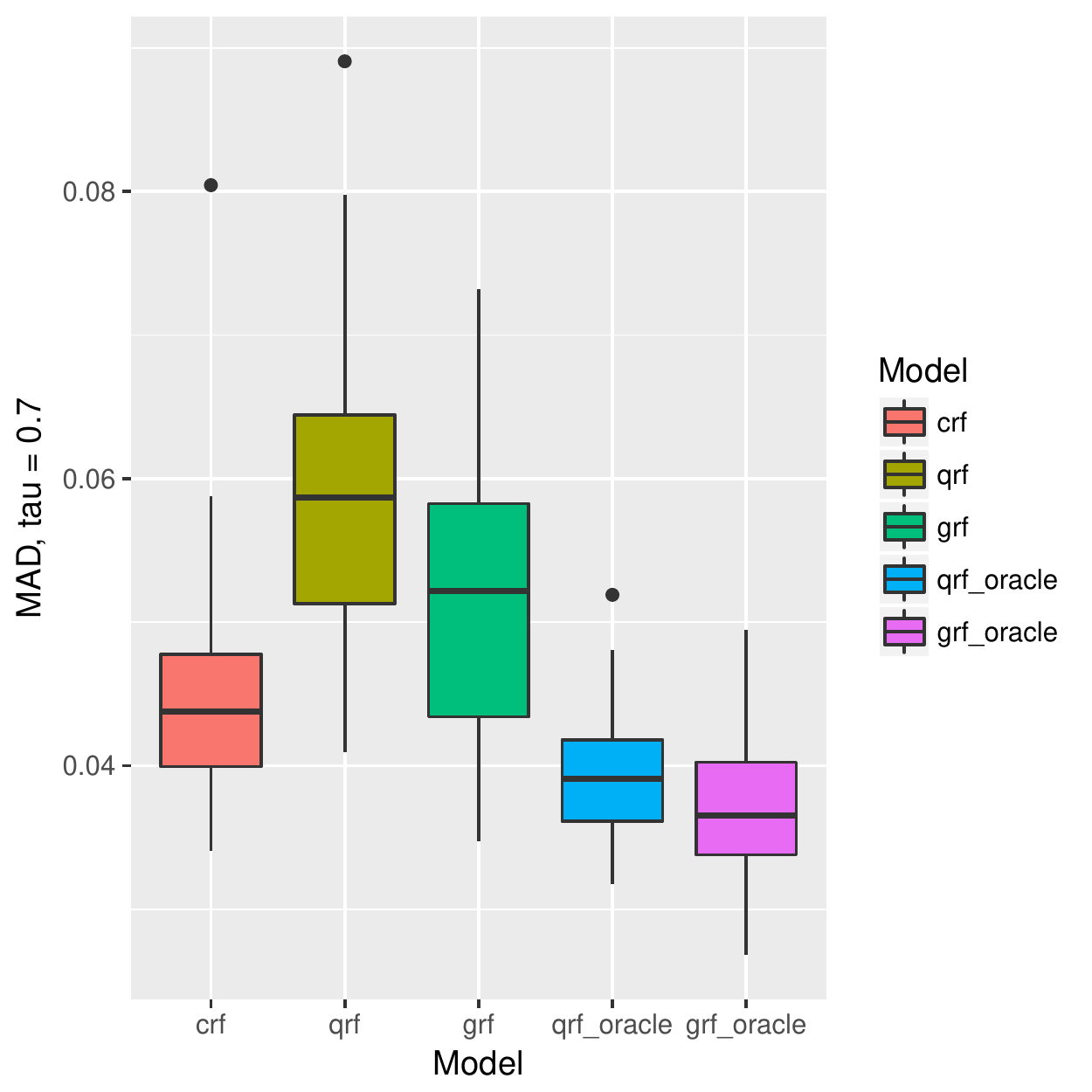}
        \caption{MAD: $\tau = 0.7$}
    \end{subfigure}
    ~
    \begin{subfigure}[b]{0.182\linewidth}
        \includegraphics[width=\textwidth]{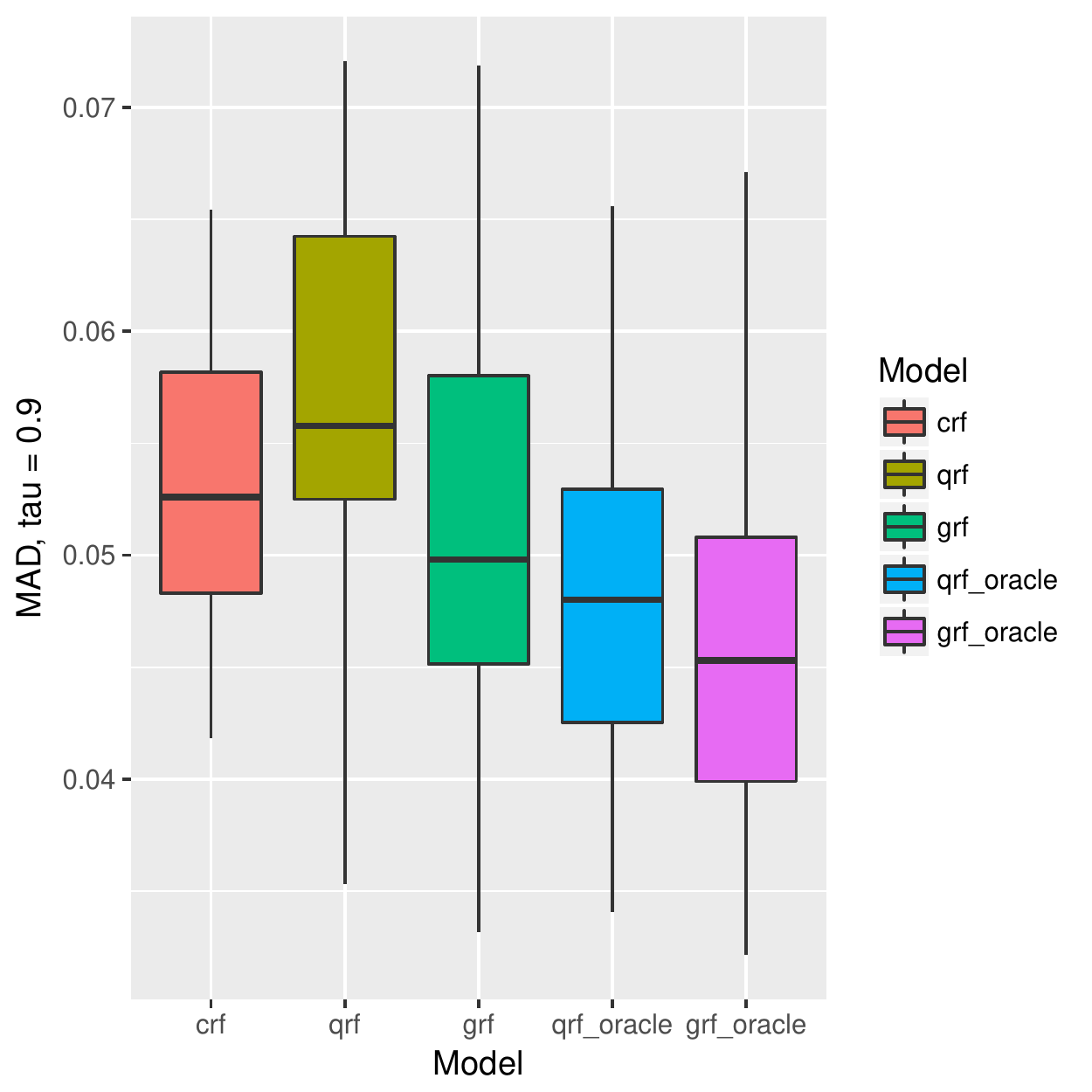}
        \caption{MAD: $\tau = 0.9$}
    \end{subfigure}
    
    \vspace{-0.05in}
    
    \begin{subfigure}[b]{0.18\linewidth}
        \includegraphics[width=\textwidth]{fig/mse_result_1.pdf}
        \caption{MSE: $\tau = 0.1$}
    \end{subfigure}
    ~
    \begin{subfigure}[b]{0.18\linewidth}
        \includegraphics[width=\textwidth]{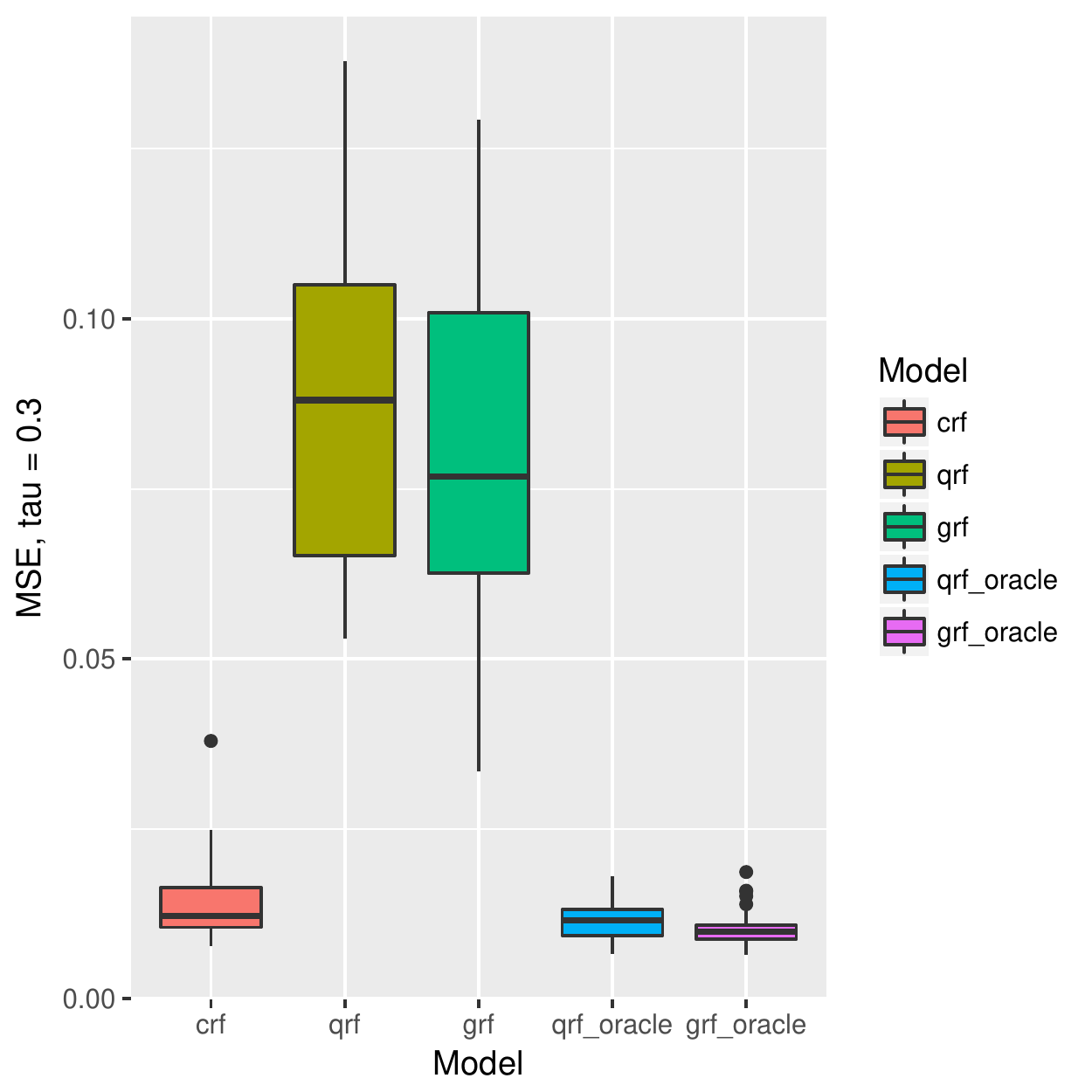}
        \caption{MSE: $\tau = 0.3$}
    \end{subfigure}
    ~
    \begin{subfigure}[b]{0.18\linewidth}
        \includegraphics[width=\textwidth]{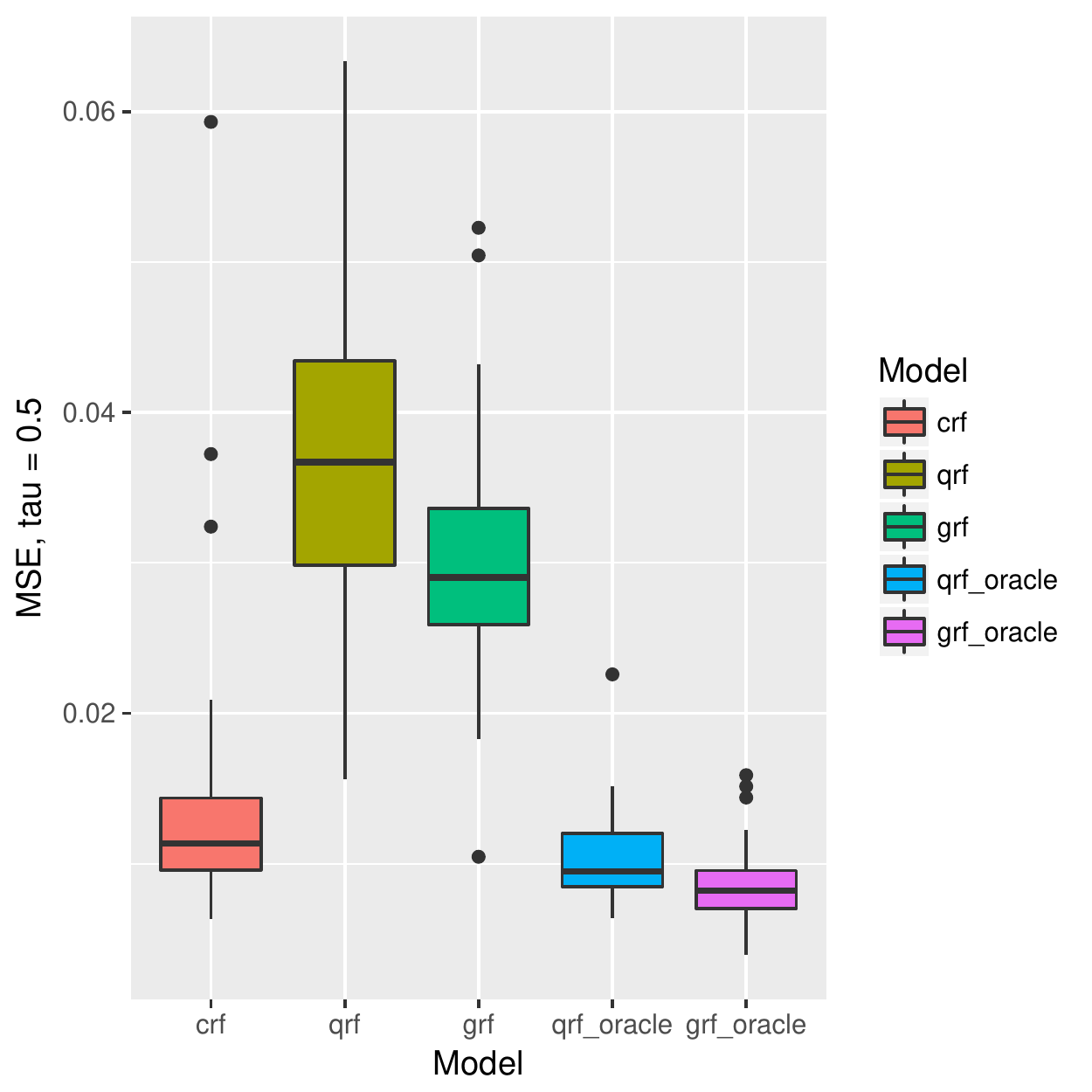}
        \caption{MSE: $\tau = 0.5$}
    \end{subfigure}
    ~
    \begin{subfigure}[b]{0.18\linewidth}
        \includegraphics[width=\textwidth]{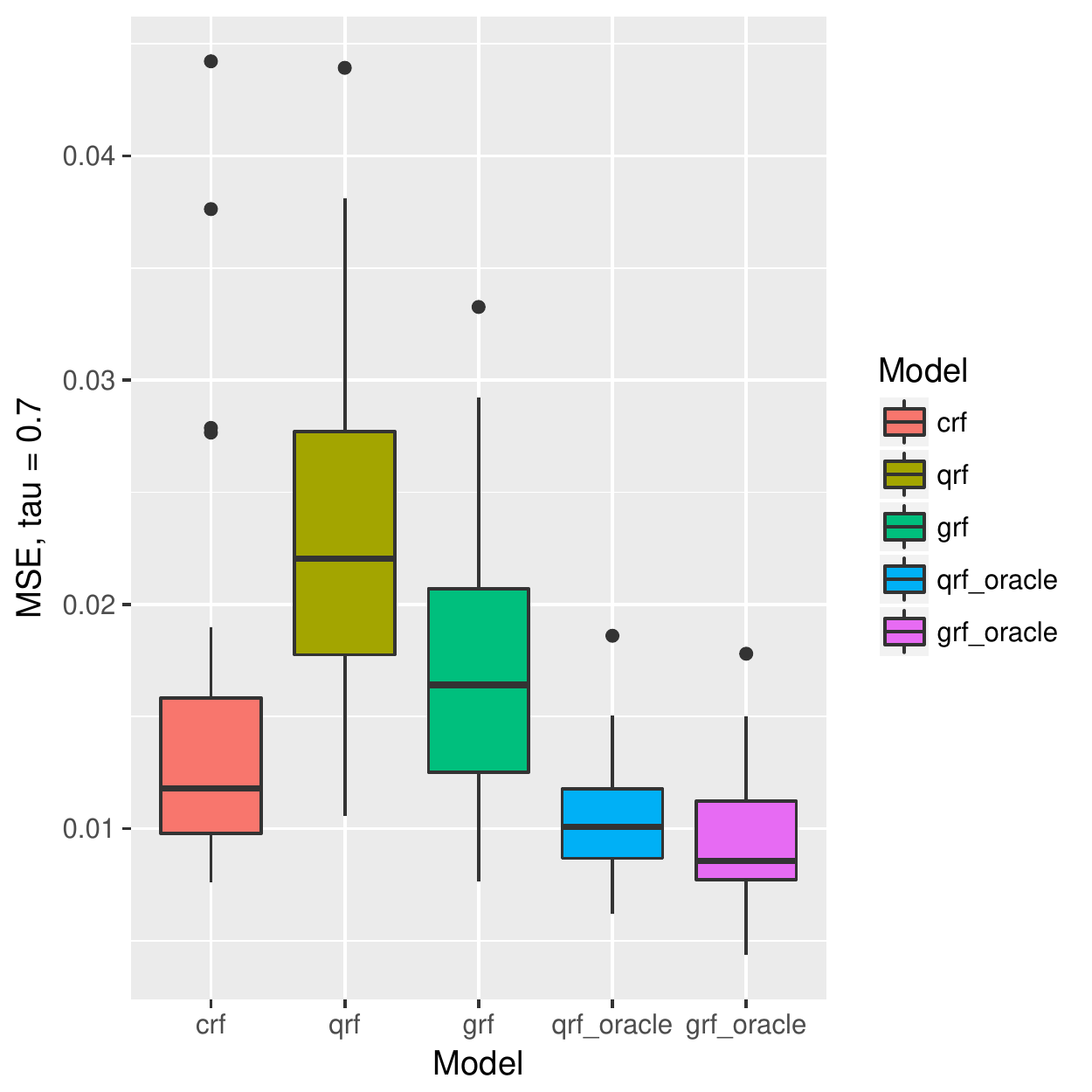}
        \caption{MSE: $\tau = 0.7$}
    \end{subfigure}
    ~
    \begin{subfigure}[b]{0.18\linewidth}
        \includegraphics[width=\textwidth]{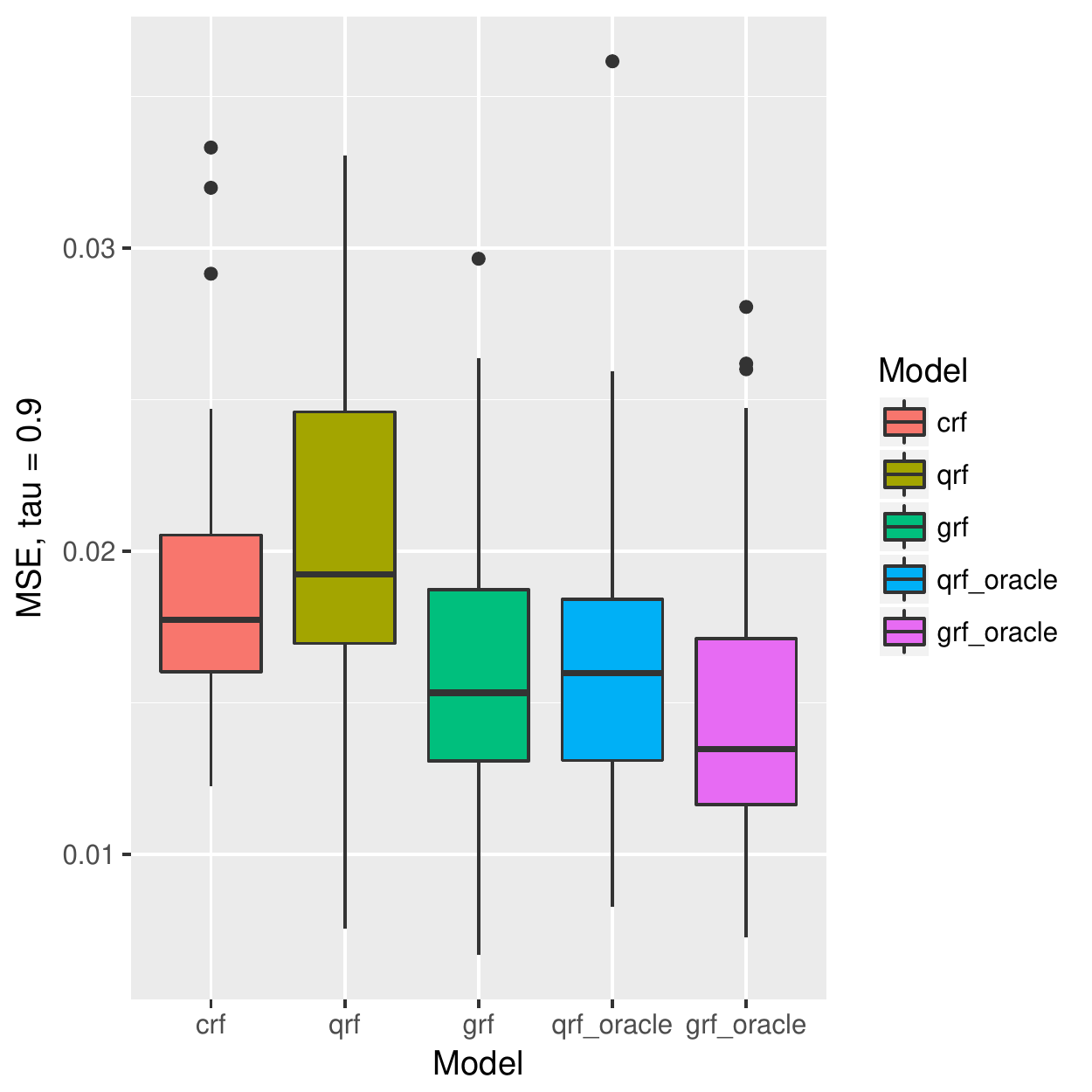}
        \caption{MSE: $\tau = 0.9$}
    \end{subfigure}
    
    \vspace{-0.05in}
    
    \begin{subfigure}[b]{0.18\linewidth}
        \includegraphics[width=\textwidth]{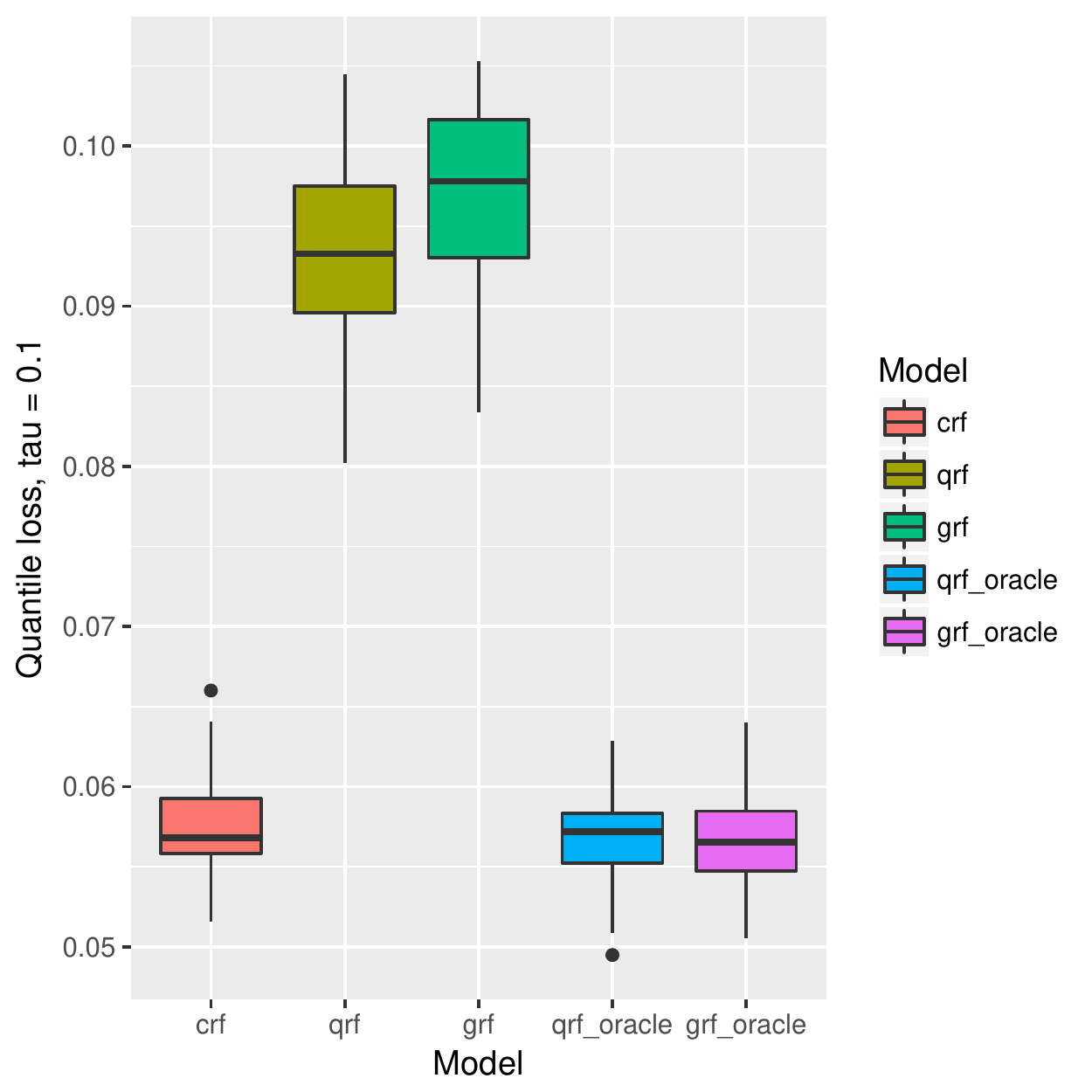}
        \caption{Qauntile loss: $\tau = 0.1$}
    \end{subfigure}
    ~
    \begin{subfigure}[b]{0.18\linewidth}
        \includegraphics[width=\textwidth]{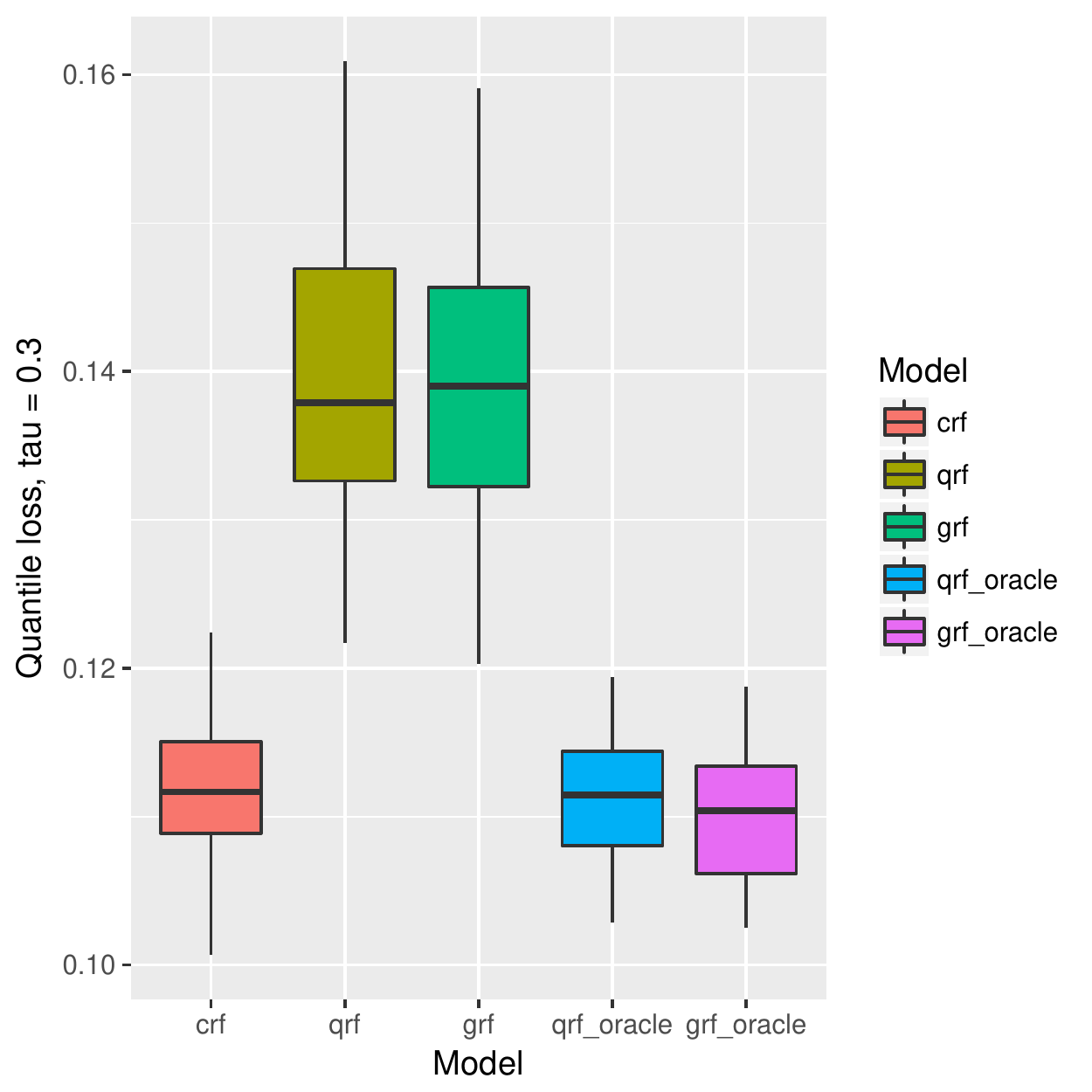}
        \caption{Qauntile loss: $\tau = 0.3$}
    \end{subfigure}
    ~ %add desired spacing between images, e. g. ~, \quad, \qquad, \hfill etc. 
    \begin{subfigure}[b]{0.18\linewidth}
        \includegraphics[width=\textwidth]{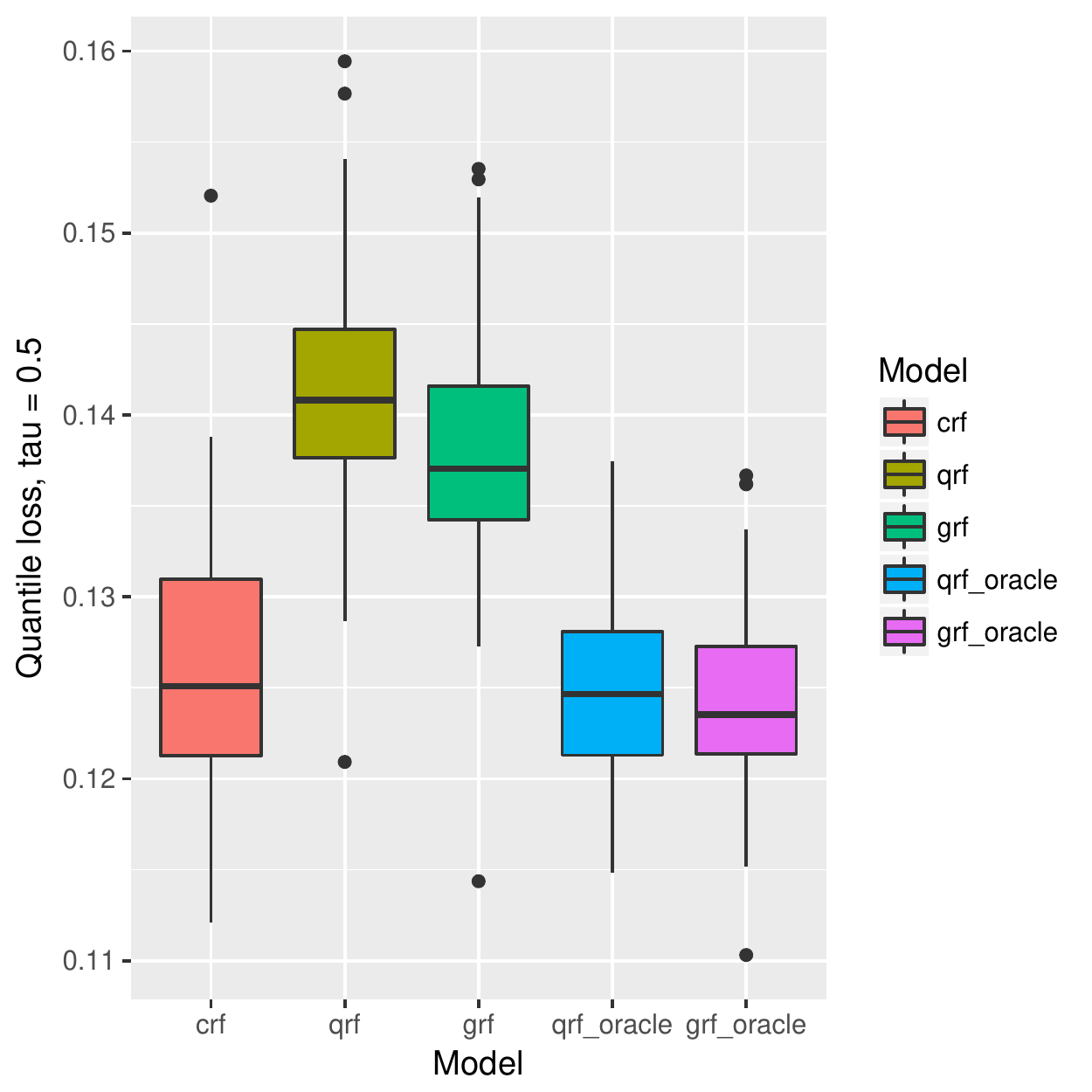}
        \caption{Quantile loss: $\tau = 0.5$}
    \end{subfigure}
    ~
    \begin{subfigure}[b]{0.18\linewidth}
        \includegraphics[width=\textwidth]{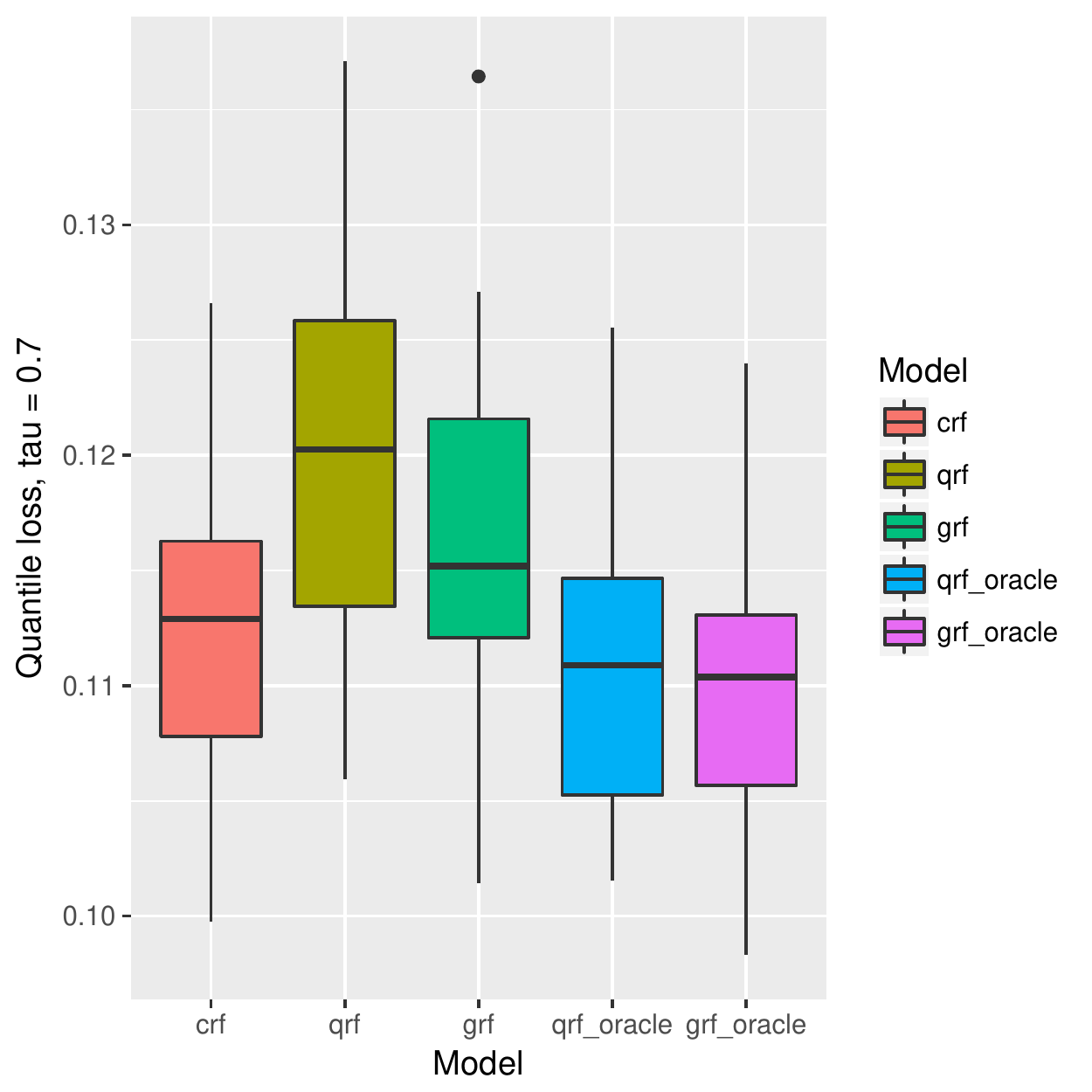}
        \caption{Quantile loss: $\tau = 0.7$}
    \end{subfigure}
    ~
    \begin{subfigure}[b]{0.18\linewidth}
        \includegraphics[width=\textwidth]{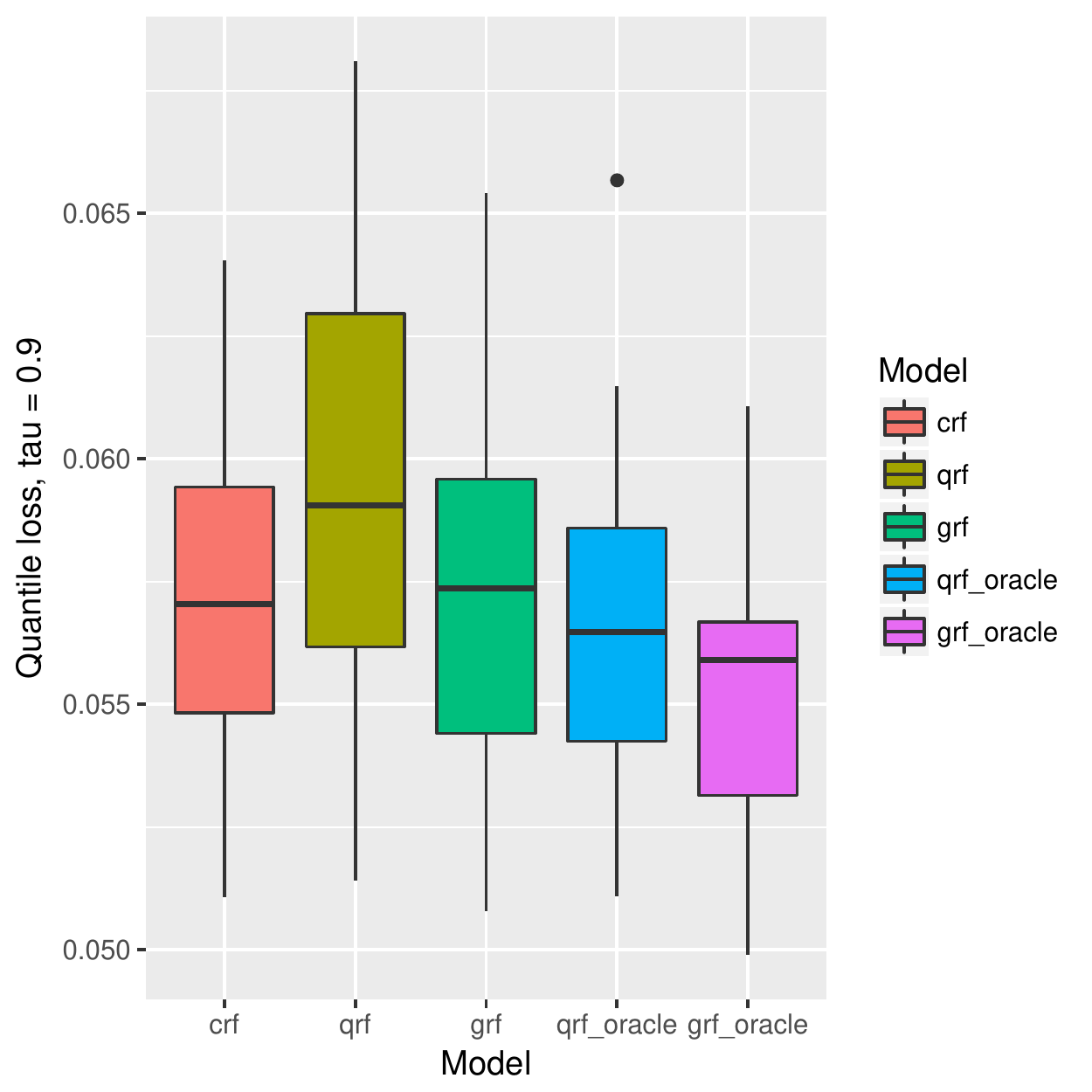}
        \caption{Qauntile loss: $\tau = 0.9$}
    \end{subfigure}
    
    \vspace{-0.05in}
    
    \begin{subfigure}[b]{0.18\linewidth}
        \includegraphics[width=\textwidth]{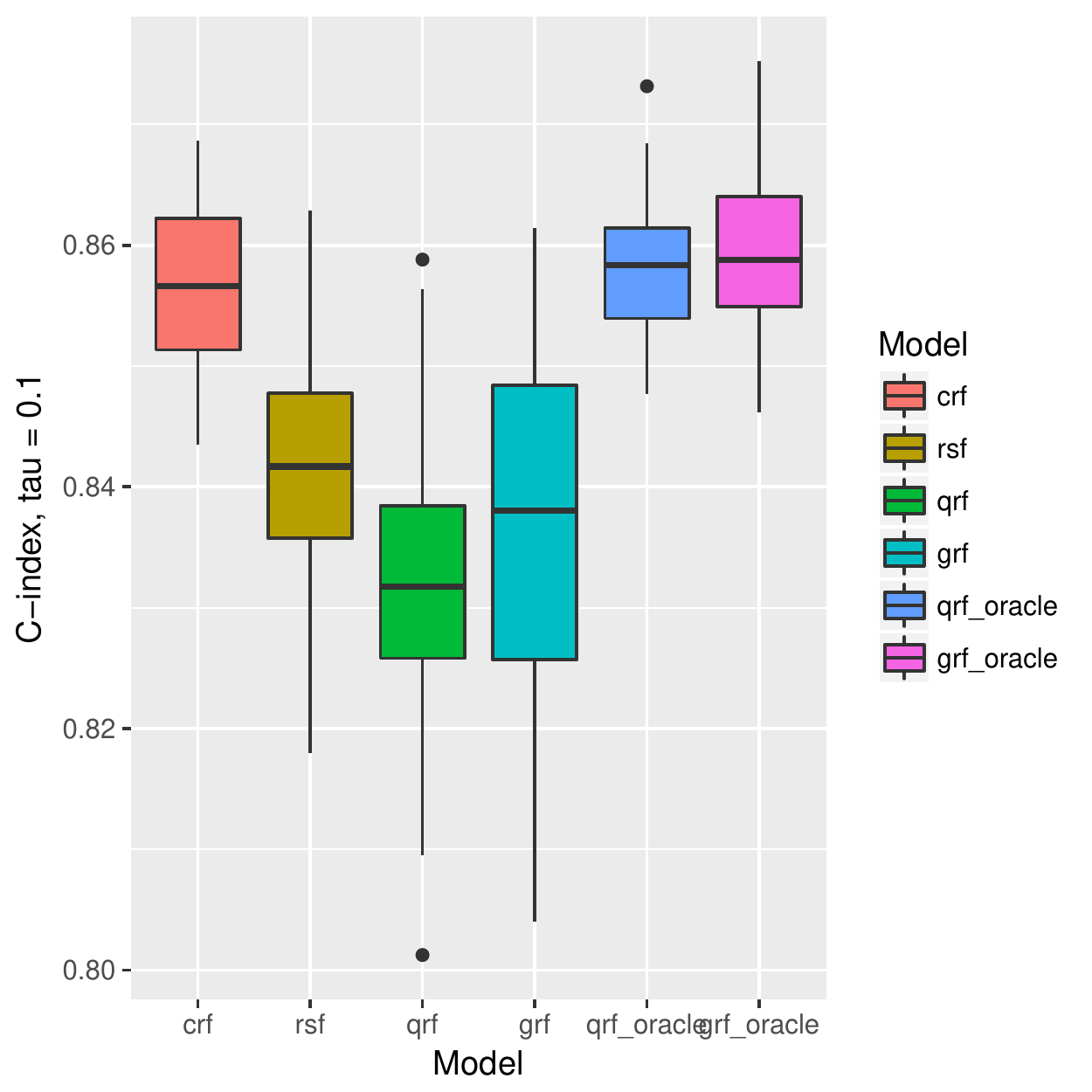}
        \caption{C-index: $\tau = 0.1$}
    \end{subfigure}
    ~
    \begin{subfigure}[b]{0.18\linewidth}
        \includegraphics[width=\textwidth]{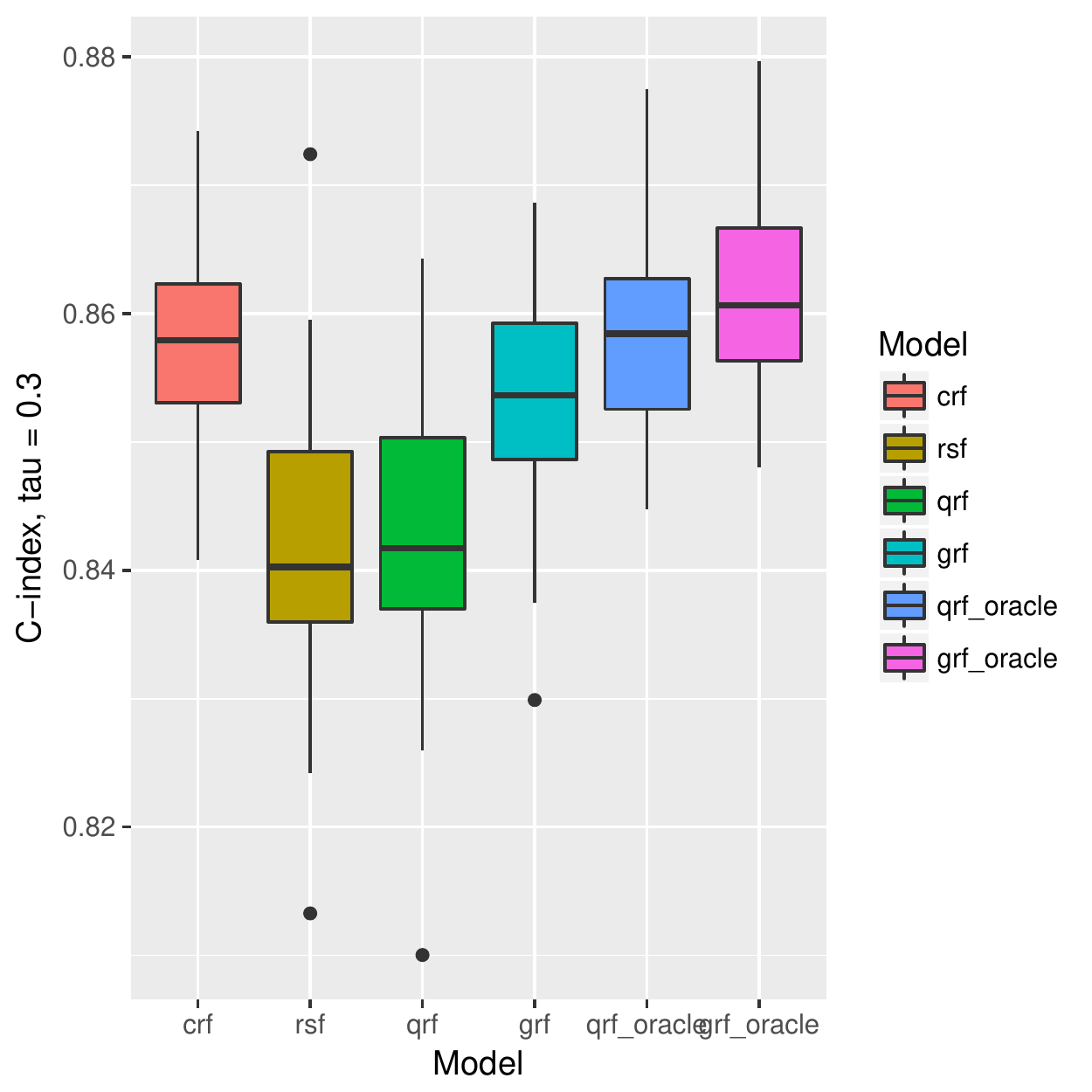}
        \caption{C-index: $\tau = 0.3$}
    \end{subfigure}
    ~ %add desired spacing between images, e. g. ~, \quad, \qquad, \hfill etc. 
    \begin{subfigure}[b]{0.18\linewidth}
        \includegraphics[width=\textwidth]{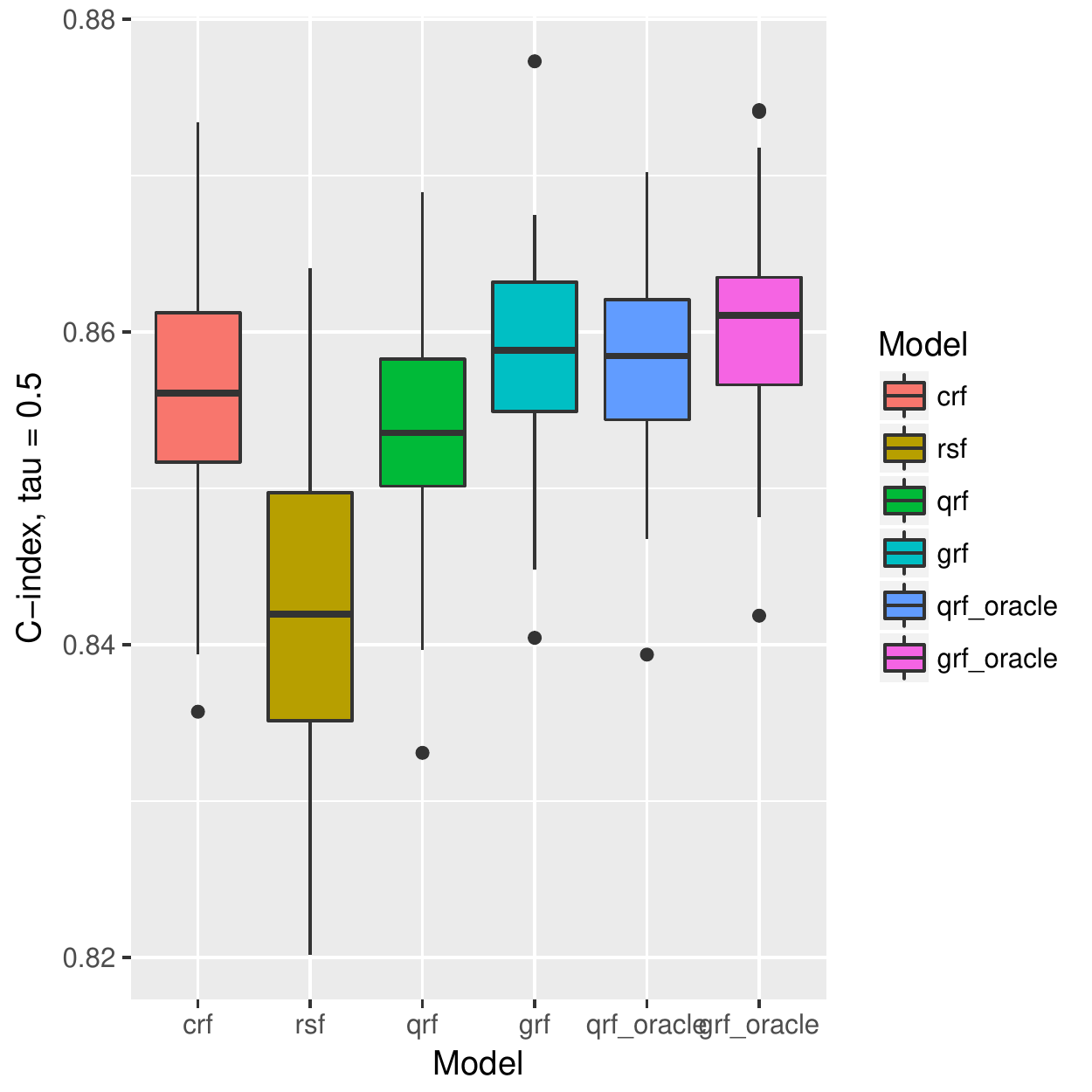}
        \caption{C-index: $\tau = 0.5$}
    \end{subfigure}
    ~
    \begin{subfigure}[b]{0.18\linewidth}
        \includegraphics[width=\textwidth]{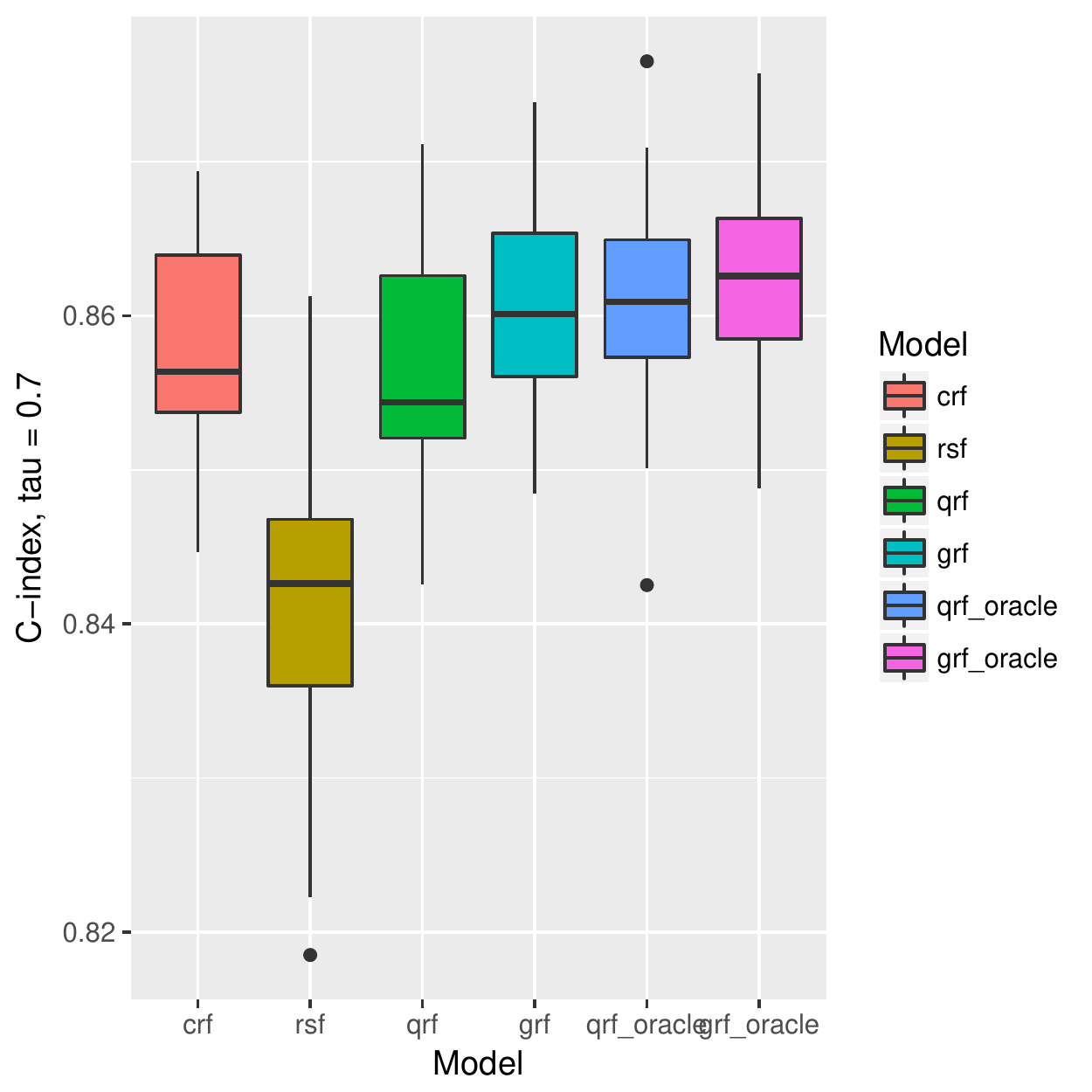}
        \caption{C-index: $\tau = 0.7$}
    \end{subfigure}
    ~
    \begin{subfigure}[b]{0.18\linewidth}
        \includegraphics[width=\textwidth]{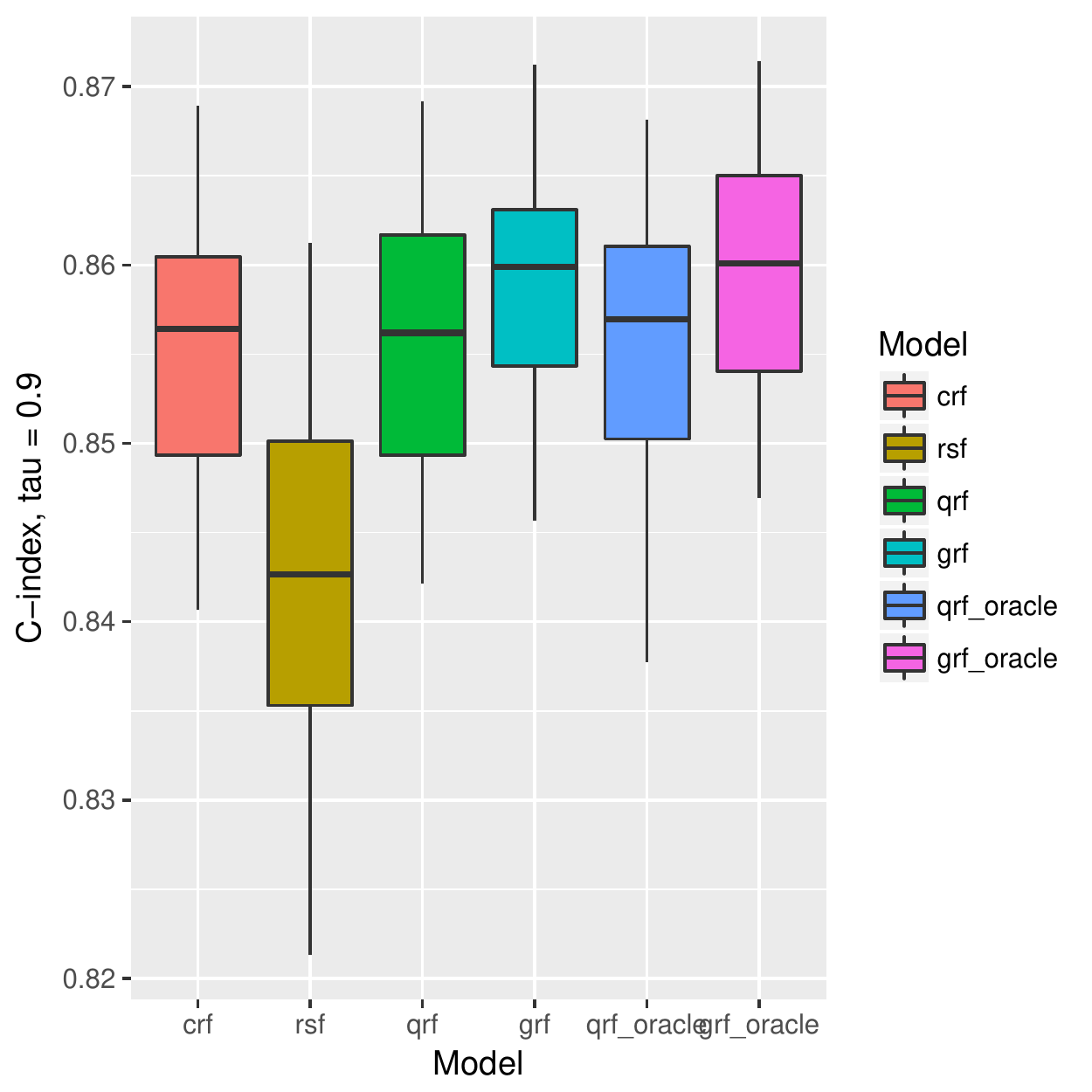}
        \caption{C-index: $\tau = 0.9$}
    \end{subfigure}
    
    \caption{One-dimensional censored sine model box plots with $n=300$ and $B=1000$. For the metrics MSE \eqref{eq:L_MSE}, MAD \eqref{eq:L_MAD} and quantile loss \eqref{eq:L_quantile}, the smaller the value is the better. For C-index, the larger it is the better.}
    \label{fig:sine_1d_box}
\end{figure}

\subsubsection{Multi-dimensional AFT model results}
In this section, we test our algorithm on a multi-dimensional AFT model
\begin{equation*}
    \log(T) = X^\top \beta + \epsilon,
\end{equation*}
where $\bm{\beta} = (0.1,0.2,0.3,0.4,0.5)$, $X_{\cdot,j} \sim \textrm{Unif}(0, 2)$, and $\epsilon \sim \mathcal{N}(0, 0.3^2)$. The censoring variable $C \sim \textrm{Exp}(\lambda = 0.05)$, and $Y = \min(T, C)$. The censoring level is about $22\%$. The number of training data is 500 and the number of test points is 300. All the forests consist of 1000 trees. The result is in Figure \ref{fig:aft_multi_box}. Our model \textit{crf} still outperforms \textit{qrf} and \textit{grf} significantly, and is comparable to \textit{qrf-oracle} and \textit{grf-oracle}.

\begin{figure}[!htb]
    \small
    \centering
    \begin{subfigure}[b]{0.182\linewidth}
        \includegraphics[width=\textwidth]{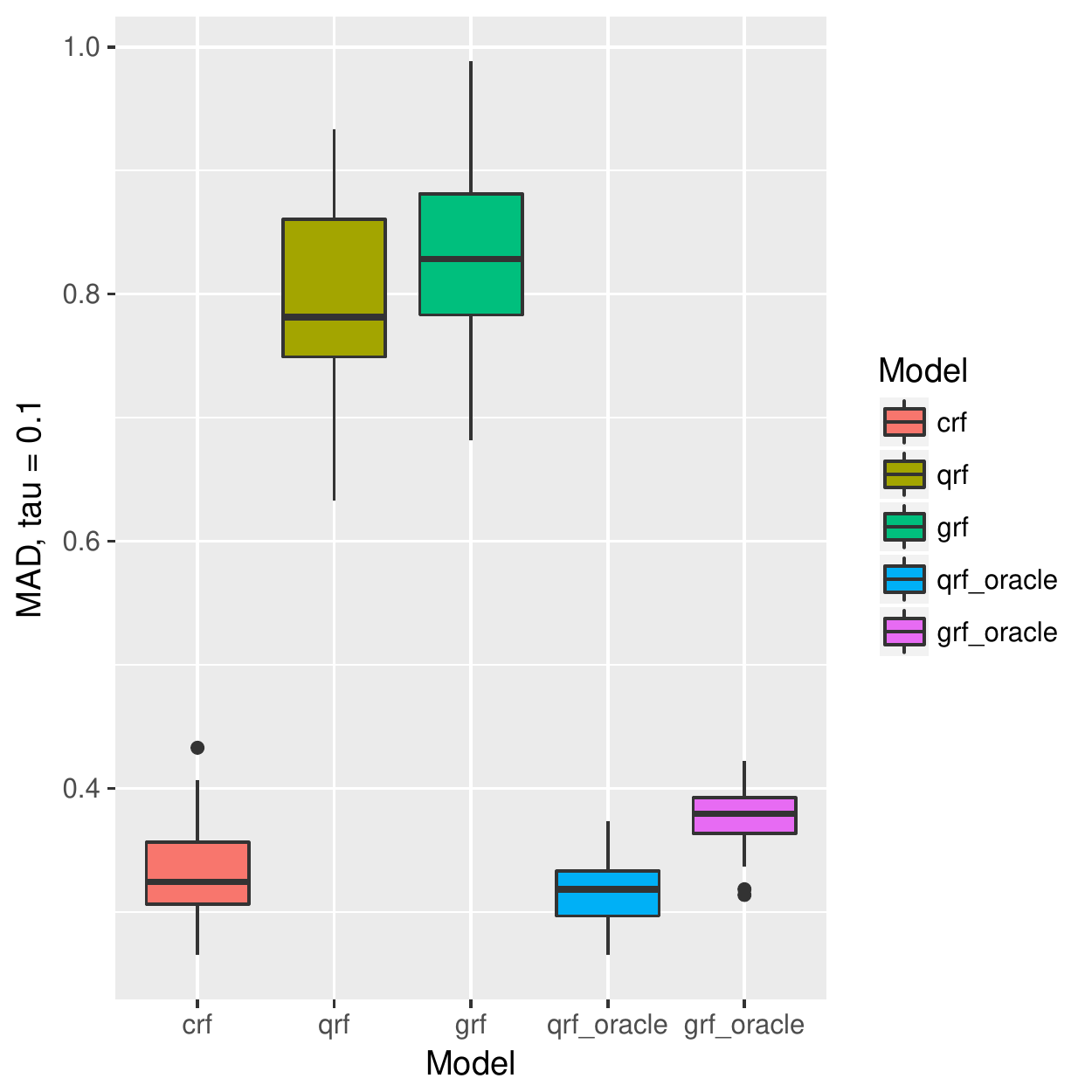}
        \caption{MAD: $\tau = 0.1$}
    \end{subfigure}
    ~
    \begin{subfigure}[b]{0.182\linewidth}
        \includegraphics[width=\textwidth]{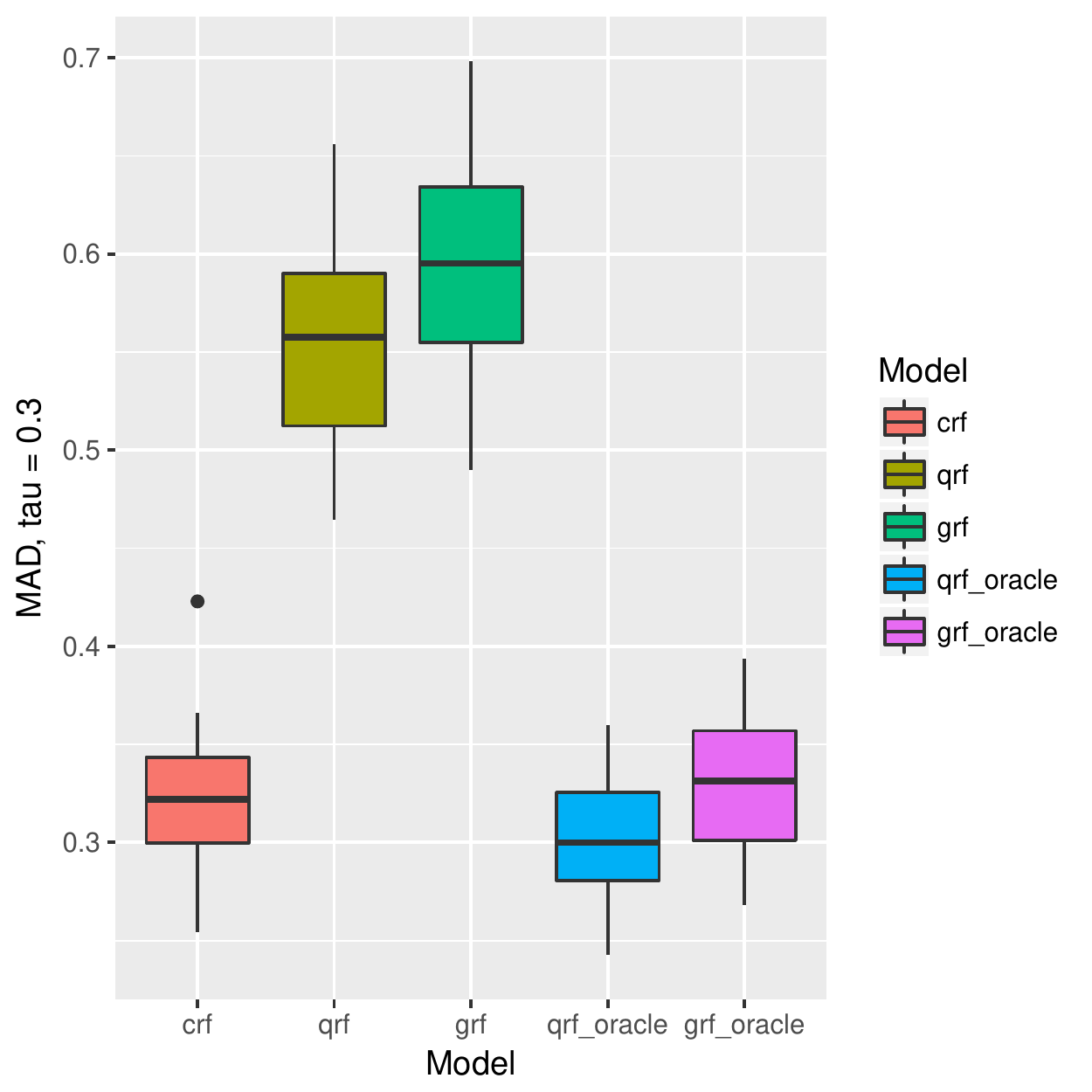}
        \caption{MAD: $\tau = 0.3$}
    \end{subfigure}
    ~ %add desired spacing between images, e. g. ~, \quad, \qquad, \hfill etc. 
    \begin{subfigure}[b]{0.182\linewidth}
        \includegraphics[width=\textwidth]{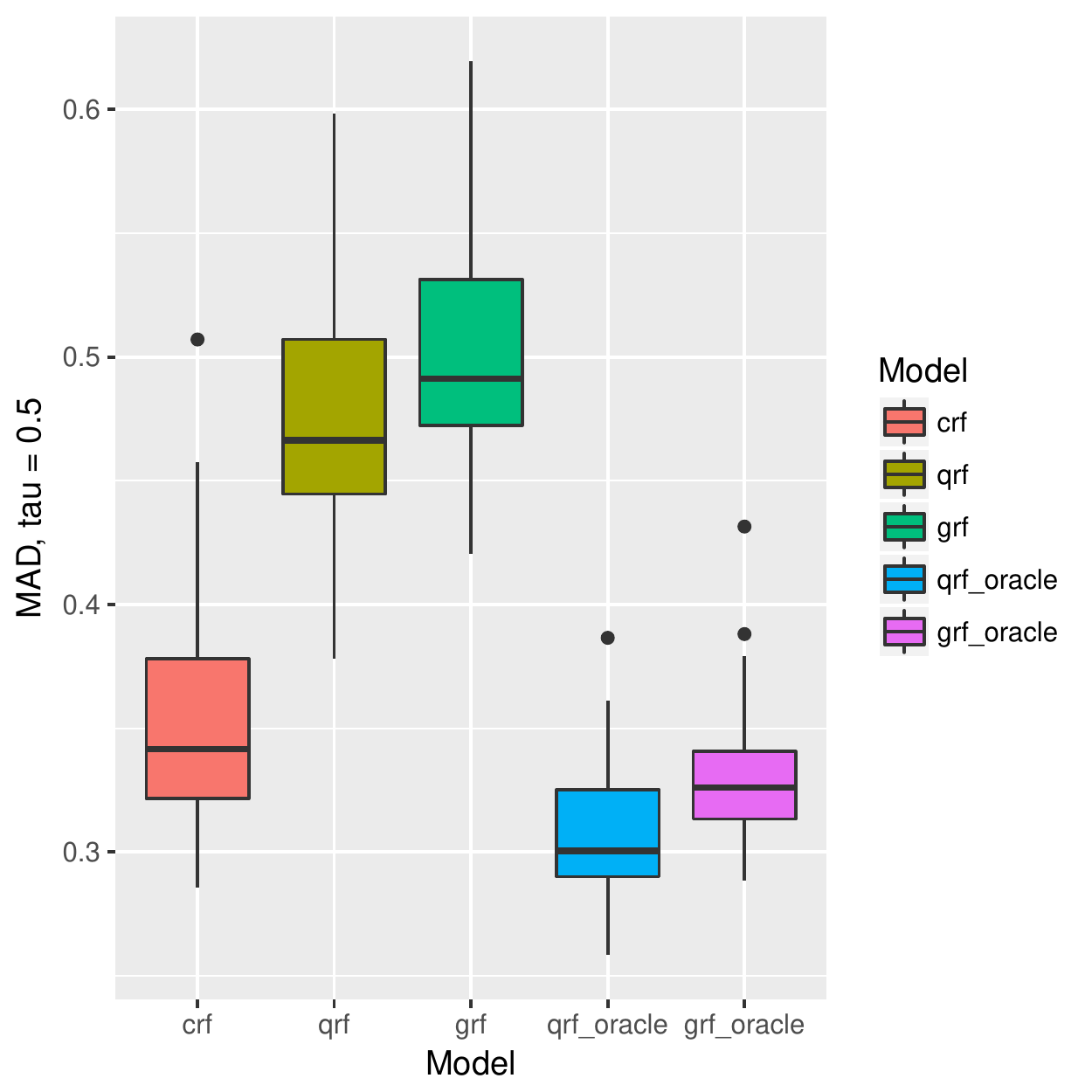}
        \caption{MAD: $\tau = 0.5$}
    \end{subfigure}
    ~
    \begin{subfigure}[b]{0.182\linewidth}
        \includegraphics[width=\textwidth]{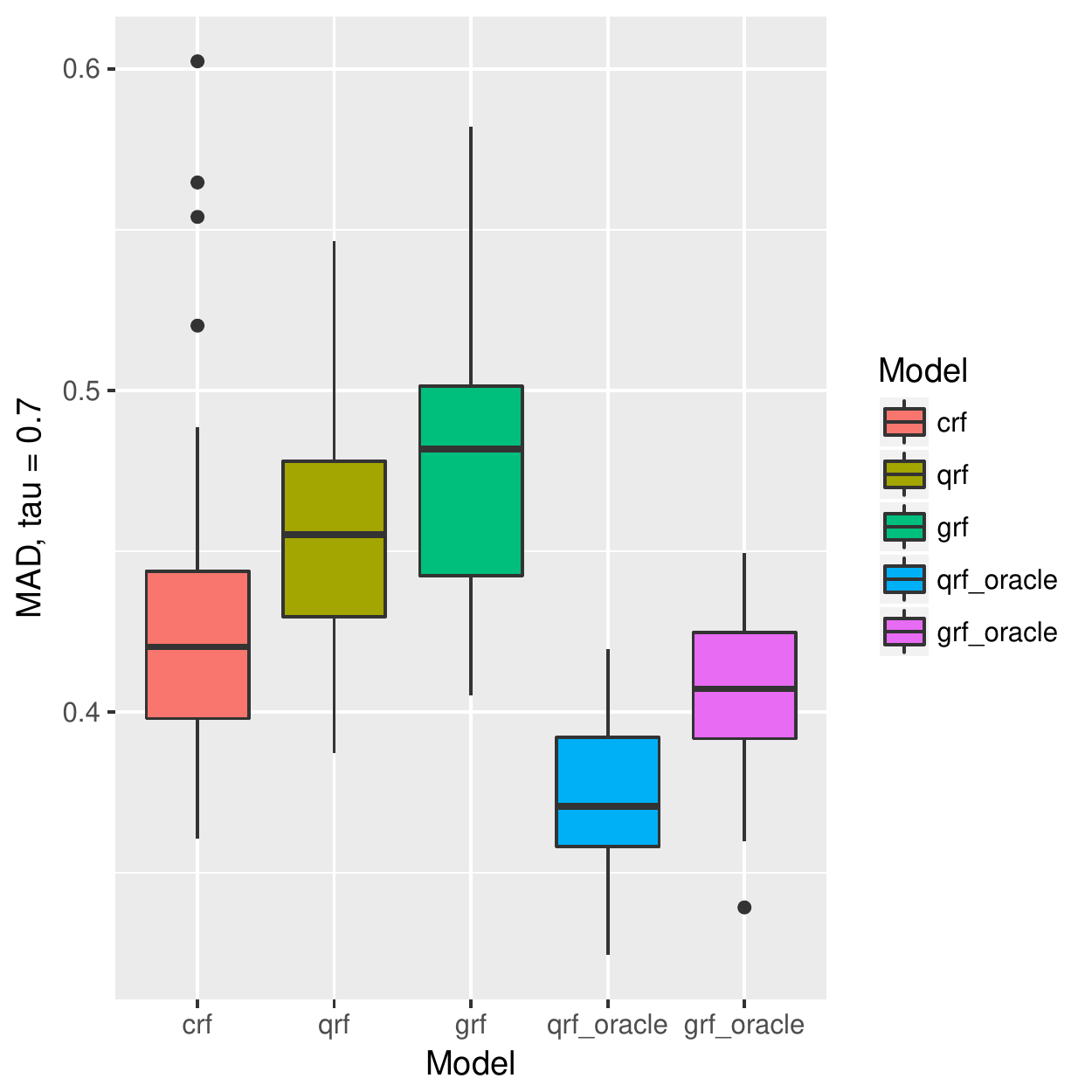}
        \caption{MAD: $\tau = 0.7$}
    \end{subfigure}
    ~
    \begin{subfigure}[b]{0.182\linewidth}
        \includegraphics[width=\textwidth]{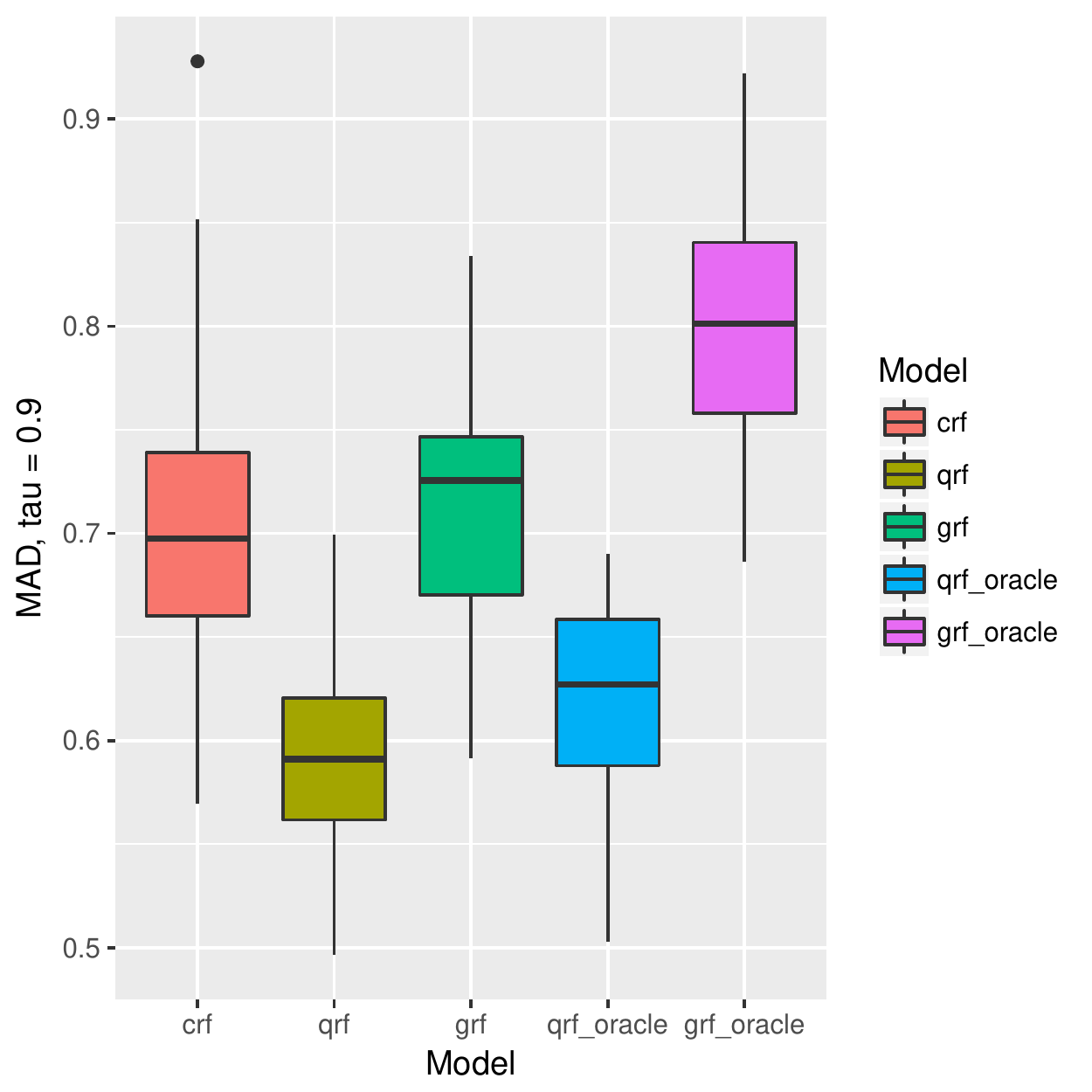}
        \caption{MAD: $\tau = 0.9$}
    \end{subfigure}
    
    \vspace{-0.05in}
    
    \begin{subfigure}[b]{0.18\linewidth}
        \includegraphics[width=\textwidth]{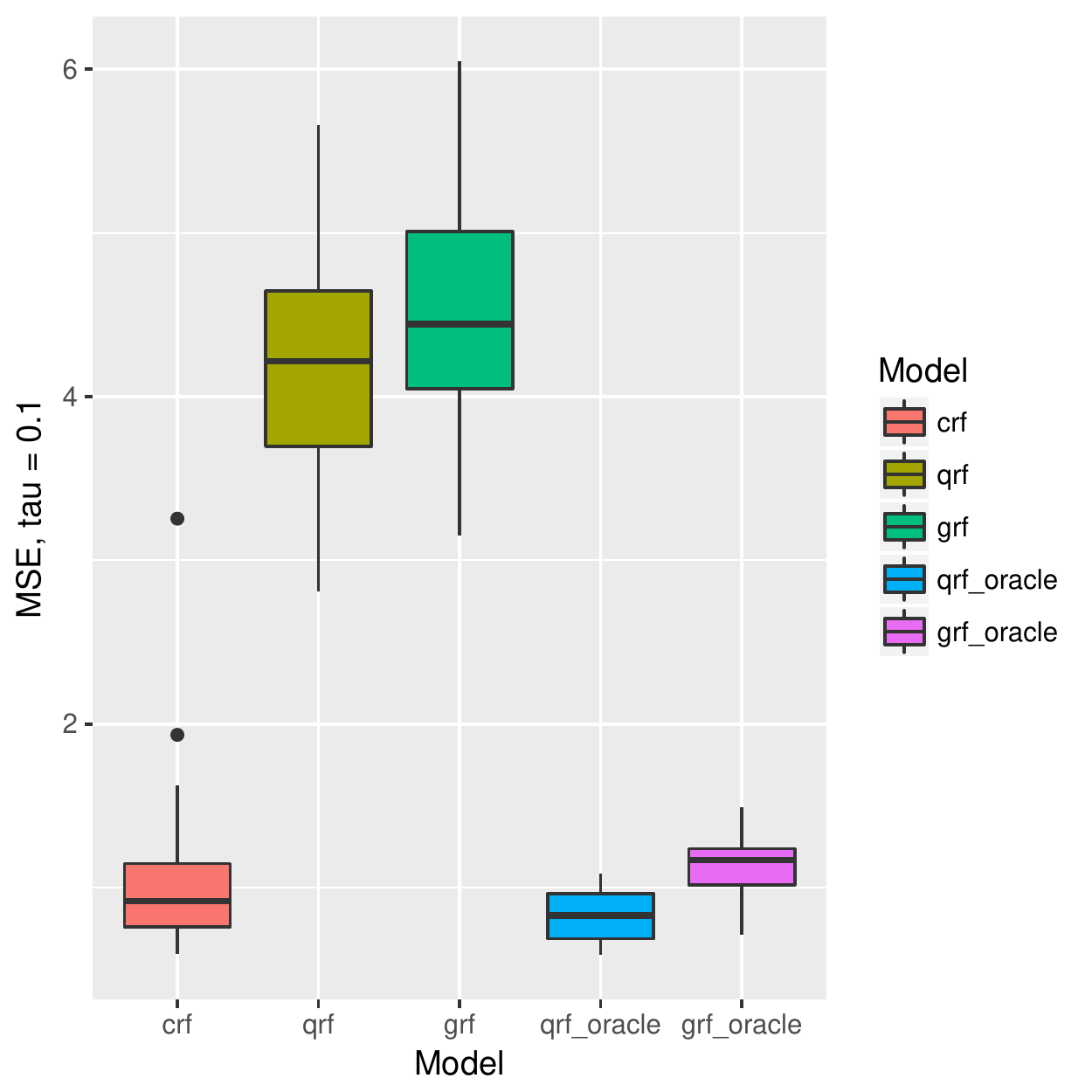}
        \caption{MSE: $\tau = 0.1$}
    \end{subfigure}
    ~
    \begin{subfigure}[b]{0.18\linewidth}
        \includegraphics[width=\textwidth]{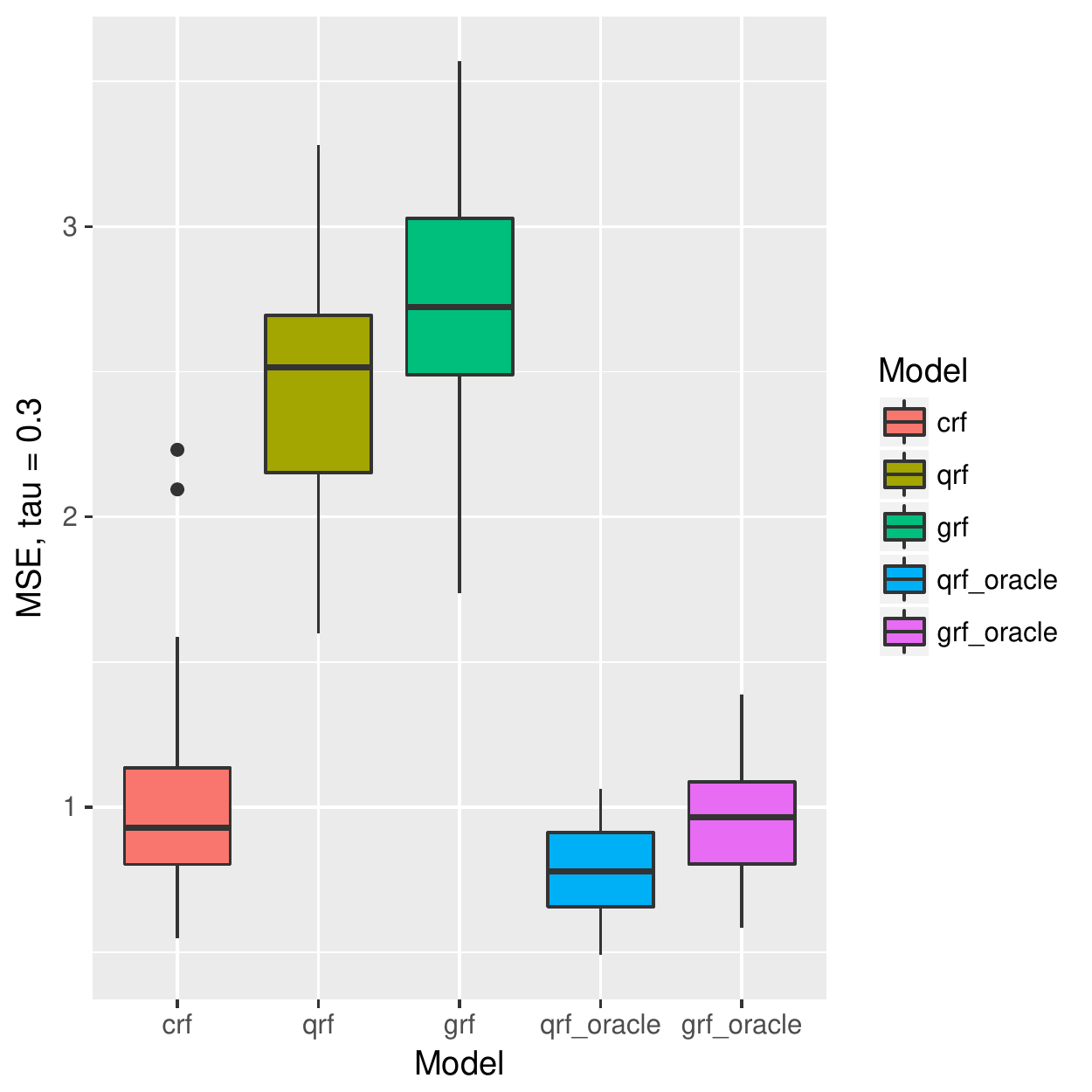}
        \caption{MSE: $\tau = 0.3$}
    \end{subfigure}
    ~ %add desired spacing between images, e. g. ~, \quad, \qquad, \hfill etc. 
    \begin{subfigure}[b]{0.18\linewidth}
        \includegraphics[width=\textwidth]{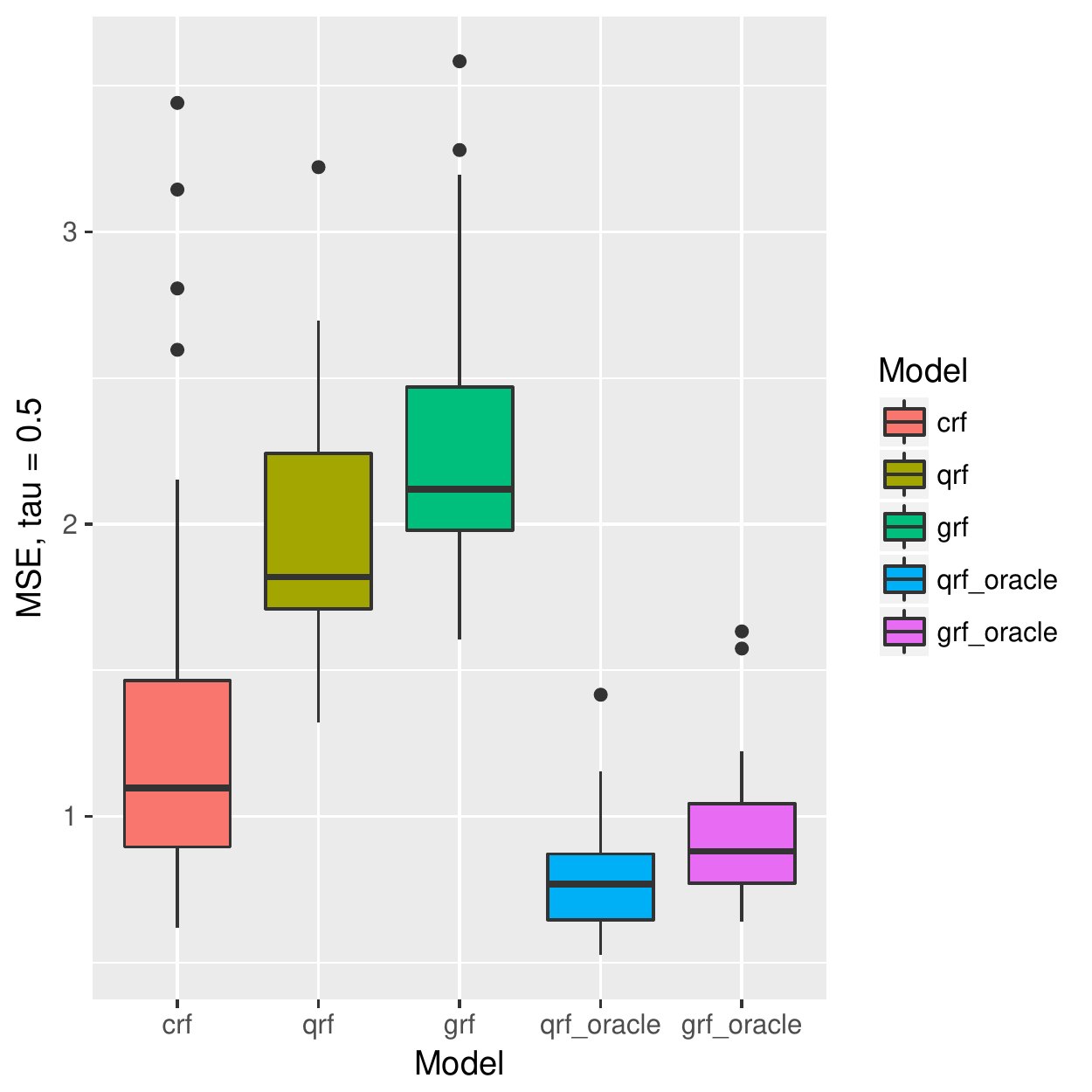}
        \caption{MSE: $\tau = 0.5$}
    \end{subfigure}
    ~
    \begin{subfigure}[b]{0.18\linewidth}
        \includegraphics[width=\textwidth]{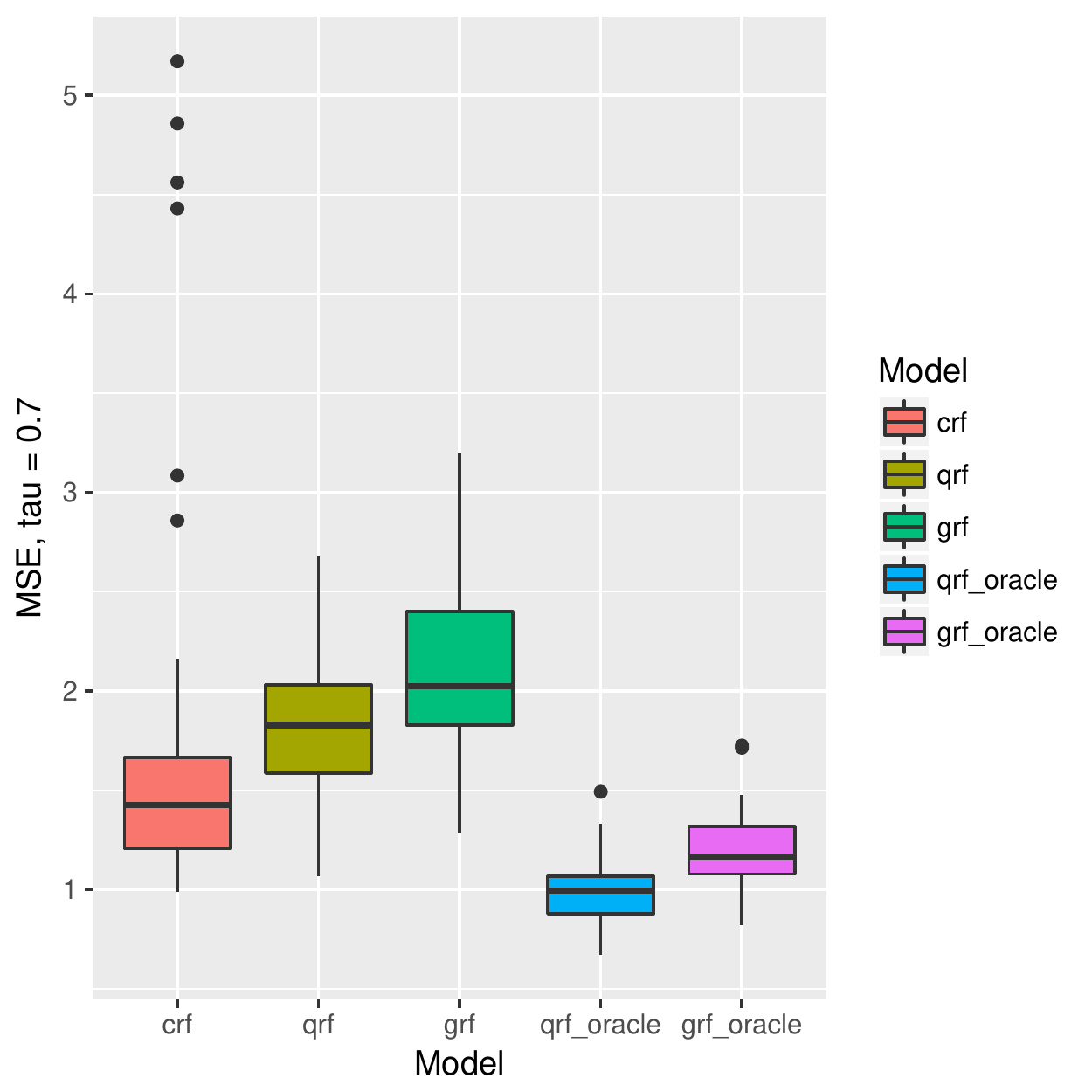}
        \caption{MSE: $\tau = 0.7$}
    \end{subfigure}
    ~
    \begin{subfigure}[b]{0.18\linewidth}
        \includegraphics[width=\textwidth]{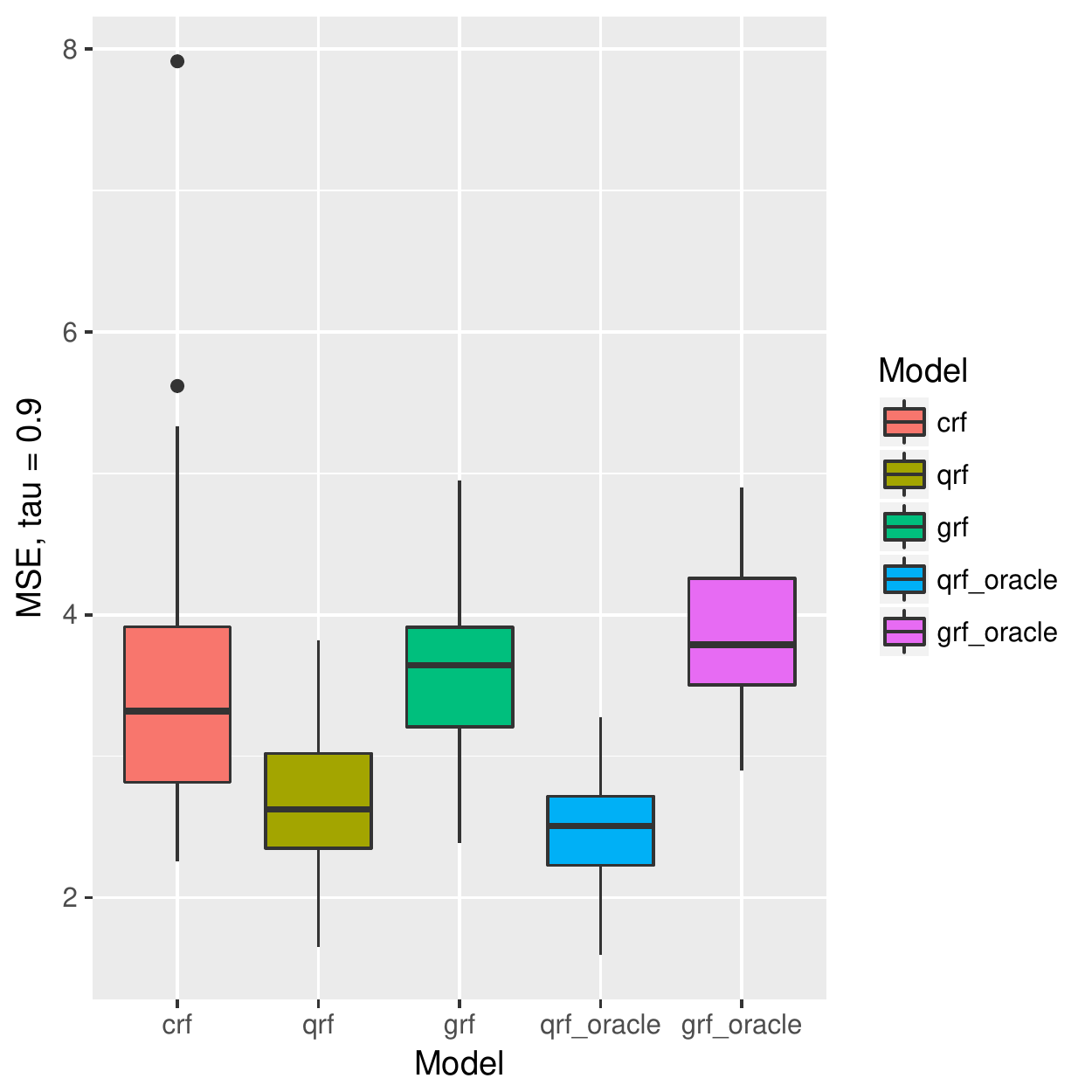}
        \caption{MSE: $\tau = 0.9$}
    \end{subfigure}
    
    \vspace{-0.05in}
    
    \begin{subfigure}[b]{0.18\linewidth}
        \includegraphics[width=\textwidth]{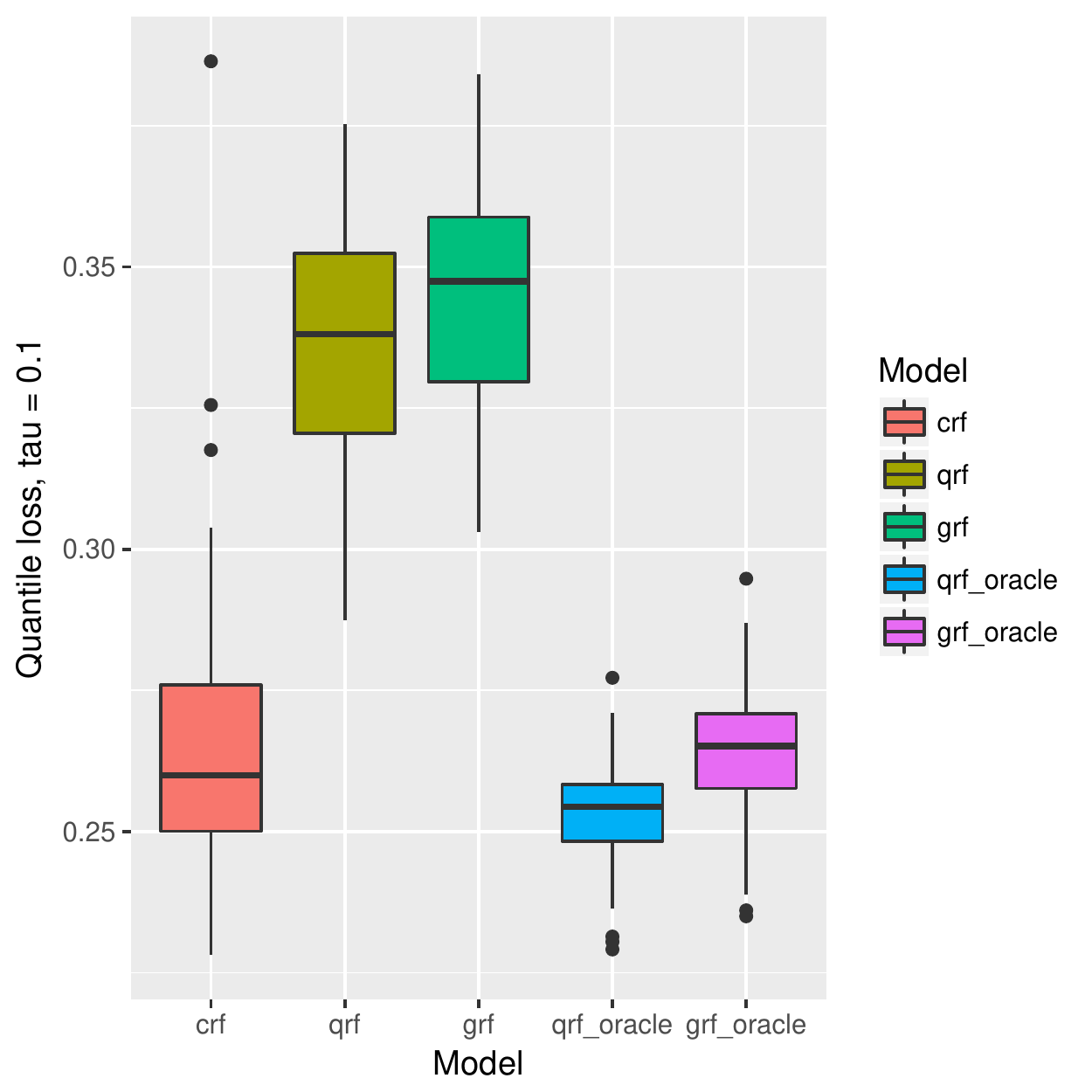}
        \caption{Quantile loss: $\tau = 0.3$}
    \end{subfigure}
    ~ %add desired spacing between images, e. g. ~, \quad, \qquad, \hfill etc. 
    \begin{subfigure}[b]{0.18\linewidth}
        \includegraphics[width=\textwidth]{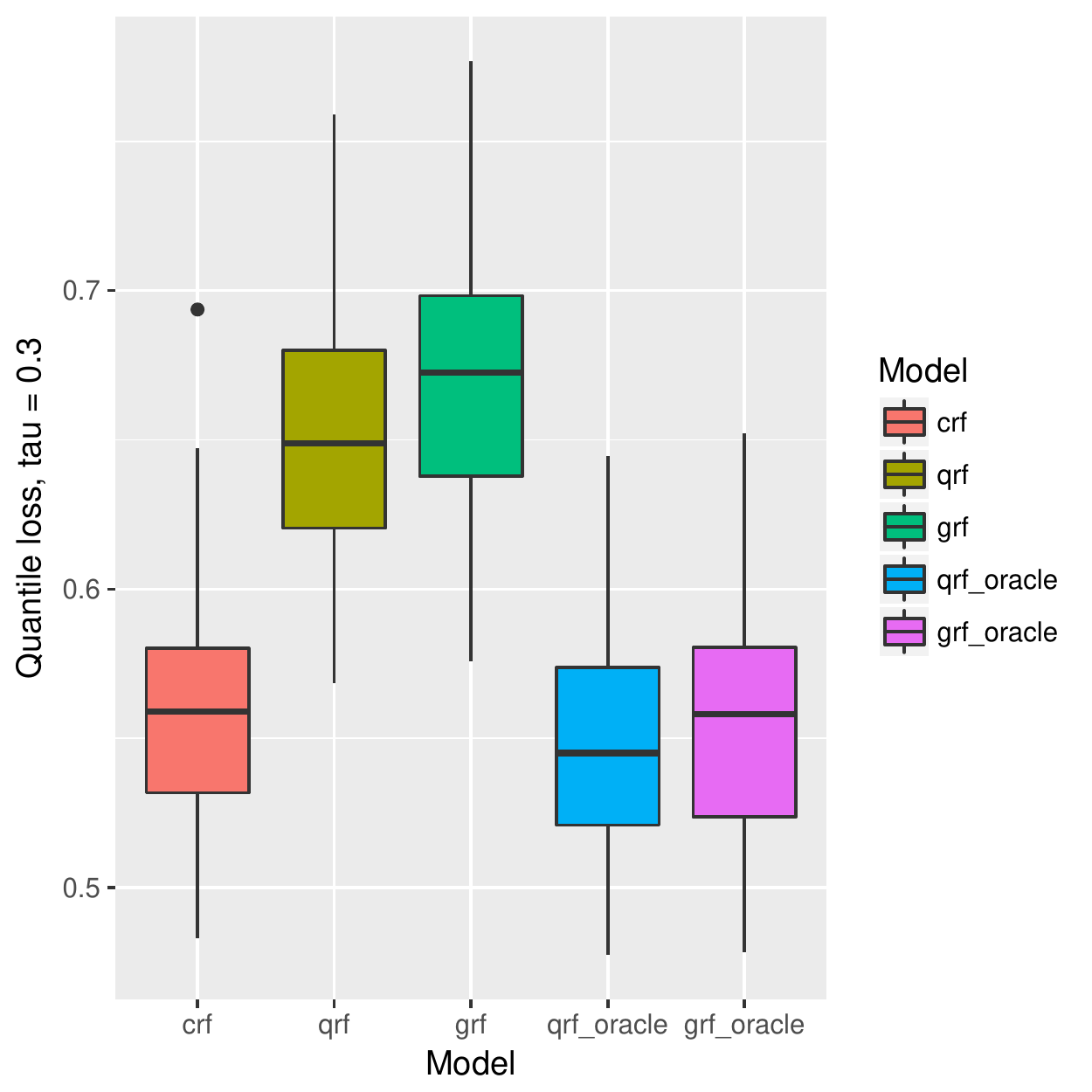}
        \caption{Quantile loss: $\tau = 0.3$}
    \end{subfigure}
    ~
    \begin{subfigure}[b]{0.18\linewidth}
        \includegraphics[width=\textwidth]{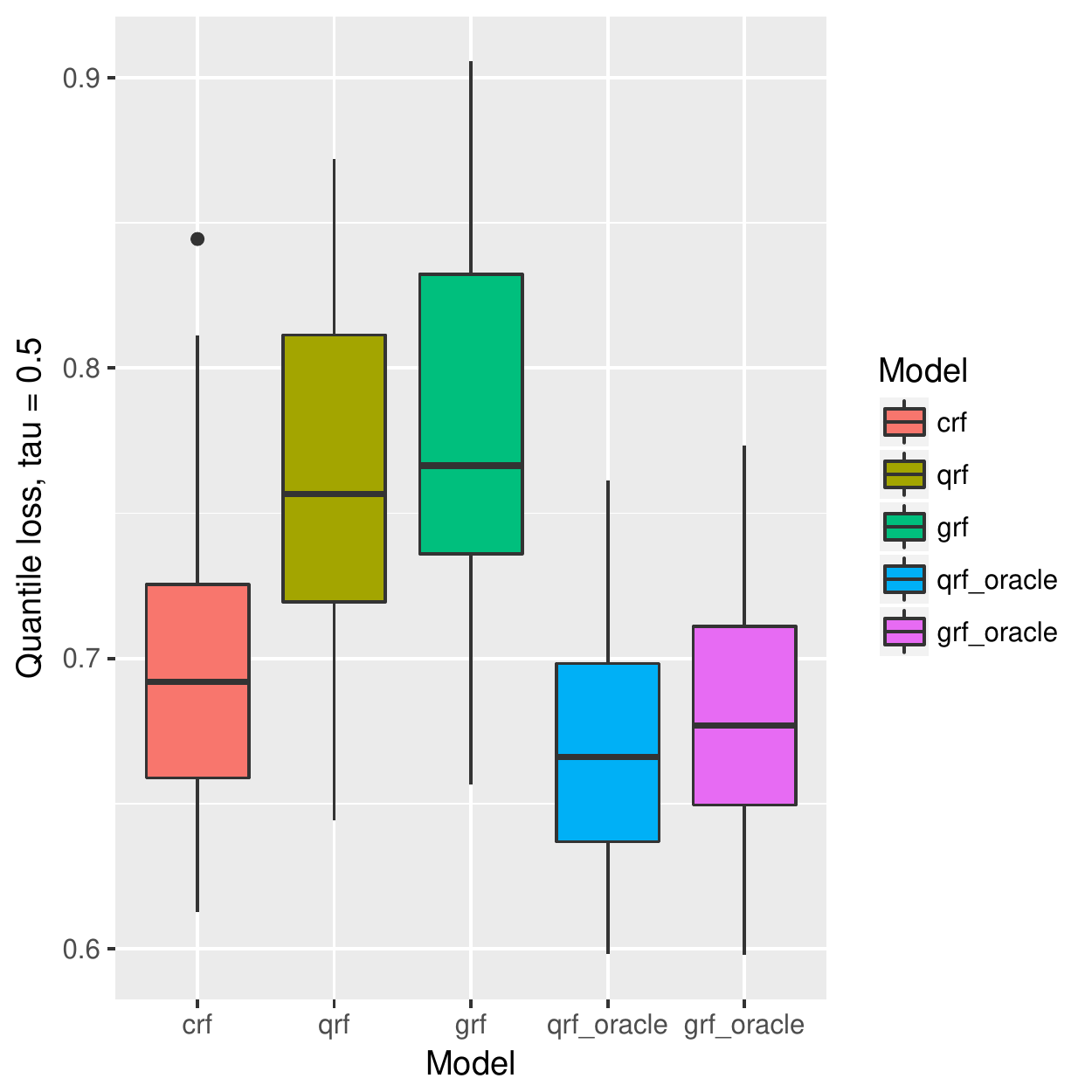}
        \caption{Quantile loss: $\tau = 0.5$}
    \end{subfigure}
    ~
    \begin{subfigure}[b]{0.18\linewidth}
        \includegraphics[width=\textwidth]{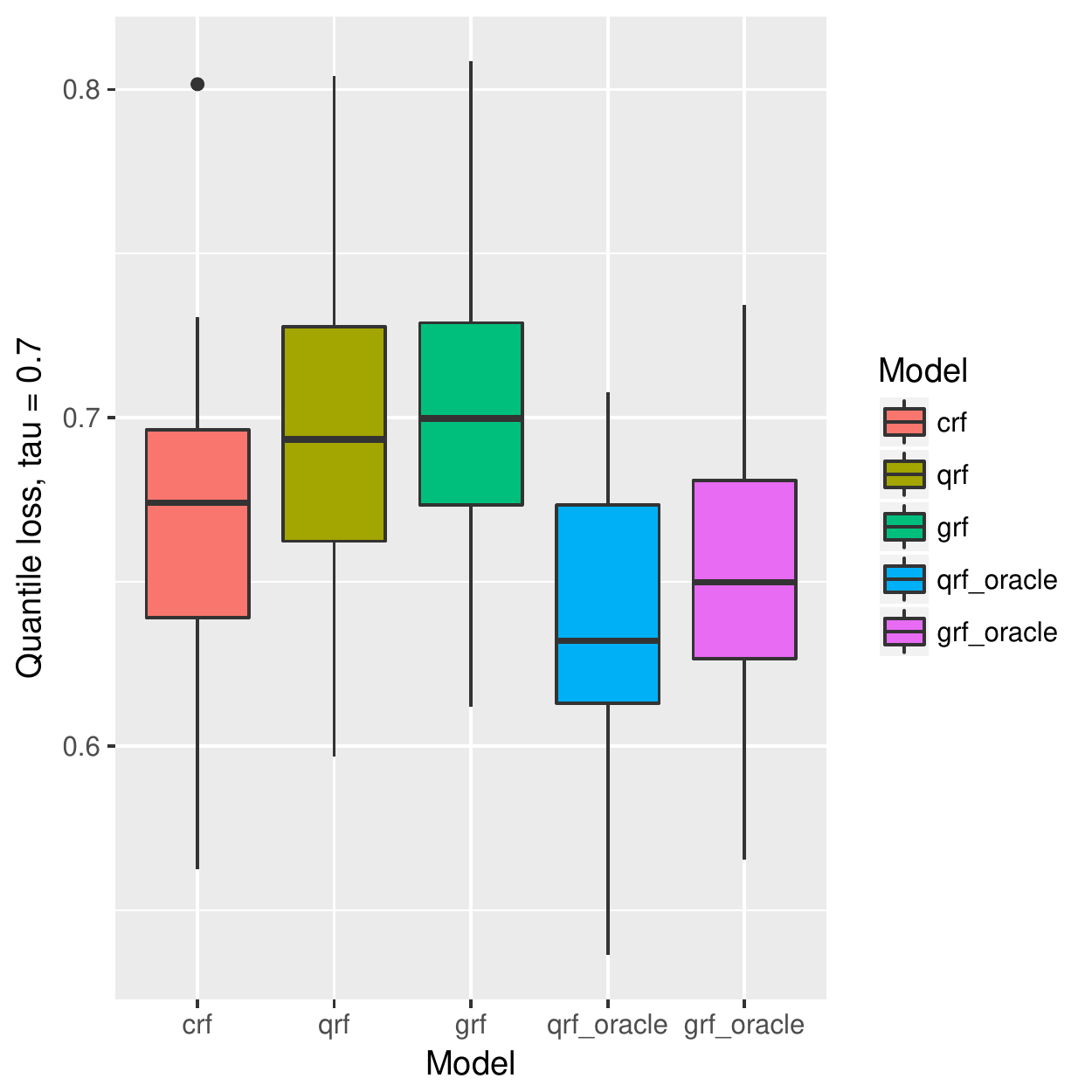}
        \caption{Quantile loss: $\tau = 0.7$}
    \end{subfigure}
    ~
    \begin{subfigure}[b]{0.18\linewidth}
        \includegraphics[width=\textwidth]{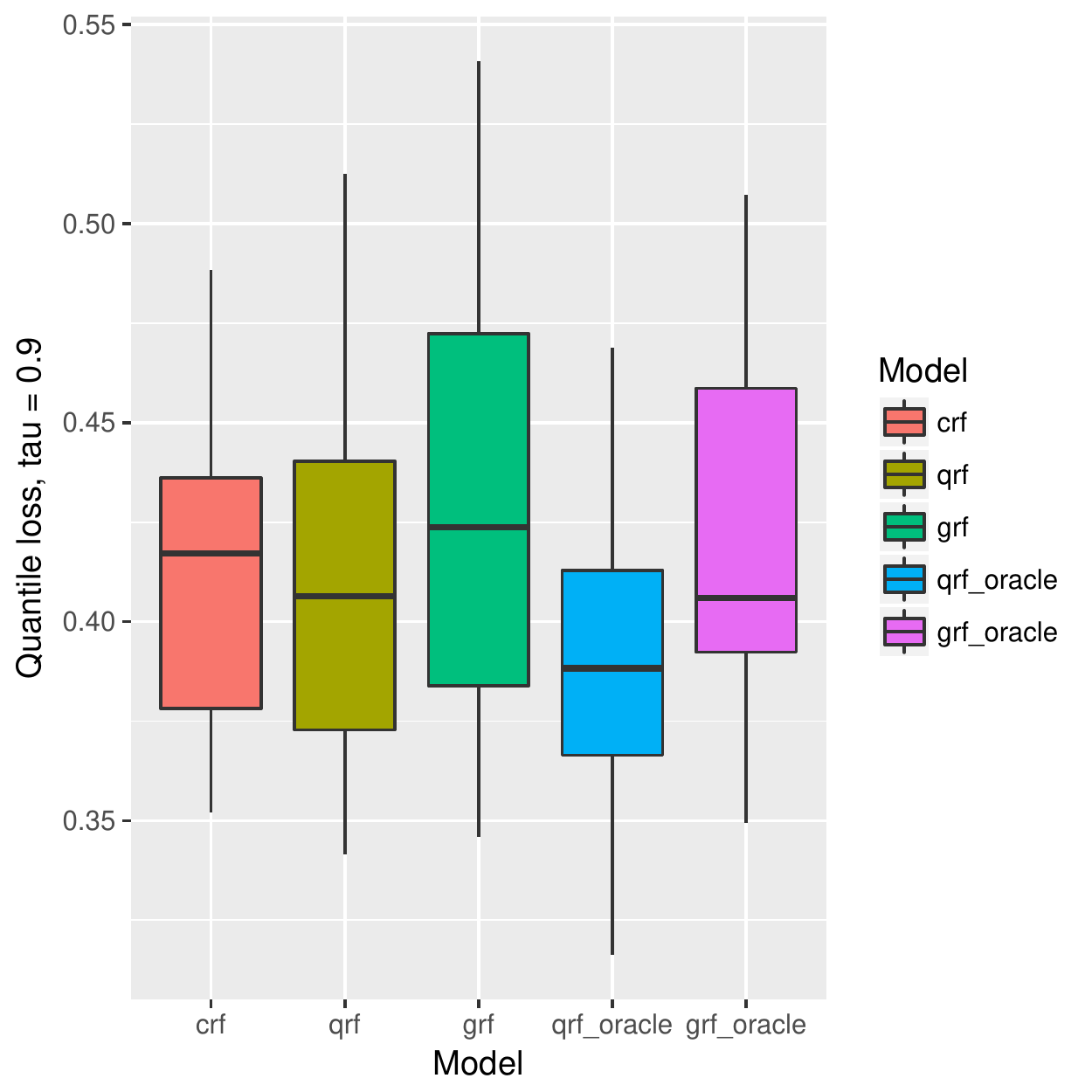}
        \caption{Qauntile loss: $\tau = 0.9$}
    \end{subfigure}
    
    \vspace{-0.05in}
    
    \begin{subfigure}[b]{0.18\linewidth}
        \includegraphics[width=\textwidth]{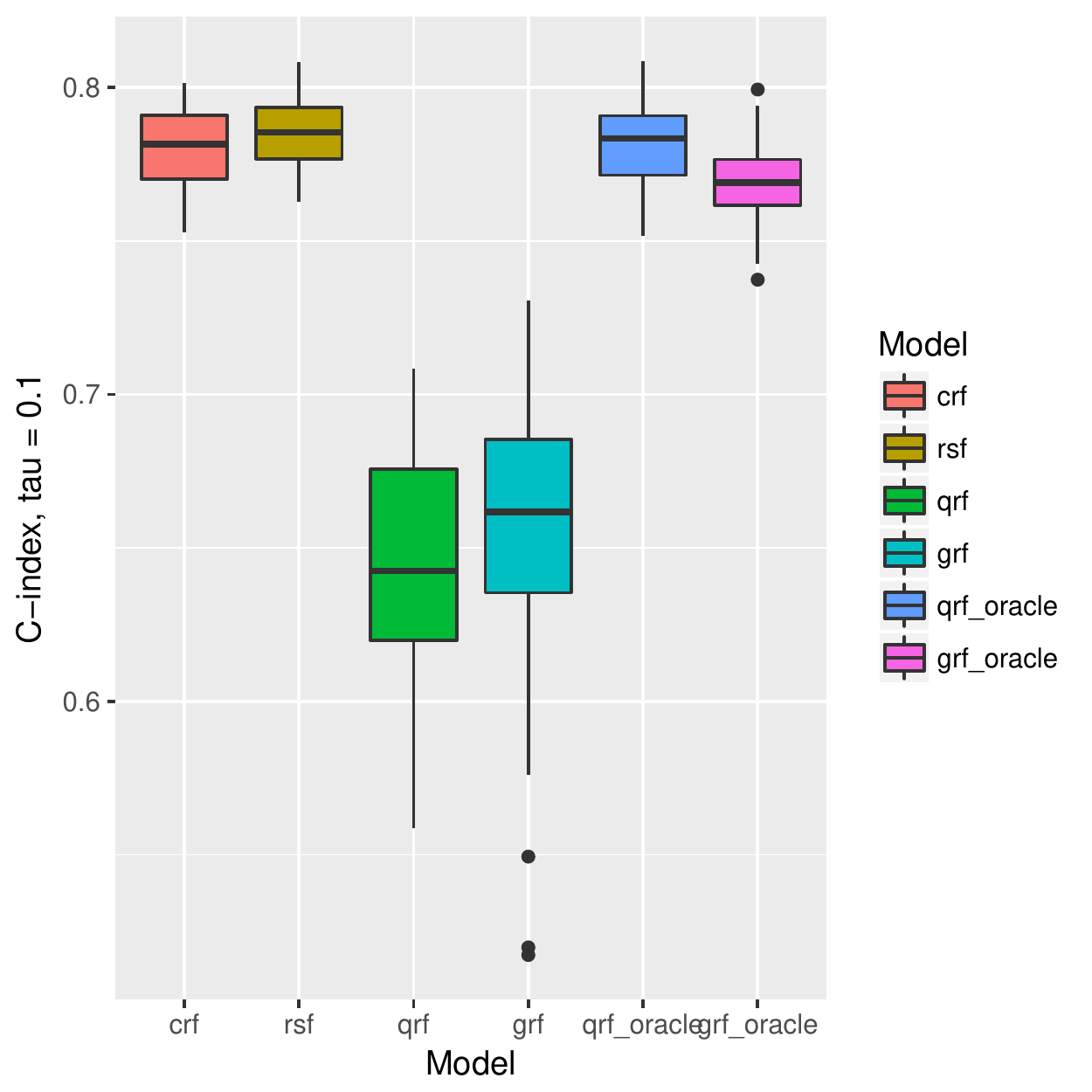}
        \caption{C-index: $\tau = 0.3$}
    \end{subfigure}
    ~ %add desired spacing between images, e. g. ~, \quad, \qquad, \hfill etc. 
    \begin{subfigure}[b]{0.18\linewidth}
        \includegraphics[width=\textwidth]{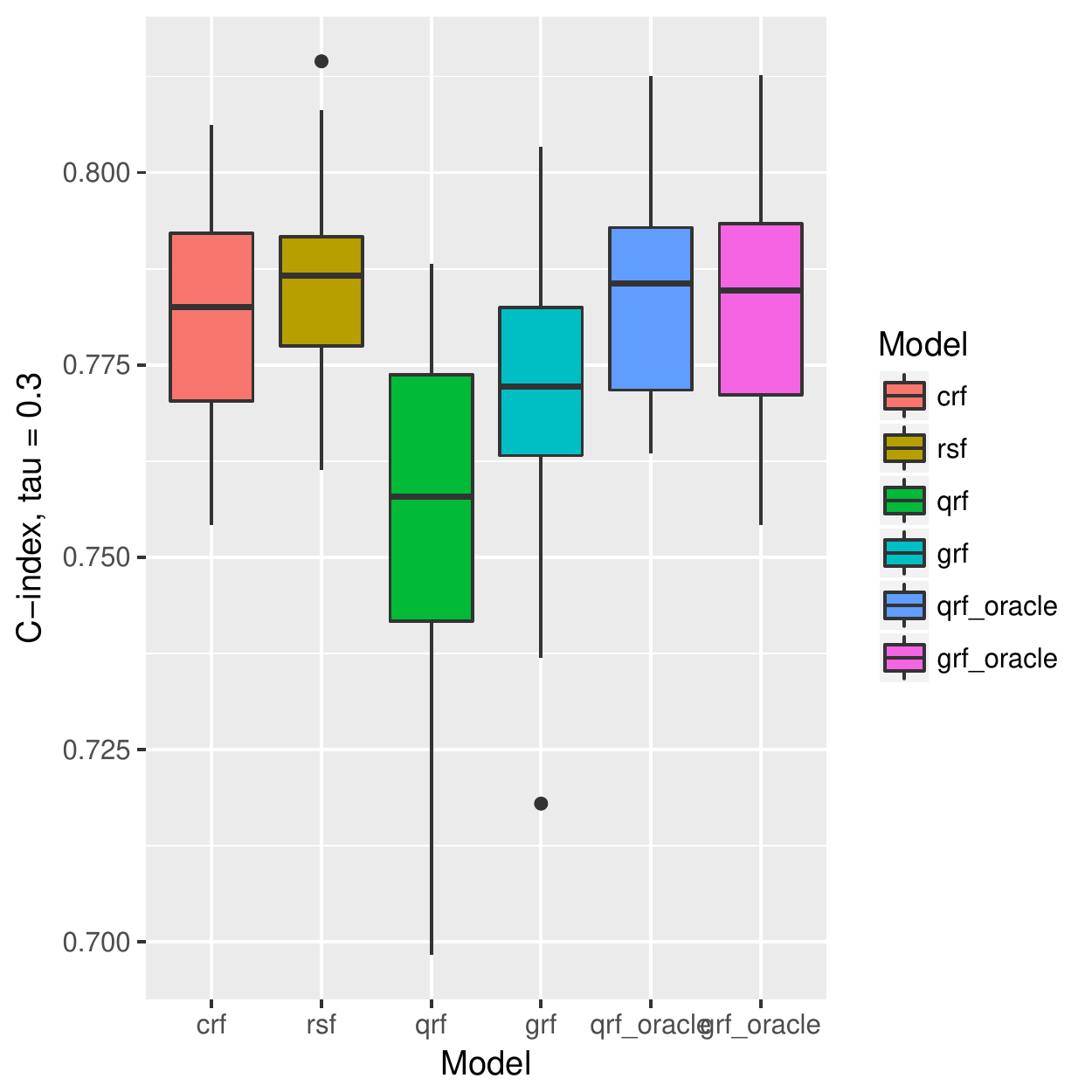}
        \caption{C-index: $\tau = 0.3$}
    \end{subfigure}
    ~
    \begin{subfigure}[b]{0.18\linewidth}
        \includegraphics[width=\textwidth]{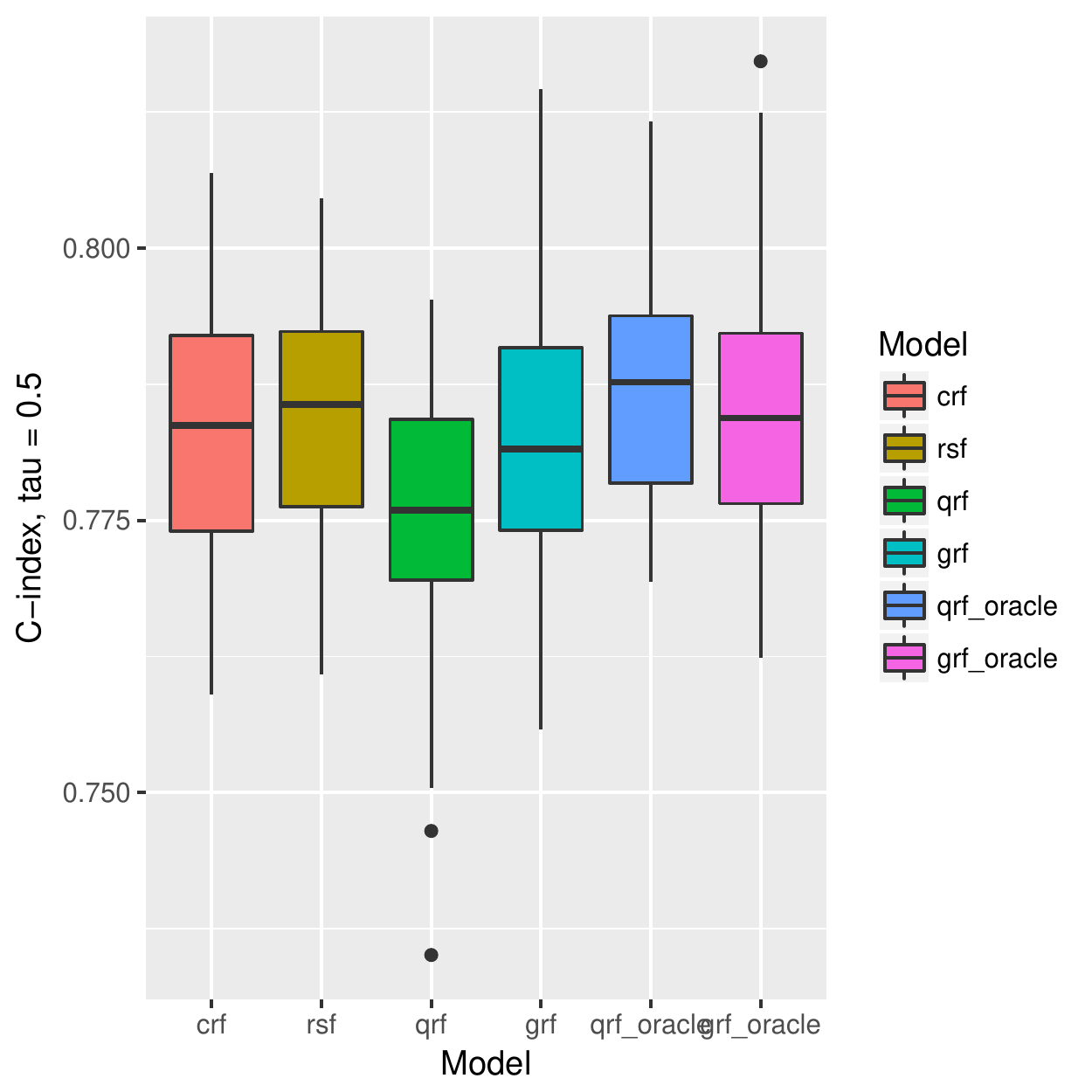}
        \caption{C-index: $\tau = 0.5$}
    \end{subfigure}
    ~
    \begin{subfigure}[b]{0.18\linewidth}
        \includegraphics[width=\textwidth]{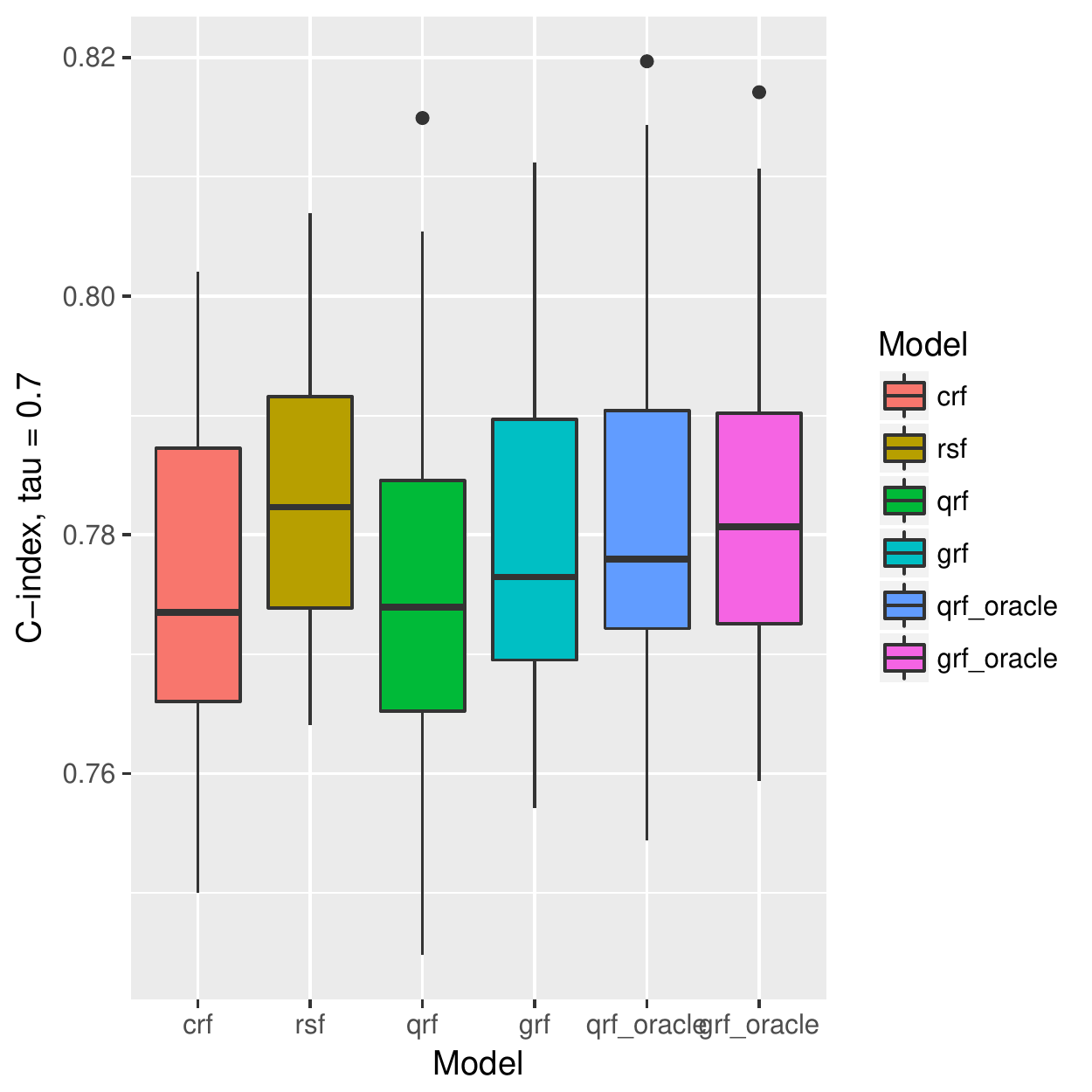}
        \caption{C-index: $\tau = 0.7$}
    \end{subfigure}
    ~
    \begin{subfigure}[b]{0.18\linewidth}
        \includegraphics[width=\textwidth]{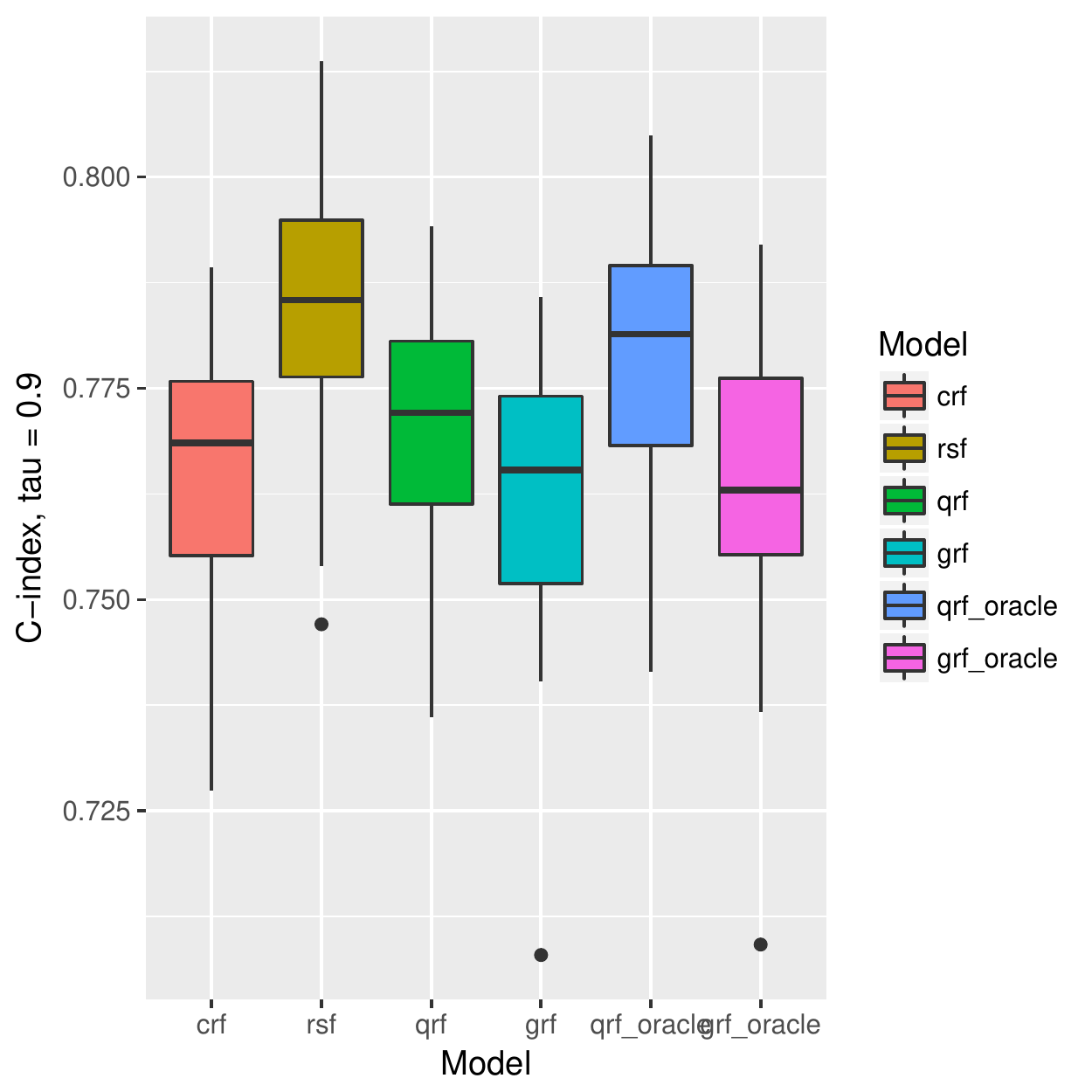}
        \caption{C-index: $\tau = 0.9$}
    \end{subfigure}
    
    \caption{Multi-dimensional AFT model box plots with dimension five and $n=800$ and $B=1000$: $500$ training and $300$ testing size. For the metrics MSE \eqref{eq:L_MSE}, MAD \eqref{eq:L_MAD} and quantile loss \eqref{eq:L_quantile}, the smaller the value is the better. For C-index, the larger it is the better.}
    \label{fig:aft_multi_box}
\end{figure}

\subsubsection{Multi-dimensional censored complex manifold}
In this section, we construct a complex model
\begin{equation*}
    T = 5 + \frac{1}{5} \left( \sin(X_{\cdot,1}) + \cos(X_{\cdot,2}) + X_{\cdot,3}^2 + \exp(X_{\cdot,4}) + X_{\cdot,5} \right) + \epsilon,
\end{equation*}
where $X_{\cdot,j}$ stands for $j$-th dimension of $\bm{X} \in \mathbb{R}^5$, and $\epsilon \sim \mathcal{N}(0,0.3^2)$. Then we consider a censoring variable independent of $\bm{X}$ and $T$: $C \sim \textrm{Exp}(\lambda = 0.015)$. The result is summarized in Figure \ref{fig:complex_multi_box}. Although the model is highly non-linear our method is able to capture that and estimate the quantile of interest consistently. However, other methods are unable to be consistent. Behavior matches that of linear models closely. In this case we notice that the oracular generalized random forests have much better performance; note that since we are estimating equations, our algorithm can be easily tweaked to match {\it grf-oracle} -- the only change would be to design the splits in the initial tree construction to match our estimating equations, use subsampling and save a separate sample for minimizing the equations themselves.

\begin{figure}[!htb]
    \small
    \centering
    \begin{subfigure}[b]{0.182\linewidth}
        \includegraphics[width=\textwidth]{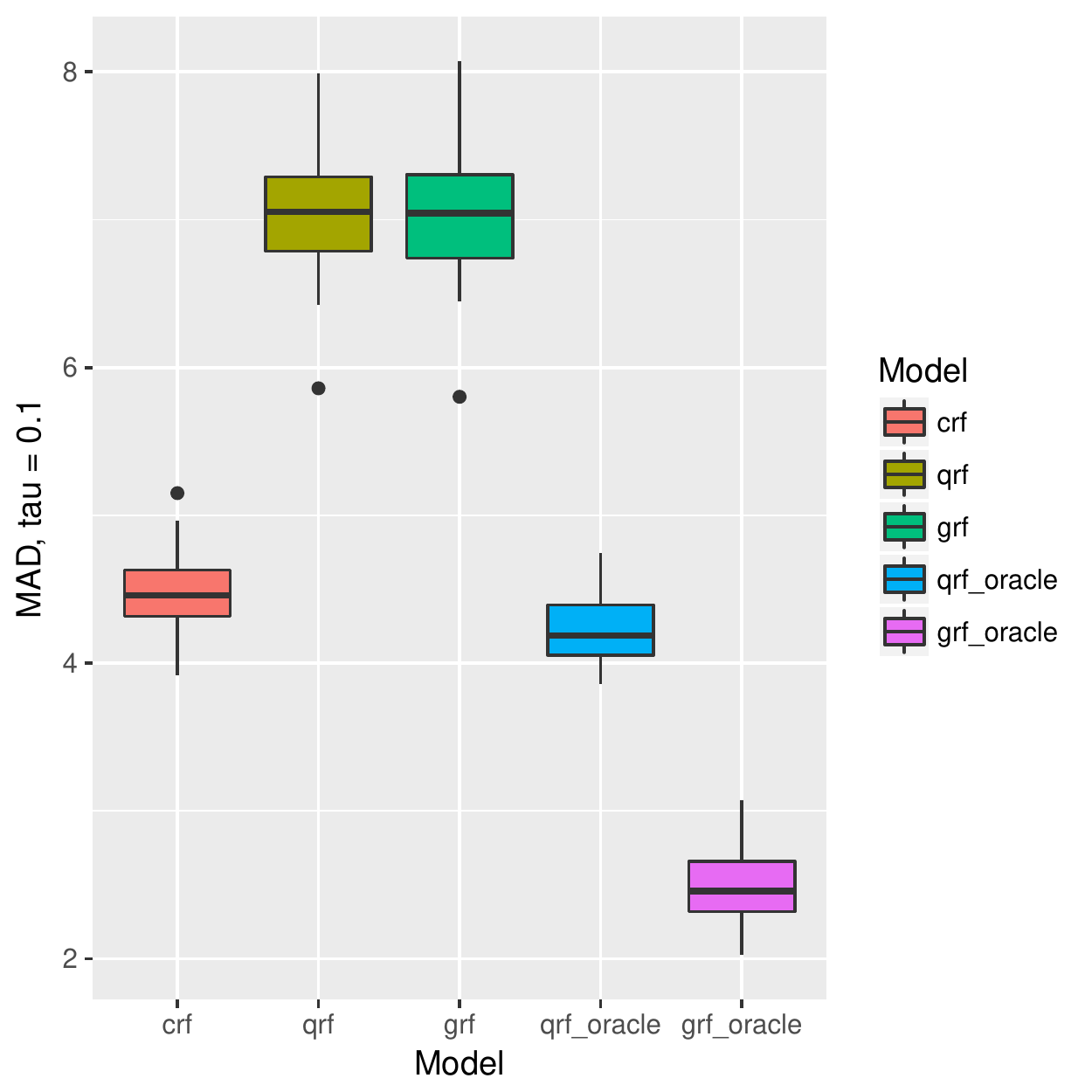}
        \caption{MAD: $\tau = 0.1$}
    \end{subfigure}
    ~
    \begin{subfigure}[b]{0.182\linewidth}
        \includegraphics[width=\textwidth]{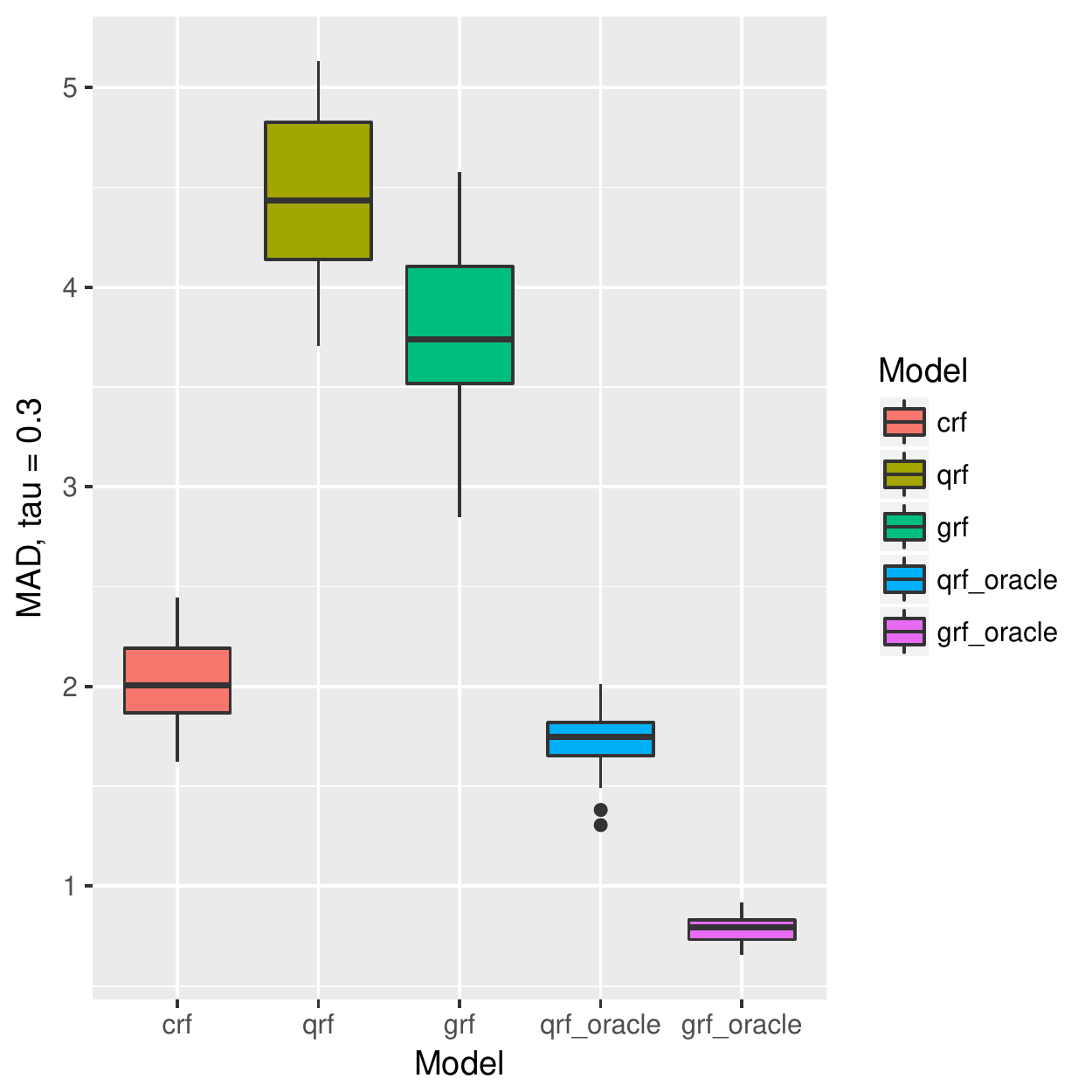}
        \caption{MAD: $\tau = 0.3$}
    \end{subfigure}
    ~ %add desired spacing between images, e. g. ~, \quad, \qquad, \hfill etc. 
    \begin{subfigure}[b]{0.182\linewidth}
        \includegraphics[width=\textwidth]{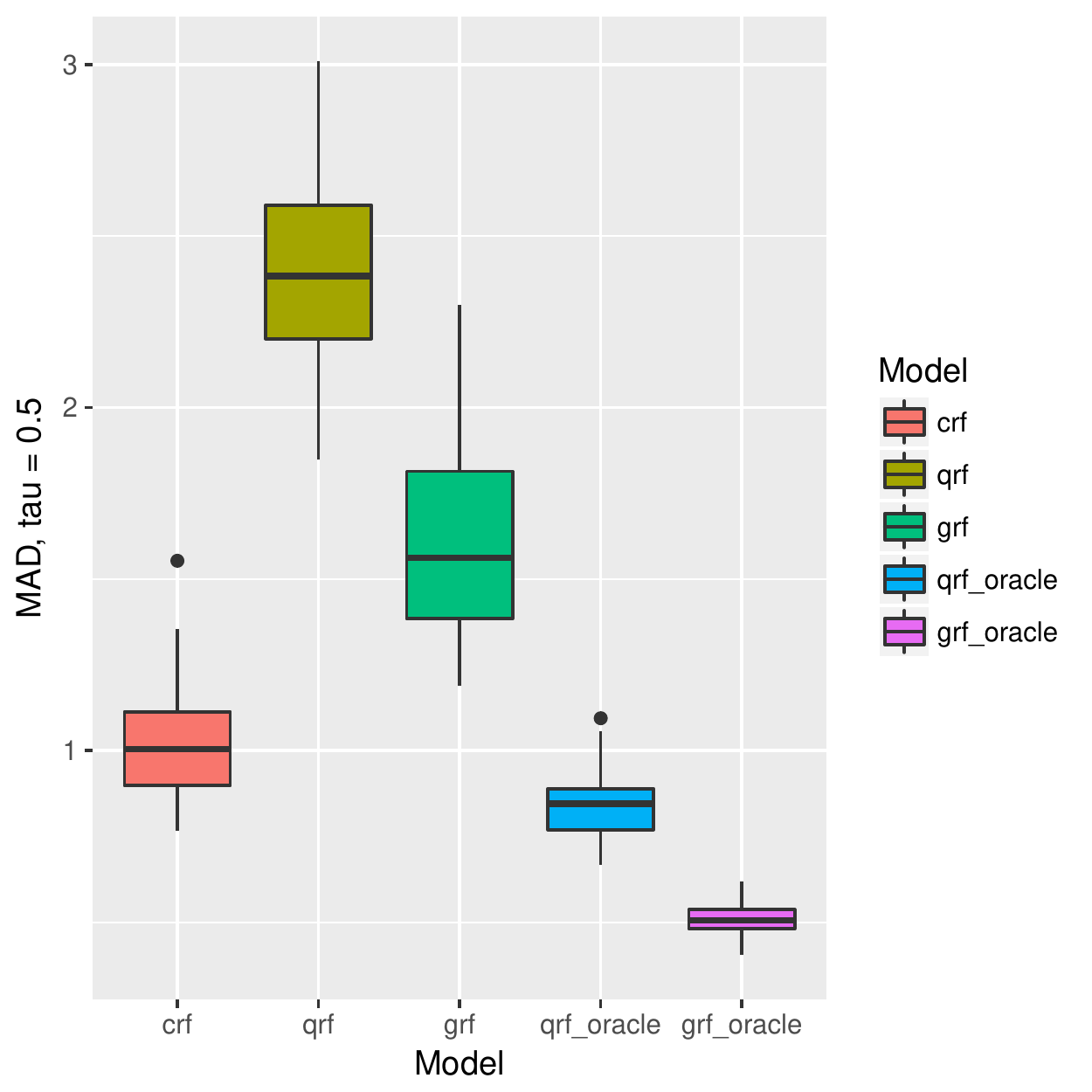}
        \caption{MAD: $\tau = 0.5$}
    \end{subfigure}
    ~
    \begin{subfigure}[b]{0.182\linewidth}
        \includegraphics[width=\textwidth]{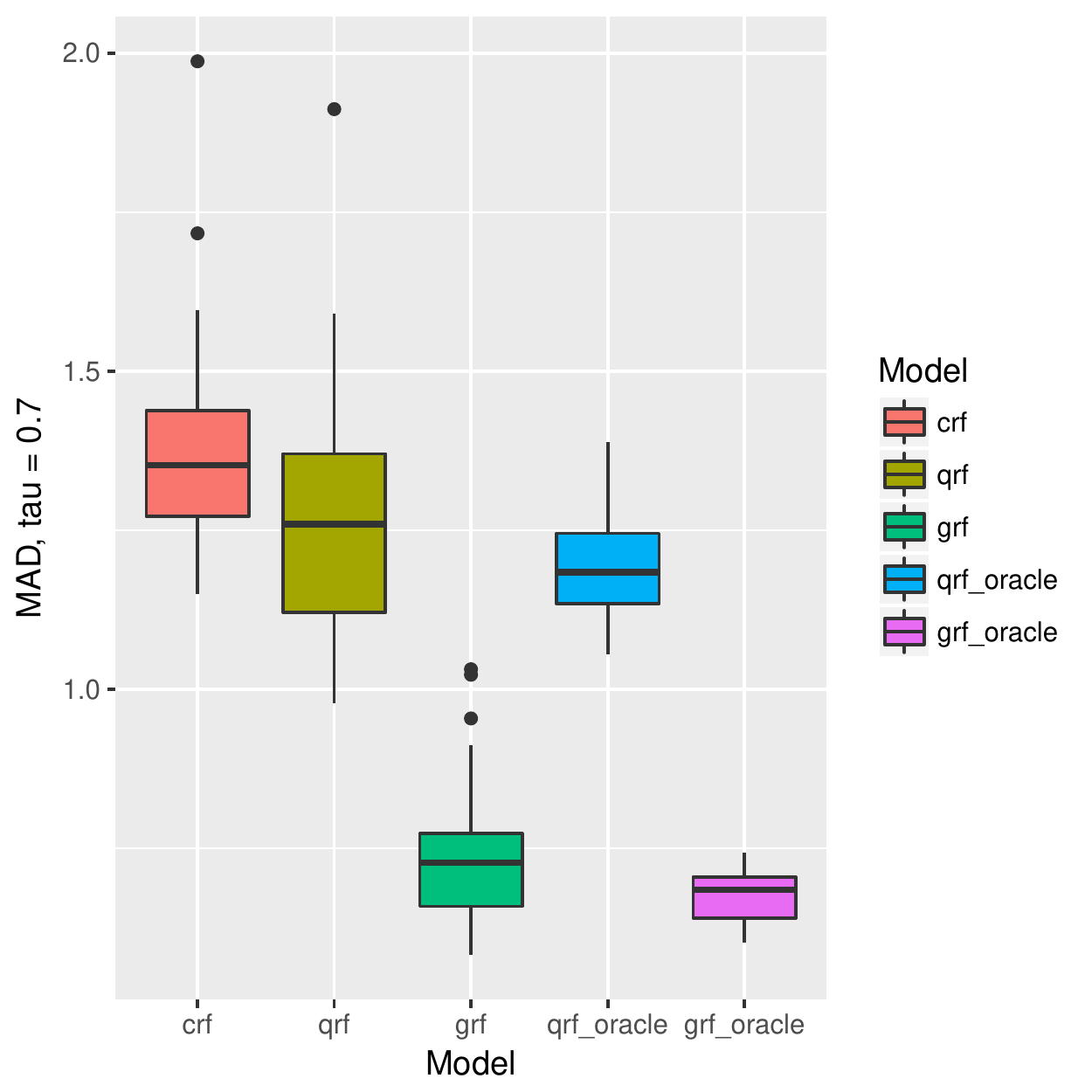}
        \caption{MAD: $\tau = 0.7$}
    \end{subfigure}
    ~
    \begin{subfigure}[b]{0.182\linewidth}
        \includegraphics[width=\textwidth]{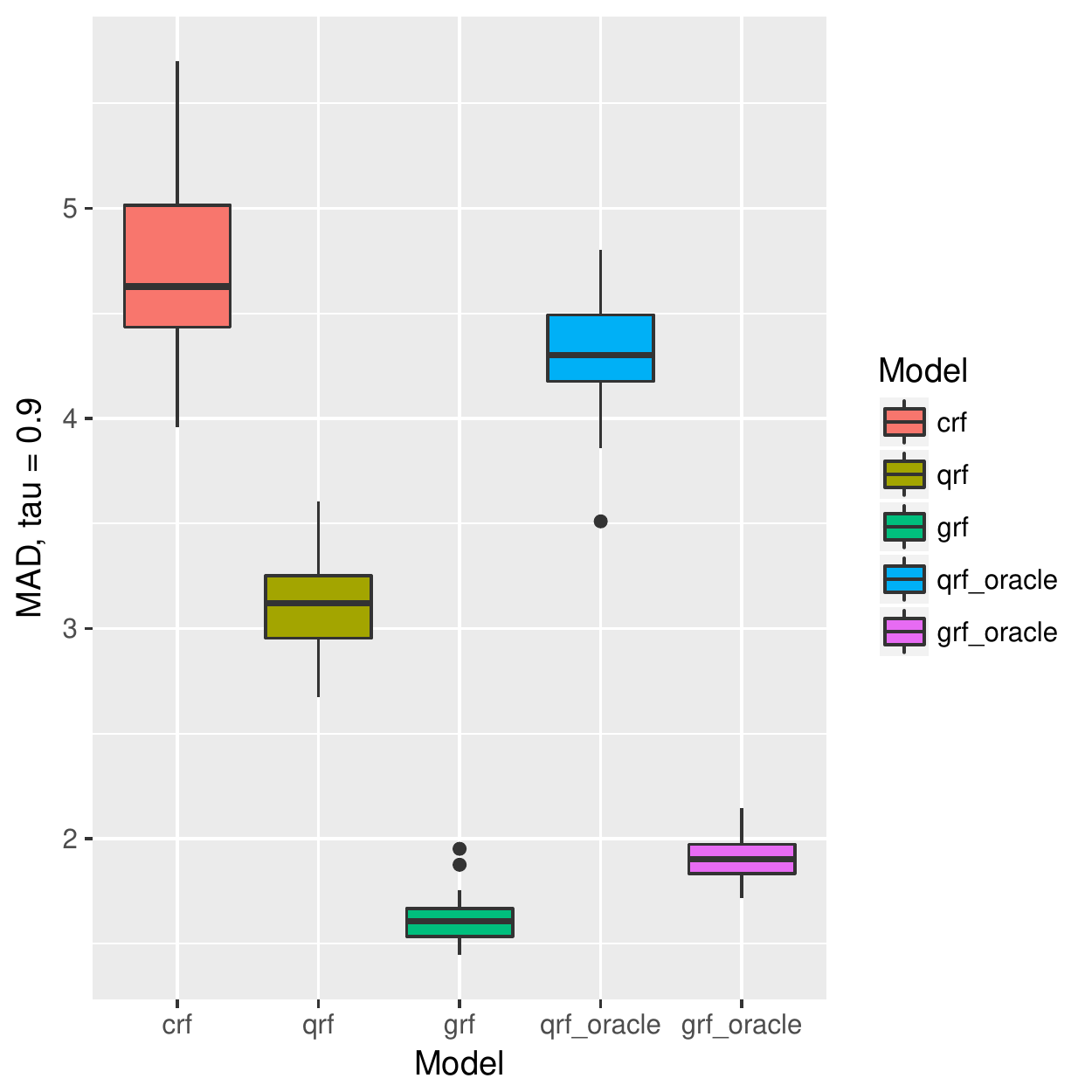}
        \caption{MAD: $\tau = 0.9$}
    \end{subfigure}
    
    \vspace{-0.05in}
    
    \begin{subfigure}[b]{0.18\linewidth}
        \includegraphics[width=\textwidth]{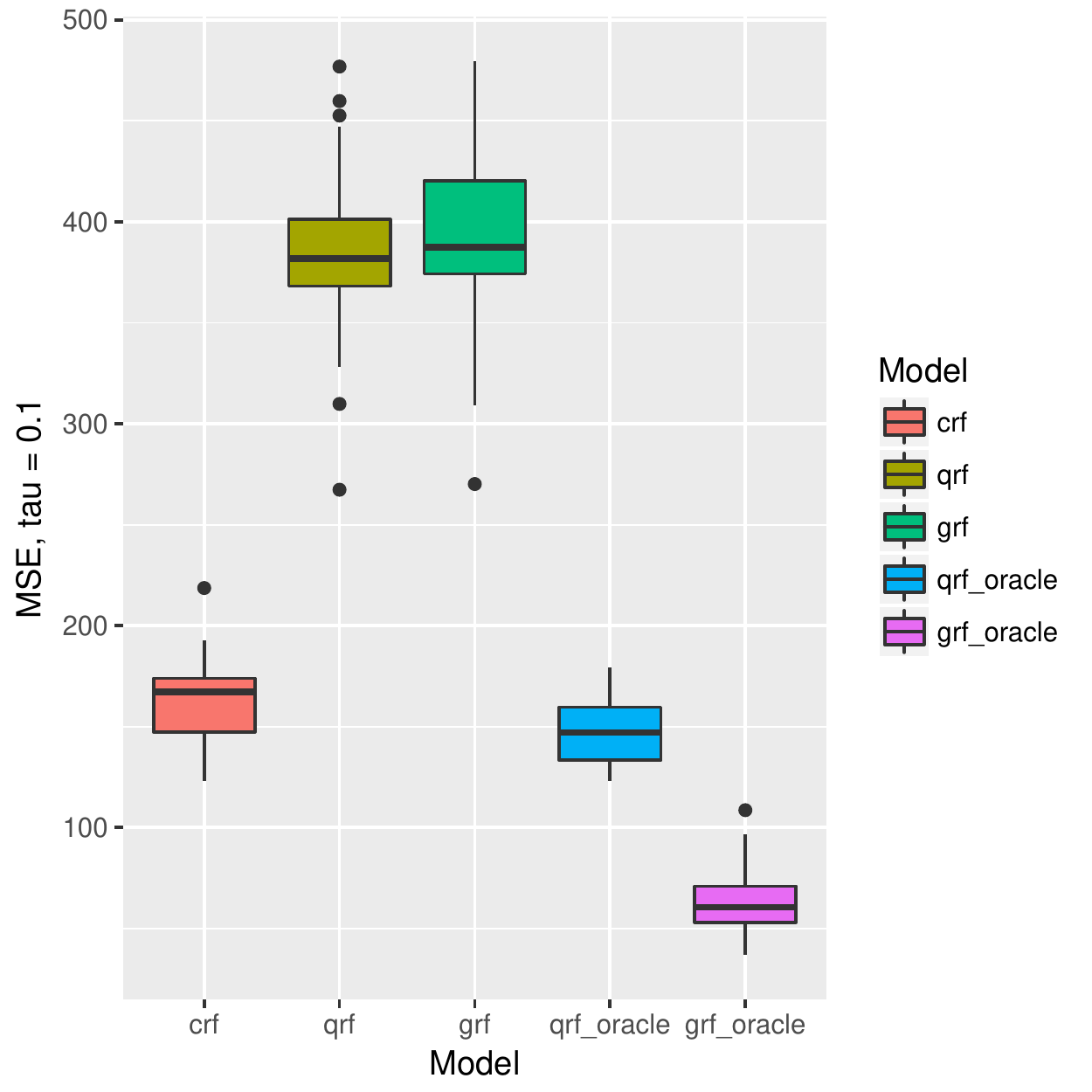}
        \caption{MSE: $\tau = 0.1$}
    \end{subfigure}
    ~
    \begin{subfigure}[b]{0.18\linewidth}
        \includegraphics[width=\textwidth]{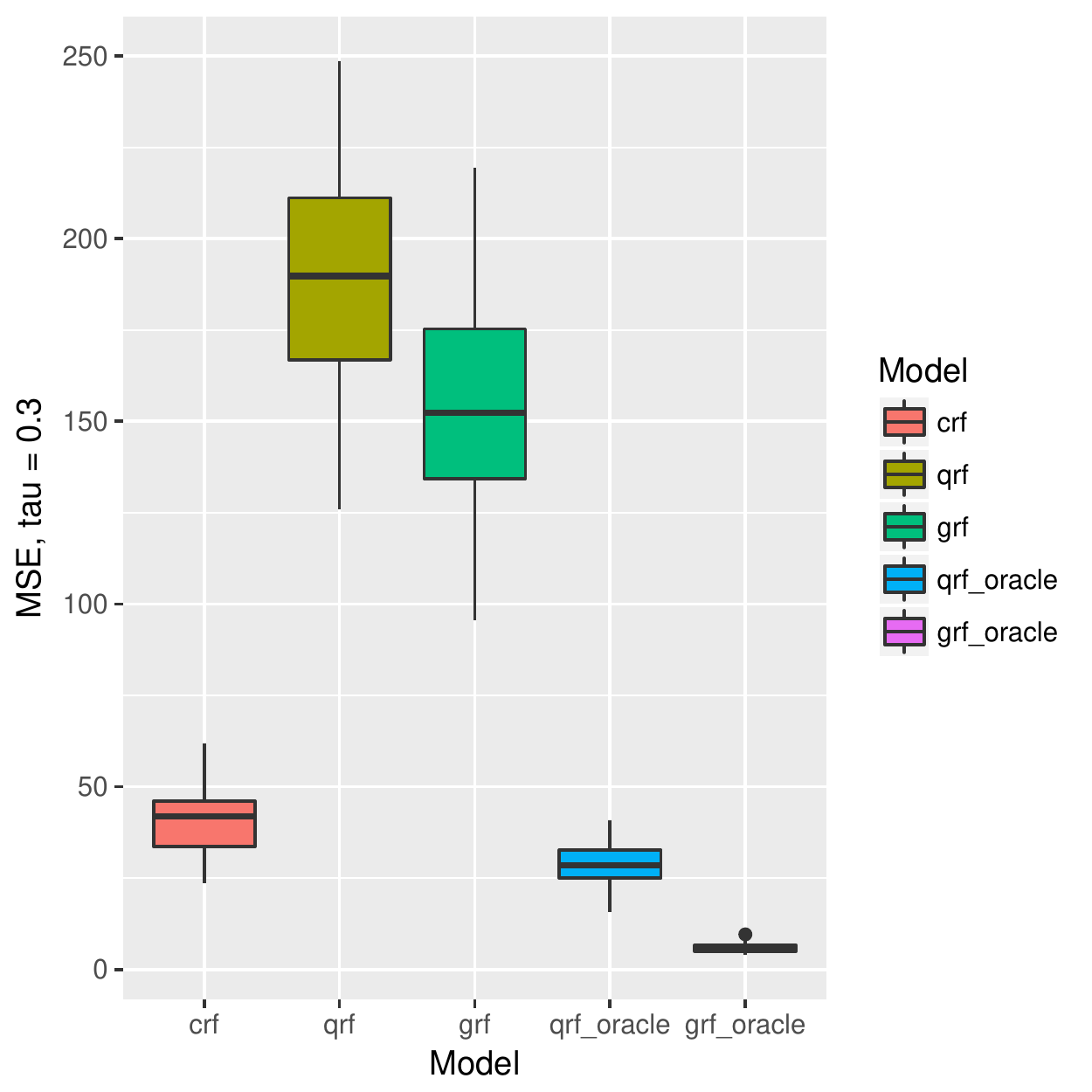}
        \caption{MSE: $\tau = 0.3$}
    \end{subfigure}
    ~ %add desired spacing between images, e. g. ~, \quad, \qquad, \hfill etc. 
    \begin{subfigure}[b]{0.18\linewidth}
        \includegraphics[width=\textwidth]{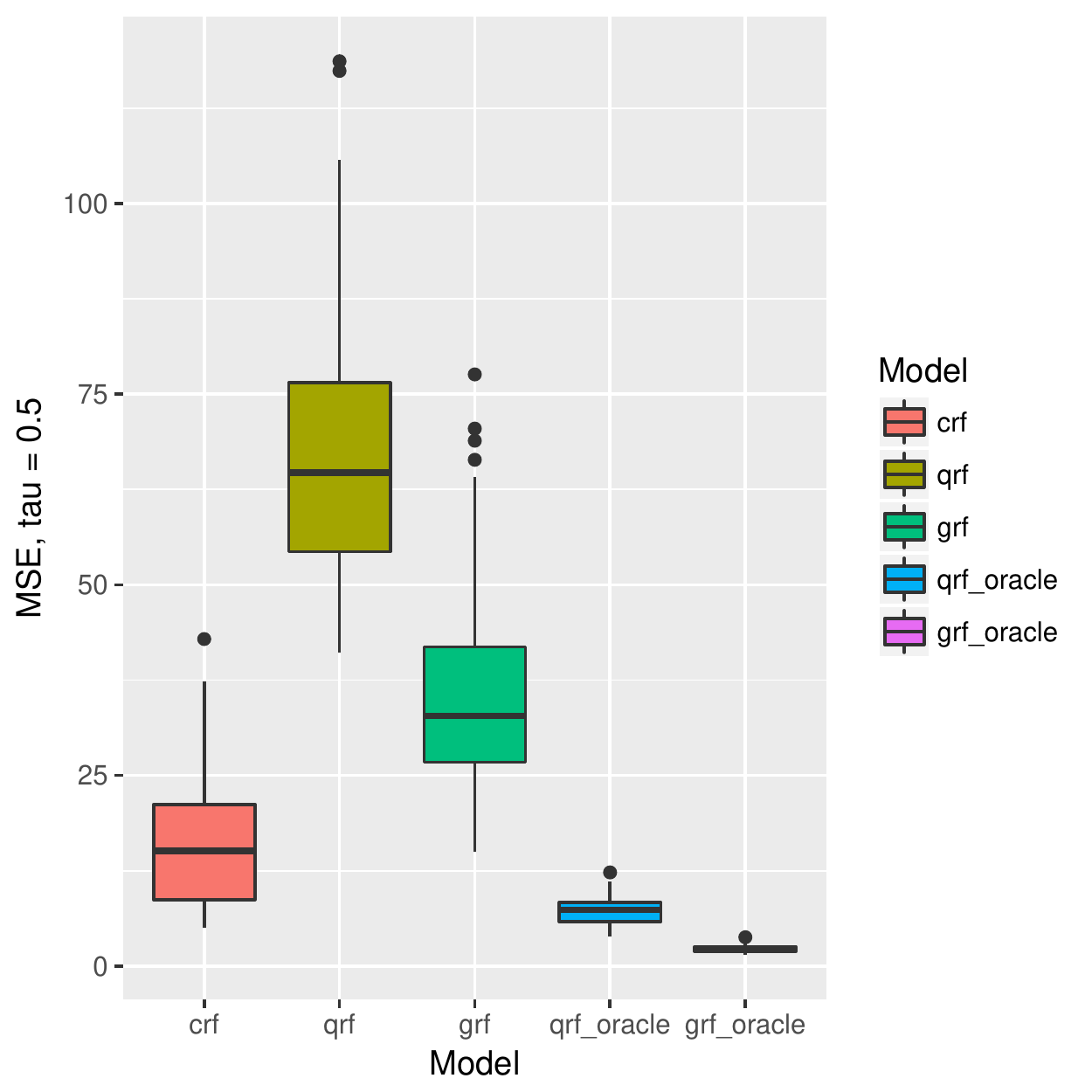}
        \caption{MSE: $\tau = 0.5$}
    \end{subfigure}
    ~
    \begin{subfigure}[b]{0.18\linewidth}
        \includegraphics[width=\textwidth]{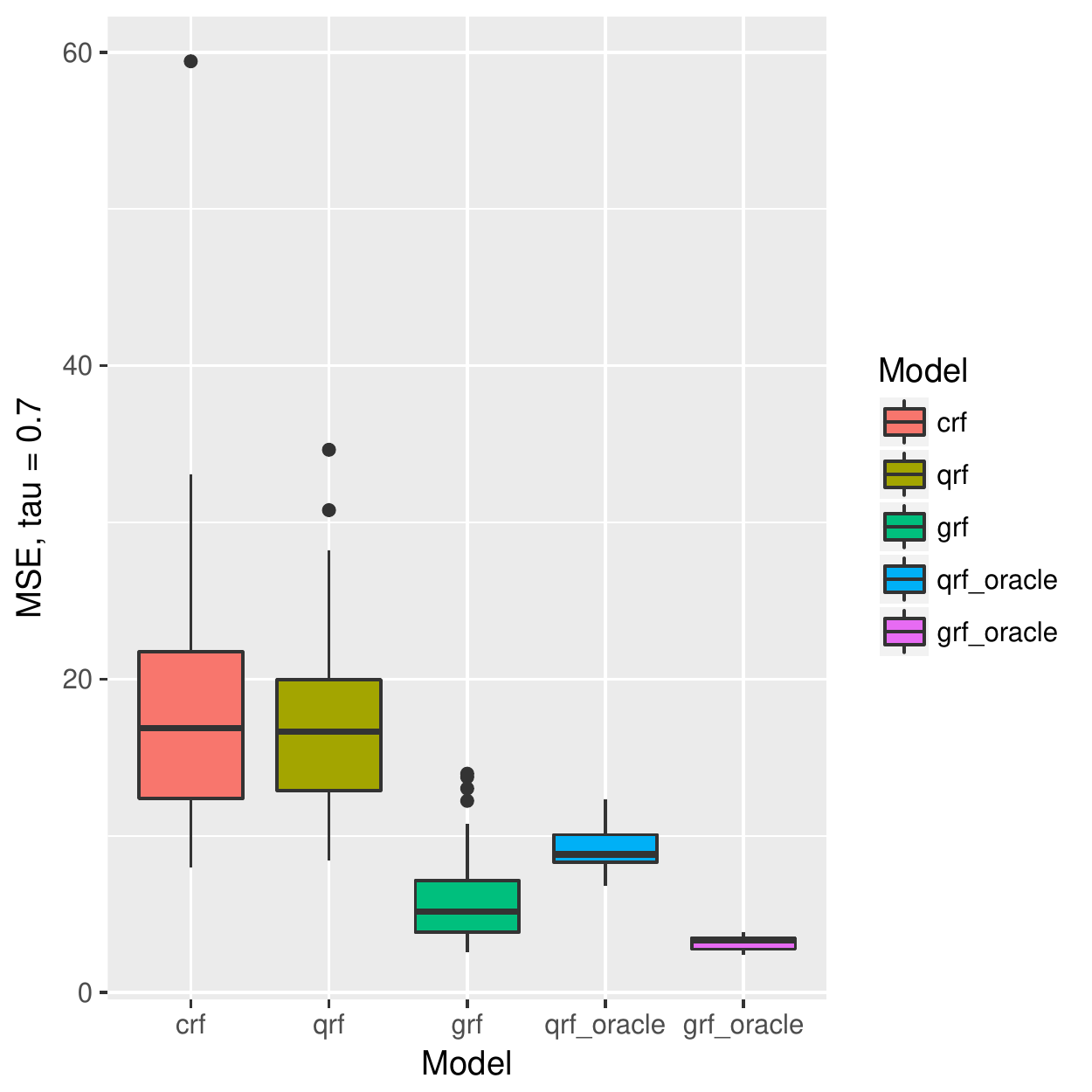}
        \caption{MSE: $\tau = 0.7$}
    \end{subfigure}
    ~
    \begin{subfigure}[b]{0.18\linewidth}
        \includegraphics[width=\textwidth]{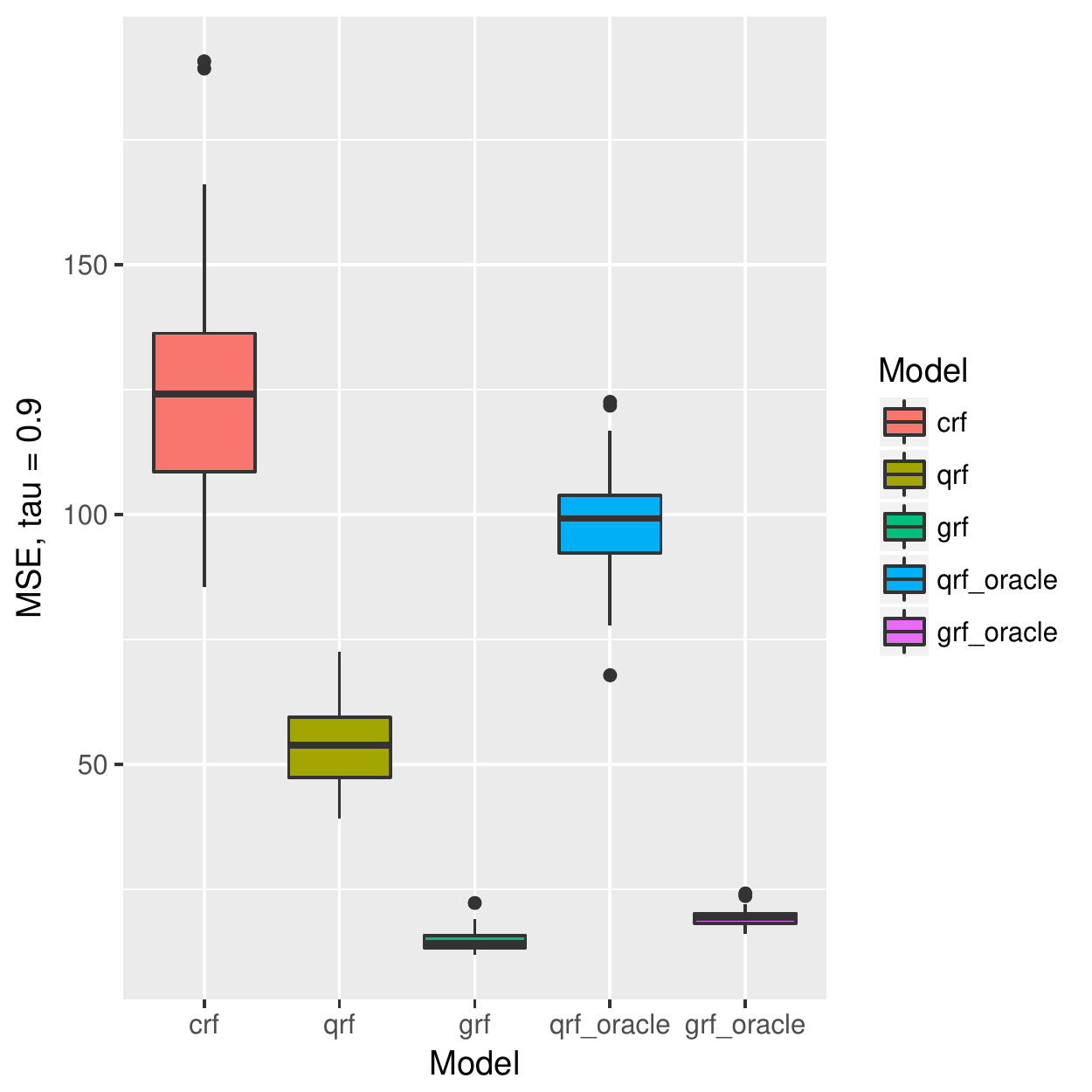}
        \caption{MSE: $\tau = 0.9$}
    \end{subfigure}
    
    \vspace{-0.05in}
    
    \begin{subfigure}[b]{0.18\linewidth}
        \includegraphics[width=\textwidth]{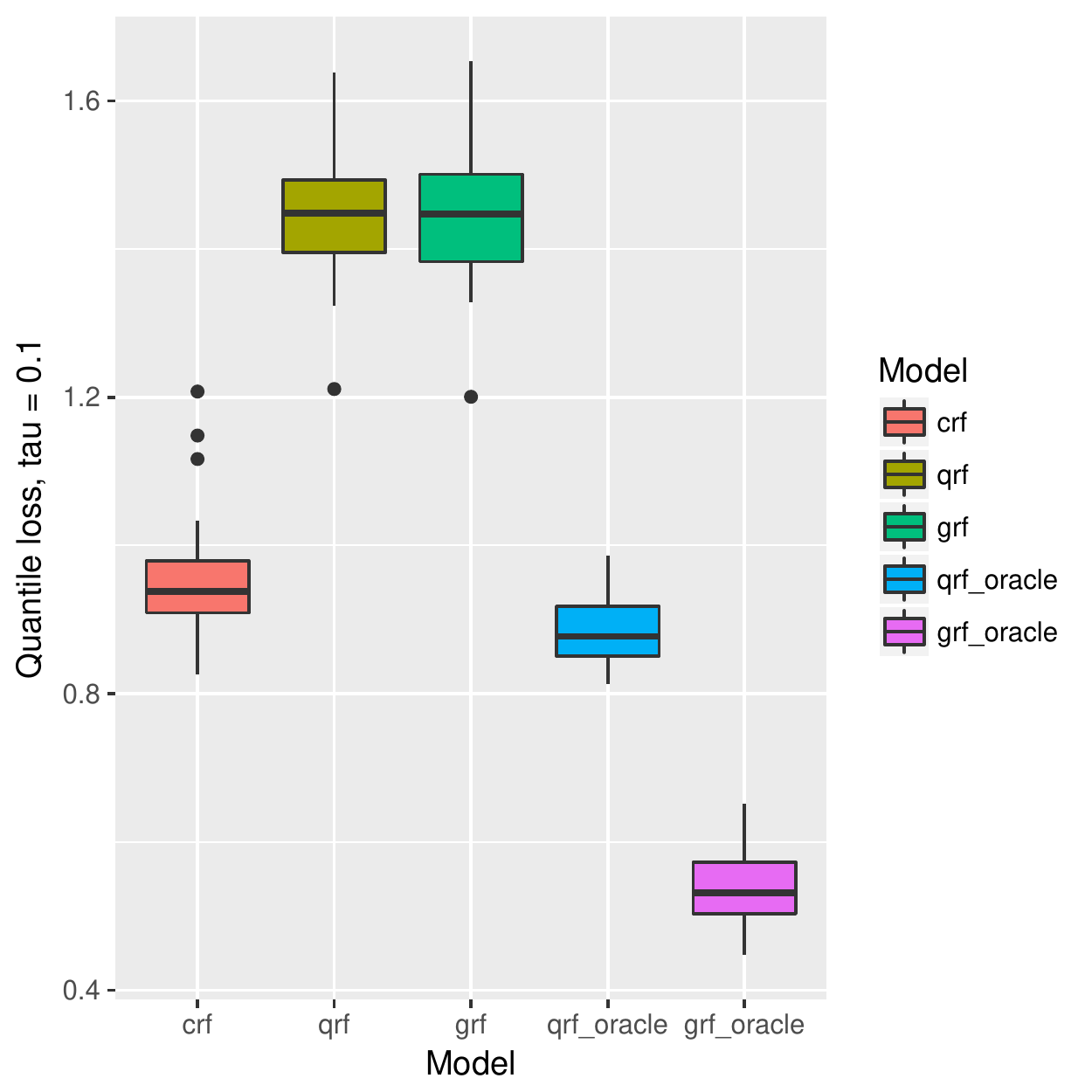}
        \caption{Quantile loss: $\tau = 0.3$}
    \end{subfigure}
    ~ %add desired spacing between images, e. g. ~, \quad, \qquad, \hfill etc. 
    \begin{subfigure}[b]{0.18\linewidth}
        \includegraphics[width=\textwidth]{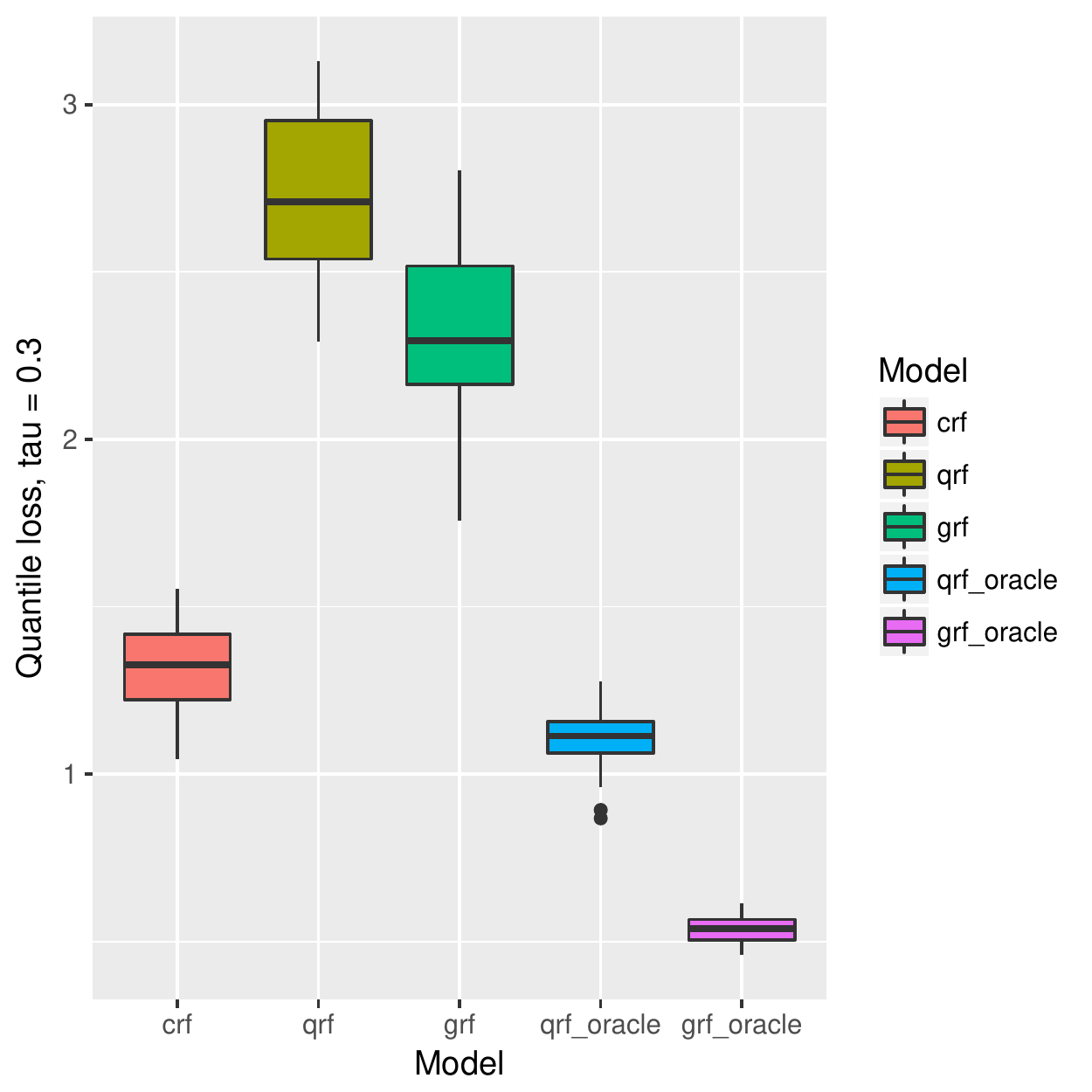}
        \caption{Quantile loss: $\tau = 0.3$}
    \end{subfigure}
    ~
    \begin{subfigure}[b]{0.18\linewidth}
        \includegraphics[width=\textwidth]{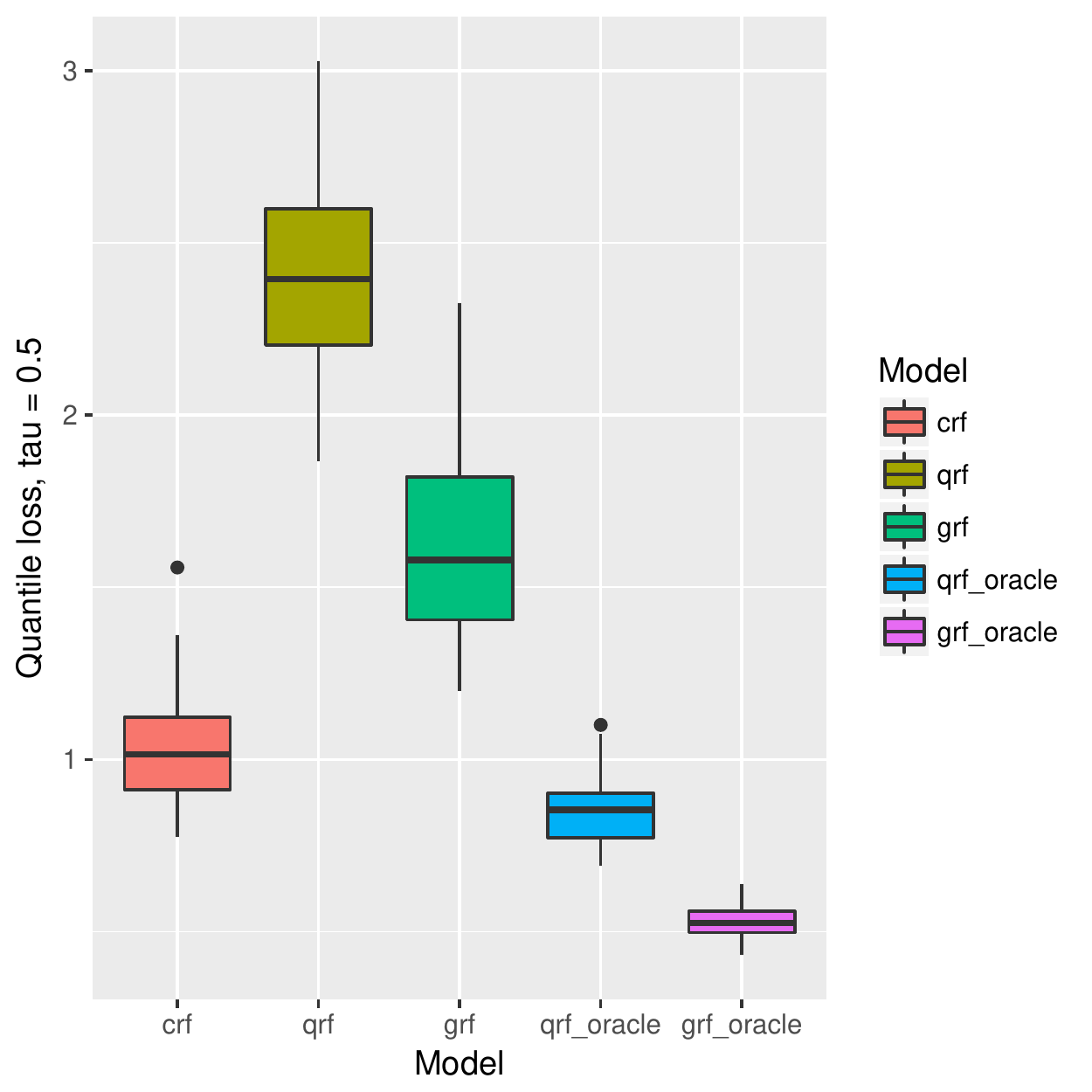}
        \caption{Quantile loss: $\tau = 0.5$}
    \end{subfigure}
    ~
    \begin{subfigure}[b]{0.18\linewidth}
        \includegraphics[width=\textwidth]{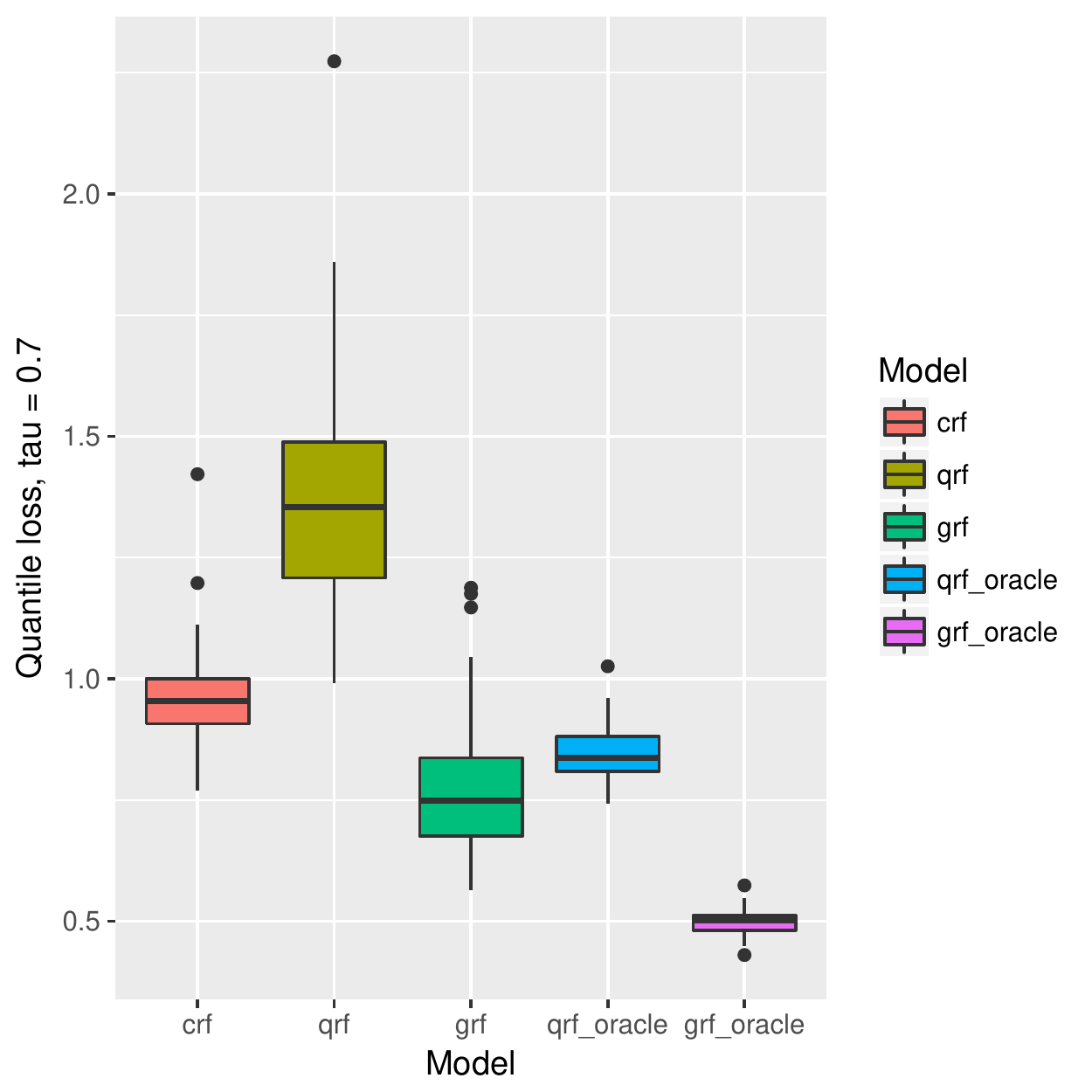}
        \caption{Quantile loss: $\tau = 0.7$}
    \end{subfigure}
    ~
    \begin{subfigure}[b]{0.18\linewidth}
        \includegraphics[width=\textwidth]{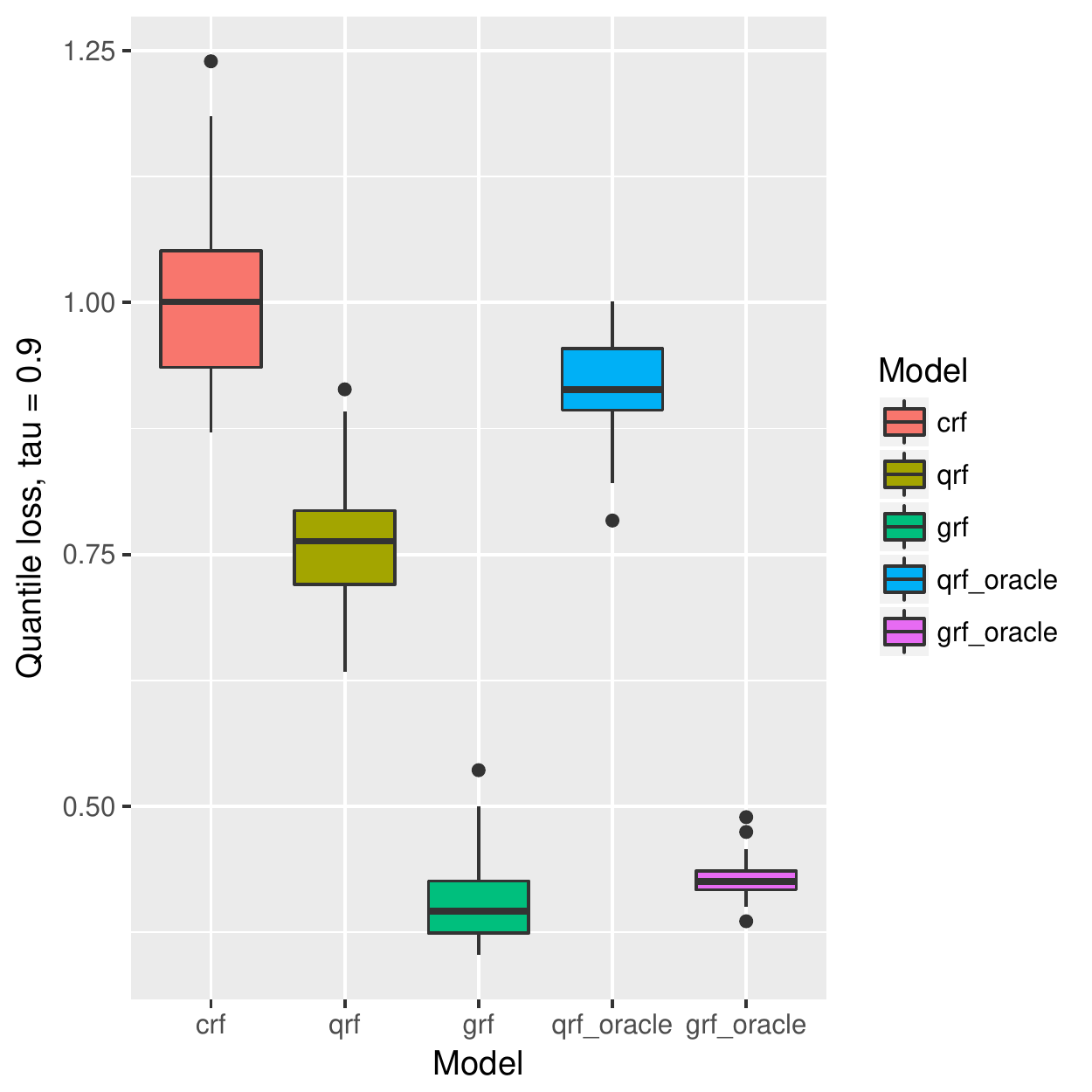}
        \caption{Qauntile loss: $\tau = 0.9$}
    \end{subfigure}
    
    \vspace{-0.05in}
    
    \begin{subfigure}[b]{0.18\linewidth}
        \includegraphics[width=\textwidth]{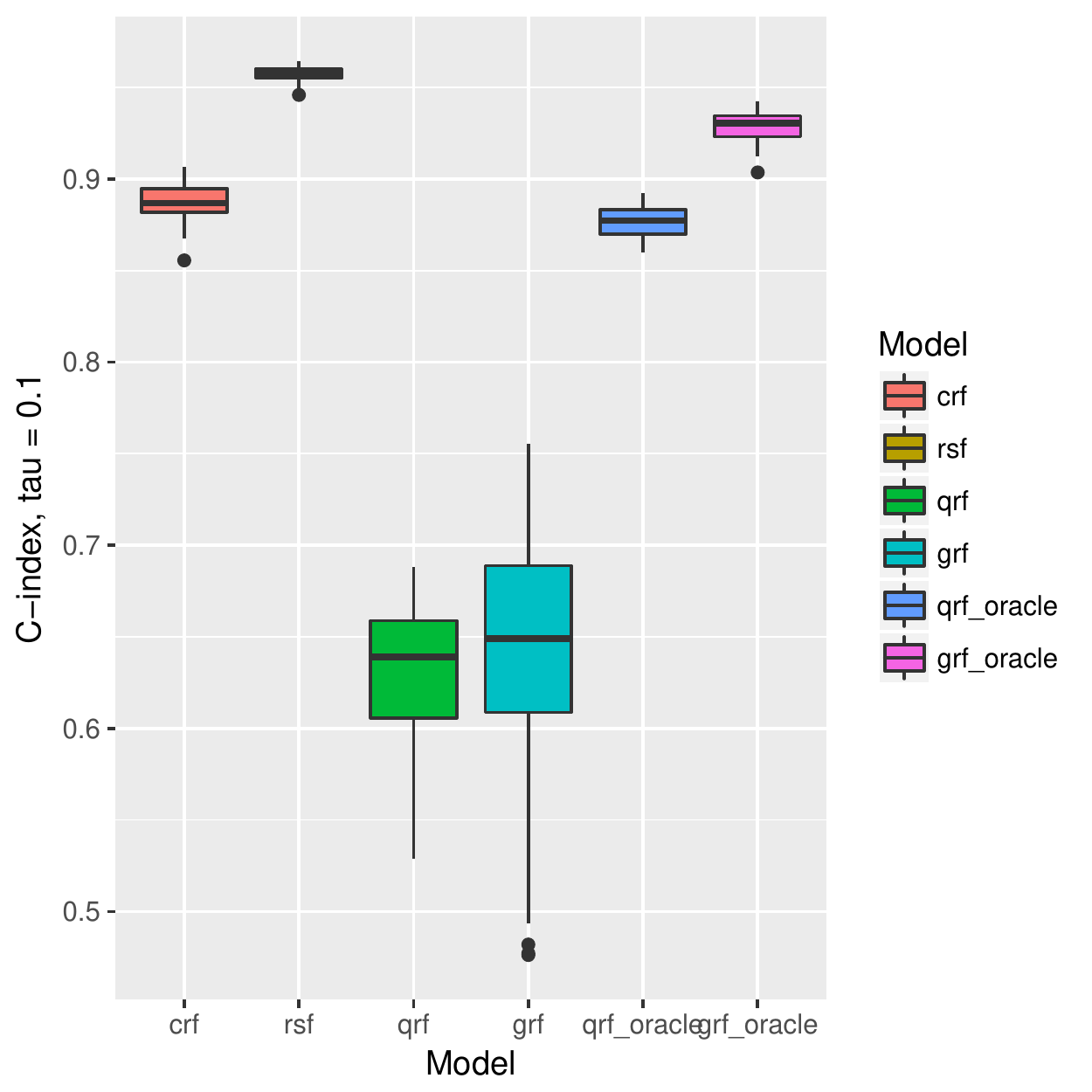}
        \caption{C-index: $\tau = 0.3$}
    \end{subfigure}
    ~ %add desired spacing between images, e. g. ~, \quad, \qquad, \hfill etc. 
    \begin{subfigure}[b]{0.18\linewidth}
        \includegraphics[width=\textwidth]{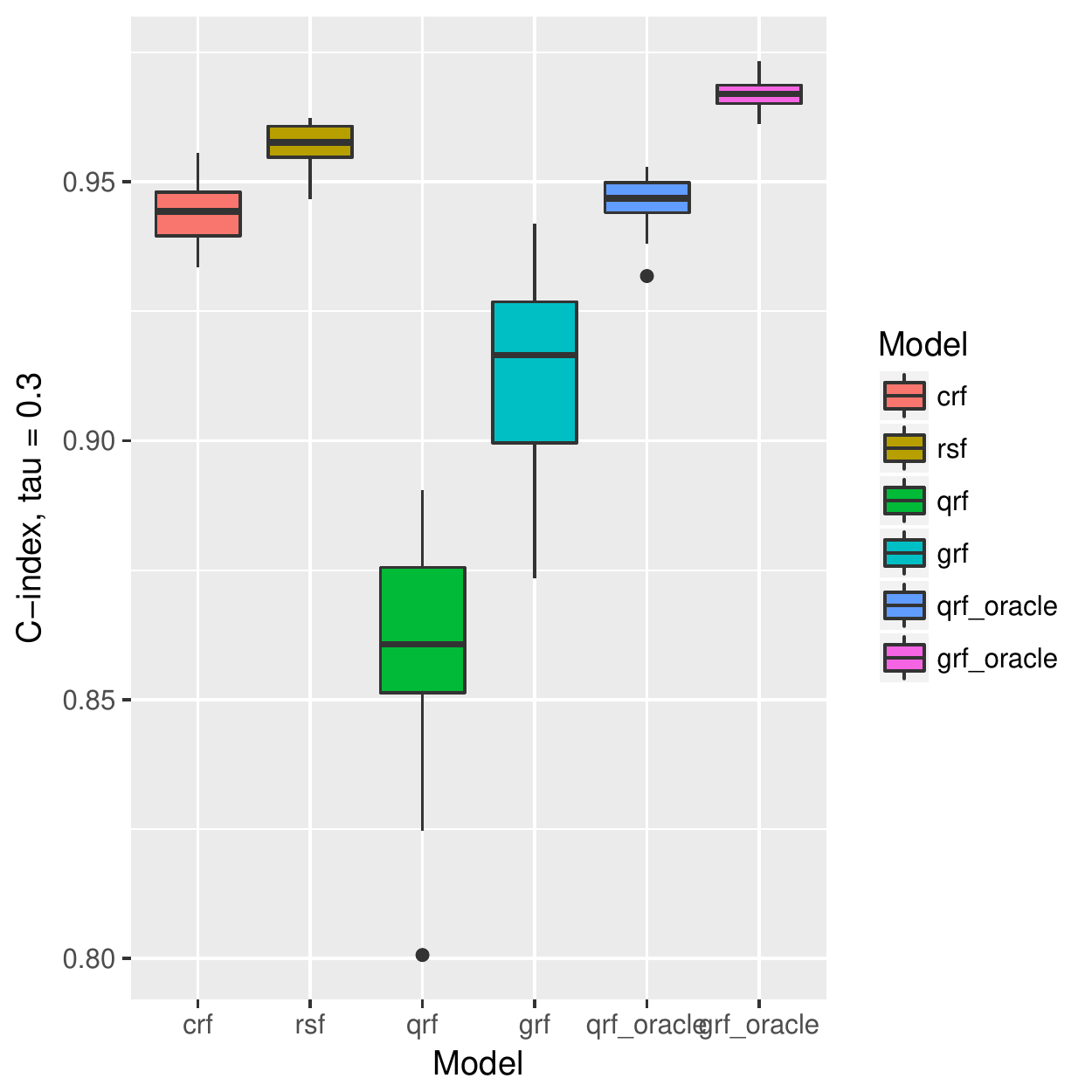}
        \caption{C-index: $\tau = 0.3$}
    \end{subfigure}
    ~
    \begin{subfigure}[b]{0.18\linewidth}
        \includegraphics[width=\textwidth]{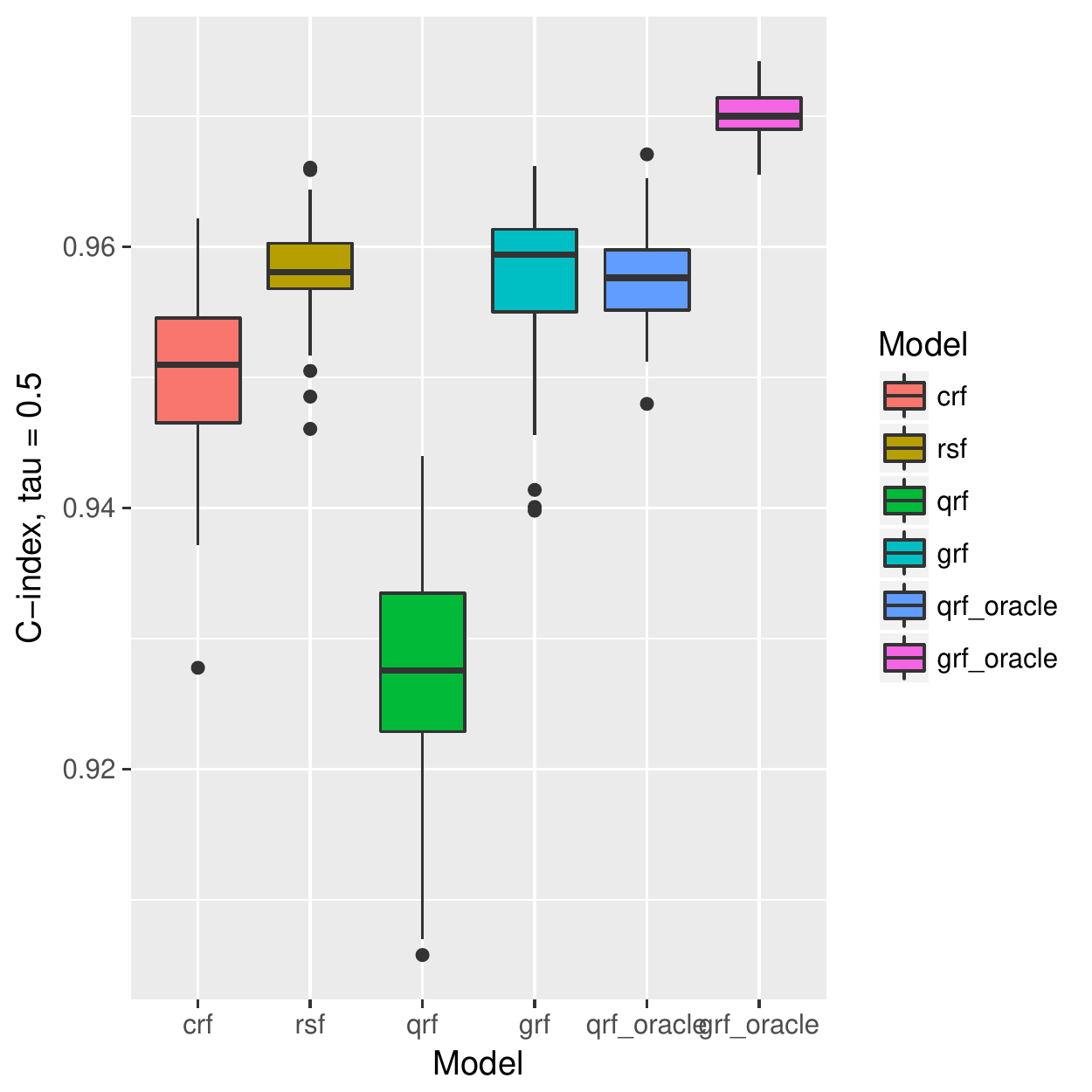}
        \caption{C-index: $\tau = 0.5$}
    \end{subfigure}
    ~
    \begin{subfigure}[b]{0.18\linewidth}
        \includegraphics[width=\textwidth]{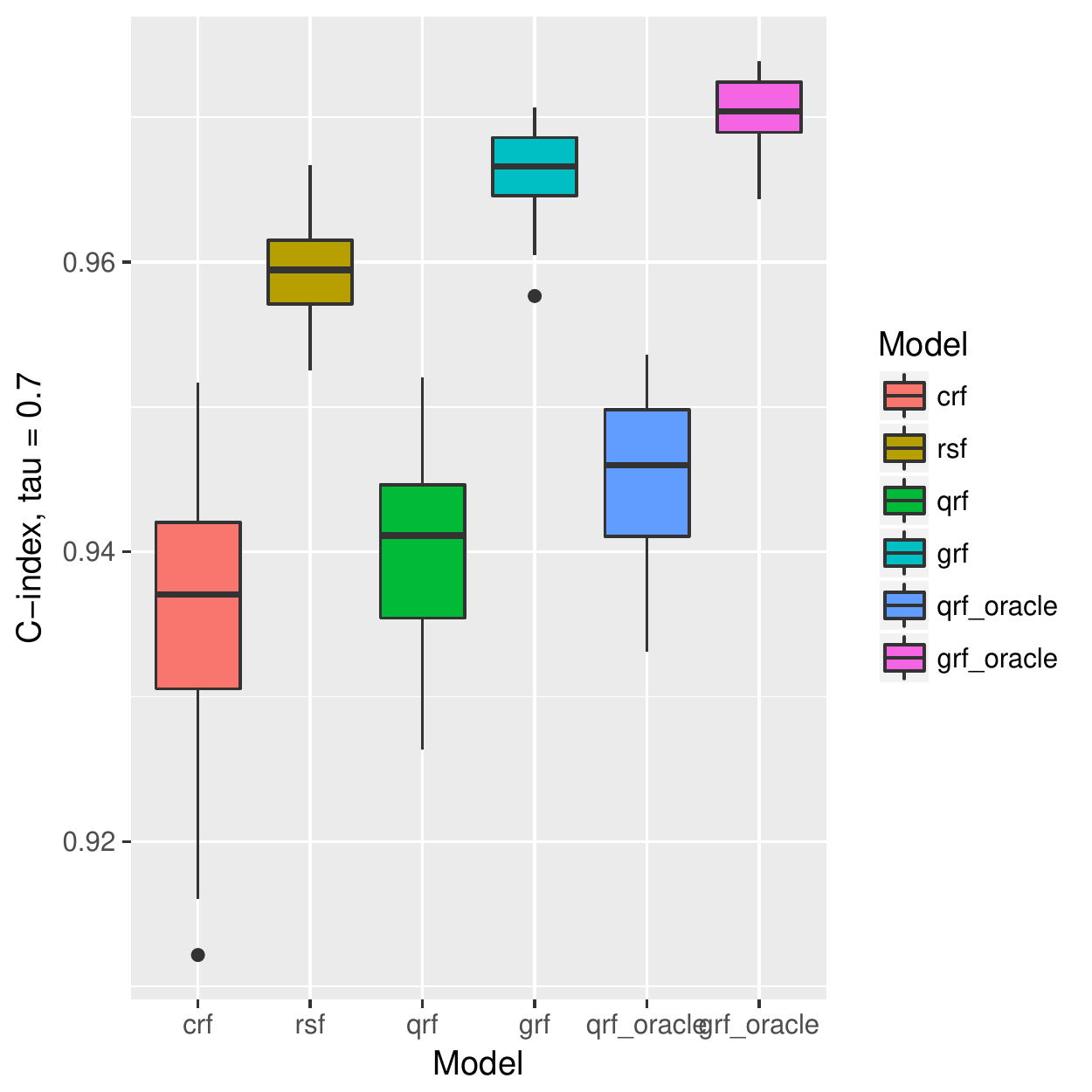}
        \caption{C-index: $\tau = 0.7$}
    \end{subfigure}
    ~
    \begin{subfigure}[b]{0.18\linewidth}
        \includegraphics[width=\textwidth]{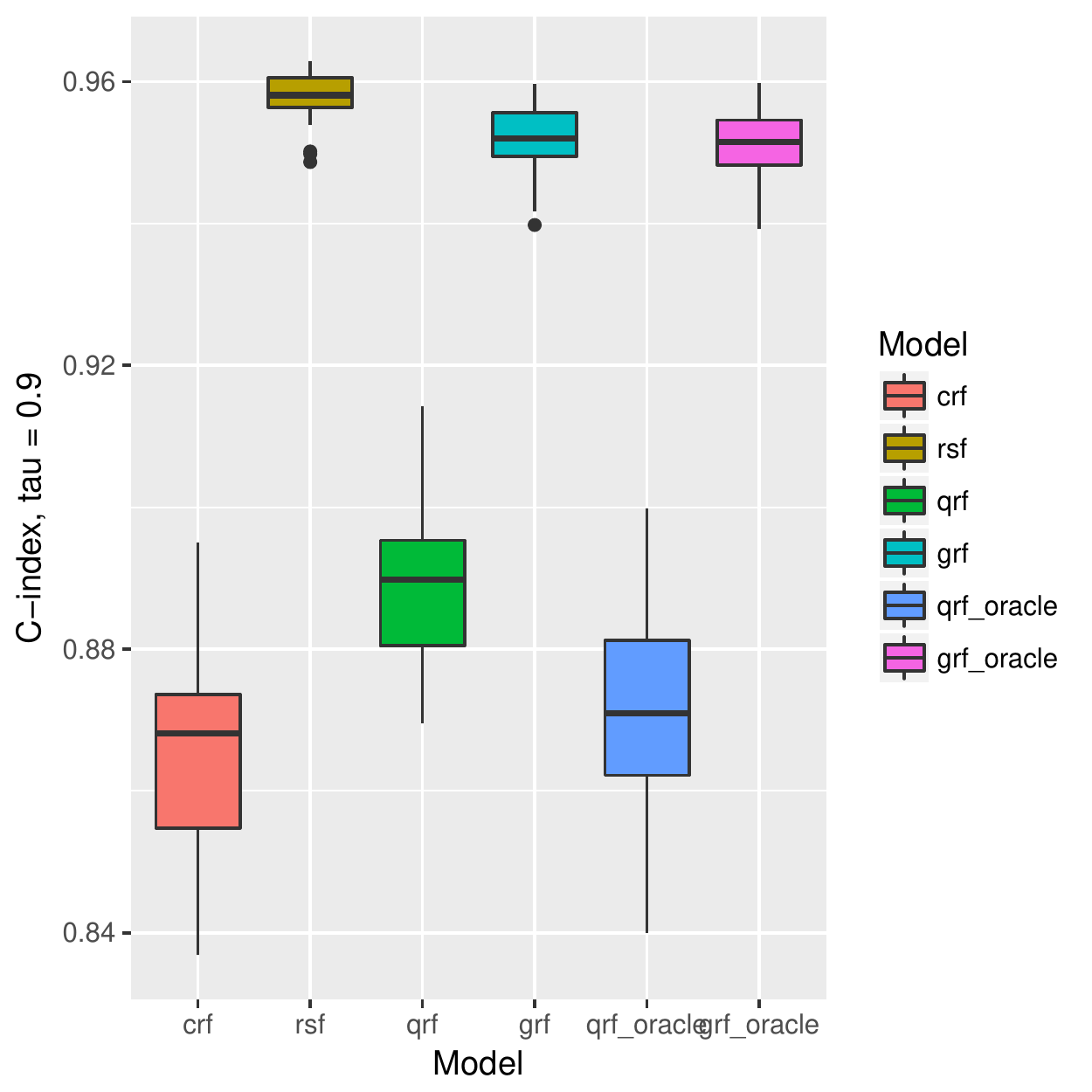}
        \caption{C-index: $\tau = 0.9$}
    \end{subfigure}
    
    \caption{Censored regression model with five-dimensional complex manifold structure: box plots with $n=300$ and $B=1000$. For the metrics MSE \eqref{eq:L_MSE}, MAD \eqref{eq:L_MAD} and quantile loss \eqref{eq:L_quantile}, the smaller the value is the better. For C-index, the larger it is the better.}
    \label{fig:complex_multi_box}
\end{figure}

\subsubsection{Node size}
\label{sssec:nodesize}
In this section, we investigate the impact of the choice of the node size on different methods. The data we use will be generated from the one-dimensional and multi-dimensional AFT and Sine models as defined in the previous sections. We increase the node size from 5 to 60 with step size of 5. 

% aft nodesize
\begin{figure}[!htb]
    \small
    \centering
    \begin{subfigure}[b]{0.3\linewidth}
        \includegraphics[width=\textwidth]{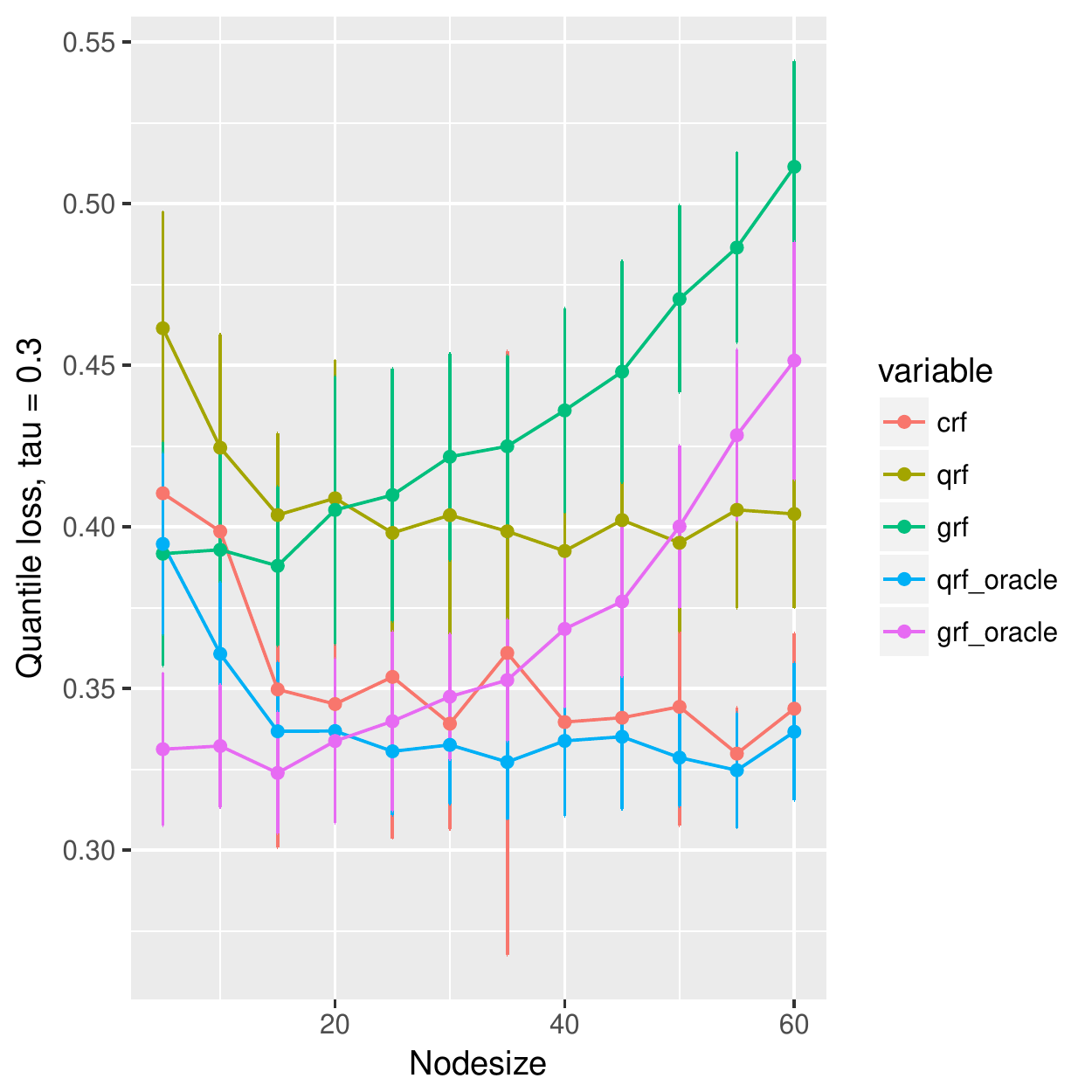}
        \caption{$\tau = 0.3$}
    \end{subfigure}
    ~
    \begin{subfigure}[b]{0.3\linewidth}
        \includegraphics[width=\textwidth]{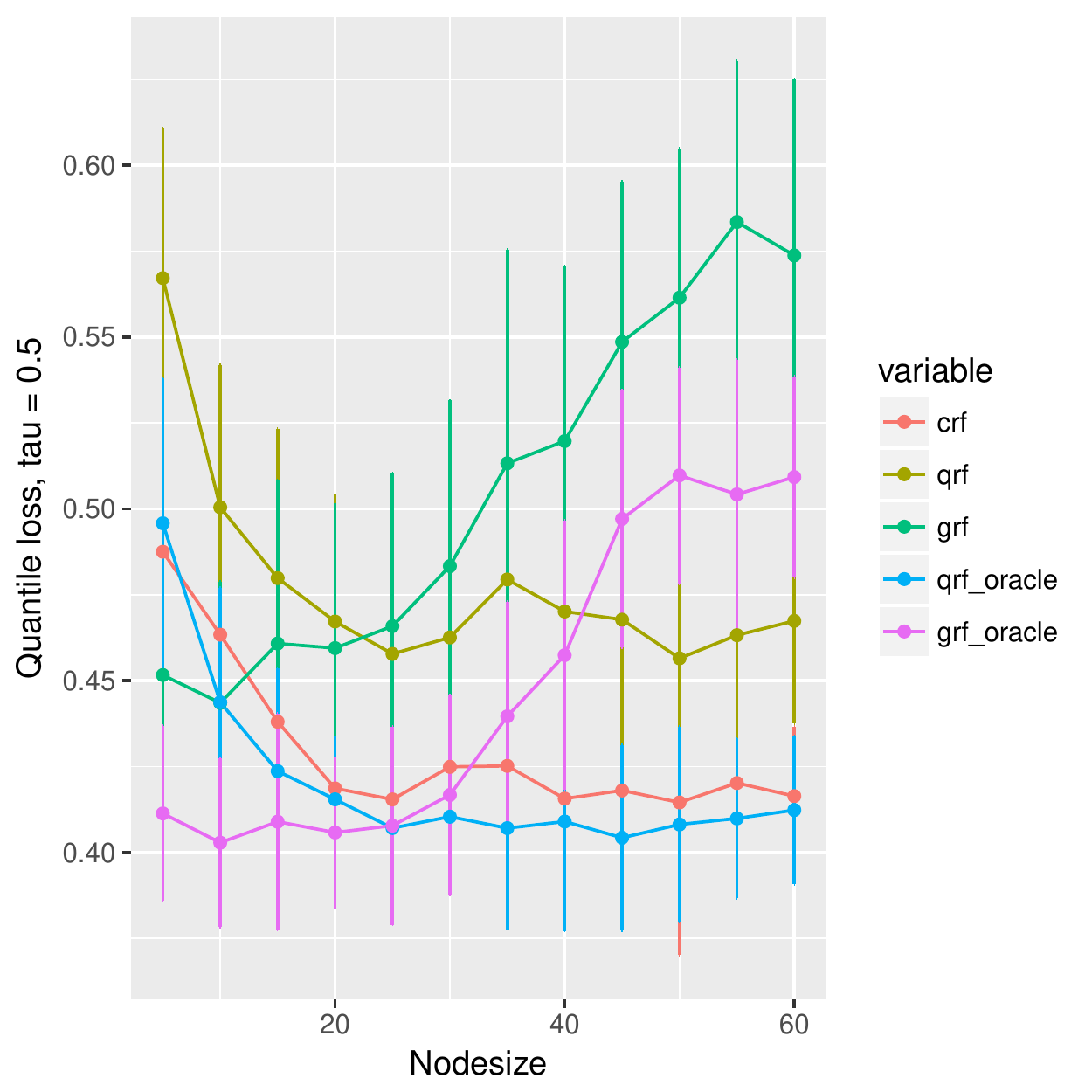}
        \caption{$\tau = 0.5$}
    \end{subfigure}
    ~
    \begin{subfigure}[b]{0.3\linewidth}
        \includegraphics[width=\textwidth]{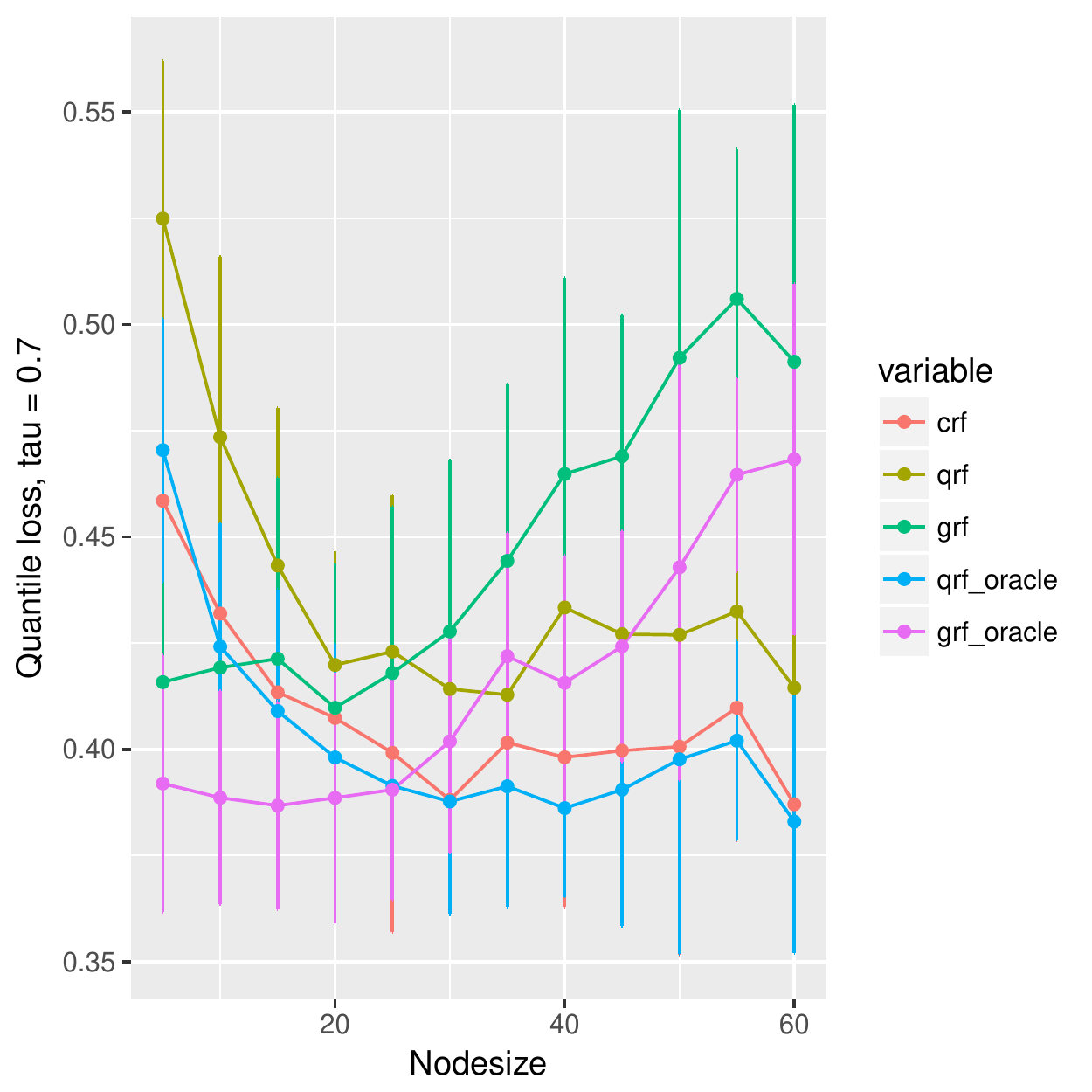}
        \caption{$\tau = 0.7$}
    \end{subfigure}
    
    \caption{Quantile losses of 1D AFT model with different node sizes.}
    \label{fig:aft_nodesize}
\end{figure}

% sine nodesize
\begin{figure}[!htb]
    \small
    \centering
    \begin{subfigure}[b]{0.3\linewidth}
        \includegraphics[width=\textwidth]{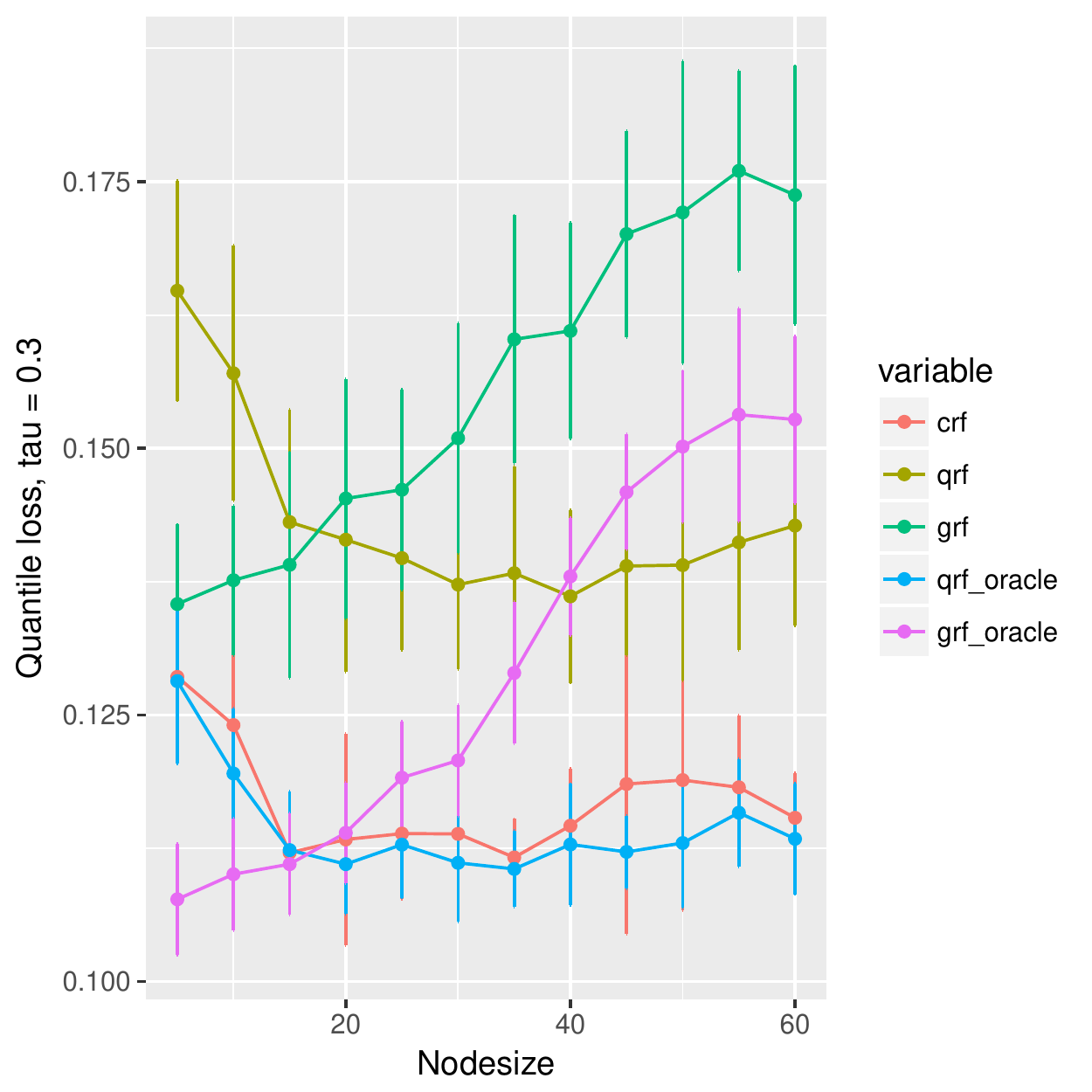}
        \caption{$\tau = 0.3$}
    \end{subfigure}
    ~
    \begin{subfigure}[b]{0.3\linewidth}
        \includegraphics[width=\textwidth]{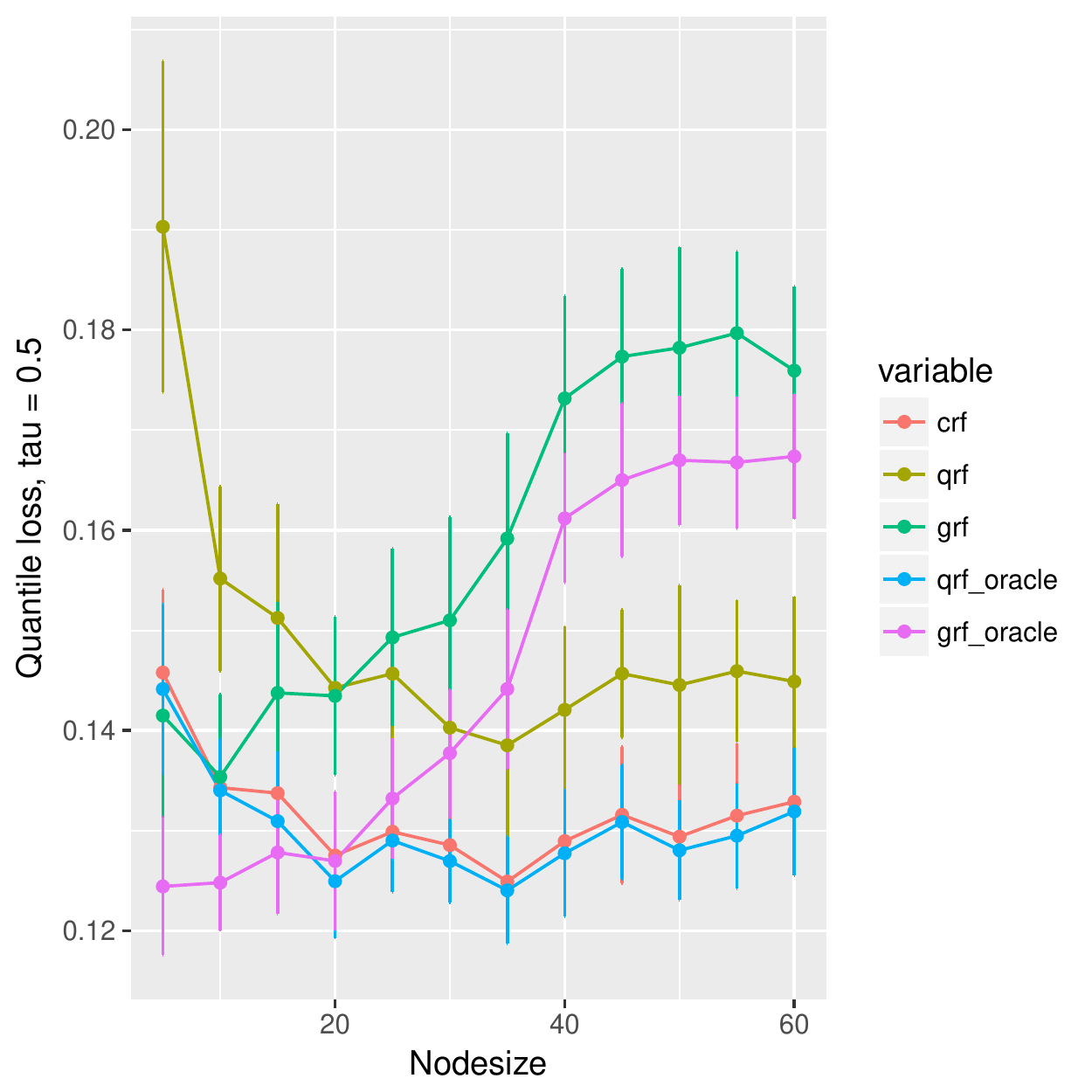}
        \caption{$\tau = 0.5$}
    \end{subfigure}
    ~
    \begin{subfigure}[b]{0.3\linewidth}
        \includegraphics[width=\textwidth]{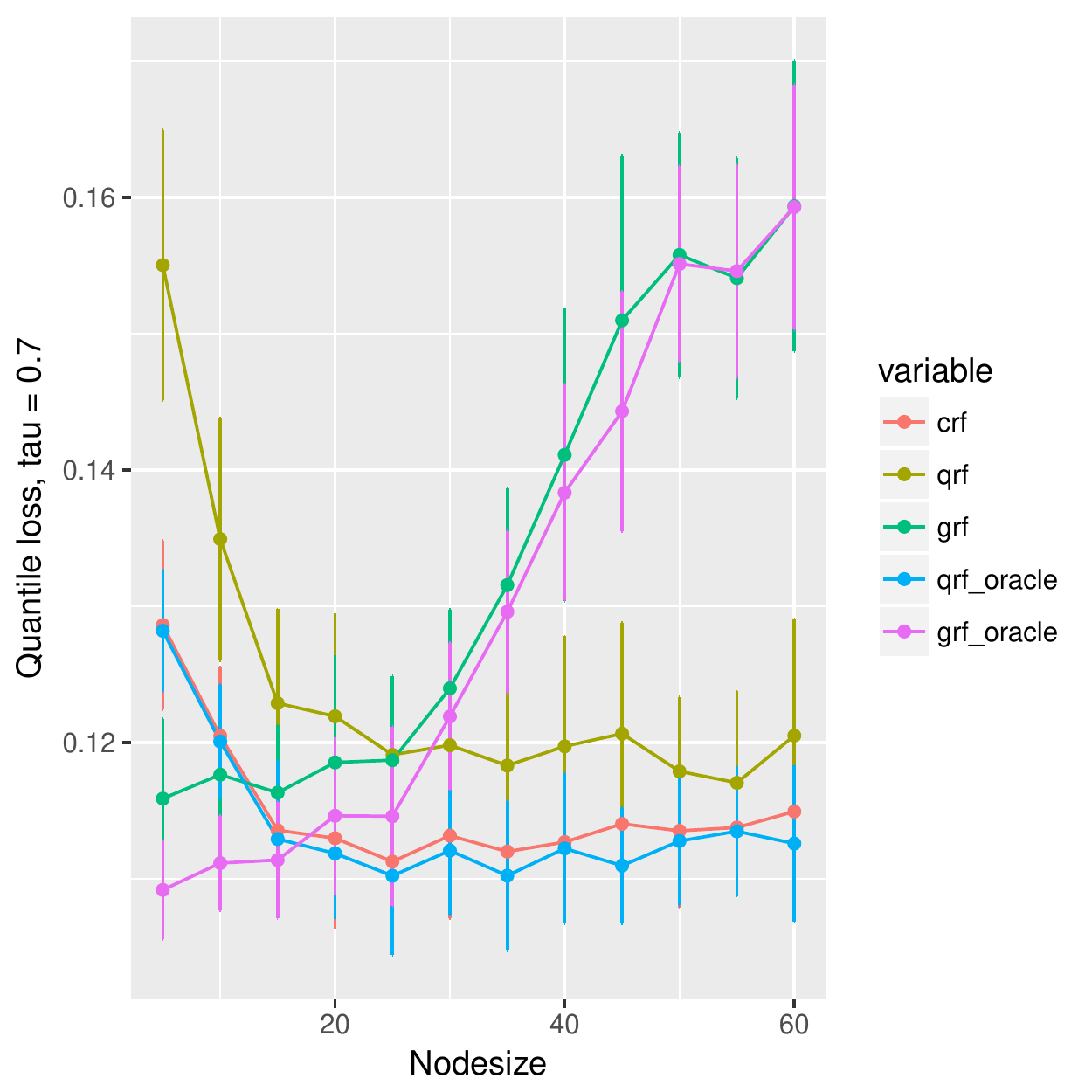}
        \caption{$\tau = 0.7$}
    \end{subfigure}
    
    \caption{Quantile losses of Sine model with different node sizes.}
    \label{fig:sine_nodesize}
\end{figure}

\paragraph{One-dimensional AFT and Sine models}
The result of sine model is summarized in Figure \ref{fig:sine_nodesize}. One can see that for both \textit{qrf} and our model, \textit{crf}, the quantile loss will first decrease when node size increases. It attains minimum around node size of 30. However, for \textit{grf}, its quantile loss is almost monotonically increasing, and attains minimum at node size of 5. But both \textit{qrf-oracle} and \textit{grf-oracle} can attain the best quantile loss of about 0.125. And one impressive observation is that our model, \textit{crf}, almost performs the same as \textit{qrf-oracle} for all node sizes. Similar conclusion can be made from the AFT result which is in Figure \ref{fig:aft_nodesize}.

% multi aft
\begin{figure}[!htb]
    \small
    \centering
    \begin{subfigure}[b]{0.3\linewidth}
        \includegraphics[width=\textwidth]{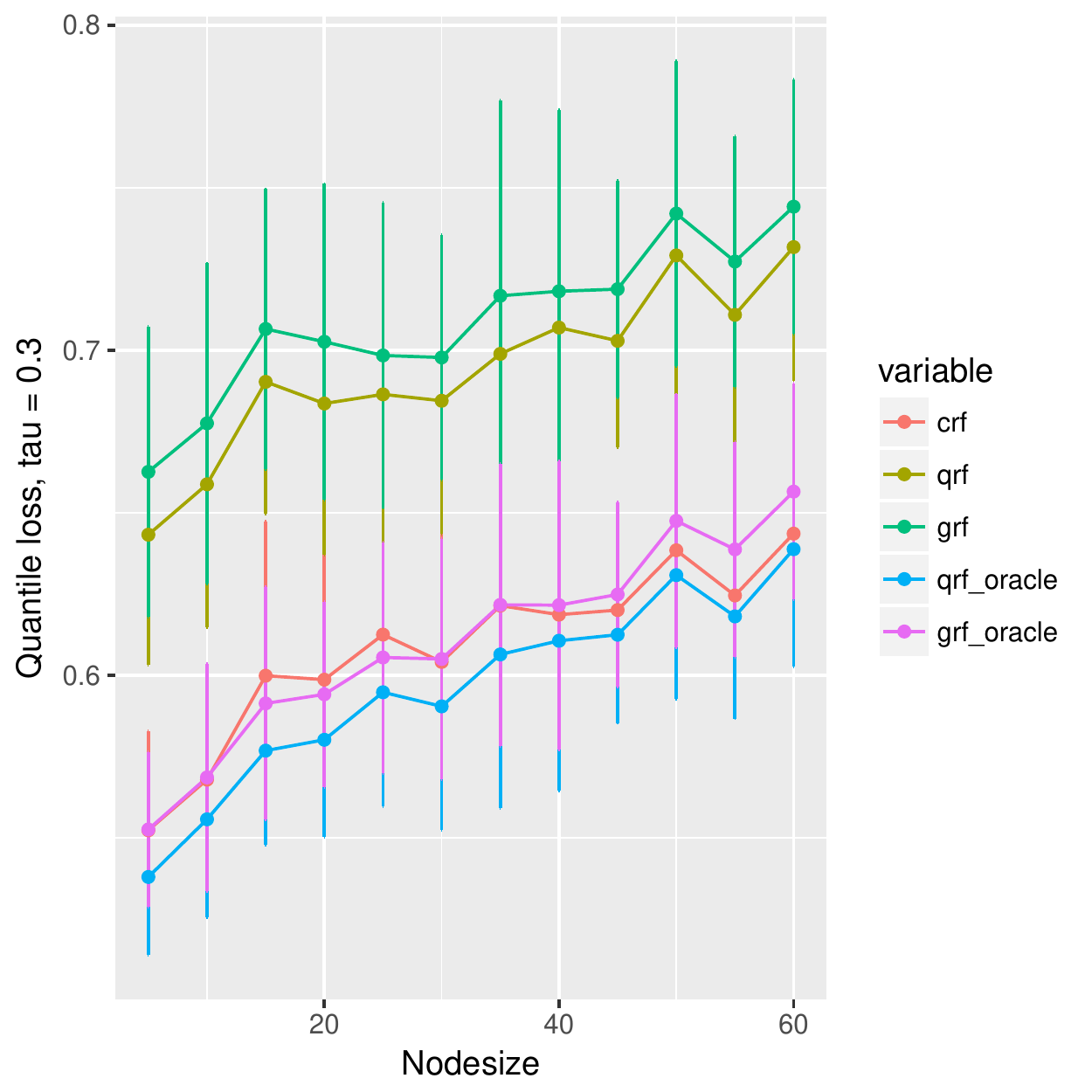}
        \caption{$\tau = 0.3$}
    \end{subfigure}
    ~
    \begin{subfigure}[b]{0.3\linewidth}
        \includegraphics[width=\textwidth]{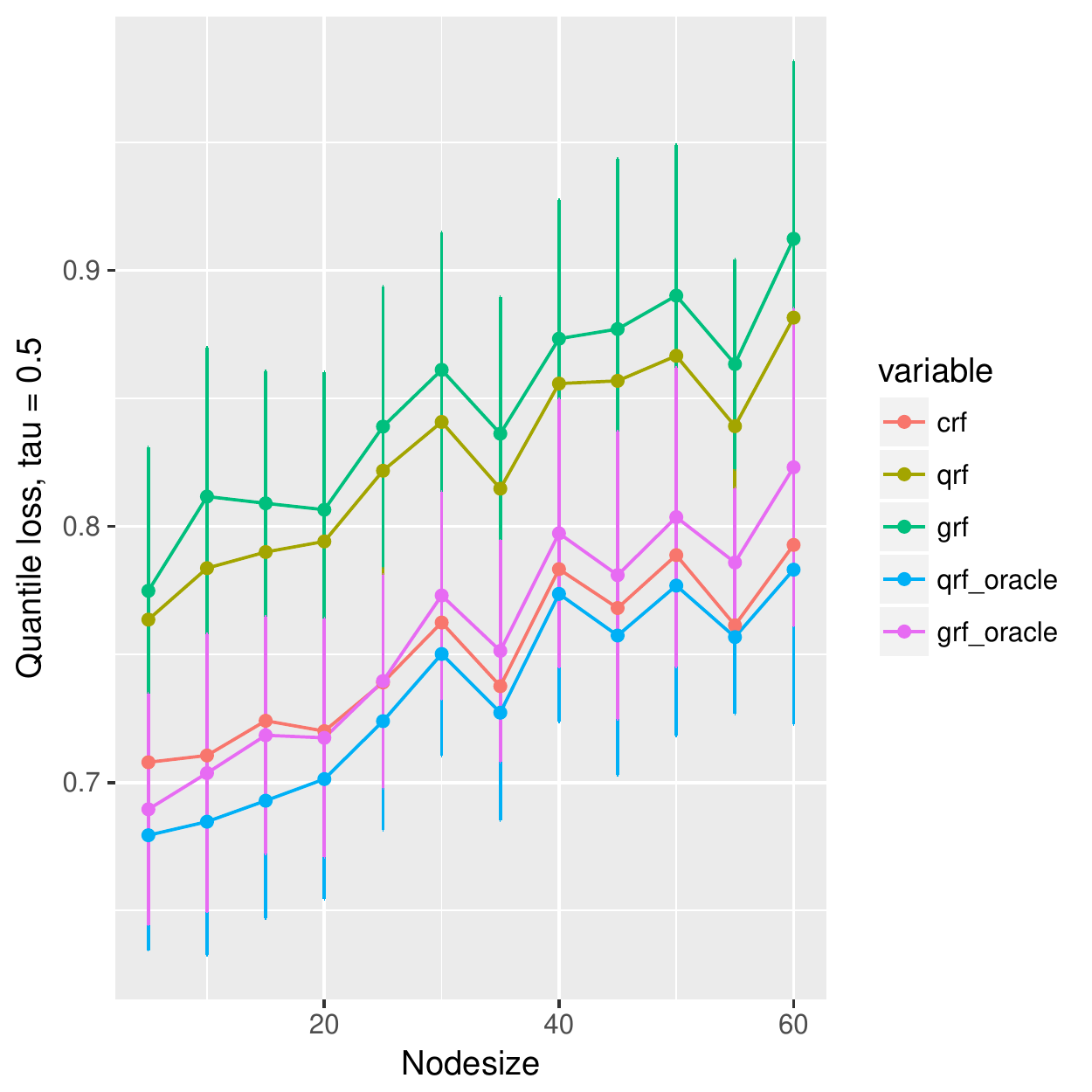}
        \caption{$\tau = 0.5$}
    \end{subfigure}
    ~
    \begin{subfigure}[b]{0.3\linewidth}
        \includegraphics[width=\textwidth]{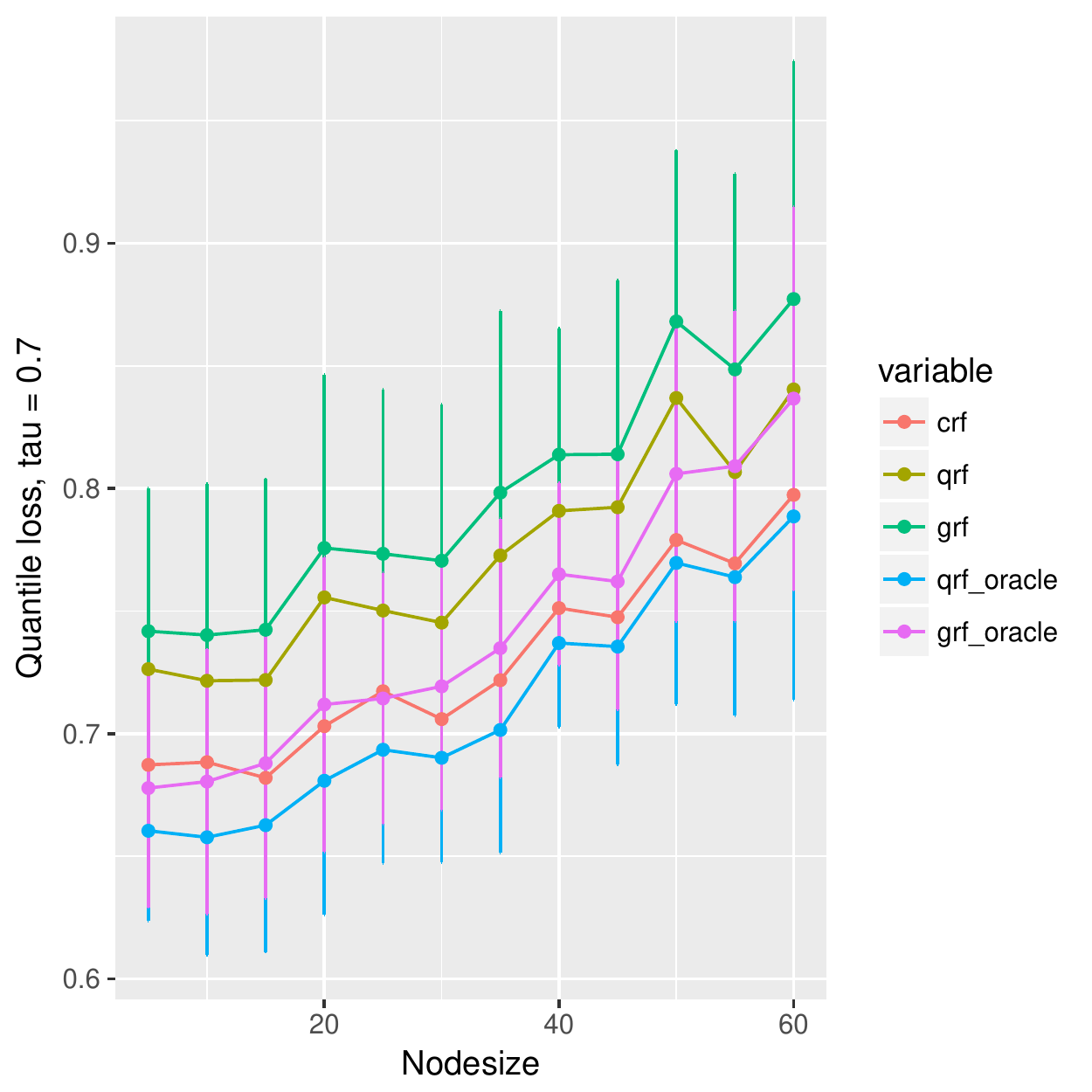}
        \caption{$\tau = 0.7$}
    \end{subfigure}
    
    \caption{Quantile losses of multi-dimensional AFT model with different node sizes.}
    \label{fig:aft_multi_nodesize}
\end{figure}

% multi complex
\begin{figure}[!htb]
    \small
    \centering
    \begin{subfigure}[b]{0.3\linewidth}
        \includegraphics[width=\textwidth]{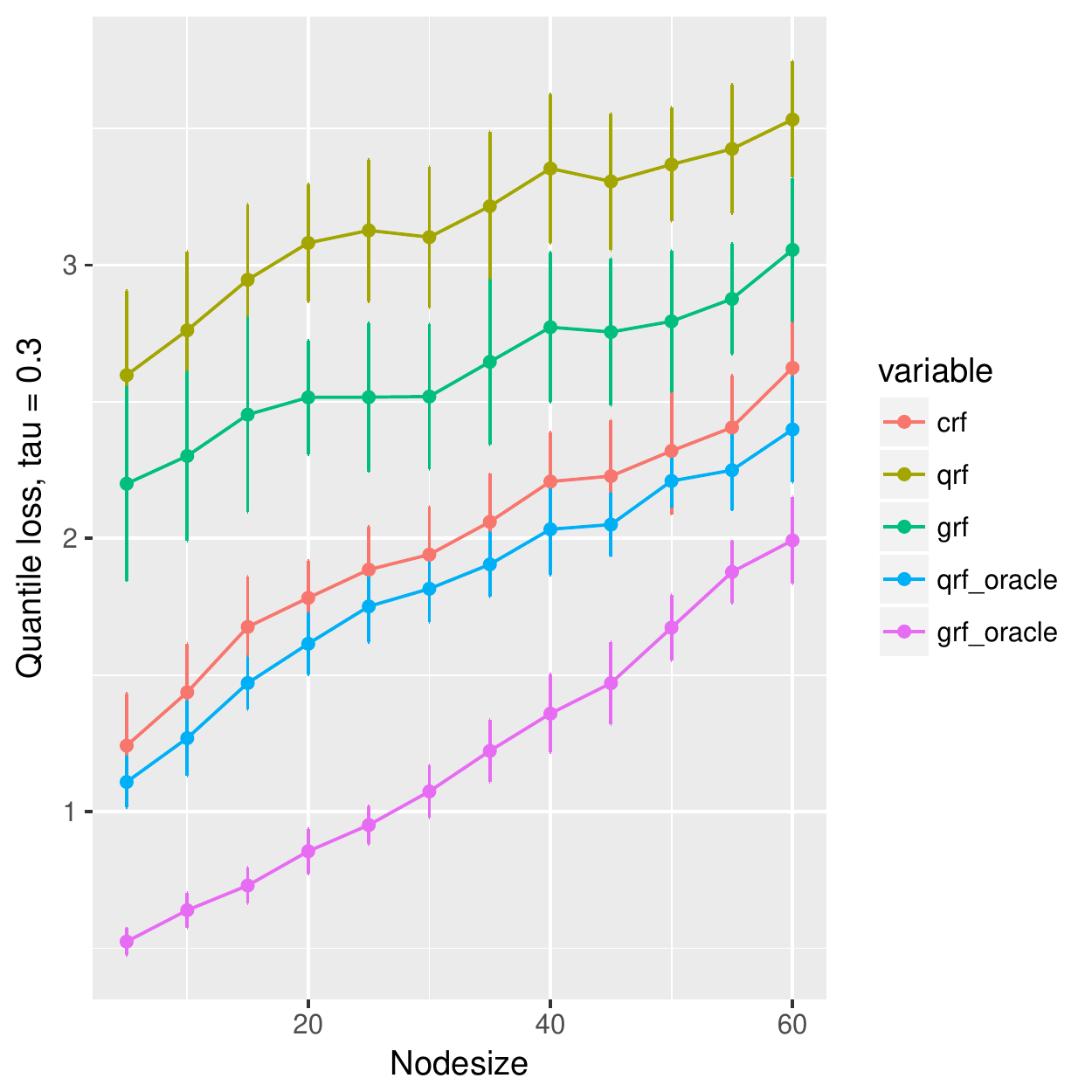}
        \caption{$\tau = 0.3$}
    \end{subfigure}
    ~
    \begin{subfigure}[b]{0.3\linewidth}
        \includegraphics[width=\textwidth]{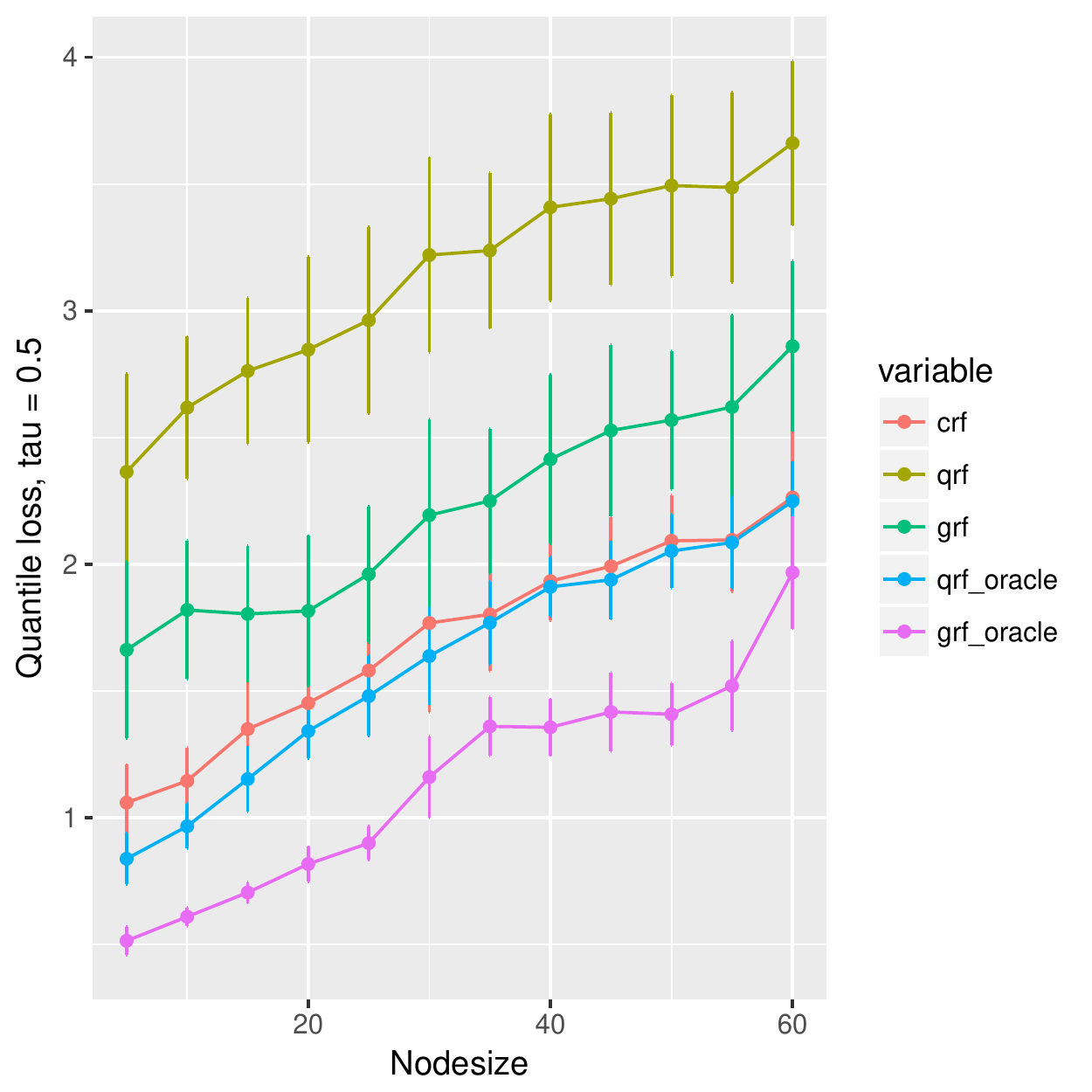}
        \caption{$\tau = 0.5$}
    \end{subfigure}
    ~
    \begin{subfigure}[b]{0.3\linewidth}
        \includegraphics[width=\textwidth]{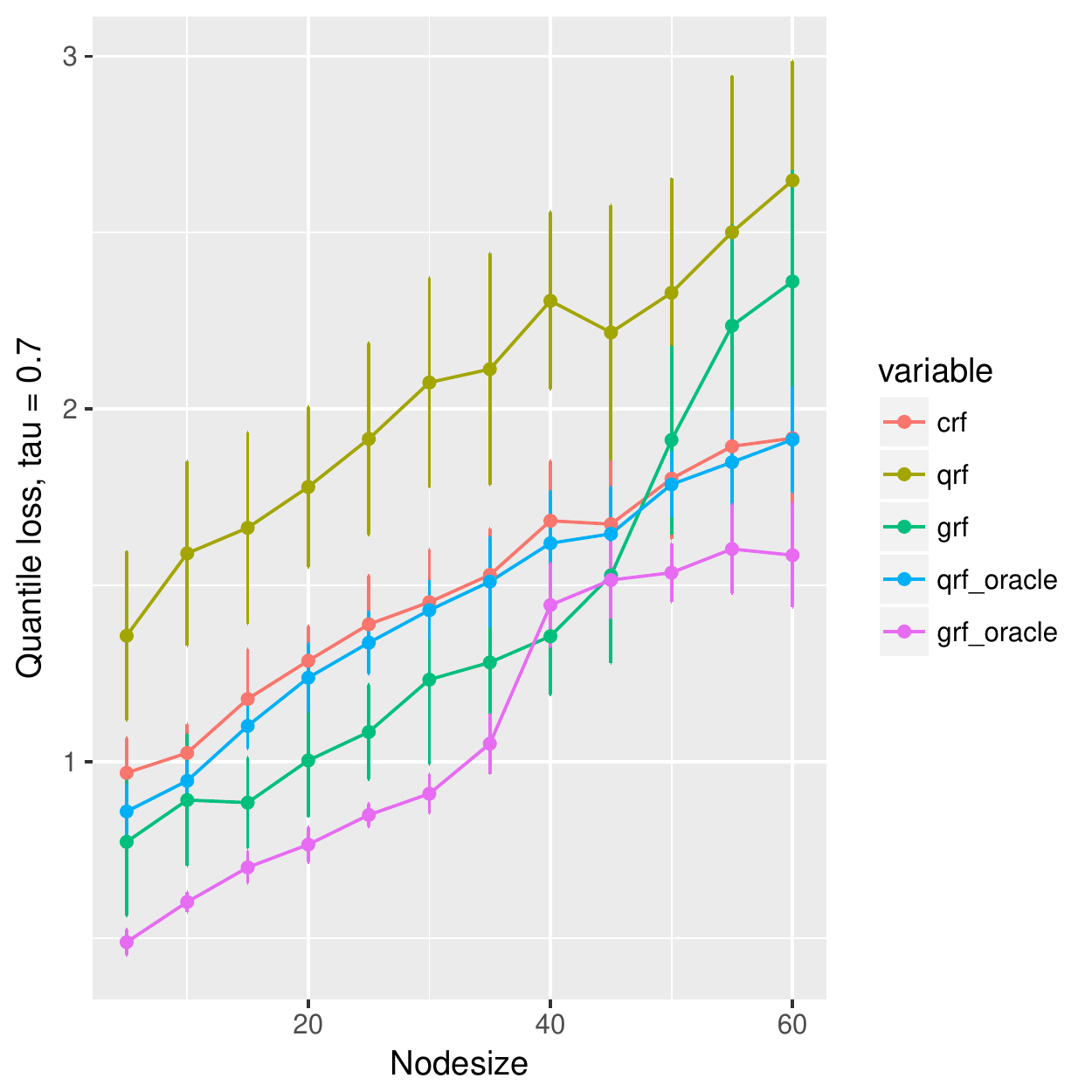}
        \caption{$\tau = 0.7$}
    \end{subfigure}
    
    \caption{Quantile losses of multi-dimensional complex model with different node sizes.}
    \label{fig:complex_multi_nodesize}
\end{figure}

\paragraph{Multi-dimensional AFT model}
From Figure \ref{fig:aft_multi_nodesize}, we observe that \textit{qrf}, \textit{qrf-oracle} and \textit{grf-oracle} all give similar results. The performance of our model \textit{crf} is only slightly worse than \textit{qrf-oracle}, but is even better than \textit{grf-oracle}.

\paragraph{Multi-dimensional complex model}
The result is summarized in Figure \ref{fig:complex_multi_nodesize}. The censoring level is about $25\%$. From the figure, we observe that the behavior of \textit{crf} is still only slightly worse than \textit{qrf-oracle}. In this experiment, \textit{grf-oracle} behaves the best. All of \textit{crf}, \textit{qrf-oracle} and \textit{grf-oracle} are significantly better than the biased models \textit{qrf} and \textit{grf}. When $\tau=0.7$, \textit{grf} behaves slightly better than \textit{qrf-oracle} when node size is small. The reason is that the conditional quantiles of $Y$ and $T$ are closer when $\tau$ is larger, and \textit{grf} is more stable and smooth on the data in this experiment. But we still observe that the performance of \textit{crf} and \textit{qrf-oracle} are very close.

\subsection{Real Data}
In this section, we compare our censored forest (crf) with quantile random forest (qrf) \citep{meinshausen2006quantile} and generalized forest (grf) \citep{athey2016generalized} on real datasets. In order to evaluate the performances unbiasedly, we manually add censoring to the data. In addition, we apply qrf and grf to the data without censoring and we call the resulted models qrf-oracle and grf-oracle, respectively.

For all these methods, bagged versions of the training data are used for each of the $1000$ trees. We use 5-fold cross validation to select the best node sizes for different methods. For all the other parameters, we keep the default settings.

\paragraph{Datasets}
We use datasets \textit{BostonHousing}, \textit{Ozone} from the R packages \textit{mlbench} and \textit{alr3}. For all the datasets, we sample censoring variables from Exponential distributions with $\lambda$ set so that the censoring level is roughly $20\%$. For \textit{BostonHousing} dataset, we set $\lambda = 0.01$. For \textit{Ozone}, $\lambda = 0.025$. For \textit{Abalone} dataset, we random sample $1000$ observations and take the log-transformation of the response variable \textit{rings}. We then set $\lambda = 0.10$.

\paragraph{Evaluation}
For each dataset, we train our model on bootstrapped version of the data, and test the performance on out-of-bag observations. This process is repeated for 40 times, and we calculate the mean and standard deviation of the prediction errors. In our context, the prediction error is measured by the $\tau$-th quantile loss for $\tau$-th quantile estimation. The results are illustrated in Figure \ref{fig:real_data}.

On all data sets, our proposed method behaves better than quantile forest and generalized forest in terms of quantile losses. Especially when $\tau = 0.1$, $0.3$ or $0.5$, the performance of our method is significantly better than \textit{qrf} and \textit{grf}, and is even comparable to that of oracle \textit{qrf} and \textit{grf}. It agrees with our observation in the one-dimensional example (Figure \ref{fig:aft_1d} and \ref{fig:aft_1d_box}). While estimating larger quantiles, the true $\tau$-th quantile of $T_i$ and $Y_i$ are close, and hence the performance of all five models are similar. But when $\tau$ is small, the $\tau$-th quantile of $T_i$ and $Y_i$ are different because of the censoring, and in this case, our model has superior advantage and find the true quantiles of $T_i$ almost as good as the oracle methods.

\begin{figure}[!htb]
    \small
    \centering
    \begin{subfigure}[b]{0.18\linewidth}
        \includegraphics[width=\textwidth]{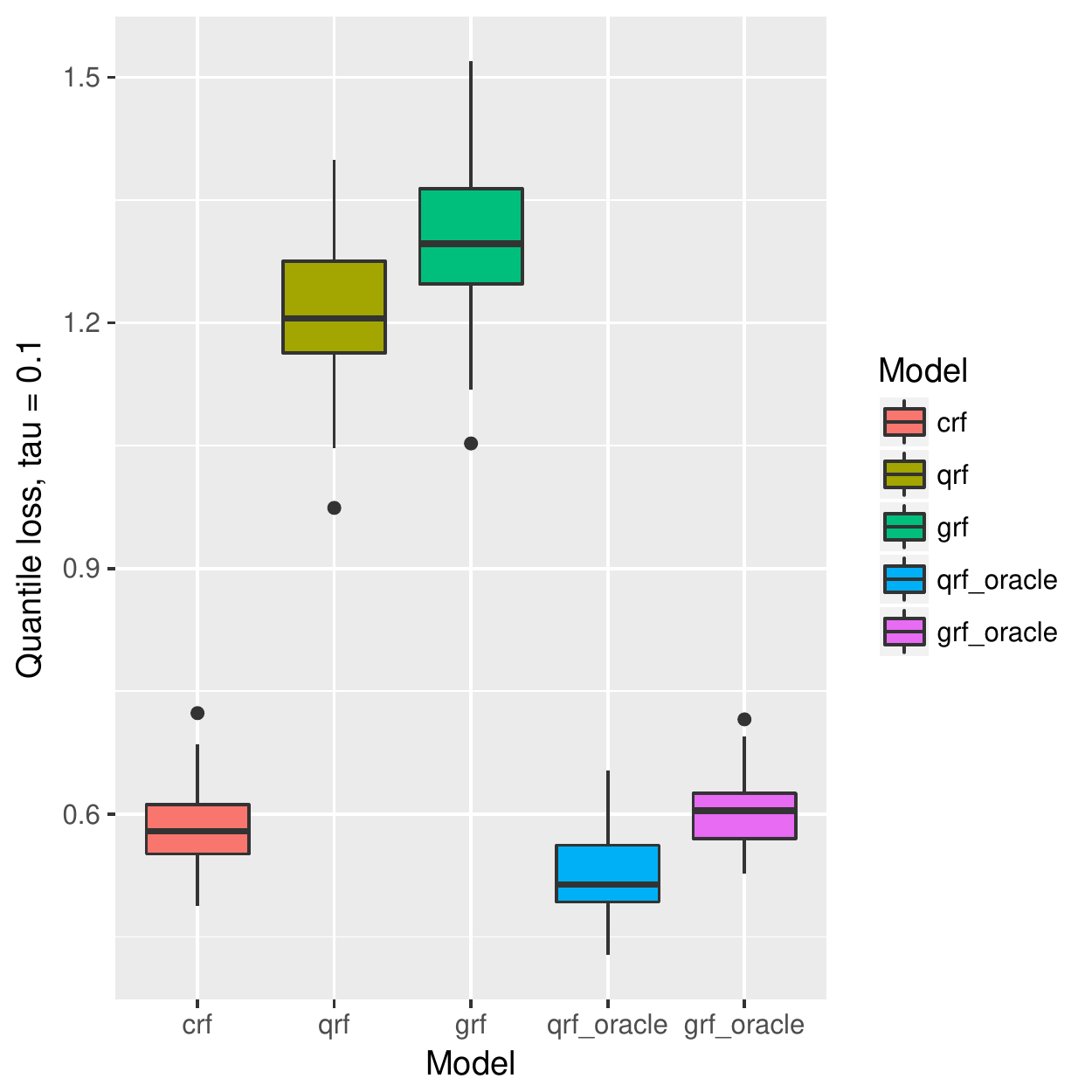}
        \caption{BostonHousing: $\tau = 0.1$}
    \end{subfigure}
    ~
    \begin{subfigure}[b]{0.18\linewidth}
        \includegraphics[width=\textwidth]{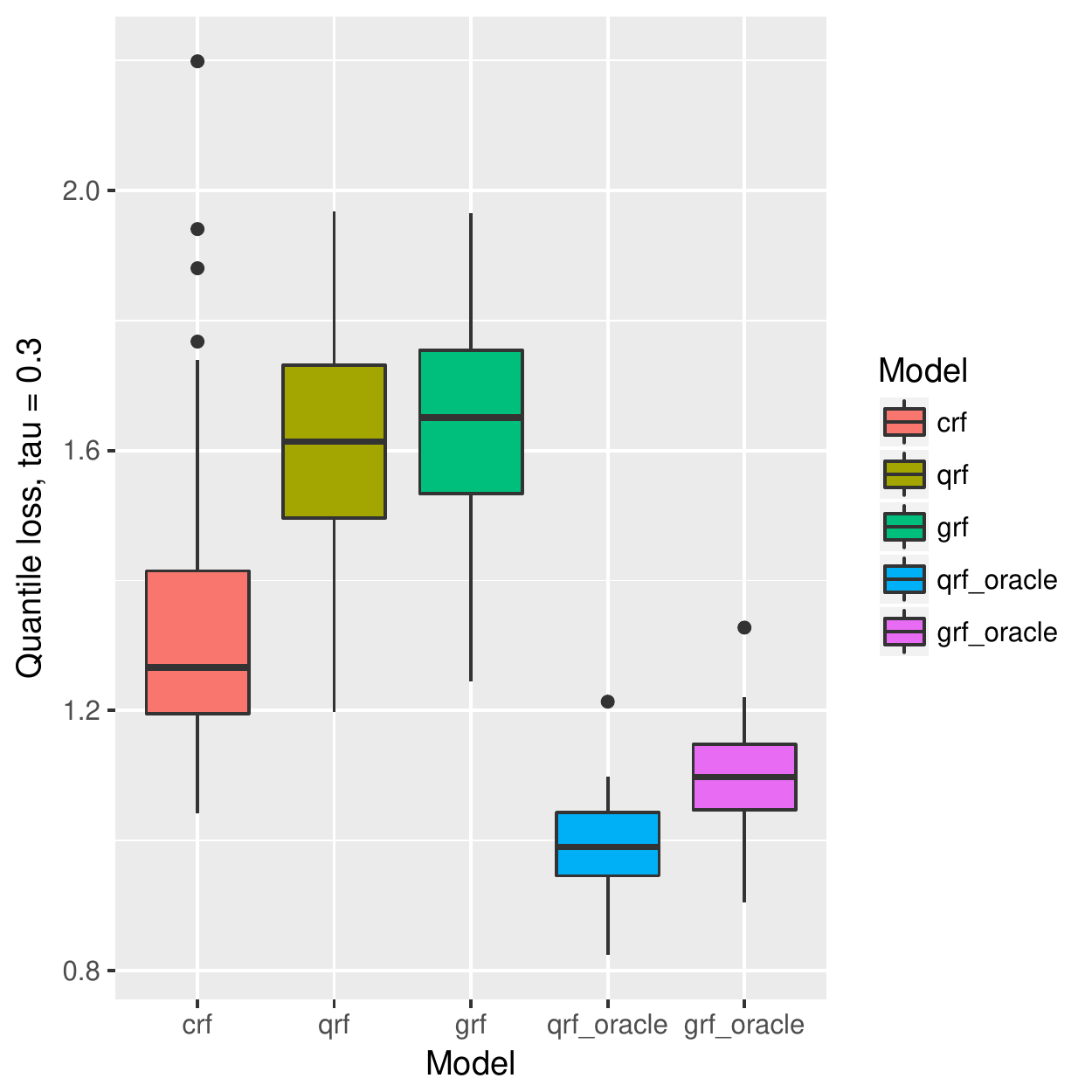}
        \caption{BostonHousing: $\tau = 0.3$}
    \end{subfigure}
    ~ %add desired spacing between images, e. g. ~, \quad, \qquad, \hfill etc. 
    \begin{subfigure}[b]{0.18\linewidth}
        \includegraphics[width=\textwidth]{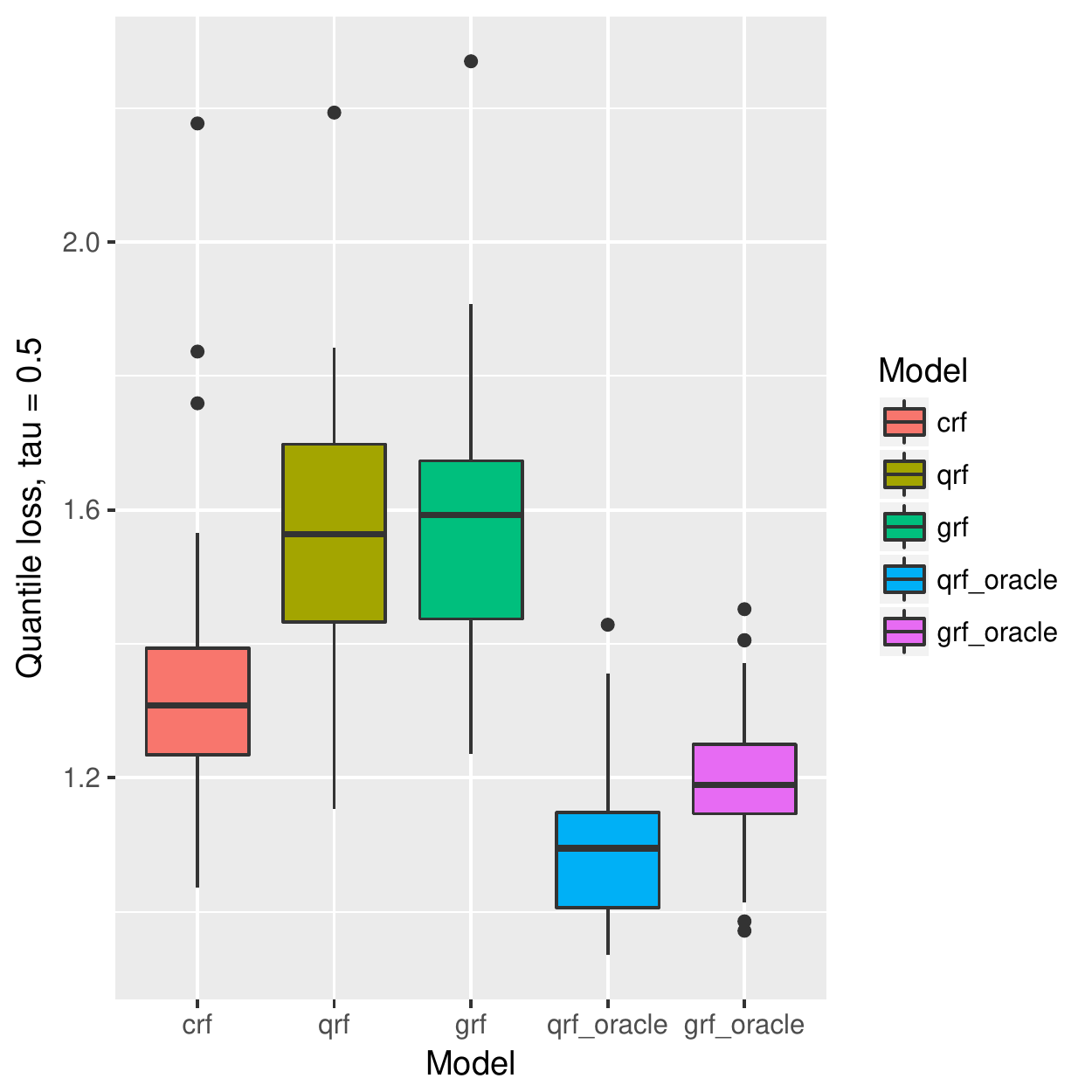}
        \caption{BostonHousing: $\tau = 0.5$}
    \end{subfigure}
    ~
    \begin{subfigure}[b]{0.18\linewidth}
        \includegraphics[width=\textwidth]{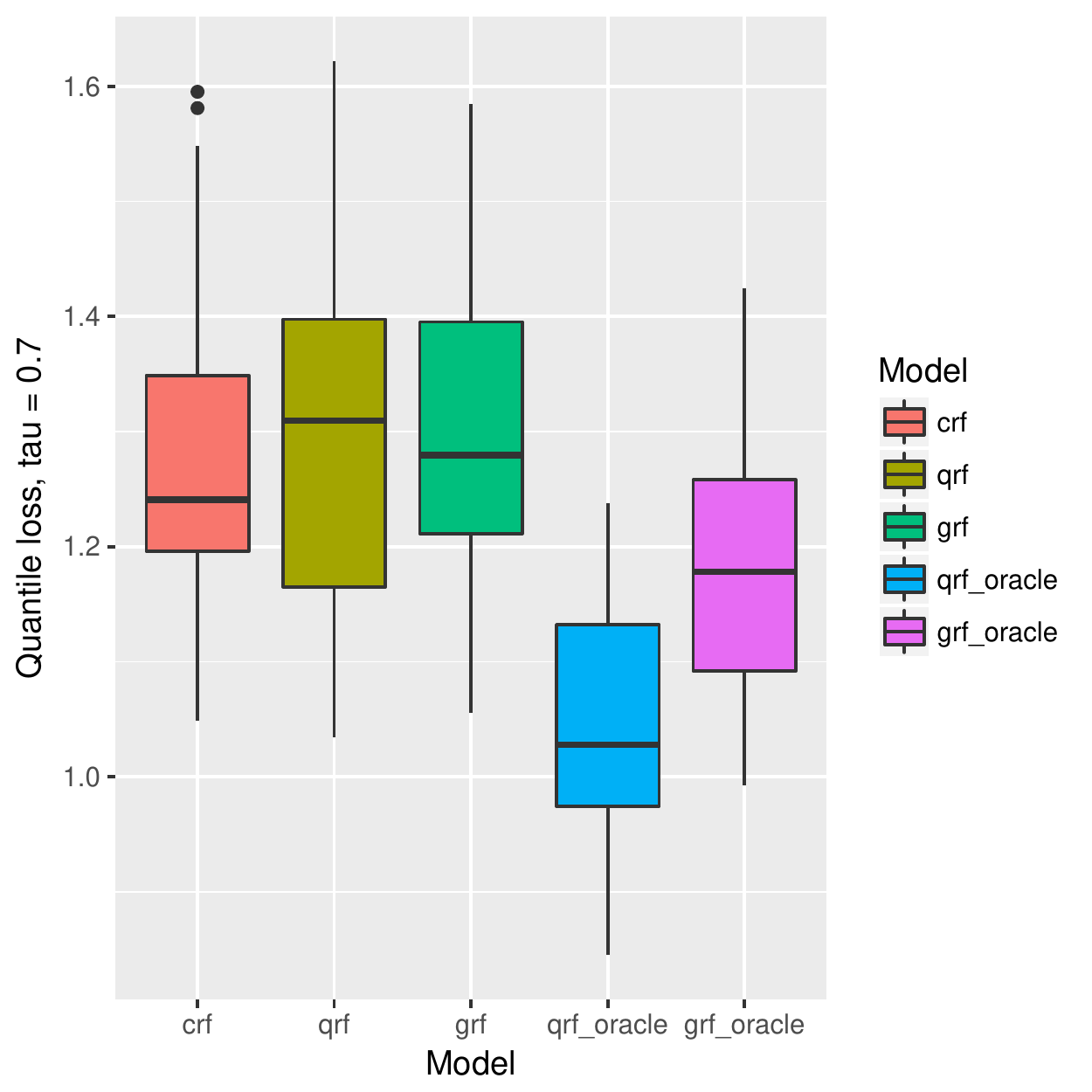}
        \caption{BostonHousing: $\tau = 0.7$}
    \end{subfigure}
    ~
    \begin{subfigure}[b]{0.18\linewidth}
        \includegraphics[width=\textwidth]{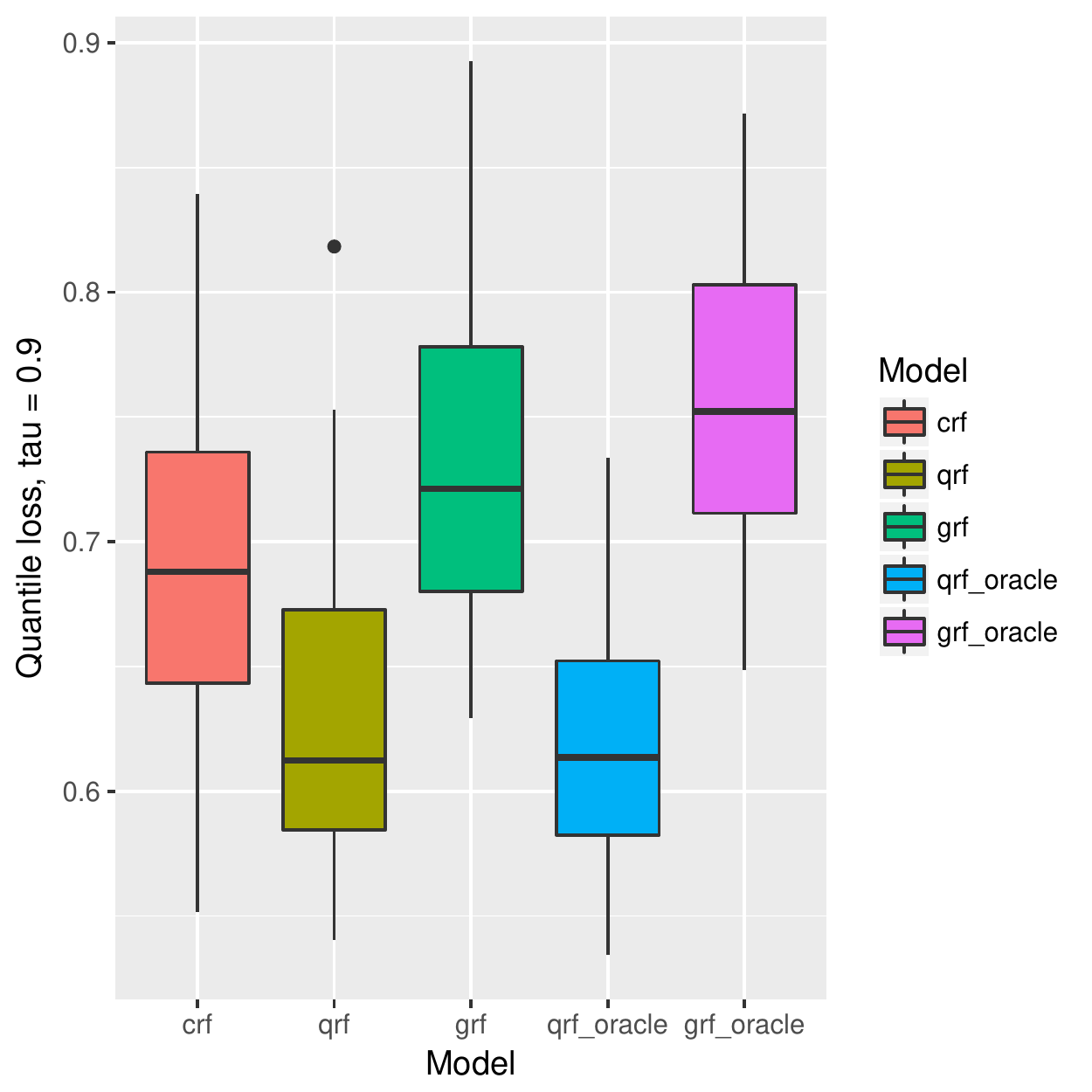}
        \caption{BostonHousing: $\tau = 0.9$}
    \end{subfigure}

    \vspace{-0.05in}
    
    \begin{subfigure}[b]{0.18\linewidth}
        \includegraphics[width=\textwidth]{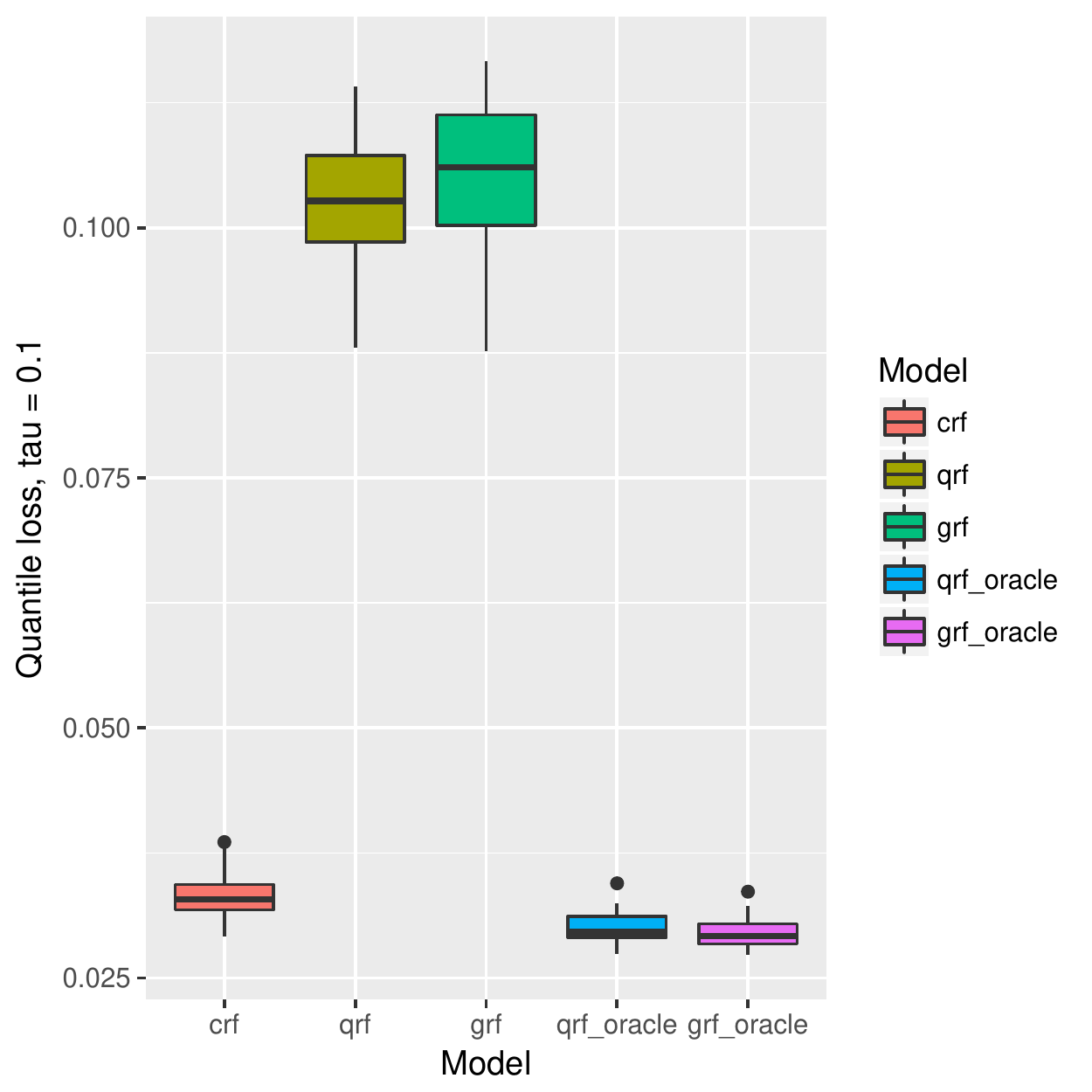}
        \caption{Abalone: $\tau = 0.1$}
    \end{subfigure}
    ~
    \begin{subfigure}[b]{0.18\linewidth}
        \includegraphics[width=\textwidth]{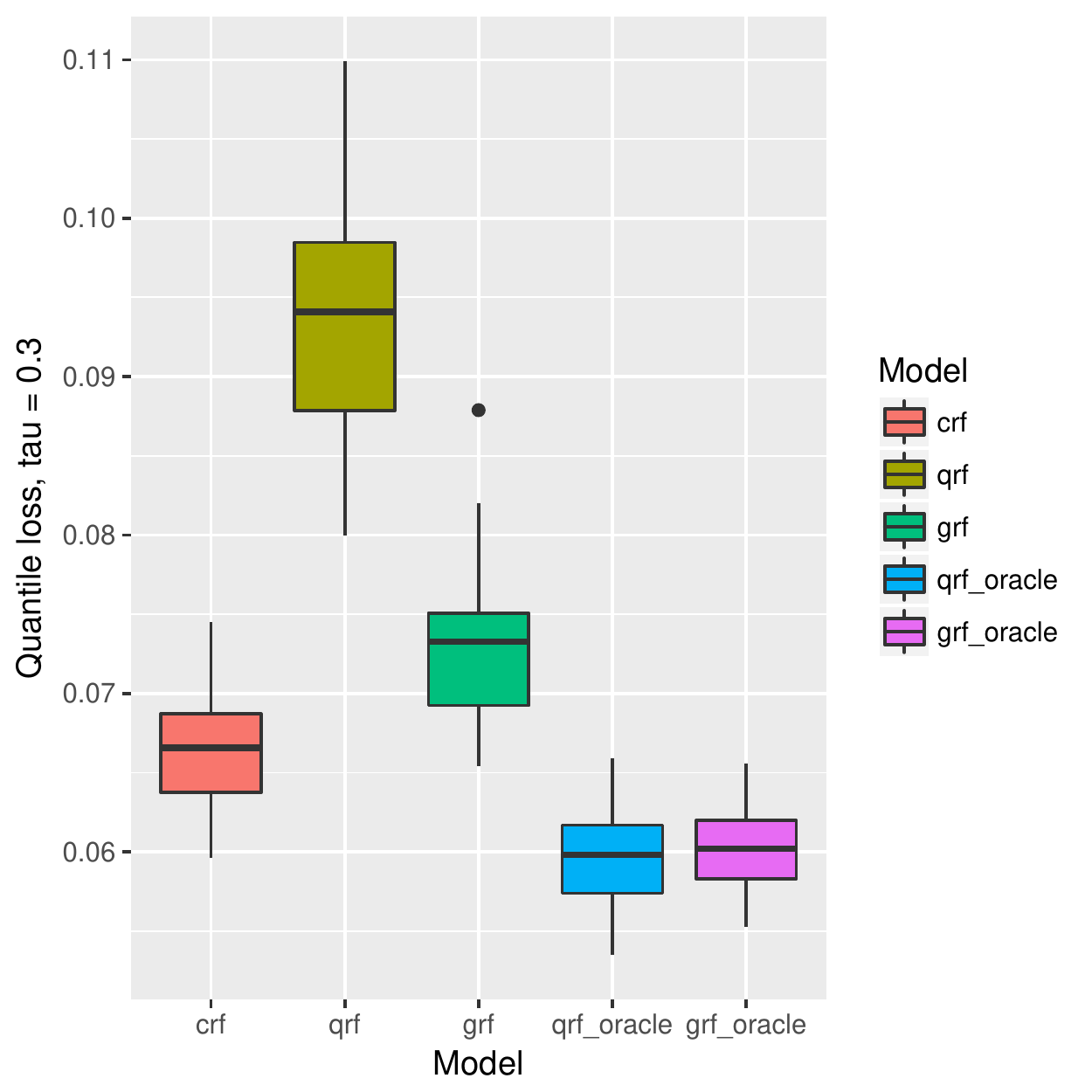}
        \caption{Abalone: $\tau = 0.3$}
    \end{subfigure}
    ~ %add desired spacing between images, e. g. ~, \quad, \qquad, \hfill etc. 
    \begin{subfigure}[b]{0.18\linewidth}
        \includegraphics[width=\textwidth]{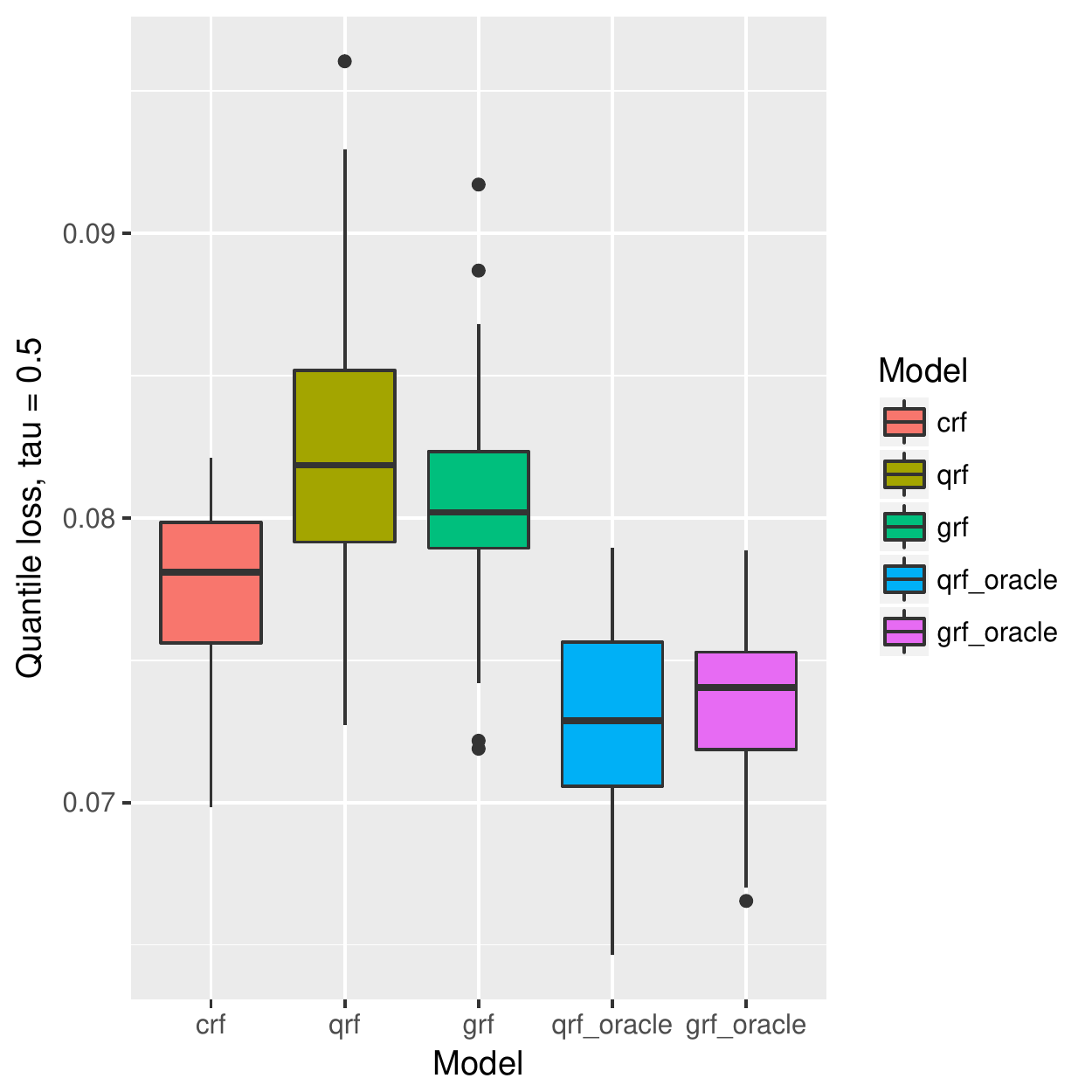}
        \caption{Abalone: $\tau = 0.5$}
    \end{subfigure}
    ~
    \begin{subfigure}[b]{0.18\linewidth}
        \includegraphics[width=\textwidth]{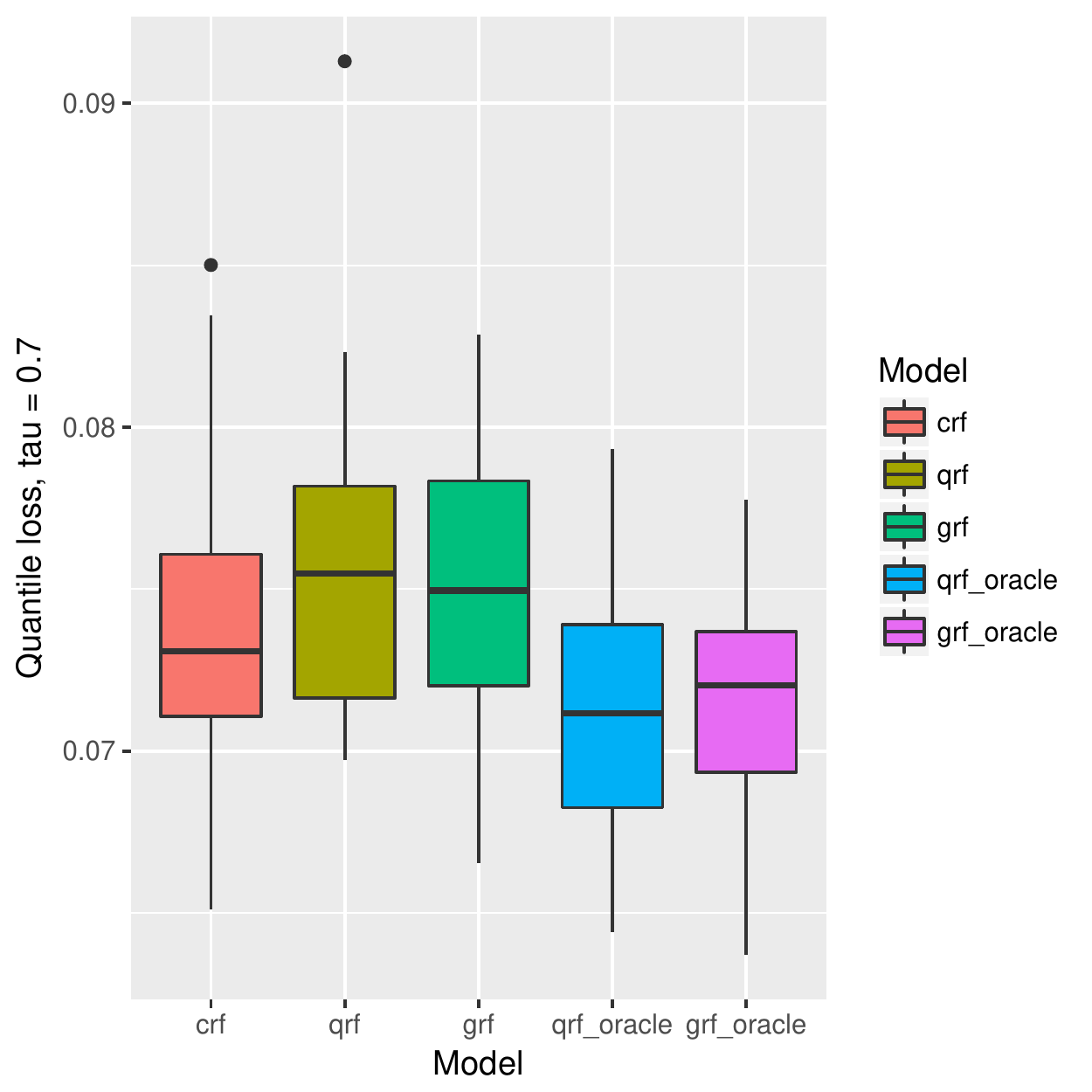}
        \caption{Abalone: $\tau = 0.7$}
    \end{subfigure}
    ~
    \begin{subfigure}[b]{0.18\linewidth}
        \includegraphics[width=\textwidth]{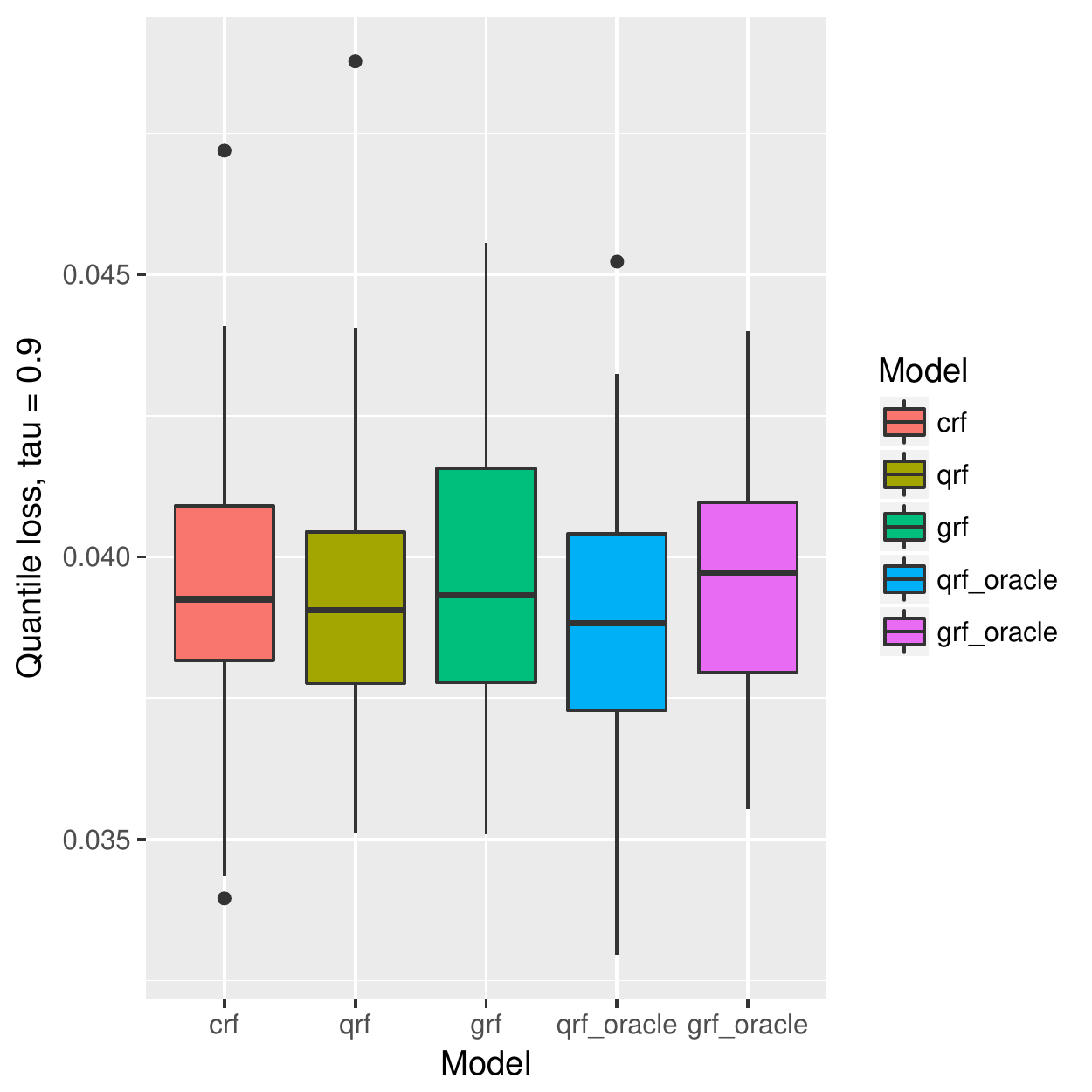}
        \caption{Abalone: $\tau = 0.9$}
    \end{subfigure}

    \vspace{-0.05in}
    
    \begin{subfigure}[b]{0.18\linewidth}
        \includegraphics[width=\textwidth]{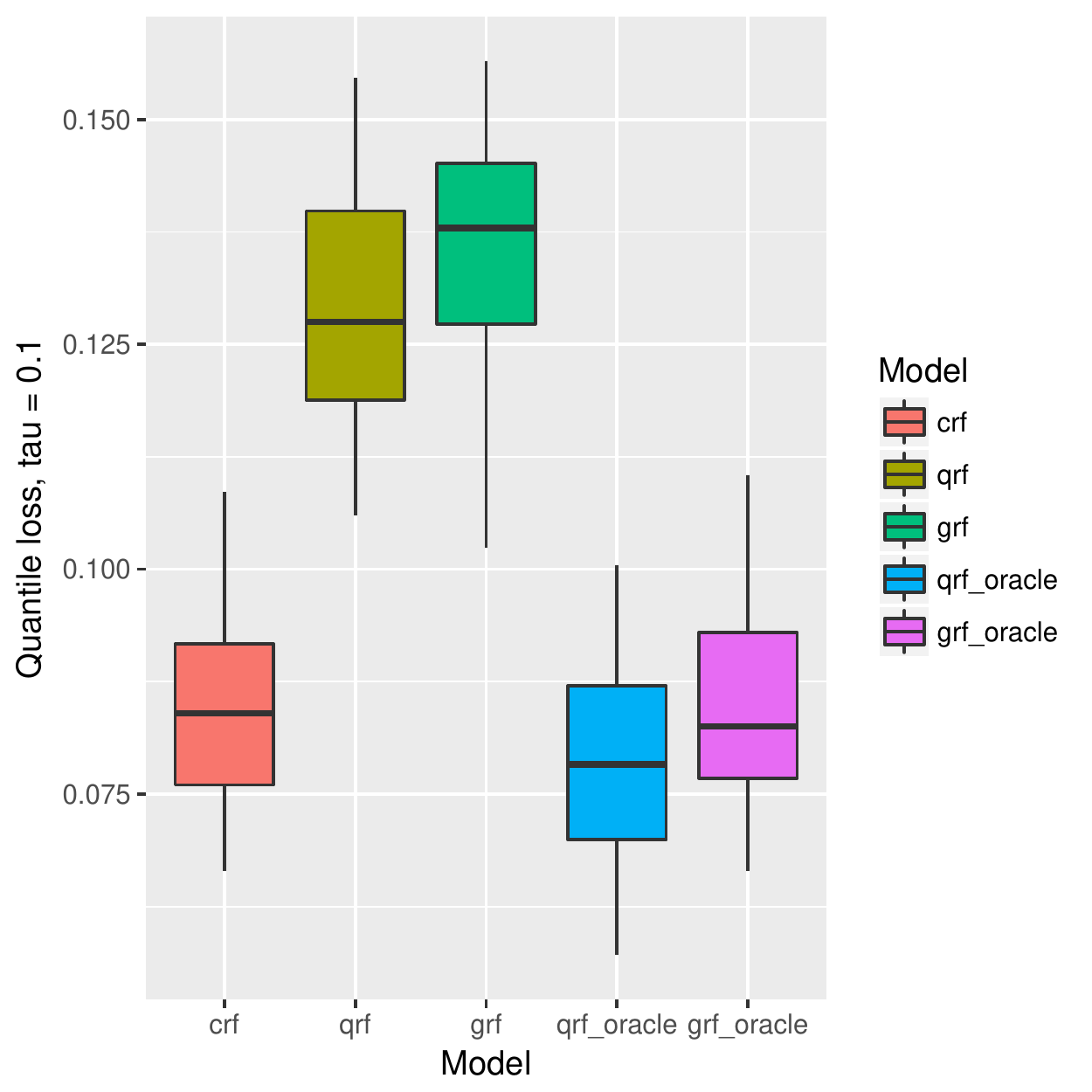}
        \caption{Ozone: $\tau = 0.1$}
    \end{subfigure}
    ~
    \begin{subfigure}[b]{0.18\linewidth}
        \includegraphics[width=\textwidth]{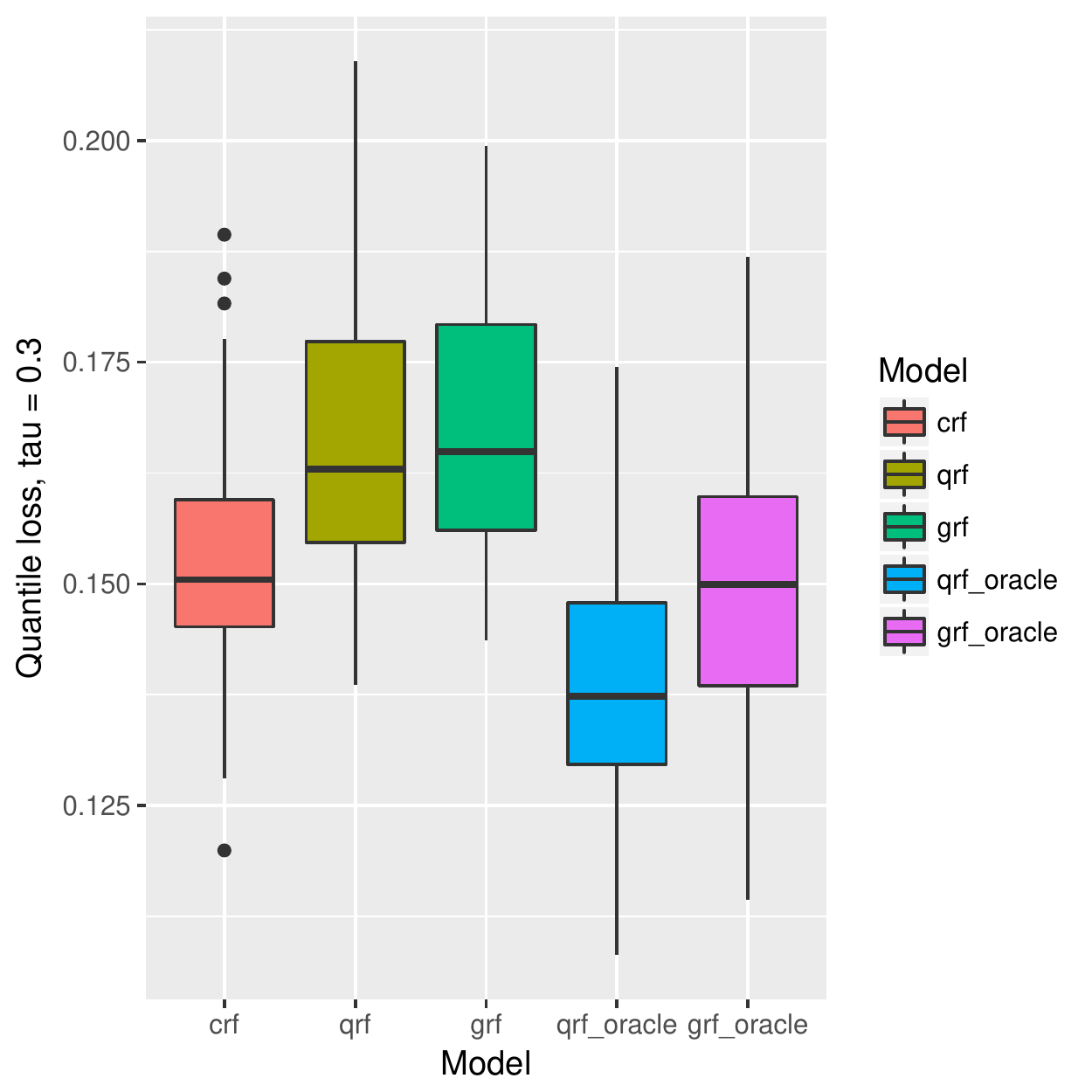}
        \caption{Ozone: $\tau = 0.3$}
    \end{subfigure}
    ~ %add desired spacing between images, e. g. ~, \quad, \qquad, \hfill etc. 
    \begin{subfigure}[b]{0.18\linewidth}
        \includegraphics[width=\textwidth]{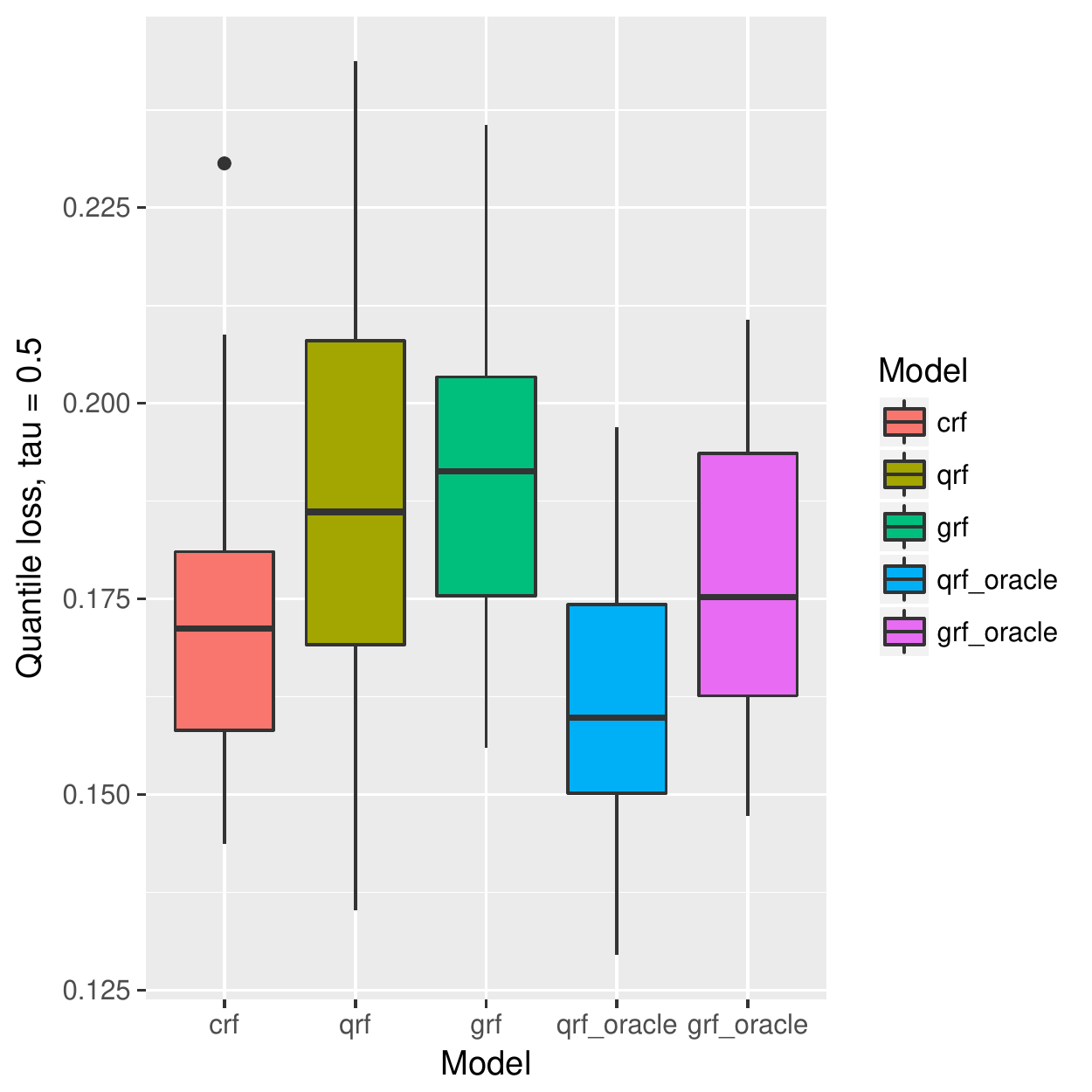}
        \caption{Ozone: $\tau = 0.5$}
    \end{subfigure}
    ~
    \begin{subfigure}[b]{0.18\linewidth}
        \includegraphics[width=\textwidth]{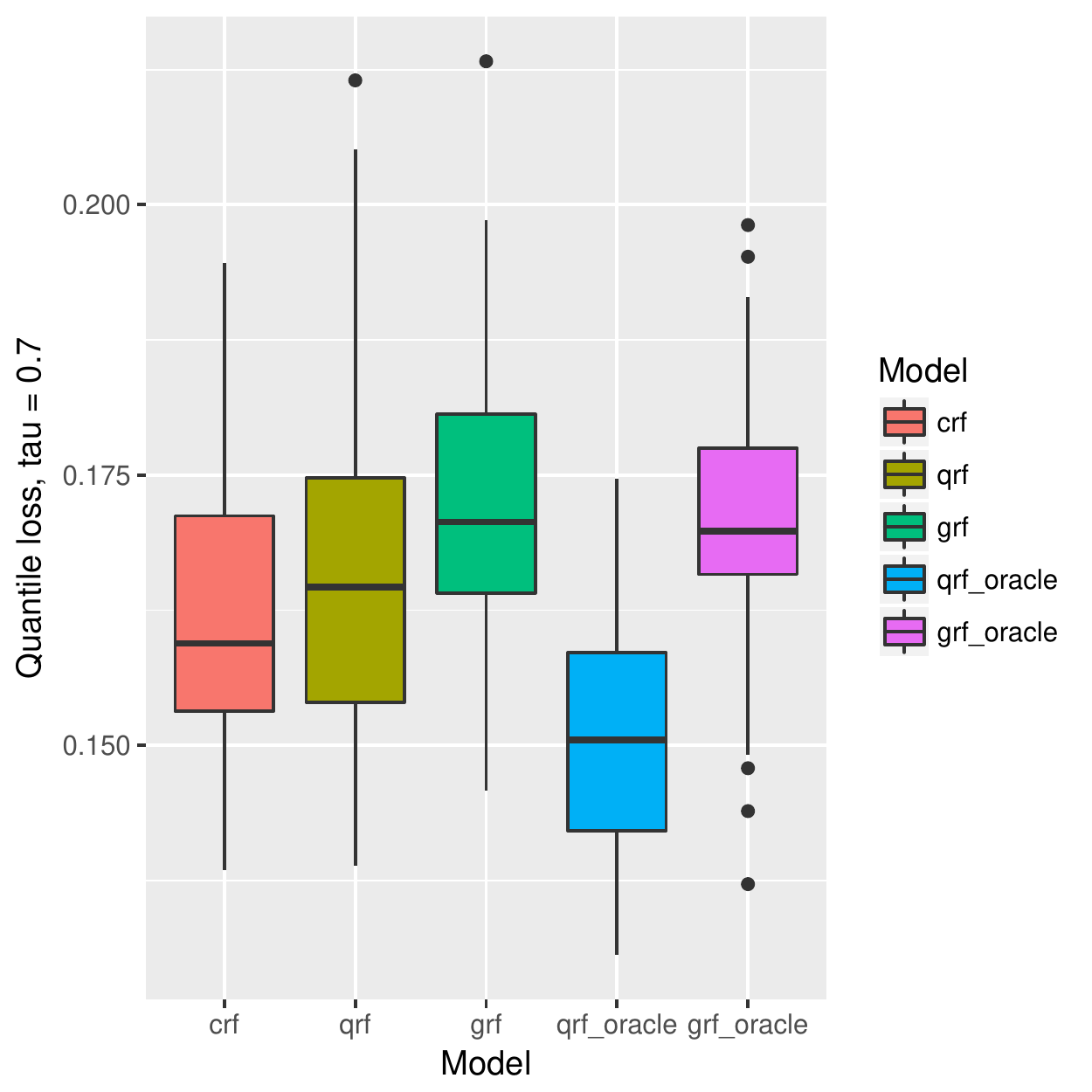}
        \caption{Ozone: $\tau = 0.7$}
    \end{subfigure}
    ~
    \begin{subfigure}[b]{0.18\linewidth}
        \includegraphics[width=\textwidth]{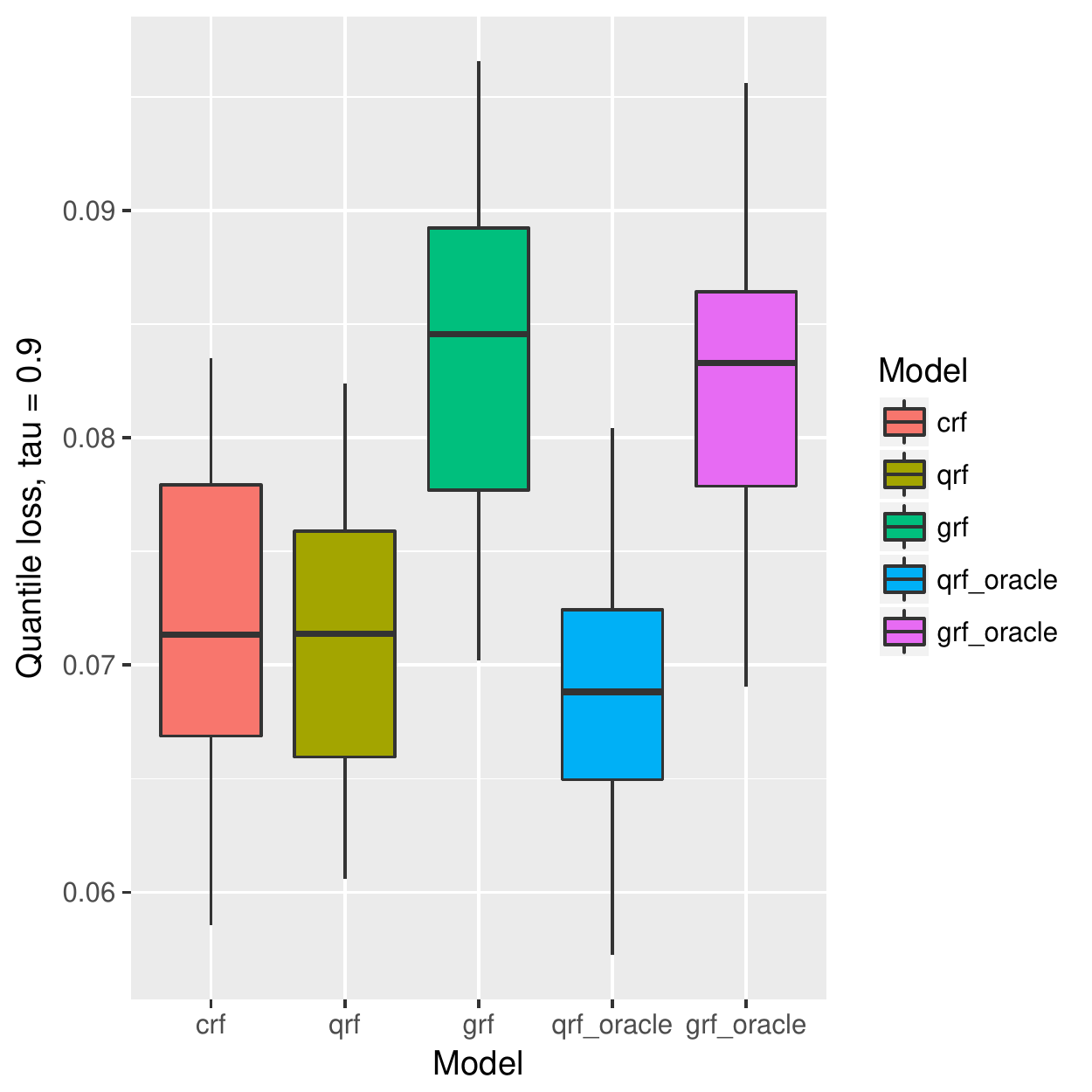}
        \caption{Ozone: $\tau = 0.9$}
    \end{subfigure}
    
    \caption{Quantile losses evaluated on a hold-out set of the three real data sets: each one presented in a corresponding row. Censoring level is approximately 20\%. We compare our method with {\it grf }(green) and {\it qrf }(olive) and their corresponding oracle versions, {\it grf-oracle }(purple) and {\it qrf-oracle} (blue).}
    \label{fig:real_data}
\end{figure}

\begin{figure}[!htb]
    \small
    \centering
    \begin{subfigure}[b]{0.3\linewidth}
        \includegraphics[width=\textwidth]{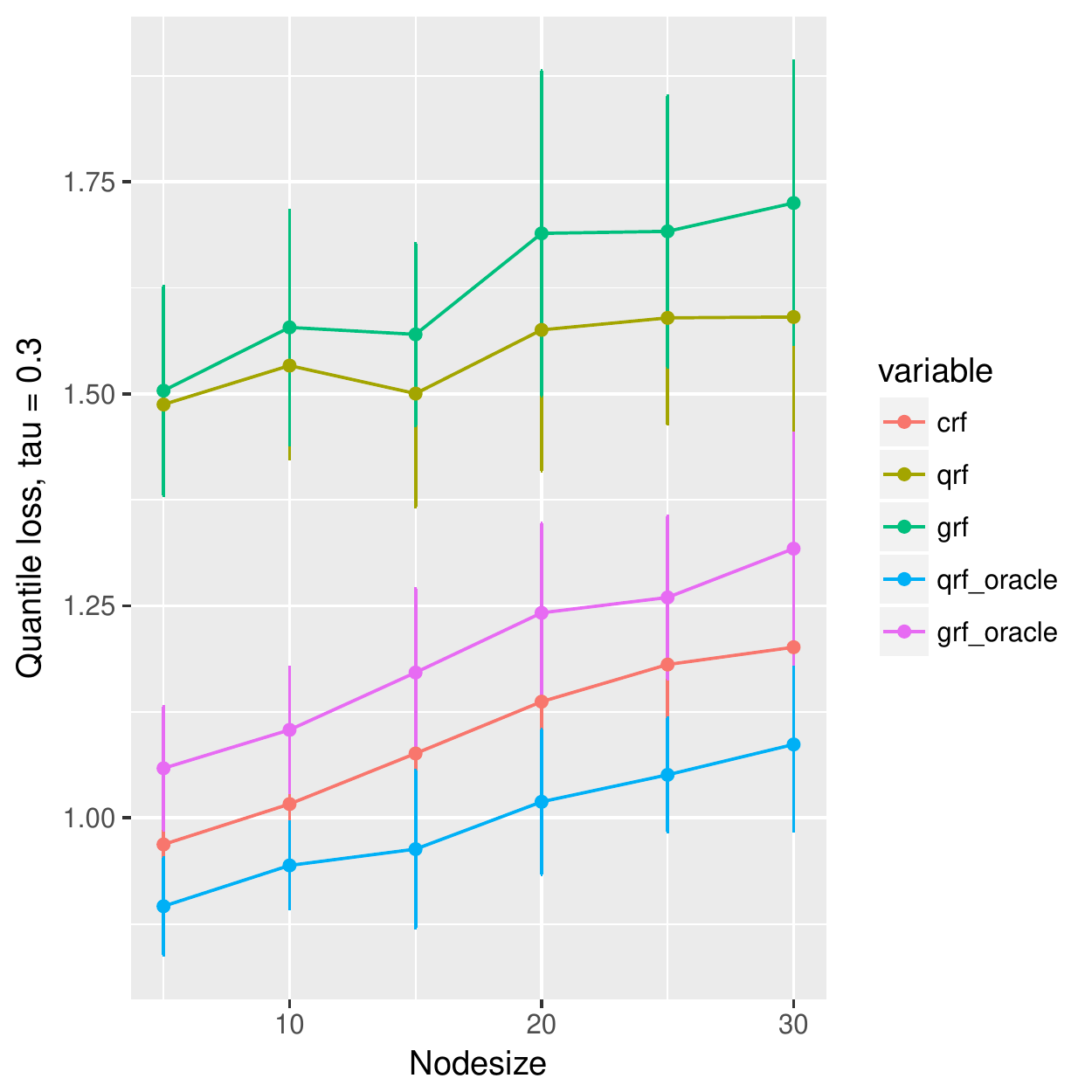}
        \caption{BostonHousing: $\tau = 0.3$}
    \end{subfigure}
    ~
    \begin{subfigure}[b]{0.3\linewidth}
        \includegraphics[width=\textwidth]{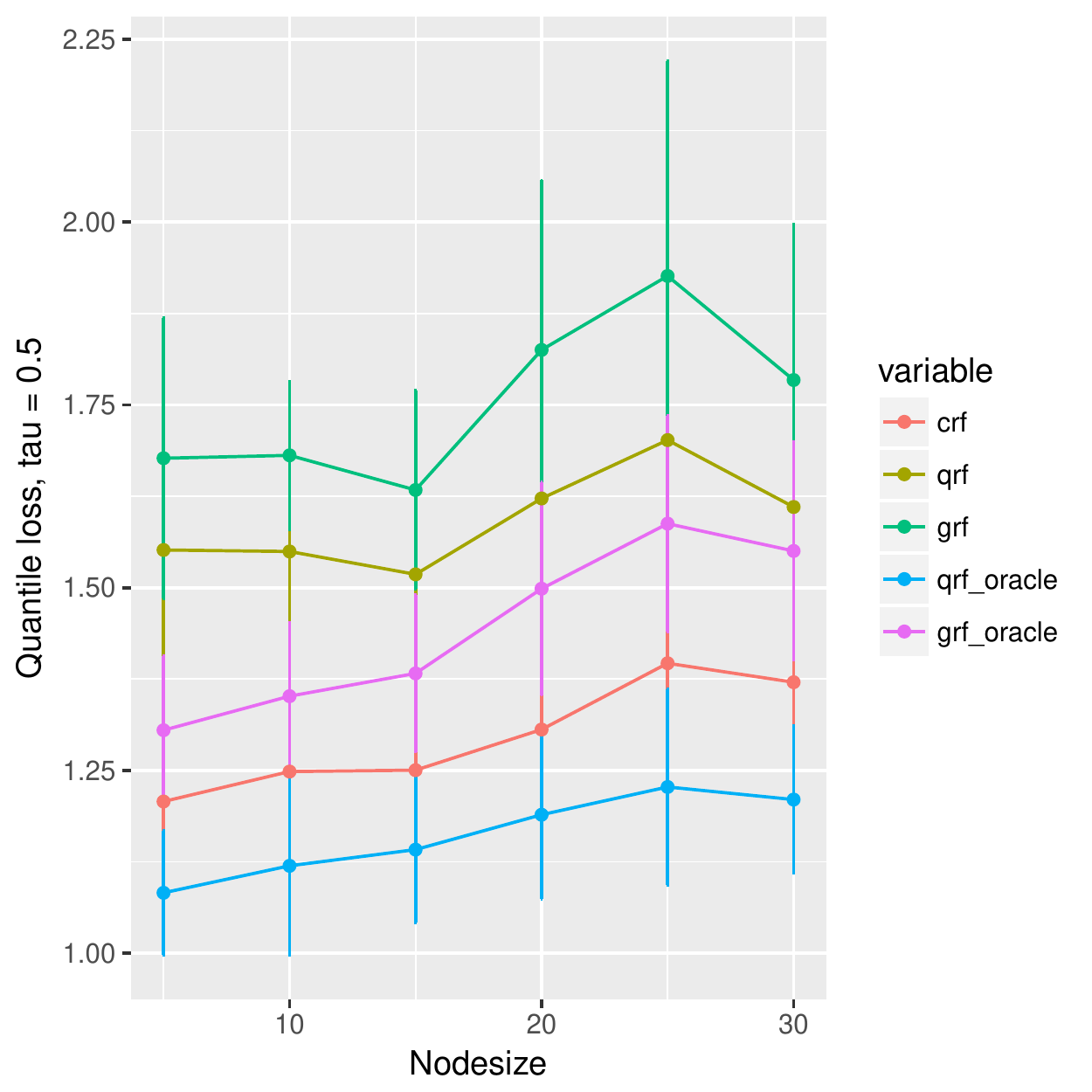}
        \caption{BostonHousing: $\tau = 0.5$}
    \end{subfigure}
    ~
    \begin{subfigure}[b]{0.3\linewidth}
        \includegraphics[width=\textwidth]{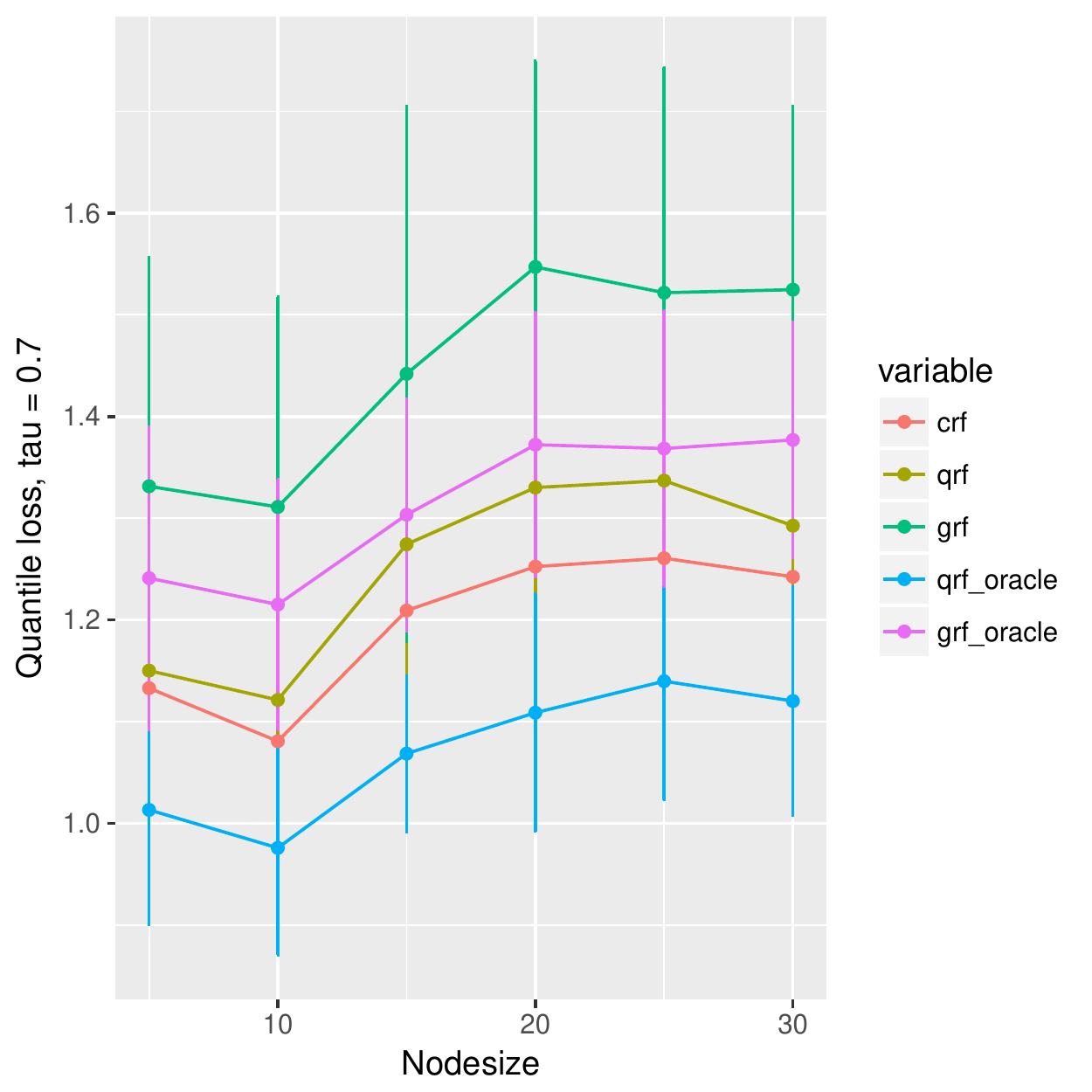}
        \caption{BostonHousing: $\tau = 0.7$}
    \end{subfigure}
    ~
    
    \vspace{-0.05in}
    
    \begin{subfigure}[b]{0.3\linewidth}
        \includegraphics[width=\textwidth]{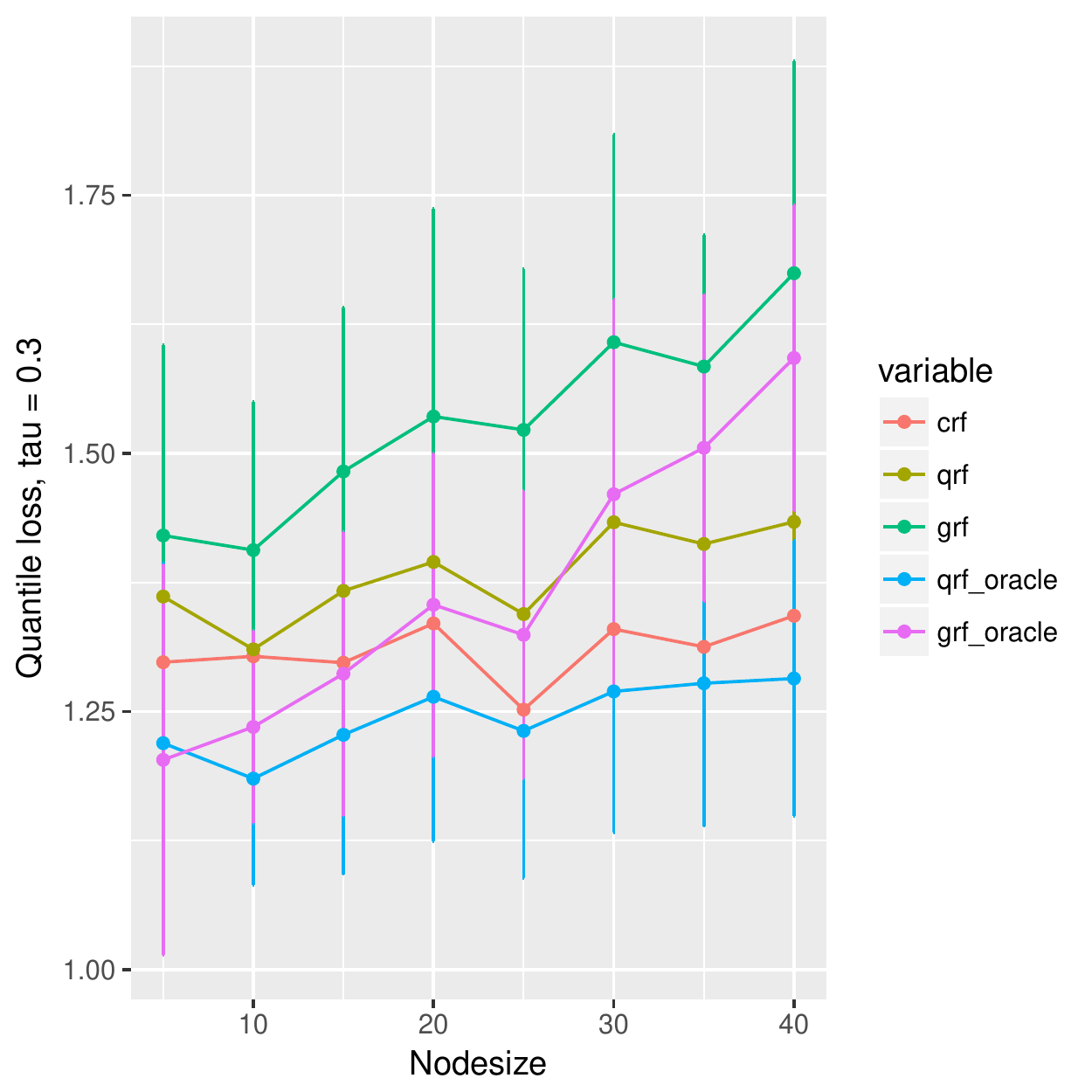}
        \caption{Ozone: $\tau = 0.3$}
    \end{subfigure}
    ~
    \begin{subfigure}[b]{0.3\linewidth}
        \includegraphics[width=\textwidth]{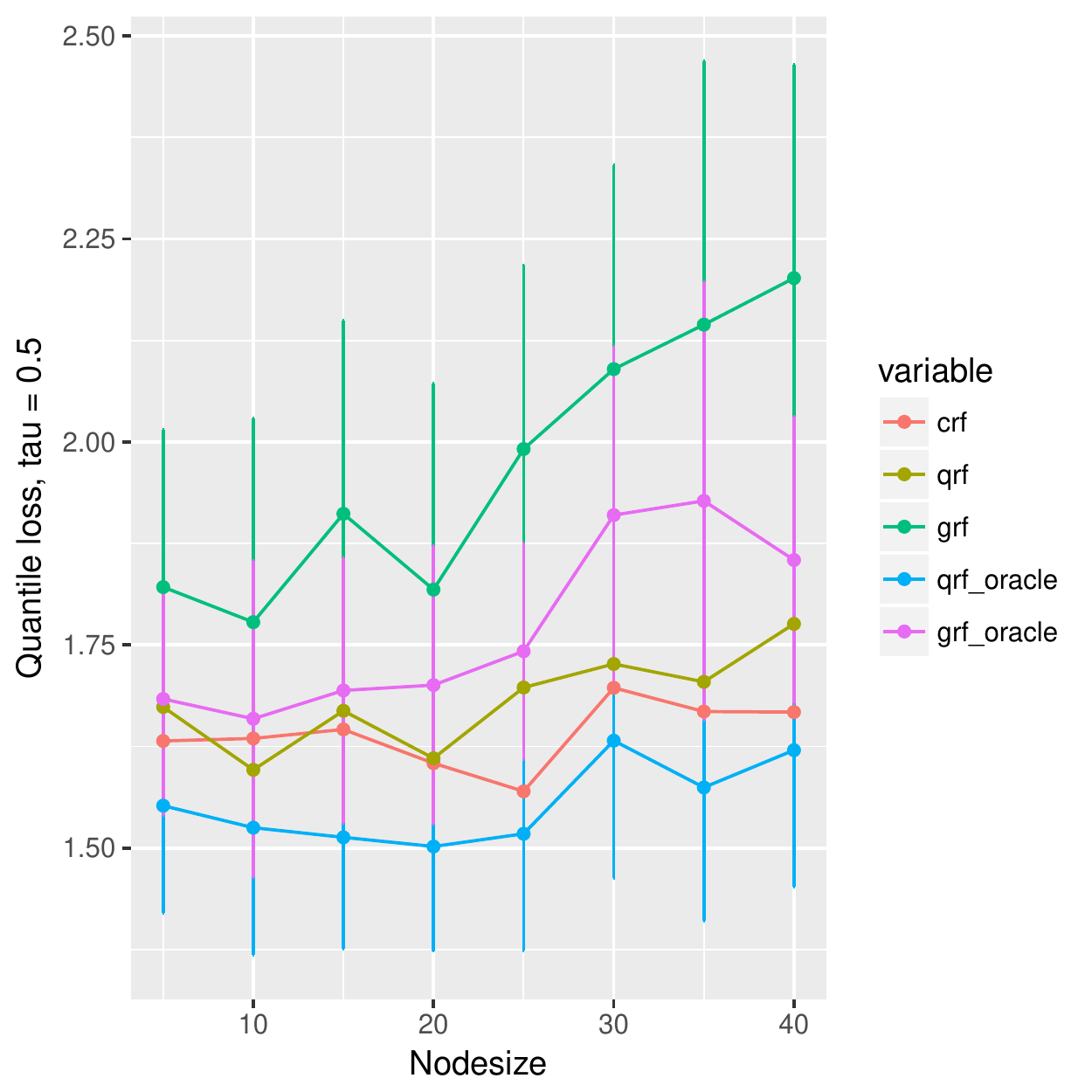}
        \caption{Ozone: $\tau = 0.5$}
    \end{subfigure}
    ~
    \begin{subfigure}[b]{0.3\linewidth}
        \includegraphics[width=\textwidth]{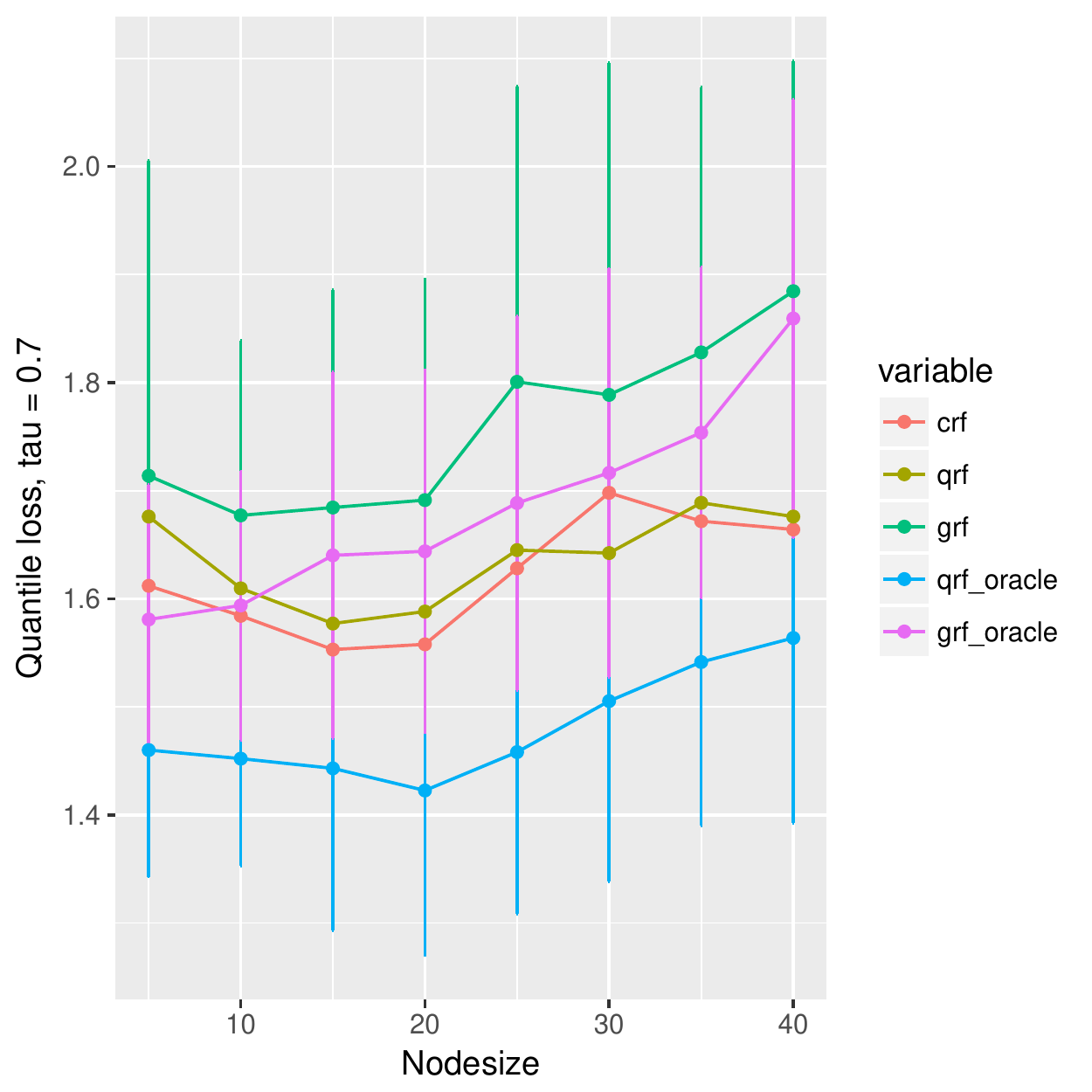}
        \caption{Ozone: $\tau = 0.7$}
    \end{subfigure}
    ~
    
    \vspace{-0.05in}
    
    \begin{subfigure}[b]{0.3\linewidth}
        \includegraphics[width=\textwidth]{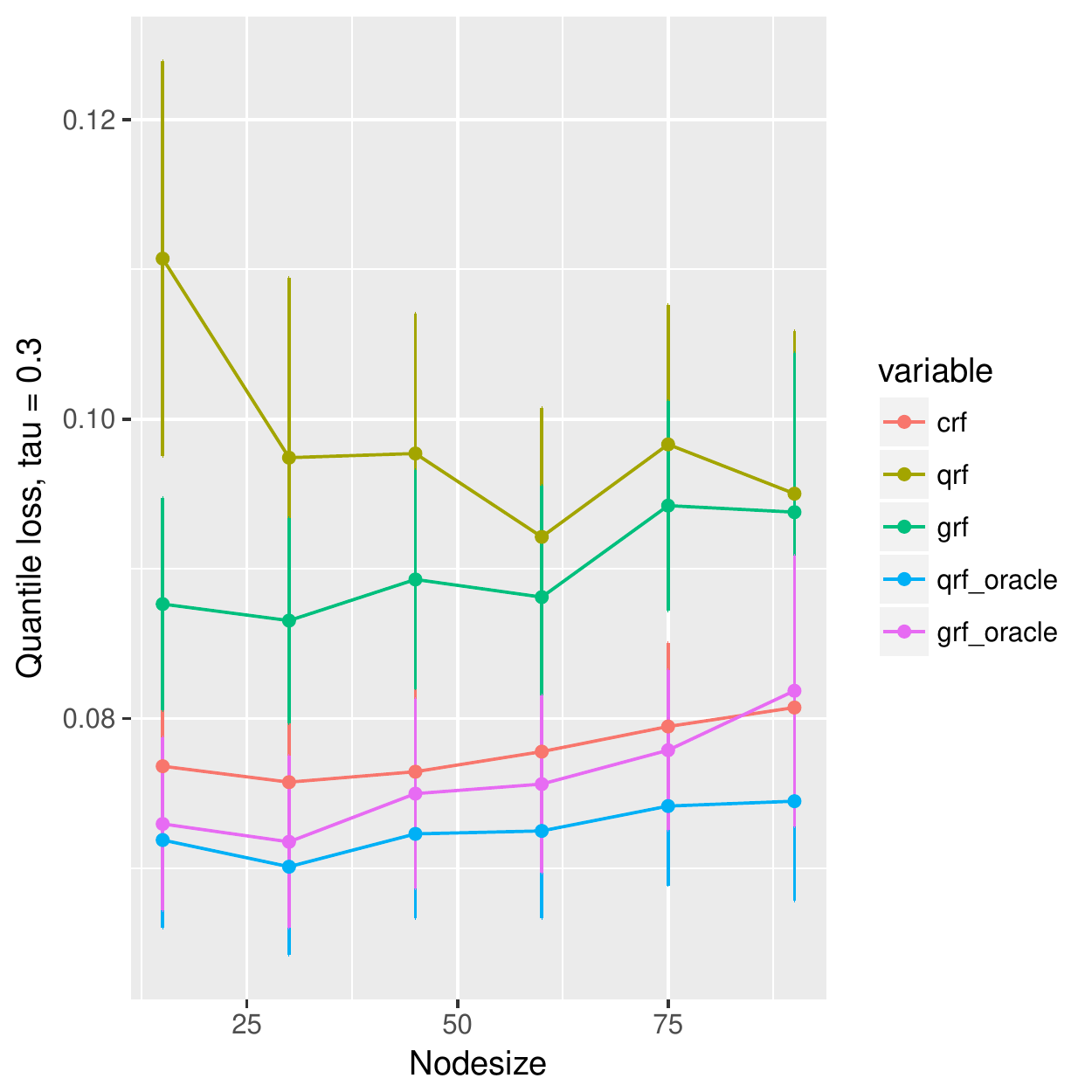}
        \caption{Abalone: $\tau = 0.3$}
    \end{subfigure}
    ~
    \begin{subfigure}[b]{0.3\linewidth}
        \includegraphics[width=\textwidth]{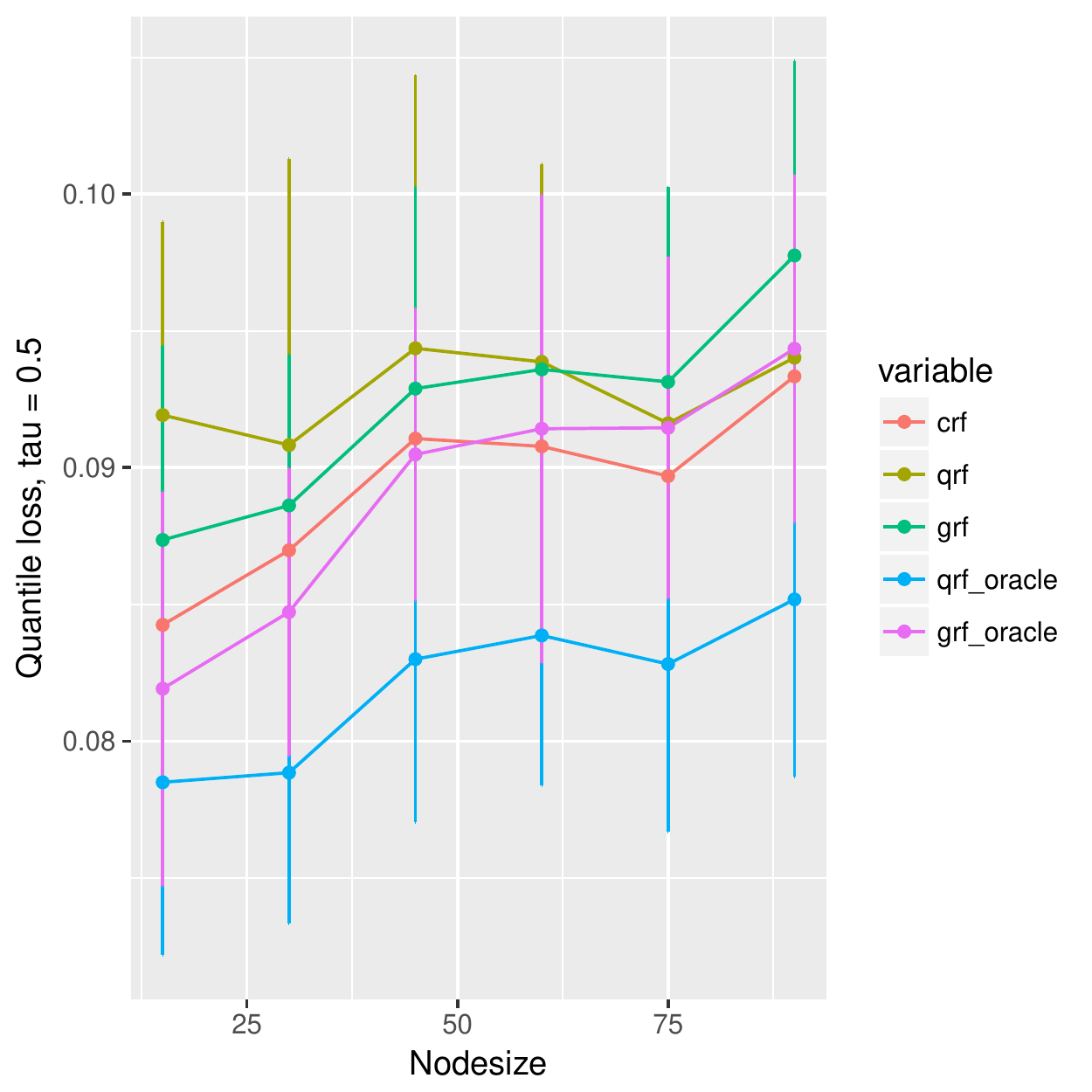}
        \caption{Abalone: $\tau = 0.5$}
    \end{subfigure}
    ~
    \begin{subfigure}[b]{0.3\linewidth}
        \includegraphics[width=\textwidth]{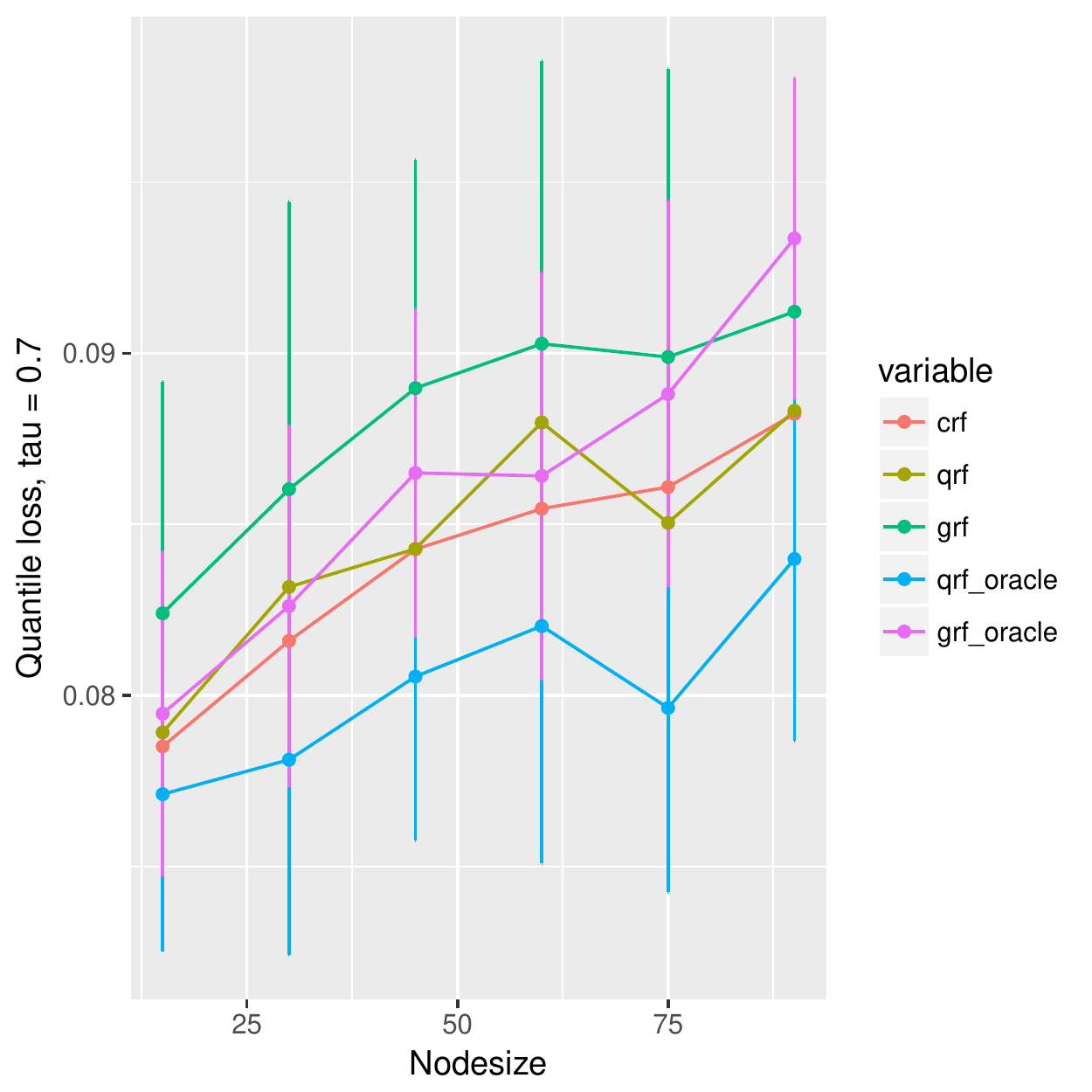}
        \caption{Abalone: $\tau = 0.7$}
    \end{subfigure}
    
    \caption{Hold-out Quantile losses  with different node sizes on the three real data sets: each one represented in one row. }
    \label{fig:real_data_nodesize}
\end{figure}

\paragraph{Nodesize}
For each dataset, we train different models using different nodesizes and compare the performance. For each node size, we bootstrap the data and repeat the experiments for 20 times, and we calculate the mean and standard deviation of the quantile predictions for quantiles $\tau = 0.3$, $0.5$, and $0.7$. The result is in Figure \ref{fig:real_data_nodesize}. We observe that our method, \textit{crf}, is uniformly better than \textit{qrf} and \textit{grf}, proving that \textit{crf} is able to correct the bias introduced by censoring. Moreover, the quantile loss of \textit{crf} is always competitive to that of \textit{qrf-oracle} and \textit{grf-oracle}, and is remarkably  always better than \textit{grf-oracle}, only slightly worse than \textit{qrf-oracle}.

\subsection{Prediction Intervals}
All the forest methods can be used to get $95\%$ prediction intervals by predicting the $0.025$ and $0.975$ quantiles of the true response variable. Then for any location $x \in \mathcal{X}$, a straightforward confidence interval will be $[Q(x;0.025), Q(x;0.975)]$. The result is illustrated in Figure \ref{fig:real_data_ci} for the case of univariate censored sine model. For each data set, we bootstrap the data and calculate the $0.025$ and $0.975$ quantile for the out of bag points. Then for each node size, we repeat this process for 20 times and calculate the average coverage rate of the confidence intervals.

We observe that in all of the cases, our method \textit{crf} and \textit{qrf-oracle} give the coverage closest to $95\%$. As can be seen from Figure \ref{fig:real_data_nodesize}, both \textit{qrf} and \textit{grf} perform much worse on predicting lower quantiles. They tend to under-estimate the lower quantiles and hence make the confidence intervals much wider than the true ones. Nevertheless, as seen in Figure \ref{fig:real_data_ci}, out method has great coverage that is never below $95\%$, indicating certain efficiency of the proposed method.

\begin{figure}[!htb]
    \small
    \centering
    \begin{subfigure}[b]{0.3\linewidth}
        \centering
        \includegraphics[width=\textwidth]{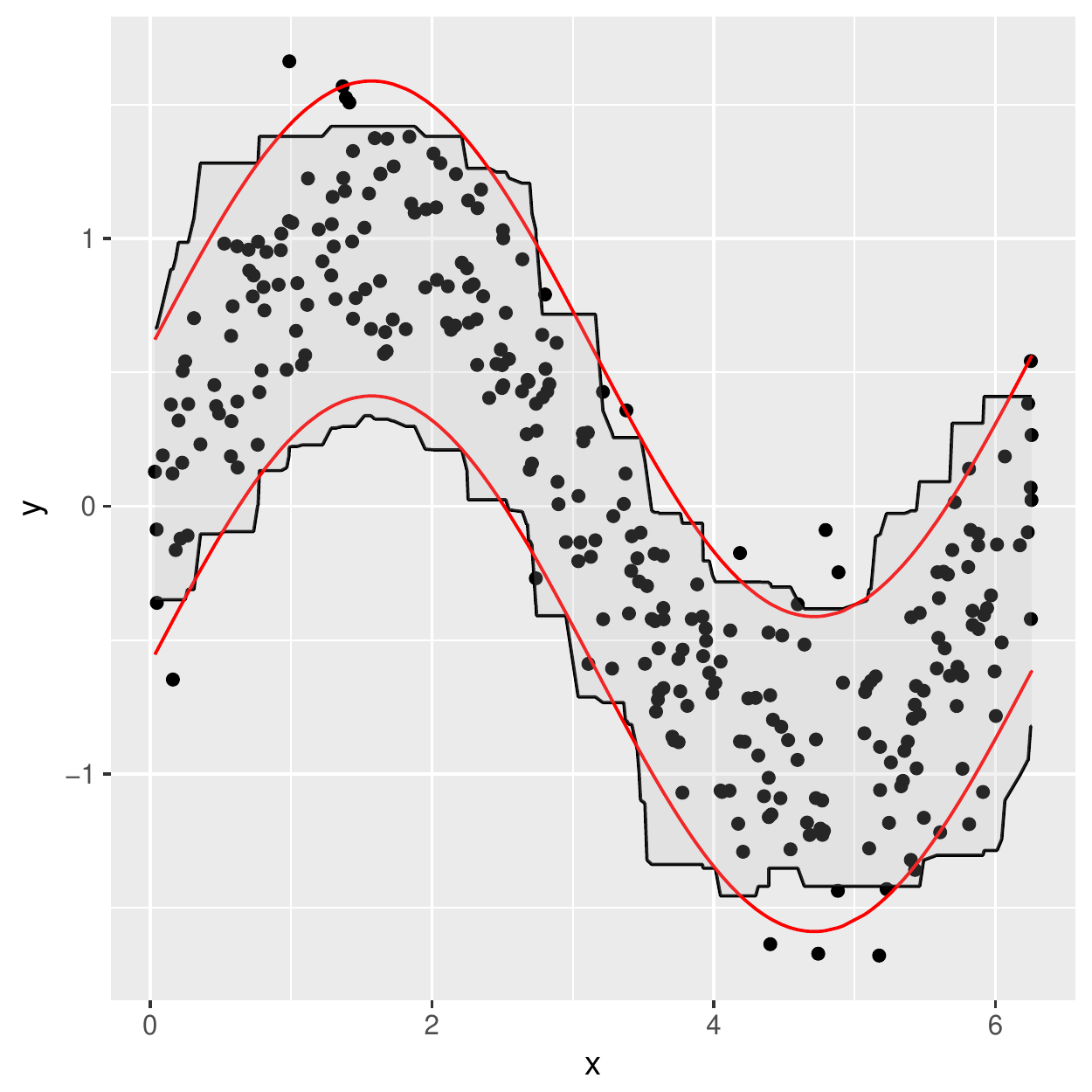}
        \caption{crf}
    \end{subfigure}
    ~
    \begin{subfigure}[b]{0.3\linewidth}
        \centering
        \includegraphics[width=\textwidth]{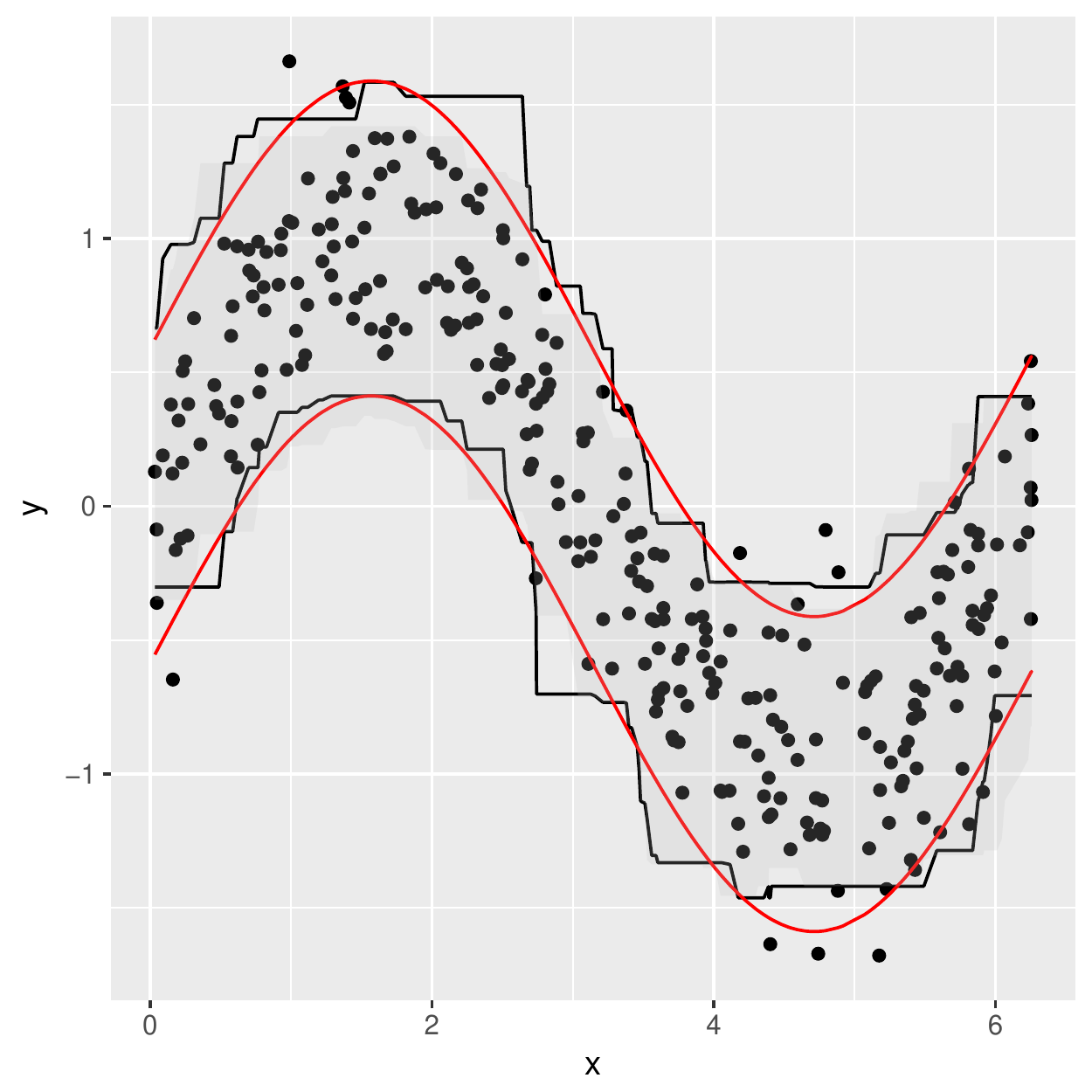}
        \caption{qrf-oracle}
    \end{subfigure}
    ~
    \begin{subfigure}[b]{0.3\linewidth}
        \centering
        \includegraphics[width=\textwidth]{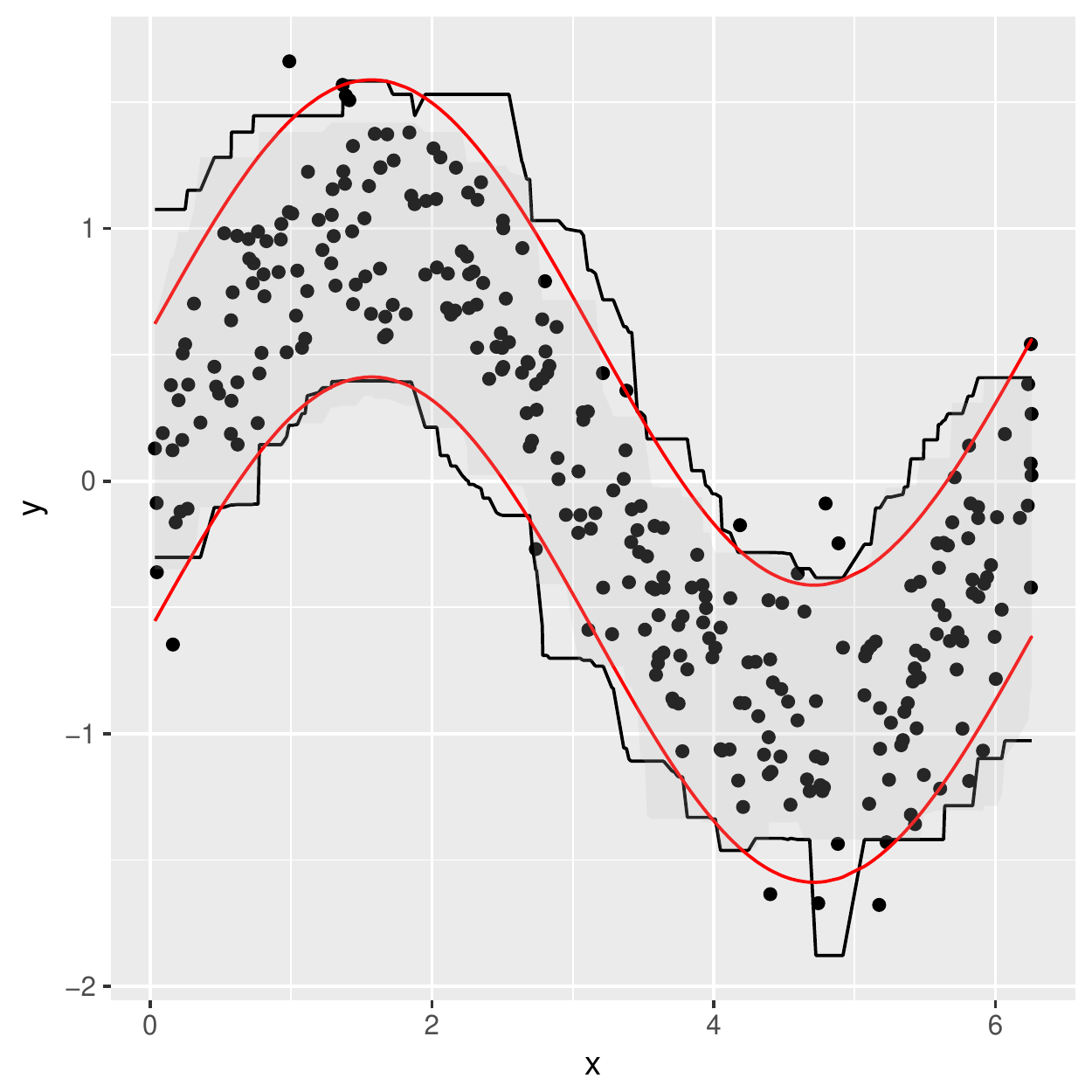}
        \caption{grf-oracle}
    \end{subfigure}
    ~
    
    \vspace{-0.05in}
    
    \begin{subfigure}[b]{0.3\linewidth}
        \centering
        \includegraphics[width=\textwidth]{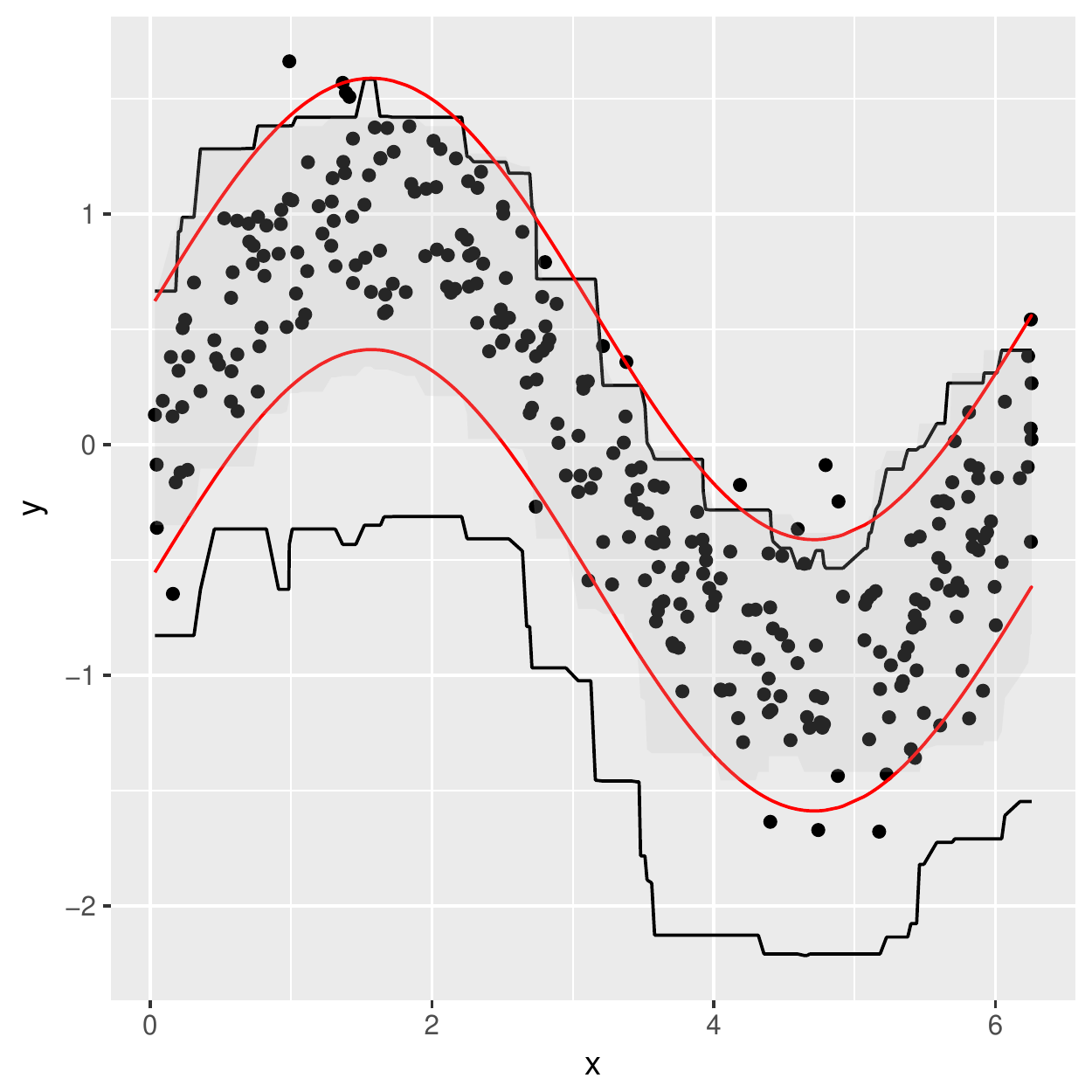}
        \caption{qrf}
    \end{subfigure}
    ~
    \begin{subfigure}[b]{0.3\linewidth}
        \centering
        \includegraphics[width=\textwidth]{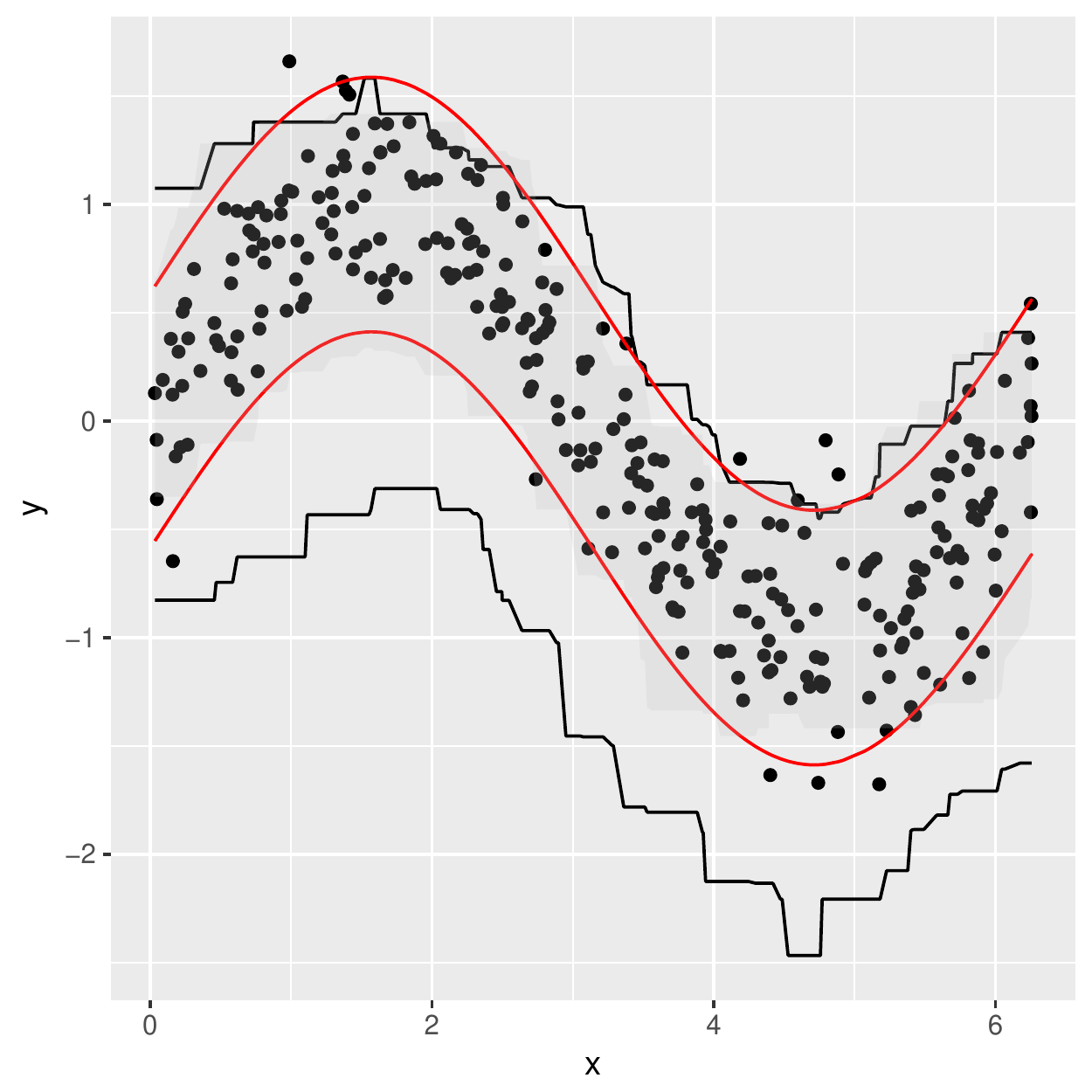}
        \caption{grf}
    \end{subfigure}
    
    \caption{Prediction intervals of the univariate censored since model.}
    \label{fig:sine_ci_plot}
\end{figure}

\begin{figure}[!htb]
    \small
    \centering
    \begin{subfigure}[b]{0.3\linewidth}
        \includegraphics[width=\textwidth]{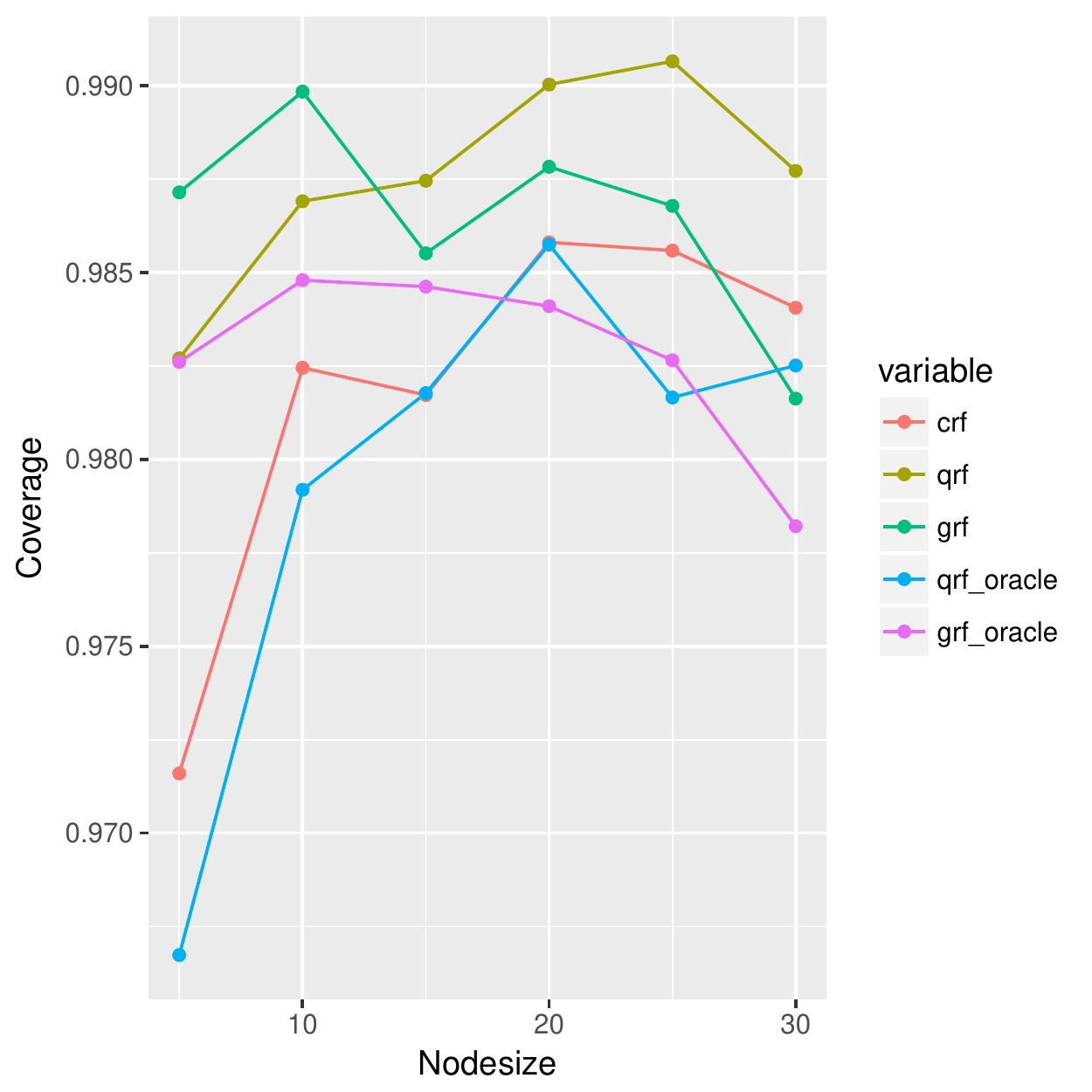}
        \caption{BostonHousing}
    \end{subfigure}
    ~
    \begin{subfigure}[b]{0.3\linewidth}
        \includegraphics[width=\textwidth]{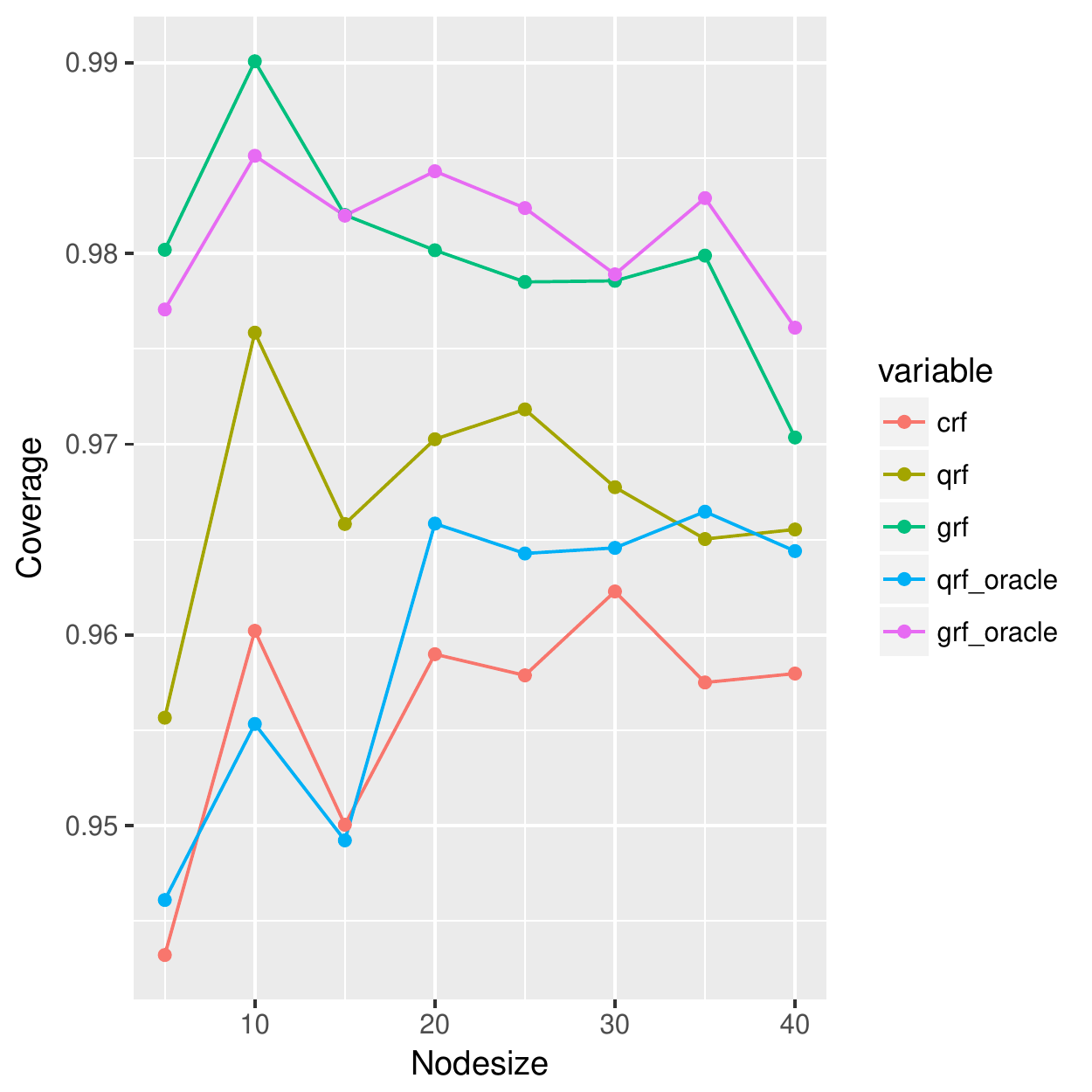}
        \caption{Ozone}
    \end{subfigure}
    ~
    \begin{subfigure}[b]{0.3\linewidth}
        \includegraphics[width=\textwidth]{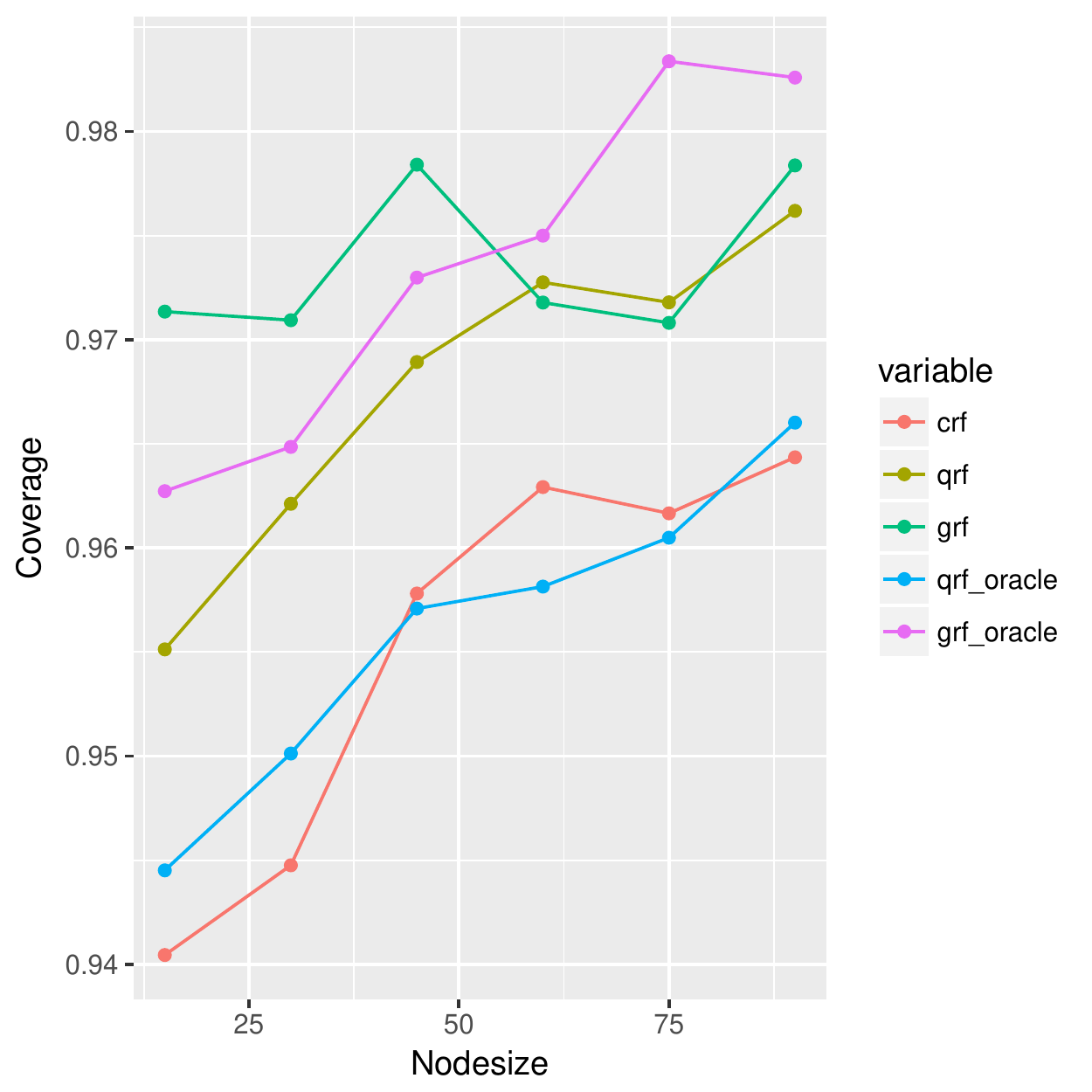}
        \caption{Abalone}
    \end{subfigure}
    
    \caption{Average coverage of prediction intervals with different node sizes on the three real data sets: each one represented in one plot. The length of the confidence intervals is much larger for all of the other considered methods.}
    \label{fig:real_data_ci}
\end{figure}

\section{Discussion}
\label{sec:conc}

In this article, we introduced a censored quantile random forest, a novel non-parametric method for quantile regression problems that is integrated with the censored nature of the observations. While preserving information carried by the censored observations, the novel estimating equations maintain the flexibility of general random forest approaches. The estimating equations are equipped with both censoring information as well as random forest information simultaneously.

The statistical mechanism, particularly asymptotic normality, of ensemble tree-based methods is still not fully understood. Some insightful discussions and first steps towards this goal can be found in \cite{zhu2012recursively}. As noted in \cite{athey2016generalized} even consistency properties of censored based methods present significant theoretical challenges. Generalization of the results of \cite{athey2016generalized} to the case of censored observations involves non-trivial extensions of the theoretical advancements introduced therein. A generalization would require not only adaptation to censoring but as well as to an infinite-dimensional nuisance parameter; current results focus on finite dimensional nuisance parameters.  Once the latter generalization is achieved then, utilizing our estimating equations, will allow for the generalization of the former.  

One of the promising applications of the introduced method is in the estimation of heterogeneous treatment effects when the response variable is censored; treatment discovery with right-censored observations is important and yet poorly understood research area. Equipping this literature with a fully non-parametric approach would lead to a significant broadening of the now more known parametric approaches. Our method appears to be a nice fit for this setting; observe that our estimating equations can be easily replaced with another kind that targets treatment effects directly.

\bibliographystyle{chicago}
\bibliography{forestcqr}

\clearpage
\section{Appendix}
\label{sec:appex}
\begin{proof}[Proof of Theorem \ref{thm:complexity}]
To get the candidate set $\mathcal{C}$, if we use the k-nearest neighbor estimator \eqref{eq:KM_kNN}, then the first step is to sort $n$ weights and choose the largest $k$ elements. This is in general a $O(n \log(n))$ procedure. If we use the Beran estimator \eqref{eq:Beran_rf}, then the time complexity is $O(n)$ because we need to find all the nonzero weights.

After we have the candidate set $\mathcal{C}$, evaluating $S_n(q;\tau)$ for all $q \in \mathcal{C}$ and finding the minimum is a $O(n |\mathcal{C}|)$ procedure. For \eqref{eq:KM_kNN}, $|\mathcal{C}| = k$; and for \eqref{eq:Beran_rf}, $|\mathcal{C}|$ is in the order of $m \log(n)^{p-1}$ by \cite{lin2006random}.
\end{proof}

\begin{proof}[Proof of Theorem \ref{thm:part1}]
Conditional on $X_1, \cdots, X_n$, the random variable $U_i = F(Y_i|X_i), i=1,\cdots,n$ are i.i.d. uniform on $[0,1]$. By Condition \ref{cond:4}, for a given $X_i$,
\begin{eqnarray*}
\ind(Y_i \le q) = \ind(U_i \le F(q|X_i)).
\end{eqnarray*}
Then we can decompose
\begin{eqnarray*}
&& \sum_{i=1}^n w(X_i, x) \ind(Y_i \le q) \\
&=& \sum_{i=1}^n w(X_i, x) \ind(U_i \le F(q|X_i)) \\
&=& \sum_{i=1}^n w(X_i, x) \ind(U_i \le F(q|x)) + \sum_{i=1}^n w(X_i, x)\Bigl\{ \ind(U_i \le F(q|X_i)) - \ind(U_i \le F(q|x)) \Bigr\}.
\end{eqnarray*}
The difference between the empirical distribution function and the truth can then be bounded by
\begin{eqnarray*}
&& \left|\sum_{i=1}^n w(X_i, x) \ind(Y_i \le q) - F(q|x)\right| \\
&\le& \underbrace{\left|\sum_{i=1}^n w(X_i, x) \ind(U_i \le F(q|x)) - F(q|x)\right|}_{\textrm{(I)}}
\\
&+& \underbrace{\left|\sum_{i=1}^n w(X_i, x)\Bigl\{ \ind(U_i \le F(q|X_i)) - \ind(U_i \le F(q|x)) \Bigr\}\right|}_{\textrm{{(II)}}}.
\end{eqnarray*}
For part (I), since $U_i$ is uniform, we have
\begin{equation*}
\sup_{q \in \mathbb{R}} \left|\sum_{i=1}^n w(X_i, x) \ind(U_i \le F(q|x)) - F(q|x)\right| = \sup_{z \in [0,1]} \left|\sum_{i=1}^n w(X_i, x) \ind(U_i \le z) - z \right|
\end{equation*}
Now since $0 \le w(X_i, x) \le 1/m$ and $\sum_{i=1}^n w(X_i, x) = 1$, we have
\begin{equation*}
\sum_{i=1}^n w(X_i, x)^2 \le \max_{i=1,\cdots,n} w(X_i, x) \le \frac{1}{m} \to 0
\end{equation*}
as $n \to \infty$, by Condition \ref{cond:2}. Hence, by Chebyshev inequality, for every $z \in [0,1]$ and $x \in \mathcal{X}$,
\begin{equation*}
\left|\sum_{i=1}^n w(X_i, x) \ind(U_i \le z) - z \right| = o_p(1).
\end{equation*}
Then by Bonferroni's inequality,
\begin{equation*}
\sup_{z \in [0,1]} \left|\sum_{i=1}^n w(X_i, x) \ind(U_i \le z) - z \right| = o_p(1).
\end{equation*}
The proof of part (II)
\begin{equation*}
\left|\sum_{i=1}^n w(X_i, x)\Bigl\{ \ind(U_i \le F(q|X_i)) - \ind(U_i \le F(q|x)) \Bigr\}\right| = o_p(1)
\end{equation*}
follows the same argument of Theorem 1 and Lemma 2 in \cite{meinshausen2006quantile} by invoking Condition \ref{cond:2}. Finally, we notice that by Condition \ref{cond:5}, $\sup_{q\in [-r,r]} |\hat{G}(q|x) - G(q|x)| = o(1)$ because $[-r,r]$ is compact.
\end{proof}

\begin{proof}[Proof of Theorem \ref{thm:consistency}]
By \cite{van2000asymptotic}, we only need to show for any $\tau \in (0,1)$, $x \in \mathcal{X}$,
\begin{enumerate}
\item \label{part:1}
$\sup_{q \in [-r,r]} | S_n(q; \tau) - S(q; \tau) | = o_p(1)$.
\item \label{part:2}
For any $\epsilon > 0$,
$\inf\{|S(q;\tau)|: |q - q^*| \ge \epsilon, q \in [-r, r] \} > 0$.
\item \label{part:3}
$S_n(q_n; \tau) = o_p(1)$.
\end{enumerate}
Part \ref{part:1} has been proved by Theorem \ref{thm:part1}. For part \ref{part:2}, note that
\begin{eqnarray*}
S(q;\tau) &=& (1-\tau)G(q|x) - \mathbb{P}(Y > q|x) \\
&=& (1-\tau)G(q|x) - \mathbb{P}(T > q|x) \mathbb{P}(C > q|x) \\
&=& ((1-\tau) - \mathbb{P}(T > q|x)) G(q | x) \\
&=& (\mathbb{P}(T \le q|x) - \tau) G(q | x).
\end{eqnarray*}
The second equality is because of the conditionally independency between $T$ and $C$. Fix an $\epsilon > 0$, and denote $$E = \{|S(q;\tau)|: |q - q^*| \ge \epsilon, q \in [-r, r]\}.$$ Since $0 < \tau < 1$, by Condition \ref{cond:4}, there exists some $l > 0$ such that $G(q|x) \ge l$ and $$|\mathbb{P}(T \le q | x) - \tau| \ge l$$ for $q \in E$. Now for part \ref{part:3}, by the definition of $q_n$, we know 
$$|S_n(q_n; \tau)| = \min_{q \in [-r,r]} |S_n(q;\tau)|.$$ Also by definition of $q^*$, 
$$0 = |S(q^*;\tau)| = \min_{q \in [-r,r]} |S(q;\tau)|.$$ Then we get
\begin{eqnarray*}
    |S_n(q_n;\tau)| &=& |S_n(q_n;\tau)| - |S_n(q^*;\tau)| + |S_n(q^*;\tau)| - |S(q^*;\tau)| \\
    &\le& | S_n(q^*;\tau) - S(q^*;\tau) | \\
    &\le& \sup_{q \in [-r,r]} |S_n(q;\tau) - S(q;\tau)| \\
    &=& o_p(1)
\end{eqnarray*}
where the first inequality is because of the definition of $q_n$ and the triangular inequality.
\end{proof}

\clearpage
%\newpage
%\bigskip
%\begin{center}
%{\large\bf SUPPLEMENTARY MATERIAL}
%\end{center}
%
%\begin{description}
%
%\item[Title:] Brief description. (file type)
%
%\item[R-package for  MYNEW routine:] R-package ÒMYNEWÓ containing code to perform the diagnostic methods described in the article. The package also contains all datasets used as examples in the article. (GNU zipped tar file)
%
%\item[HIV data set:] Data set used in the illustration of MYNEW method in Section~ 3.2. (.txt file)
%
%\end{description}

\end{document}